\documentclass[lettersize,journal,romanappendices]{IEEEtran}

\usepackage{color}
\usepackage{comment}

\usepackage[noadjust]{cite}
\usepackage{hyperref}

\usepackage{times}
\usepackage[mathcal]{euscript}

\usepackage[cmex10]{amsmath}
\usepackage{amssymb} 
\let\proof\relax
\let\endproof\relax
\usepackage{amsthm} 
  
\usepackage{dsfont} 
\usepackage{mathptmx}
\usepackage{bm}

\usepackage{epsfig} 
\usepackage{graphicx} 

\usepackage{multirow}
\usepackage{enumerate}

\usepackage[ruled, vlined, linesnumbered]{algorithm2e}
\SetNlSty{text}{}{:}

\usepackage{dblfloatfix}

\usepackage{scalerel}
\newcommand\reallywidehat[1]{\arraycolsep=0pt\relax%
\begin{array}{c}
\stretchto{
  \scaleto{
    \scalerel*[\widthof{\ensuremath{#1}}]{\kern-.5pt\bigwedge\kern-.5pt}
    {\rule[-\textheight/2]{1ex}{\textheight}} 
  }{\textheight} %
}{0.5ex}\\[-1mm]           
#1\\                 
\rule{-1ex}{0ex}
\end{array}
} 

\newcommand{\cdim}		{d} 
\newcommand{\cdegree}	{n} 
\newcommand{\corder}	{n} 
\newcommand{\clength}    {L} 

\newcommand{\bcurve}	{\mathrm{B}} 
\newcommand{\bpoly}		{b} 
\newcommand{\bdegree}	{\cdegree} 
\newcommand{\bdim} 		{\cdim} 
\newcommand{\bpoint}	{\vect{p}} 
\newcommand{\bpmat}		{\mat{P}} 
\newcommand{\bbasis}	{\vect{b}} 
\newcommand{\bbmat}		{\mat{b}} 
\newcommand{\bwmat}		{\mat{W}} 

\newcommand{\mcurve} 	{\mathrm{M}} 
\newcommand{\mpoint}	{\vect{q}} 
\newcommand{\mpmat}		{\mat{Q}} 
\newcommand{\mbasis}	{\vect{m}} 
\newcommand{\mbmat}		{\mat{m}} 

\newcommand{\tcurve}	{\Upsilon} 
\newcommand{\tpoint}	{\vect{y}} 
\newcommand{\tpmat}		{\mat{Y}} 
\newcommand{\tbasis}	{\vect{\tau}} 
\newcommand{\tbmat}		{\boldsymbol{\tau}} 
\newcommand{\toffset}	{t_{o}} 

\newcommand{\tfbasis}	{\mat{T}} 

\newcommand{\bdist}	{d} 
\newcommand{\bdistL}{d_{L2}} 
\newcommand{\bdistF}{d_F} 
\newcommand{\bdistH}{d_H} 
\newcommand{\bdistM}{d_M} 
\newcommand{\bdistC}{d_C} 

\newcommand{\emat}  {\mat{E}} 
\newcommand{\rmat}	{\mat{R}} 

\newtheoremstyle{plain}
	  {}
	  {}
	  {\itshape}
	  {}
	  {\bfseries}
	  {}
	  {5pt plus 1pt minus 1pt}
	  {}

\newtheoremstyle{definition}
  	  {}
	  {}
	  {\normalfont}
	  {}
	  {\bfseries}
	  {}
	  {5pt plus 1pt minus 1pt}
	  {}
  
\theoremstyle{plain}

\newtheorem{lemma}{Lemma}
\newtheorem{proposition}{Proposition}

\newtheorem{property}{Property}
	
\theoremstyle{definition}
\newtheorem{definition}{Definition}

\newtheorem{remark}{Remark}

\theoremstyle{plain}
\newtheorem*{BezierApproximationRule}{A Rule of Thumb for B\'ezier Approximations}

\newcommand{\refeq}[1]			{(\ref{#1})} 
\newcommand{\reffig}[1]			{Fig. \ref{#1}} 
\newcommand{\refsec}[1]			{Section \ref{#1}}
\newcommand{\refapp}[1]			{Appendix \ref{#1}}
\newcommand{\reftab}[1]			{Table \ref{#1}}

\newcommand{\refprop}[1]		{Proposition \ref{#1}}

\newcommand{\reflem}[1]			{Lemma \ref{#1}}
\newcommand{\refdef}[1]			{Definition \ref{#1}}

\newcommand{\refalg}[1]			{Algorithm \ref{#1}}

\newcommand{\reffn}[1] 		    {\textsuperscript{\ref{#1}}}
\newcommand{\refpropty}[1]      {Property \ref{#1}}

\newcommand{\R}  	{\mathbb{R}}
\newcommand{\N}  	{\mathbb{N}}
\newcommand{\ball}	{\mathcal{B}} 

\let\originalleft\left
\let\originalright\right
\renewcommand{\left}{\mathopen{}\mathclose\bgroup\originalleft}
\renewcommand{\right}{\aftergroup\egroup\originalright}
	
\newcommand{\plist}[1] 	{\left(#1\right)} 
\newcommand{\blist}[1]	{\left[ #1 \right]} 
\newcommand{\clist}[1]	{\left\{#1\right\}} 

\newcommand{\vect}[1]   {\mathrm{#1}}
\newcommand{\mat}[1]    {\mathbf{#1}}
\newcommand{\tr}[1] {{#1}^{\mathrm{T}}} 
\newcommand{\norm}[1]  {\|#1\|}
\newcommand{\trace} {\mathrm{tr}} 
\newcommand{\rank}   {\mathrm{rank}} 
\newcommand{\diag}	{\mathrm{diag}} 

\newcommand{\ldf}   {:=} 
\newcommand{\argmin}{\operatornamewithlimits{argmin}} 
\newcommand{\argmax}{\operatornamewithlimits{argmax}} 
\newcommand{\diff} {\mathrm{d}} 
\newcommand{\conv}	{\mathrm{conv}} 

\begin{document}

\title{Adaptive B\'ezier Degree Reduction and Splitting \\ for Computationally Efficient Motion Planning}

\author{\"Om\"ur Arslan and Aron Tiemessen
\thanks{The authors are with the Department of Mechanical Engineering, Eindhoven University of Technology, P.O. Box 513, 5600 MB Eindhoven, The Netherlands. The first author is also affiliated with the Eindhoven AI Systems Institute. Emails: o.arslan@tue.nl, a.j.c.tiemessen@student.tue.nl}%
}


\markboth{Technical Report, January~2022}%
{Arslan and Tiemessen: Adaptive B\'ezier Degree Reduction and Splitting}


\maketitle

\begin{abstract}
As a parametric polynomial curve family, B\'ezier curves are widely used in safe and smooth motion design of intelligent robotic systems from flying drones to autonomous vehicles to robotic manipulators. 
In such motion planning settings, the critical features of high-order B\'ezier curves such as curve length, distance-to-collision, maximum curvature/velocity/acceleration are either numerically computed at a high computational cost or inexactly approximated by discrete samples.
To address these issues, in this paper we present a novel computationally efficient approach for adaptive approximation of high-order B\'ezier curves by multiple low-order B\'ezier segments at any desired level of accuracy  that is specified in terms of a B\'ezier metric.
Accordingly, we introduce a new B\'ezier degree reduction method, called \emph{parameterwise matching reduction}, that approximates B\'ezier curves more accurately compared to the standard least squares and Taylor reduction methods.  
We also propose a new B\'ezier metric, called the \emph{maximum control-point distance}, that can be computed analytically, has a strong equivalence relation with other existing B\'ezier metrics, and defines a geometric relative bound between B\'ezier curves.
We provide extensive numerical evidence to demonstrate the effectiveness of our proposed B\'ezier approximation approach.
As a rule of thumb, based on the degree-one matching reduction error, we conclude that  an $n^\text{th}$-order B\'ezier curve can be accurately approximated by $3(n-1)$ quadratic and $6(n-1)$ linear B\'ezier segments, which is fundamental for B\'ezier discretization.         
\end{abstract}

\begin{IEEEkeywords}
Smooth motion planning, path smoothing, polynomial trajectory optimization, path discretization, B\'ezier curves
\end{IEEEkeywords}

\section{Introduction}
\label{sec.Introduction}
 
Safe and smooth motion planning is essential for many autonomous robots.
As a parametric smooth  motion representation, polynomial curves find significant applications in safe robot motion design from flying drones \cite{mellinger_kumar_ICRA2011,richter_bry_roy_ISRR016,ding_gao_wang_shen_TRO2019,gao_wu_lin_shen_ICRA2018, tordesillas_etal_TRO2021}  to autonomous vehicles \cite{gonzalez_et_al_TITS2016,ding_zhang_chen_shen_RAL2019, qian_etal_ITSC2016, perez_godoy_villagra_onieva_ICRA2013} to robotic manipulators \cite{ozaki_lin_ICRA1996,hauser_ng-thow-hing_ICRA2010,scheiderer_thun_meisen_FAIM2019, zhao_etal_CYBER2019}.
Polynomials expressed in different (e.g., monomial, Taylor, and Bernstein) bases offer different  useful functional and geometric properties for computationally efficient motion planning. 
While the monomial (a.k.a. power) basis yields quadratic trajectory optimization objectives \cite{mellinger_kumar_ICRA2011}, polynomial  B\'ezier curves in Bernstein basis have useful convexity and interpolation properties \cite{gao_wu_lin_shen_ICRA2018}: a B\'ezier curve is contained in the convex hull of its control points (i.e., parameters), and it smoothly interpolates between the first and last control point.
A well known challenge of motion planning with polynomial and so B\'ezier curves is  that the computational complexity increases with increasing curve degree \cite{gonzalez_et_al_TITS2016}. 
Because critical  curve features such as curve length, distance-to-collision, and maximum curvature/velocity/acceleration can be analytically determined only for low-order (e.g., linear and quadratic) polynomial curves and are numerically computed or inexactly approximated using discrete samples for higher-order polynomials, as summarized in \reftab{tab.BezierCurveFeatures}.

\begin{figure}[t]
\centering
\begin{tabular}{@{}c@{\hspace{1mm}}c@{}}
\includegraphics[width=0.24\textwidth]{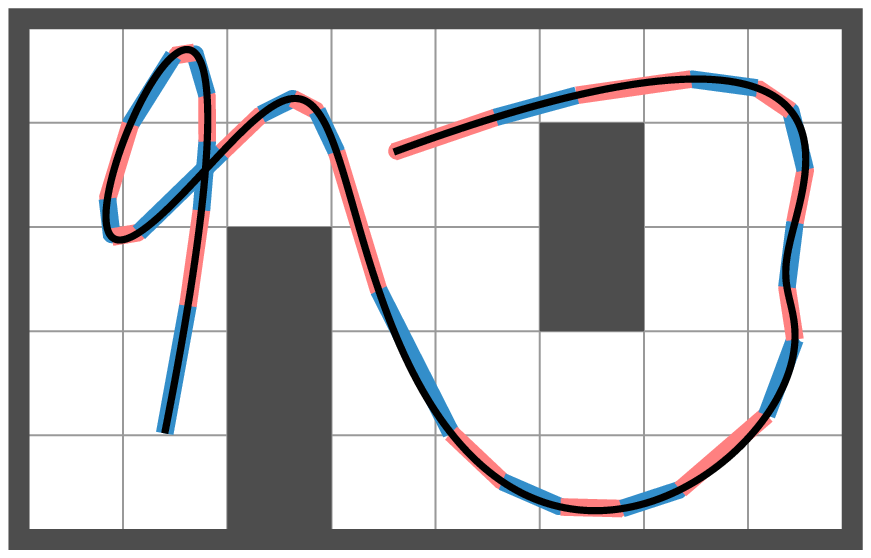}&
\includegraphics[width=0.24\textwidth]{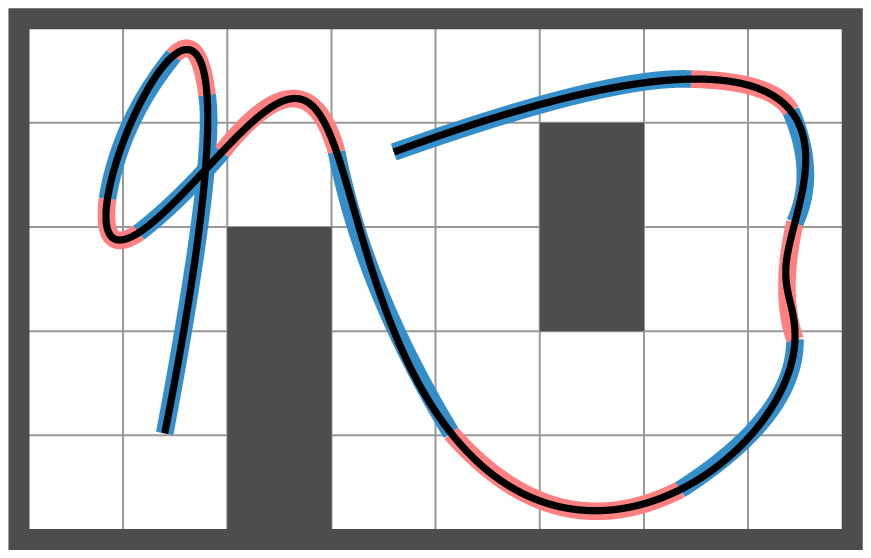} 
\end{tabular}
\vspace{-2mm}
\caption{Adaptive approximation of an $8^{\text{th}}$-order B\'ezier curve (black line) by multiple (left) linear and (right) quadratic B\'ezier  segments (red and blue patches) whose length, maximum velocity/acceleration/curvature, and distance to obstacles (dark gray) can be computed analytically. 
The B\'ezier degree reduction and splitting is automatically performed by uniform matching reduction and binary search for a maximum control-point distance of~$0.1$~units.}
\label{fig.BezierApproximation}
\vspace{-2mm}
\end{figure} 
 
 
\begin{table}[b]
\caption{Computation of Bezier Curve Features in Motion Planning}
\label{tab.BezierCurveFeatures}
\centering
\vspace{-1mm}
\begin{tabular}{lccc}
\hline \hline
Bezier Curve Feature & $\corder \leq 2$ & $\corder = 3$  &$\corder > 3 $ \\
\hline
Arc Length & Analytic & Numeric & Numeric  \\
Maximum Velocity & Analytic & Analytic & Numeric \\
Maximum Curvature & Analytic & Numeric & Numeric \\
Distance-to-Point & Analytic & Numeric & Numeric \\
Distance-to-Line-Segment & Analytic & Numeric & Numeric \\
\hline
\hline
\end{tabular}
\end{table}  

In this paper, we propose a new computationally efficient approach for adaptive approximation of high-order B\'ezier curves by multiple low-order B\'ezier segments at any desired level of accuracy specified in terms of a B\'ezier metric, as illustrated in \reffig{fig.BezierApproximation}.
Our approach is based on an unexplored functional property of B\'ezier curves in motion planning: distance between B\'ezier curves can be measured analytically in terms of control points.
Accordingly, we introduce a new analytic B\'ezier metric, called the \emph{maximum control-point distance}, that can be used to geometrically bound B\'ezier curves with respect to each other, and defines tight upper bounds on other existing B\'ezier metrics.  
We also propose a new B\'ezier degree reduction method,  called \emph{parameterwise matching reduction}, that allows preserving certain curve points (e.g., end points) while performing degree reduction. 
Based on the degree-one parameterwise matching reduction error, we conclude that an $n^{\text{th}}$-order B\'ezier curve can be accurately approximated by $3(n-1)$ quadratic and $6(n-1)$ linear B\'ezier segments, which is a fundamental rule of thumb for B\'ezier discretization.  
In numerical simulations, we demonstrate the effectiveness of approximating high-order B\'ezier curves by linear and quadratic B\'ezier segments for fast and accurate computation of common curve features used in \mbox{motion planning}.

\vspace{-2mm}

\subsection{Motivation and Related Literature}

Autonomous robots and people interacting with them enjoy smooth motion in practice: jerky robot motion does not only cause mechanical and electrical failures and malfunctions, but also causes discomfort for the user.
Most existing smooth motion planning methods follow a two-step approach: first find a piecewise linear path for a simplified version of the system to achieve a simplified version of a given task; and then perform path smoothing as post-processing to satisfy the actual task and system requirements \cite{ravankar_etal_Sensors2018}.
The first step, piecewise linear motion planning,  is well established with many computationally effective (search- and sampling-based) planning algorithms for the fully actuated kinematic robot model \cite{lavalle_PlanningAlgorithms2006}.  
The second step, path smoothing that aims to convert a piecewise linear reference plan into a smooth dynamically feasible trajectory satisfying both system  and task constraints, is an active research topic, especially for real-time operation requirements.
Due to their compact parametric form and functional properties, polynomial curves have recently received significant attention with promising potentials for computationally efficient path smoothing, especially for differential flat systems \cite{nieuwstadt_murray_IJRNC1998} such as cars \cite{ding_zhang_chen_shen_RAL2019, qian_etal_ITSC2016}, quadrotors \cite{mellinger_kumar_ICRA2011, richter_bry_roy_ISRR016, ding_gao_wang_shen_TRO2019}, and fixed-wing aircrafts \cite{bry_richter_bachrach_roy_IJRR2015}, to name a few, whose control inputs can be expressed as a function of \emph{flat system outputs} (represented by polynomials) and their derivatives.
For example, while polynomials of degree 3-5 are often used for autonomous vehicles, polynomials of degree 5-10 are required for quadrotors. 
The major reason for the use of relatively low-order polynomials in practice is that the computational cost of planning with polynomials increases with increasing degree of polynomials \cite{ravankar_etal_Sensors2018, gonzalez_et_al_TITS2016}.
Our proposed approach enables handling high-order polynomials efficiently by approximating them with multiple low-order polynomial~segments.

Convex optimization plays a key role in polynomial path smoothing. 
In polynomial trajectory optimization, the standard optimization objectives of total  squared velocity, acceleration, jerk, and snap (i.e., the first, second, third and fourth time derivatives of the position) of a robotic system can be written as a quadratic objective function of polynomial parameters \cite{mellinger_kumar_ICRA2011, bry_richter_bachrach_roy_IJRR2015}. 
In order to take the full advantage of quadratic programming, the system and task constraints are often represented as linear or quadratic inequalities.
For example, a piecewise linear reference plan can be used to construct a convex safe corridor around the reference plan to represent planning constraints as a collection of convex polytopes \cite{liu_etal_RAL2017} or spheres \cite{ding_gao_wang_shen_TRO2019}.
Accordingly, polynomial trajectory optimization is often formulated as a quadratic optimization problem, for example, by simply using a polynomial discretization \cite{mellinger_kumar_ICRA2011, richter_bry_roy_ISRR016, zhao_etal_CYBER2019}.
This naturally raises a question about polynomial discretization: how many sample points along a polynomial are needed for a proper and accurate representation of planning constraints. 
The existing methods use either manual or heuristic approaches to add extra samples if polynomial discretization fails \cite{mellinger_kumar_ICRA2011, richter_bry_roy_ISRR016}.  
In this sense, our results offer a systematic solution for determining a proper discretization of polynomials to model planning constraints at any desired level of accuracy.

In polynomial trajectory optimization, the convexity property of B\'ezier curves makes them an attractive choice for handling convex system constraints within quadratic programming.
Since B\'ezier curves are contained in the convex hull of their control points, trajectory optimization constraints are often enforced by  constraining B\'ezier control points inside convex constraint sets \cite{gao_wu_lin_shen_ICRA2018, honig_preiss_kumar_sukhatme_ayanian_TRO2018, choi_curry_elkaim_JAM2010}. 
This approach is effectively applied for smooth trajectory generation with Bezier curves over safe corridors \cite{liu_etal_RAL2017} in various application settings; for example, for drone navigation in unknown environments \cite{gao_wu_lin_shen_ICRA2018, tordesillas_etal_TRO2021}, autonomous driving \cite{ding_zhang_chen_shen_RAL2019, gonzalez_etal_ITSC2014, choi_curry_elkaim_JAM2010}, multirobot coordination \cite{honig_preiss_kumar_sukhatme_ayanian_TRO2018, tang_kumar_IROS2016}, and  perception-aware navigation \cite{preiss_hausman_sukhatme_weiss_RSS2017}.
Although it performs reasonably well for low-order B\'ezier curves in practice, this simple but conservative approach is suboptimal for high-order B\'ezier curves since the convex hull of B\'ezier control points significantly overestimates the smallest convex region containing by the actual curve, especially for higher-order polynomials.
On the other hand,  exact and fast continuous constraint verification with polynomial curves is possible based on the separation of polynomial extremes \cite{bucki_mueller_IROS2019}, the sign change of polynomials \cite{tang_tong_wang_manocha_TG2014}, and their root existence test based on Sturm's theorem \cite{wang_zhou_xu_chu_gao_RAL2020}, but these methods result in highly complex nonlinear optimization constraints.  
Our approach for approximating high-order B\'ezier curves by low-order B\'ezier segments allows one to use the convexity of low-order B\'ezier curves in high-order polynomial trajectory optimization in a  less conservative way.

Another appealing feature of B\'ezier curves for smooth robot motion design is that they smoothly interpolate between the first and last control points.
This interpolation property is often leveraged for motion planning of nonholonomic systems with boundary conditions; for example, for waypoint smoothing \cite{artunedo_godoy_villagra_IVS2017, li_luo_wu_IEEEAccess2019, perez_godoy_villagra_onieva_ICRA2013} and smooth steering control \cite{bae_etal_ITSC2013, chen_etal_ICRA2014} in  autonomous driving \cite{han_etal_IVS2010, zheng_etal_ITS2020}, and  path smoothing in sampling-based motion planning \cite{lau_sprunk_burgard_IROS2009, han_liu_DDCLS2020, pan_zhang_manocha_IJRR2012, hauser_ng-thow-hing_ICRA2010}.
Continuous curvature path smoothing with curvature constraints is applied for increasing passenger comfort while ensuring dynamical feasibility in autonomous vehicles for smooth lane change \cite{chen_zhao_mei_liang_ICVES2013} and urban driving \cite{qian_etal_ITSC2016, cimurs_hwang_suh_IRC2017, bu_su_zou_wang_ROBIO2015, elbanhawi_simic_jazar_JIRS2015}. 
Although path smoothing with curvature constraints can be performed analytically for low-order B\'ezier curves \cite{yang_sukkarieh_TRO2010}, the maximum curvature is numerically computed for high-order B\'ezier curves \cite{gonzalez_etal_ITSC2014}.
Thus, one can use our adaptive B\'ezier approximation approach to take the analytic advantages of low-order B\'eziers in path smoothing with high-order B\'eziers.

As a smooth motion primitive, polynomial curves are also used in search-based and sampling-based smooth motion planning of  
nonholonomic systems \cite{elbanhawi_simic_jazar_TITS2016, yang_etal_JIRS2014} and robotic manipulators \cite{yang_etal_JIRS2014, han_liu_DDCLS2020, zhao_etal_CYBER2019} as well as their reinforcement learning \cite{scheiderer_thun_meisen_FAIM2019}.
A challenge of planning with polynomial motion primitives is finding an informative and computationally efficient local metric for measuring the connectivity and travel cost. 
A natural travel cost measure is the arc length of polynomials, which can be analytically determined only for linear and quadratic polynomials.
Using the proposed  B\'ezier approximation method, one can accurately  and efficiently measure the arc length of high-order polynomial curves by dividing them into multiple low-order polynomial segments.

B\'ezier curves are widely used in computer graphics and computer aided design (CAD) for efficiently representing complex shapes with few parameters \cite{farouki_CAGD2012, farin_CurvesSurfaces2002}.
Computationally efficient handling of complex shapes often requires optimal reduction of B\'ezier curves based on different metrics \cite{lee_park_BAMS1997, eck_CAD1995, lee_park_yoo_CAGD2002}.
This motivates many alternative approaches for degree reduction of B\'ezier curves \cite{sunwoo_lee_CAGD2004} and their approximate conversions \cite{park_choi_kimn_BAMS1995}  (with end point constraints \cite{chen_wang_CAGD2002}).
This present paper brings such CAD tools to the motion planning literature with important additions which, we believe, also contribute back to the CAD literature.

\vspace{-2mm}

\subsection{Contributions and Organization of the Paper}

In this paper, we present a novel systematic approach for adaptive discretization and approximation of high-order B\'ezier curves by multiple low-order B\'ezier curves for computationally efficient smooth motion planning with high-order polynomials.   
In summary, our main contributions are:
\begin{itemize}
\item a new B\'ezier metric, called the \emph{maximum control-point distance}, that defines an analytic tight upper bound on existing standard B\'ezier metrics such as the Hausdorff, parameterwise maximum, and Frobenius-norm distances of B\'ezier polynomials, and enables bounding B\'ezier curves geometrically with respect to each other,

\item a new B\'ezier degree reduction method, called \emph{parameterwise matching reduction}, that approximates  B\'ezier geometry more accurately (e.g., by preserving end points) compared to the least squares and Taylor reductions,

\item a new adaptive B\'ezier approximation approach for  representing  high-order B\'ezier curves by multiple low-order B\'ezier segments at any desired level of accuracy that is specified in terms of a B\'ezier metric,  

\item a new rule of thumb for accurately approximating high-order B\'ezier curves with a fixed finite collection of linear and quadratic B\'ezier curves. 
\end{itemize}
With extensive numerical simulations, we demonstrate the effectiveness of the newly proposed methods. 
At a more conceptual level, this paper for the first time introduces the use of B\'ezier metrics and degree reduction methods for local low-order approximation of high-order B\'ezier curves in order to enable computationally efficient smooth motion planning.

The rest of the paper is organized as follows. 
In \refsec{sec.BezierCurves}, we provide a background overview of B\'ezier curves, and the matrix representation, basis transformation and reparametrization of polynomial curves. 
In \refsec{sec.BezierMetric}, we describe how to measure the distance between B\'ezier curves and introduce a new B\'ezier metric.
In \refsec{sec.BezierElevationReduction}, we present how to (approximately) represent a B\'ezier curve with more or fewer control points via degree elevation and reduction operations, and introduce a new degree reduction method.
In \refsec{sec.BezierAdaptiveApproximation}, we describe how to approximate high-order B\'ezier curves by low-order B\'ezier curves at any desired accuracy level, and present a rule of thumb for accurate B\'ezier approximations.
In \refsec{sec.NumericalAnalysis}, we present numerical results to demonstrate the role of polynomial degree and the number of B\'ezier segments on  approximation accuracy.   
In \refsec{sec.Conclusions}, we conclude with a summary of our research highlights and future directions.


\section{B\'ezier Curves}
\label{sec.BezierCurves}

In this section, we first briefly introduce B\'ezier curves and their important properties, and then continue with the matrix representation, basis transformation and affine reparameterization of polynomial B\'ezier, monomial and Taylor curves.   

\vspace{-2mm}

\subsection{Characteristic Properties of B\'ezier Curves}
\label{sec.GeneralBezierCurves}

\begin{definition}\label{def.BezierCurve}
(\emph{B\'ezier Curve}) In a $\bdim$-dimensional Euclidean space $\R^{\bdim}$, a \emph{B\'ezier curve} $\bcurve_{\bpoint_0, \ldots \bpoint_{\bdegree}} (t)$ of degree $\bdegree \in \N$, associated with \emph{control points} $\bpoint_0, \ldots, \bpoint_\bdegree \in \R^{\bdegree}$,  is  a parametric polynomial curve defined for $0 \leq t \leq 1$ as\footnote{The standard definition of B\'ezier curves is over the unit interval, and they are mathematical well defined over all reals.}
\begin{align} \label{eq.BezierCurve}
\bcurve_{\bpoint_0, \ldots, \bpoint_\bdegree}(t) \ldf \sum_{i=0}^\bdegree \bpoly_{i,\bdegree}(t) \bpoint_i, 
\end{align}
where  $\bpoly_{i,\bdegree}(t)$ denotes the $i^{\text{th}}$ Bernstein basis polynomial of degree $\bdegree$ that is  defined for $ i = 0,1, \ldots, \bdegree$ as
\begin{align}\label{eq.BernsteinPolynomial}
\bpoly_{i,\bdegree}(t) \ldf \scalebox{1.2}{$\binom{n}{i}$} t^i (1-t)^{n-i}.
\end{align}
\end{definition}

Key characteristics of B\'ezier and Bernstein polynomials  are their recursion, derivative and convexity  properties \cite{farouki_CAGD2012, farin_CurvesSurfaces2002}.

\begin{property} (\emph{Recursion})
A B\'ezier curve can be recursively determined as a convex combination of two B\'ezier curves of one degree lower as
\begin{align}
\bcurve_{\bpoint_0, \bpoint_1, \ldots, \bpoint_\bdegree}(t) = (1-t)\bcurve_{\bpoint_0, \ldots, \bpoint_{\bdegree -1}}(t) + t \bcurve_{\bpoint_1, \ldots, \bpoint_\bdegree} (t),
\end{align}
with the base case $\bcurve_{\bpoint_0}(t) = \bpoint_0$, 
which follows from the recursive definition of Bernstein polynomials
\begin{align}
\bpoly_{i,\bdegree}(t) = (1-t) \bpoly_{i,\bdegree-1}(t) + t \bpoly_{i-1, \bdegree-1}(t),
\end{align}
with base cases $\bpoly_{0,0}(t) = 1$ and $\bpoly_{i,n}(t) = 0$ for $i < 0$ and $i > \bdegree$.
\end{property}

\begin{property}\label{propty.BezierDerivative}
(\emph{Derivative}) The derivative of a B\'ezier curve is another B\'ezier curve of one degree lower and given by
\begin{align}\label{eq.BezierDerivative}
\frac{\diff}{\diff t} \bcurve_{\bpoint_0, \bpoint_1, \ldots, \bpoint_\bdegree}(t)=  \cdegree\bcurve_{\bpoint_1 - \bpoint_0, \ldots, \bpoint_\bdegree - \bpoint_{\bdegree-1}}(t),
\end{align}
since the Bernstein derivatives satisfy  
\begin{align}
\frac{\diff}{\diff t} \bpoly_{i,\bdegree}(t) = \bdegree\plist{\bpoly_{i-1,\bdegree-1}(t) - \bpoly_{i, \bdegree-1}(t)}.
\end{align} 
\end{property}

\begin{property} \label{propty.BezierConvexity}
(\emph{Convexity}) A B\'ezier curve is contained in the convex hull, denoted by $\conv$, of its control points, i.e.,
\begin{align}
\bcurve_{\bpoint_0, \ldots, \bpoint_{\bdegree}}(t) \in \conv\plist{\bpoint_0, \ldots, \bpoint_\bdegree} \quad \quad \forall t \in [0,1], 
\end{align}
because Bernstein polynomials are nonnegative and sum to one, i.e., for any $t \in [0,1]$ 
\begin{align}
\bpoly_{i,n}(t) \geq 0, \quad \text{and} \quad \sum_{i=0}^{n} \bpoly_{i,n}(t) = 1. 
\end{align}
\end{property}

\begin{property} \label{propty.BezierInterpolation}
(\emph{Interpolation}) A B\'ezier curve smoothly interpolates between its first and last control point, i.e.,
\begin{align}
\bcurve_{\bpoint_0, \ldots, \bpoint_{\bdegree}}(0) = \bpoint_0 \quad \text{ and } \quad \bcurve_{\bpoint_0, \ldots, \bpoint_{\bdegree}}(0) = \bpoint_\bdegree, 
\end{align}
since Bernstein polynomials smoothly interpolates between 
\begin{subequations}
\begin{align}
\plist{\bpoly_{0, \bdegree}(0),  \ldots, \bpoly_{\bdegree, \bdegree}(0)} &= (1, 0, \ldots, 0),  \\
\plist{\bpoly_{0, \bdegree}(1), \ldots,  \bpoly_{\bdegree, \bdegree}(1)} &= (0, \ldots, 0, 1).
\end{align}
\end{subequations}
\end{property}

\subsection{Matrix Representation of Polynomial Curves}

To  effectively handle high-order B\'ezier curves with a large number of control points, it is convenient to use the matrix representation of B\'ezier curves in the form of
\begin{align}
\bcurve_{\bpoint_0, \ldots, \bpoint_n}(t) =  \bpmat_{\cdegree} \bbasis_{\cdegree}(t), 
\end{align}
based on the control point matrix \mbox{$\bpmat_{\bdegree} \! := \! \blist{\bpoint_0, \ldots, \bpoint_n} \! \in\! \R^{\bdim \times (\bdegree+1)}$}  and the Bernstein basis vector $\bbasis_{\bdegree}(t) \in \R^{\bdegree+1}$ that is defined~as
\begin{align} \label{eq.BernsteinBasisMatrix}
\bbasis_{\bdegree}(t) &:= \left [
\begin{array}{c}
\bpoly_{0,\bdegree}(t) \\
\bpoly_{1,\bdegree}(t) \\
\vdots \\
\bpoly_{\bdegree,\bdegree}(t)
\end{array}
 \right ].
\end{align}
\noindent Note that the Bernstein basis polynomials  $\bpoly_{0,\bdegree}(t), \ldots, \bpoly_{\bdegree,\bdegree}(t)$ form a basis of $n+1$ linearly independent polynomials for polynomials of degree $\bdegree$  \cite{farouki_CAGD2012}.
The two other widely used basis functions of $\cdegree^{\text{th}}$-order polynomials are the monomial and Taylor basis vectors, respectively, defined as
\begin{align}\label{eq.MonomialTaylorBasis}
\mbasis_\cdegree (t) &:= \blist{ 
\begin{array}{@{}c@{}}
1 \\
t \\
\vdots \\
t^\cdegree
\end{array}}, 
\quad \text{ and } \quad
\tbasis_{\cdegree,\toffset}(t) :=  \blist{ 
\begin{array}{@{}c@{}}
1 \\
t - \toffset \\
\vdots \\
(t - \toffset)^n
\end{array} },
\end{align}
where $\toffset \in \R$ is the Taylor offset term.
Accordingly, like B\'ezier curves, one can define the \emph{monomial and Taylor curves}, associated with control points $\mpmat= \blist{\mpoint_0, \ldots, \mpoint_\cdegree} \in \R^{\cdim \times (\cdegree+1)}$ and $\tpmat_\cdegree = \blist{\tpoint_0, \ldots, \tpoint_\cdegree} \in \R^{\cdim \times \cdegree+1}$, respectively, as
\begin{subequations}
\begin{align}
\mcurve_{\mpoint_0, \ldots, \mpoint_{\cdegree}}(t) &\ldf \sum_{i=0}^{\cdegree} \mpoint_i t^i = \mpmat_{\cdegree} \mbasis_{\cdegree}(t), 
\\
\tcurve_{\tpoint_0, \ldots, \tpoint_{\cdegree}}(t, \toffset) &\ldf \sum_{i=0}^{\cdegree} \tpoint_i (t - \toffset)^i = \tpmat_{\cdegree} \tbasis_{\cdegree, \toffset}(t).
\end{align}
\end{subequations}
From their very similar forms in \refeq{eq.MonomialTaylorBasis} one can observe that the monomial and Taylor basis vectors (and so curves) are strongly related, i.e.,
\begin{align}
\tbasis_{\cdegree, \toffset}(t) = \mbasis_{\cdegree}(t - \toffset).
\end{align}  

Before continuing with the basis transformations of polynomial curves, we find it useful to define the Bernstein, monomial, and Taylor basis matrices associated with any set of reals $t_0, \ldots, t_m \in \R$, respectively, as
\begin{subequations} \label{eq.BasisMatrix}
\begin{align}
\bbmat_{\cdegree}(t_0, \ldots, t_m) &:=  \blist{\bbasis_{\cdegree}(t_0), \ldots, \bbasis_{\cdegree}(t_m)}, 
\\
\mbmat_{\cdegree}(t_0, \ldots, t_m) &:=  \blist{\mbasis_{\cdegree}(t_0), \ldots, \mbasis_{\cdegree}(t_m)} ,
\\
\tbmat_{\cdegree}(t_0, \ldots, t_m) &:=  \blist{\tbasis_{\cdegree}(t_0), \ldots, \tbasis_{\cdegree}(t_m)} .
\end{align}
\end{subequations}

An important property of square polynomial basis matrices is nonsingularity.
\begin{lemma}\label{lem.InvertibleBasisMatrix}
(Invertible Polynomial Basis Matrices)
For any pairwise distinct\footnote{\label{fn.DistinctReals}For numerically stable matrix inversion, a proper choice of pairwise distinct reals $t_0, \ldots, t_n \in [0,1]$  is the uniformly spaced parameters over the unit interval, i.e., $t_i = \frac{i}{n}$ for $i = 0, \ldots, n$.}  $t_0, \ldots, t_\cdegree \in \R$ and any Taylor offset $\toffset \in \R$, the polynomial basis matrices $\bbmat_{\cdegree}(t_0, \ldots, t_\cdegree)$, $\mbmat_{\cdegree}(t_0, \ldots, t_\cdegree)$ and $\tbmat_{\cdegree, \toffset}(t_0, \ldots, t_\cdegree)$ are all invertible.
\end{lemma}
\begin{proof}
See \refapp{app.InvertibleBasisMatrix}.
\end{proof}

\subsection{Basis Transformations of Polynomial Curves}

As expected, alternative representations of polynomial curves have their advantages (e.g., the convexity of Bezier curves, the totally ordered basis\footnote{The monomial basis satisfies $1 \!< t \!< t^2\!< \ldots \!< t^n$ and $1 \!> t \!> t^2\!> \ldots \!> t^n$.} of monomial curves, and the local approximation feature of Taylor curves).
Fortunately, one can easily perform change of polynomial basis. 

\begin{lemma} \label{lem.BasisTransformation}
\emph{(Change of Basis via Parameterwise Correspondence)}
The basis transformation matrices  between Bernstein, monomial, and Taylor bases (with a Taylor offset $\toffset \in \R$)
\begin{subequations}
\begin{align}
\bbasis_{\cdegree} (t) &= \tfbasis_{\mbasis}^{\bbasis}(\cdegree)  \mbasis_{\cdegree}(t) = \tfbasis_{\tbasis}^{\bbasis}(\cdegree, \toffset) \tbasis_{\cdegree, \toffset}(t),   
\\
\mbasis_{\cdegree} (t) &= \tfbasis_{\bbasis}^{\mbasis}(\cdegree) \bbasis_{\cdegree}(t) = \tfbasis_{\tbasis}^{\mbasis}(\cdegree, \toffset) \tbasis_{\cdegree,\toffset}(t),
\\
\tbasis_{\cdegree, \toffset}(t) &= \tfbasis_{\bbasis}^{\tbasis}(\cdegree, \toffset) \bbasis_{\cdegree}(t) = \tfbasis_{\mbasis}^{\tbasis}(\cdegree, \toffset) \mbasis_{\cdegree} (t),
\end{align}
\end{subequations}
can be computed using any pairwise distinct\reffn{fn.DistinctReals} $t_0, \ldots, t_{\cdegree} \in \R$~as
\begin{subequations}
\begin{align}
\tfbasis_{\mbasis}^{\bbasis}(\cdegree) &=  \tfbasis_{\bbasis}^{\mbasis}(\cdegree)^{-1} \!= \bbmat_{\cdegree}(t_0, \ldots, t_\cdegree)\mbmat_{\cdegree}(t_0, \ldots, t_\cdegree) ^{-1},
\\
\!\!\!\tfbasis_{\tbasis}^{\bbasis}(\cdegree, \toffset) &=  \tfbasis_{\bbasis}^{\tbasis}(\cdegree, \toffset)^{-1} \!= \bbmat_{\cdegree}(t_0, \ldots, t_\cdegree)\tbmat_{\cdegree, \toffset}\!(t_0, \ldots, t_\cdegree) ^{-1}\!\!, \!\!
\\
\!\! \!\!\!\tfbasis_{\tbasis}^{\mbasis}(\cdegree, \toffset) &=  \tfbasis_{\mbasis}^{\tbasis}(\cdegree, \toffset)^{-1} \!= \mbmat_{\cdegree}(t_0, \ldots, t_\cdegree)\tbmat_{\cdegree, \toffset}\!(t_0, \ldots, t_\cdegree) ^{-1} \!\!. \!\!
\end{align}
\end{subequations}
\end{lemma}
\begin{proof}
See \refapp{app.BasisTransformation}.
\end{proof}

\noindent It is useful to highlight that the elements of the basis transformation matrices between monomial and Bernstein (Taylor, respectively) bases can be explicitly determined and these matrices are upper (lower, respectively) triangular with positive diagonal elements, see \refapp{app.ExplicitBasisTransformation} for details.

\begin{lemma} \label{lem.PolynomialCurveEquivalence}
\emph{(Polynomial Curve Equivalence)} Bezier, monomial and Taylor curves of degree $\cdegree \in \N$ (associated with a Taylor offset $\toffset \in \R$) are equivalent, i.e., for any $t \in \R$
\begin{subequations}
\begin{align}
\bcurve_{\bpoint_0, \ldots, \bpoint_\cdegree}(t) &= \mcurve_{\mpoint_0, \ldots, \mpoint_\cdegree}(t) = \tcurve_{\tpoint_0, \ldots, \tpoint_\cdegree}(t, \toffset), 
\\
\bpmat_{\cdegree} \bbmat_{\cdegree}(t) &= \mpmat_{\cdegree} \mbmat_{\cdegree}(t) = \tpmat_{\cdegree} \tbmat_{\cdegree, \toffset}(t),
\end{align}
\end{subequations}
if and only if their respective control point matrices $\bpmat_{\cdegree}=\blist{\bpoint_0, \ldots, \bpoint_{\cdegree}}$, $\mpmat_{\cdegree} = \blist{\mpoint_0, \ldots, \mpoint_\cdegree}$, and $\tpmat_{\cdegree}=\blist{\tpoint_0, \ldots, \tpoint_\cdegree}$  are related to each other by the associated basis transformations as 
\begin{subequations}
\begin{align}
\bpmat_{\cdegree} & =  \mpmat_{\cdegree} \tfbasis_{\bbasis}^{\mbasis} (\cdegree) = \tpmat_{\cdegree} \tfbasis_{\bbasis}^{\tbasis}(\cdegree, \toffset), 
\\
\mpmat_{\cdegree} &= \bpmat_{\cdegree} \tfbasis_{\mbasis}^{\bbasis} (\cdegree) =  \tpmat_{\cdegree} \tfbasis_{\tbasis}^{\mbasis}(\cdegree, \toffset), 
\\
\tpmat_{\cdegree} &= \bpmat_{\cdegree} \tfbasis_{\bbasis}^{\tbasis} (\cdegree, \toffset) =  \mpmat_{\cdegree} \tfbasis_{\mbasis}^{\tbasis}(\cdegree, \toffset). 
\end{align}
\end{subequations}
\end{lemma}
\begin{proof}
See \refapp{app.PolynomialCurveEquivalence}.
\end{proof}

\subsection{Reparametrization of Polynomial Curves}

A common polynomial curve operation is the affine reparametrization from  one parameter interval to another; for example, for proper time allocation in order to satisfy  control (e.g., velocity, acceleration, jerk) constraints \cite{gao_wu_lin_shen_ICRA2018, gao_wu_pan_zhou_shen_IROS2018}.  

\begin{lemma}\label{lem.PolynomialReparametrization}
\emph{(Polynomial Curve Reparametrization)} B\'ezier, monomial and Taylor curves  of degree $\cdegree \in \N$ with respective control point matrices  $\bpmat_{\cdegree}=\blist{\bpoint_0, \ldots, \bpoint_{\cdegree}}$, $\mpmat_{\cdegree} = \blist{\mpoint_0, \ldots, \mpoint_\cdegree}$, and $\tpmat_{\cdegree}=\blist{\tpoint_0, \ldots, \tpoint_\cdegree}$ (and a Taylor offset $\toffset \in \R$) can be affinely reparametrized from interval $[a,b]$ to $[c,d]$ (with $a < b$ and $c < d$) as
\begin{subequations}
\begin{align}
\bcurve_{\widehat{\bpoint}_0, \ldots, \widehat{\bpoint}_\cdegree}(t) &= \bcurve_{\bpoint_0, \ldots, \bpoint_\cdegree} \plist{\tfrac{b-a}{d-c} t + \tfrac{a d - b c}{d - c}},
\\
\mcurve_{\widehat{\mpoint}_0, \ldots, \widehat{\mpoint}_\cdegree}(t) &= \mcurve_{\mpoint_0, \ldots, \mpoint_\cdegree} \plist{\tfrac{b-a}{d-c} t + \tfrac{a d - b c}{d - c}},
\\
\tcurve_{\widehat{\tpoint}_0, \ldots, \widehat{\tpoint}_\cdegree}(t, \widehat{\toffset}) &= \tcurve_{\tpoint_0, \ldots, \tpoint_\cdegree} \plist{\tfrac{b-a}{d-c} t + \tfrac{a d - b c}{d - c}, \toffset},
\end{align}
\end{subequations}
with the corresponding reparametrized control point matrices $\widehat{\bpmat}_{\cdegree} = \blist{\widehat{\bpoint}_0, \ldots, \widehat{\bpoint}_\cdegree}$, $\widehat{\mpmat}_{\cdegree} = \blist{\widehat{\mpoint}_0, \ldots, \widehat{\mpoint}_\cdegree}$, and $\widehat{\tpmat}_{\cdegree} = \blist{\widehat{\tpoint}_0, \ldots, \widehat{\tpoint}_\cdegree}$ that are given by 
\begin{subequations}
\begin{align}
\widehat{\bpmat}_{\cdegree} &= \bpmat_{\cdegree} \bbmat_{\cdegree}(t_0, \ldots, t_\cdegree) \bbmat_{\cdegree}(\widehat{t}_0, \ldots, \widehat{t}_{\cdegree})^{-1},
\\
\widehat{\mpmat}_{\cdegree} &= \mpmat_{\cdegree} \mbmat_{\cdegree}(t_0, \ldots, t_\cdegree) \mbmat_{\cdegree}(\widehat{t}_0, \ldots, \widehat{t}_{\cdegree})^{-1},
\\
\widehat{\tpmat}_{\cdegree} &= \tpmat_{\cdegree} \tbmat_{\cdegree,\toffset}(t_0, \ldots, t_\cdegree) \tbmat_{\cdegree, \widehat{\toffset}}(\widehat{t}_0, \ldots, \widehat{t}_{\cdegree})^{-1}, 
\end{align}
\end{subequations}
where  $t_0, \ldots, t_\cdegree \in \R$ are arbitrary pairwise distinct reals\reffn{fn.DistinctReals}, and $\widehat{t}_i = \tfrac{d -c}{b-a} t_i - \tfrac{ a d - b c}{ b- a}$ for $i = 0, \ldots, \cdegree$, and $\widehat{\toffset} = \tfrac{d -c}{b-a} \toffset - \tfrac{ a d - b c}{ b- a}$.
\end{lemma}
\begin{proof}
See \refapp{app.PolynomialReparametrization}. 
\end{proof}

\section{B\'ezier Metrics}
\label{sec.BezierMetric}

B\'ezier curves can be compared using various distances \cite{farin_CurvesSurfaces2002}. 
In this section, we particularly consider B\'ezier distances that define a true metric over the space of B\'ezier curves and can be computed efficiently in terms of B\'ezier control points, which is critical for computationally efficient and accurate B\'ezier approximation later in \refsec{sec.BezierAdaptiveApproximation}. 
 
\begin{definition}\label{def.BezierMetric} (\emph{B\'ezier Metric})
A real-valued distance measure $\bdist(\bcurve_{\bpoint_0, \ldots, \bpoint_{n}}, \bcurve_{\mpoint_0, \ldots, \mpoint_m})$ between two B\'ezier curves $\bcurve_{\bpoint_0, \ldots, \bpoint_n}(t)$ and $\bcurve_{\mpoint_0, \ldots, \mpoint_m}$ (t)  over the unit interval $[0,1]$  is a true metric~if 
\begin{enumerate}[i)]
\item it is nonnegative, i.e.,
\begin{align}
 \bdist(\bcurve_{\bpoint_0, \ldots, \bpoint_{n}}, \bcurve_{\mpoint_0, \ldots, \mpoint_m}) \geq 0, \nonumber
\end{align}
 
\item it is zero only for Bezier curves that are identical, i.e.,
\begin{align}
&\bdist(\bcurve_{\bpoint_0, \ldots, \bpoint_{n}}, \bcurve_{\mpoint_0, \ldots, \mpoint_{m}}) = 0 \nonumber \\
&\hspace{18mm} \Longleftrightarrow \bcurve_{\bpoint_0, \ldots, \bpoint_{n}}(t) = \bcurve_{\mpoint_0, \ldots, \mpoint_{m}}(t) \quad \forall t \in [0,1], \nonumber
\end{align}

\item it is symmetric, i.e.,  
\begin{align}
d(\bcurve_{\bpoint_0, \ldots, \bpoint_{n}}, \bcurve_{\mpoint_0, \ldots,\mpoint_{m}}) = d(\bcurve_{\mpoint_0, \ldots, \mpoint_m},\bcurve_{\bpoint_0, \ldots, \bpoint_{n}}), \nonumber
\end{align}

\item it satisfies the triangle inequality, i.e., 
\begin{align}
&d(\bcurve_{\bpoint_0, \ldots, \bpoint_{n}}, \bcurve_{\mpoint_0, \ldots, \mpoint_{m}})  \nonumber \\
&\hspace{15mm}\leq  d(\bcurve_{\bpoint_0, \ldots, \bpoint_{n}}, \bcurve_{\vect{r}_0, \ldots, \vect{r}_{k}}) + d(\bcurve_{\vect{r}_0, \ldots, \vect{r}_{k}}, \bcurve_{\mpoint_0, \ldots, \mpoint_{m}}). \nonumber
\end{align}
\end{enumerate}
\end{definition}

\subsection{L2-Norm \& Frobenius-Norm Distances of B\'ezier Curves}

A widely used B\'ezier metric is the L2-norm distance which can be analytically computed in terms of B\'ezier control points \cite{lee_park_BAMS1997, eck_CAD1995}.  
\begin{definition}\label{def.L2Distance}
(\emph{B\'ezier L2-Norm Distance}) The L2-norm distance of two B\'ezier curves $\bcurve_{\bpoint_0, \ldots, \bpoint_n}(t)$ and $\bcurve_{\mpoint_0, \ldots, \mpoint_m}(t)$ over the unit interval $[0, 1]$ is defined as
\begin{align}\label{eq.L2Distance}
\bdistL(\bcurve_{\bpoint_0, \ldots, \bpoint_n}, \bcurve_{\mpoint_0, \ldots, \mpoint_m})\! := \!\!\plist{\!\int\nolimits_{0}^{1} \!\!\!\!\norm{\bcurve_{\bpoint_0, \ldots, \bpoint_n}\!(t) \! -\! \bcurve_{\mpoint_0, \ldots, \mpoint_m}\!(t)\!}^2 \diff t \!\!\!}^{\!\!\!\frac{1}{2}}\!\!,
\end{align}
where $\norm{.}$ denotes the L2 (a.k.a. Euclidean) norm of vectors.
\end{definition}

\begin{proposition} \label{prop.L2Distance}
(\emph{Analytic Form of Bezier L2-Norm Distance}) The L2-norm distance of $\cdegree^{\text{th}}$-order B\'ezier curves $\bcurve_{\bpoint_0, \ldots, \bpoint_\cdegree}(t)$ and $\bcurve_{\mpoint_0, \ldots, \mpoint_\cdegree}(t)$,  with respective control point matrices $\bpmat_\cdegree=[\bpoint_0, \ldots, \bpoint_\cdegree]$ and $\mpmat_{\cdegree} = [\mpoint_0, \ldots, \mpoint_\cdegree]$, is explicitly given by
\begin{align}\label{eq.L2Distance}
\bdistL(\bcurve_{\bpoint_0, \ldots, \bpoint_\cdegree}, \bcurve_{\mpoint_0, \ldots, \mpoint_\cdegree})  = \trace\plist{\!\!(\bpmat_{\cdegree}-\mpmat_{\cdegree})\bwmat_{\cdegree} \tr{(\bpmat_{\cdegree} - \mpmat_{\cdegree})\!}\!}^{\!\!\frac{1}{2}}\!, \!\!\!
\end{align} 
where $\trace$ and $\tr{(.)}$ denote the trace and transpose operators, respectively,  and the B\'ezier L2-norm  weight matrix $\bwmat_{\cdegree} \in \R^{(\cdegree+1)\times (\cdegree+1)}$ is defined as
\begin{align}\label{eq.L2NormWeight}
\quad \quad \blist{\bwmat_\cdegree}_{i+1, j+1} &:= \scalebox{1.25}{$\frac{1}{2n+1}$}\frac{\binom{\cdegree}{i}\binom{\cdegree}{j}}{\binom{2\cdegree}{i+j}}, \quad  \forall i,j \in \clist{0, \ldots, \cdegree}. \!\!\!
\end{align}
%
\end{proposition}
\begin{proof}
See \refapp{app.L2Distance}.
\end{proof}

\noindent Hence, the L2-norm distance of B\'ezier curves is a weighted Frobenius norm of their control point difference, which motivates another standard B\'ezier metric.%
\footnote{Similarly, one can define alternative matrix-norm-induced distance metrics for B\'ezier curves; however, we are particularly, interested in L2-norm and Frobenius-norm distances because they are strongly related with the optimal least squares reduction of B\'ezier curves discussed in \refsec{sec.DegreeReduction}. }
\begin{definition}\label{def.FrobeniusDistance} (\emph{B\'ezier Frobenius-Norm Distance})
The Frobenius-norm distance of  $\cdegree^{\text{th}}$-order B\'ezier curves $\bcurve_{\bpoint_0, \ldots, \bpoint_{\cdegree}}(t)$ and $\bcurve_{\mpoint_0, \ldots, \mpoint_{\cdegree}}(t)$,  with respective control point matrices $\bpmat_{\cdegree}=[\bpoint_0, \ldots, \bpoint_\cdegree]$ and $\mpmat_{\cdegree} = [\mpoint_0, \ldots, \mpoint_\cdegree]$, is defined~as 
\begin{subequations} \label{eq.FrobeniusDistance}
\begin{align}
\bdistF(\bcurve_{\bpoint_0, \ldots, \bpoint_\cdegree}, \bcurve_{\mpoint_0, \ldots, \mpoint_\cdegree}) &:=  \norm{\bpmat_{\cdegree} - \mpmat_{\cdegree}}_F,
\\
& = \trace\plist{(\bpmat_{\cdegree} - \mpmat_{\cdegree}) \tr{(\bpmat_{\cdegree} - \mpmat_{\cdegree})}}^{\frac{1}{2}},
 \\
&= \sqrt{\sum_{i=0}^{\cdegree}\norm{\bpoint_i - \mpoint_i}^2},
\end{align}
\end{subequations}
where $\norm{.}_{F}$ denotes the Frobenius norm of matrices.
\end{definition}
\noindent It is useful to remark that any distance measure of $\cdegree^{\text{th}}$-order B\'ezier curves can be adapted to handle B\'ezier curves of different orders  via degree elevation (see \refsec{sec.DegreeElevation}).\footnote{\label{fn.BezierDistanceElevation}Let $d(\bcurve_{\bpoint_0, \ldots, \bpoint_{\cdegree}}, \bcurve_{\mpoint_0, \ldots, \mpoint_{\cdegree}})$ be a distance measure for $\cdegree^{\text{th}}$-order B\'ezier curves. It can be extended to any arbitrary B\'ezier curves  $\bcurve_{\bpoint_0, \ldots, \bpoint_{n}}$ and $\bcurve_{\mpoint_0, \ldots, \mpoint_{m}}$ as
\begin{align*}
d(\bcurve_{\bpoint_0, \ldots, \bpoint_{n}}, \bcurve_{\mpoint_0, \ldots, \mpoint_{m}}):= d(\bcurve_{[\bpoint_0, \ldots, \bpoint_{n}] \emat(n, \max(n,m))}, \bcurve_{[\mpoint_0, \ldots, \mpoint_{m}] \emat(n, \max(n,m))})
\end{align*}
where $\emat(n,m)$ denotes the elevation matrix defined in \refdef{def.DegreeElevation}.
}

\begin{figure*}[b]
\centering
\begin{tabular}{@{}c@{}c@{}c@{}c@{}c@{}c@{}}
\includegraphics[width=0.165\textwidth]{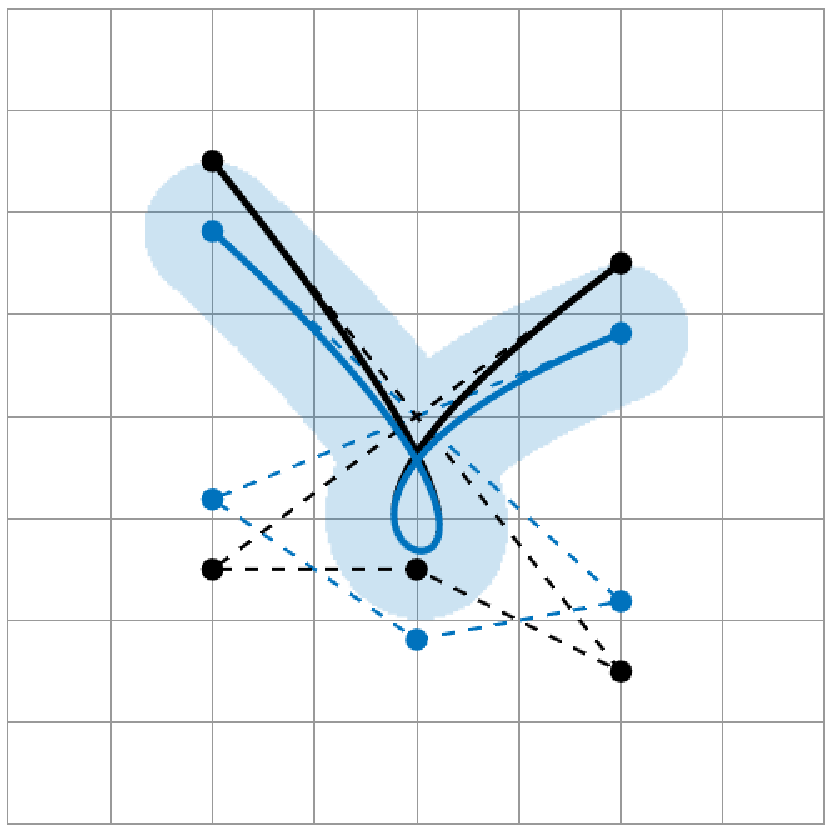} & 
\includegraphics[width=0.165\textwidth]{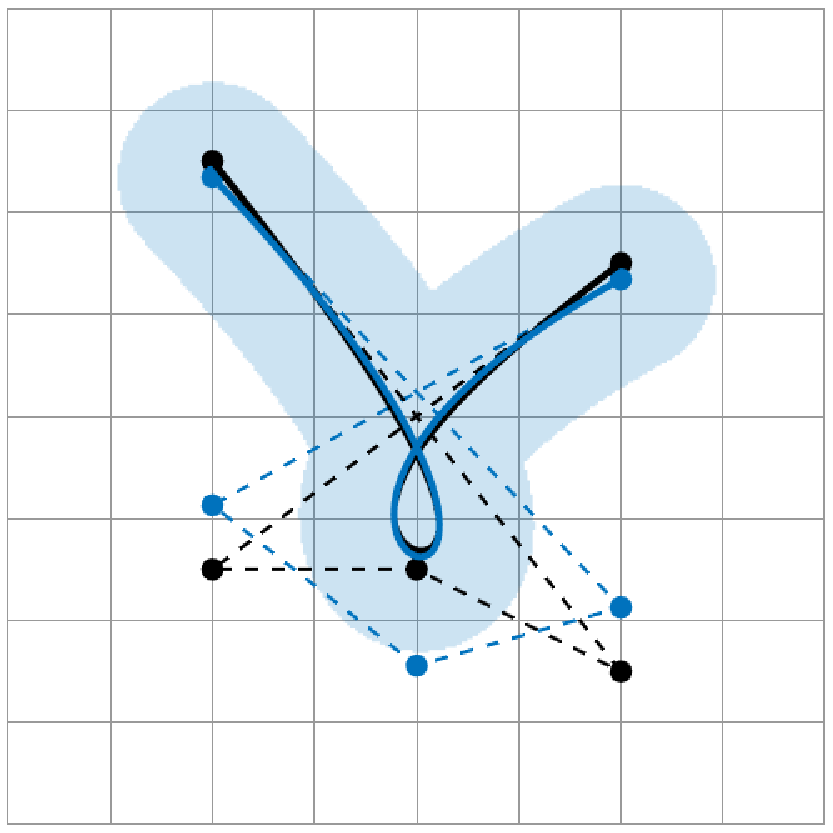} &  
\includegraphics[width=0.165\textwidth]{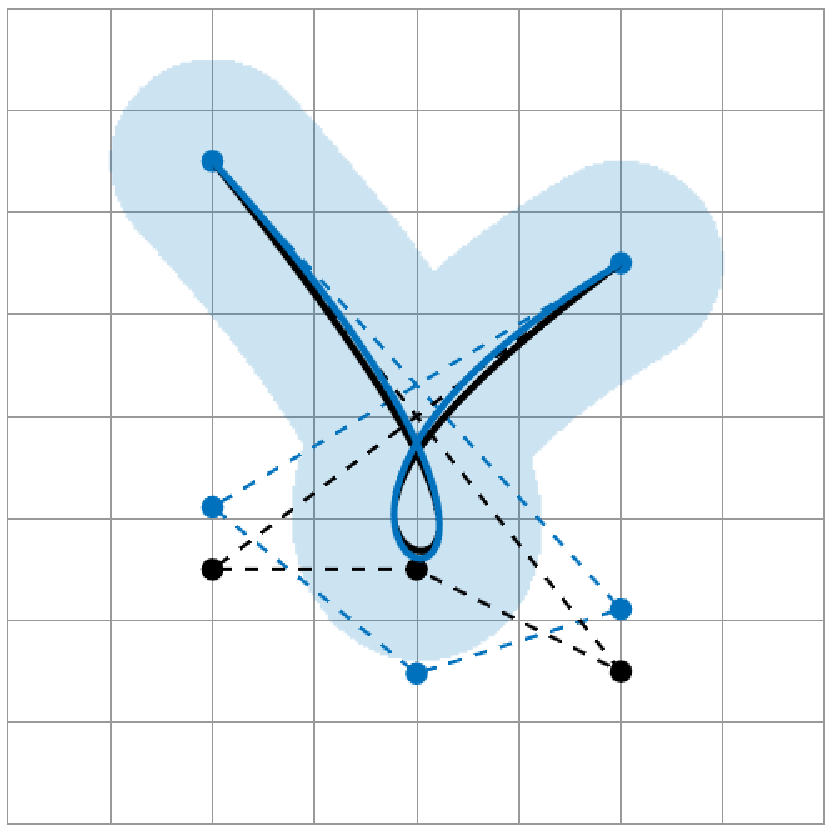} &
\includegraphics[width=0.165\textwidth]{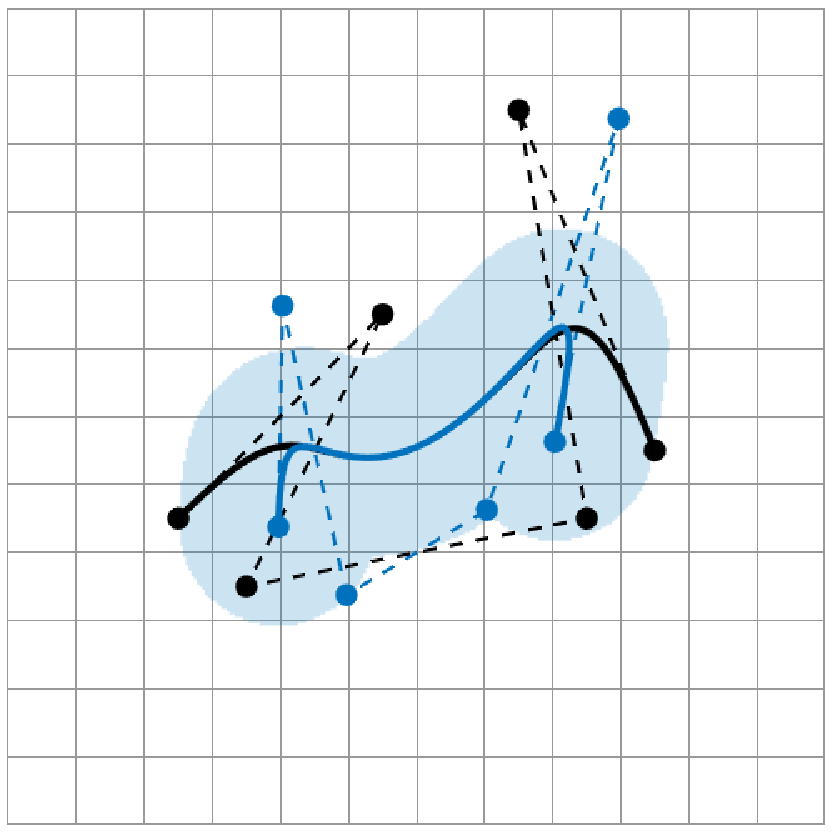} & 
\includegraphics[width=0.165\textwidth]{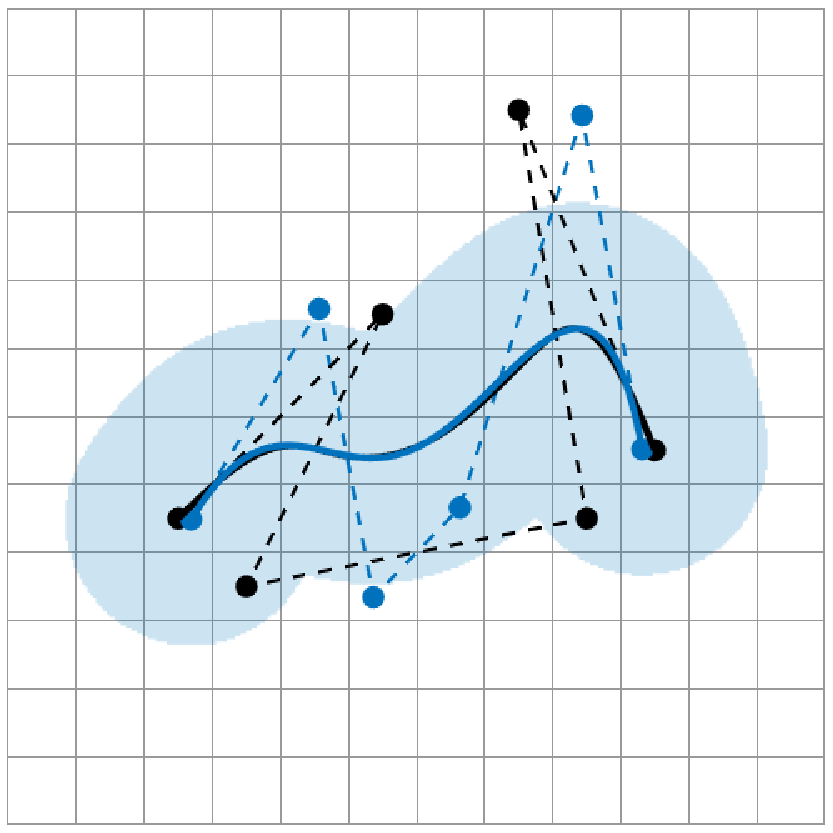} &  
\includegraphics[width=0.165\textwidth]{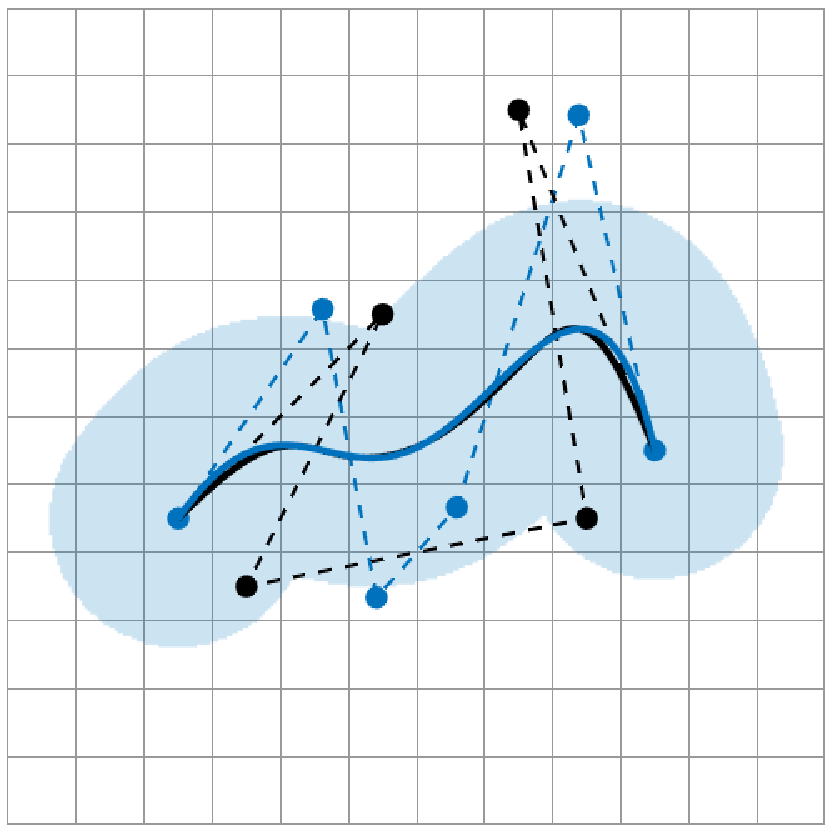} 
\\
\includegraphics[width=0.165\textwidth]{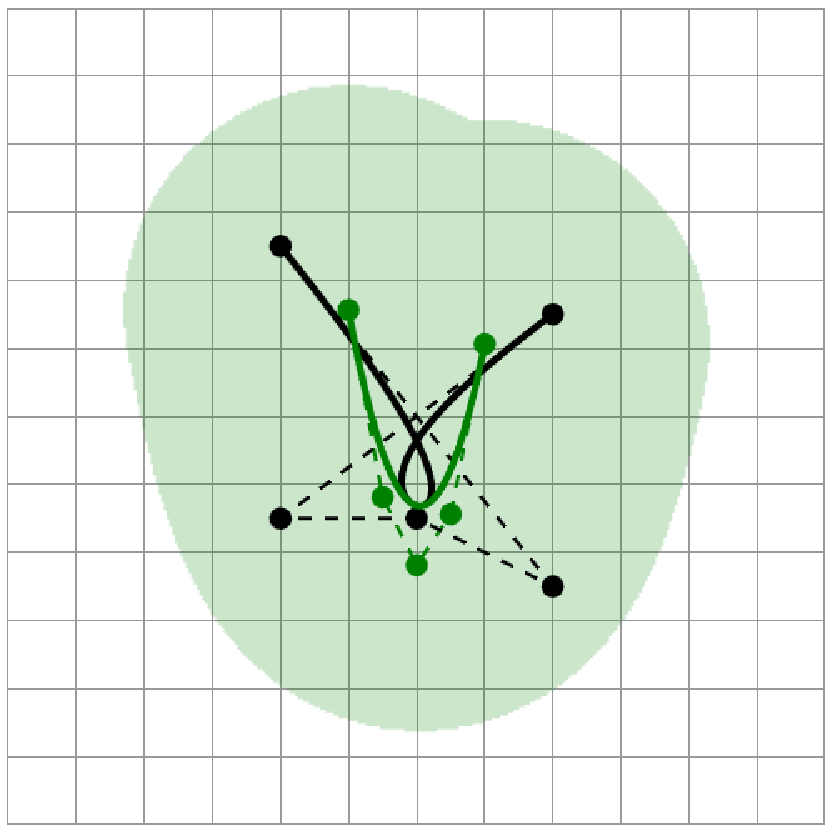} &
\includegraphics[width=0.165\textwidth]{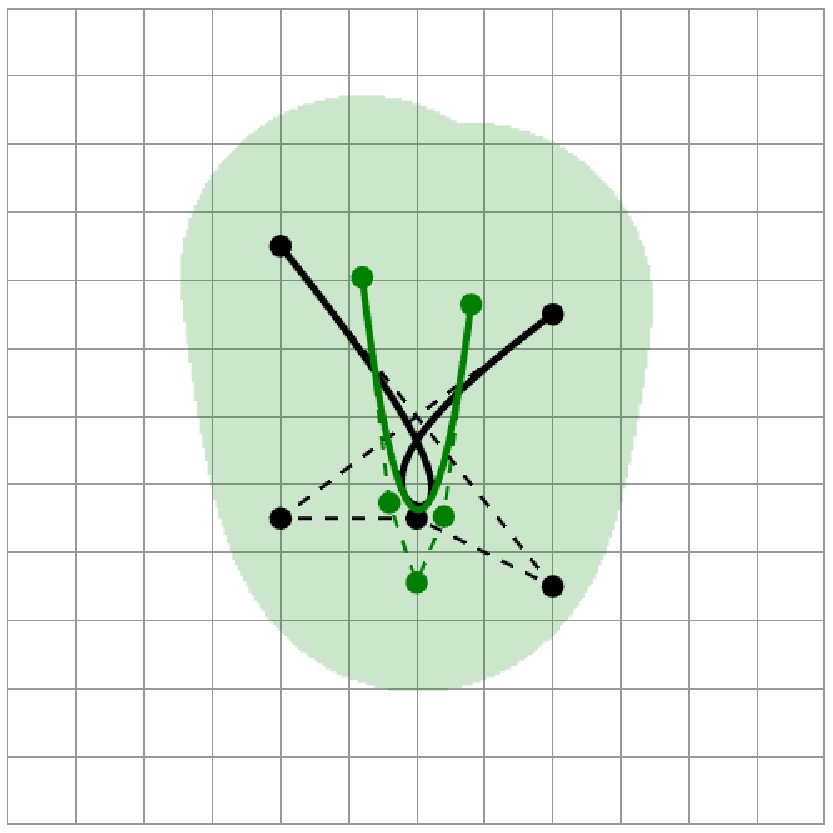} &
\includegraphics[width=0.165\textwidth]{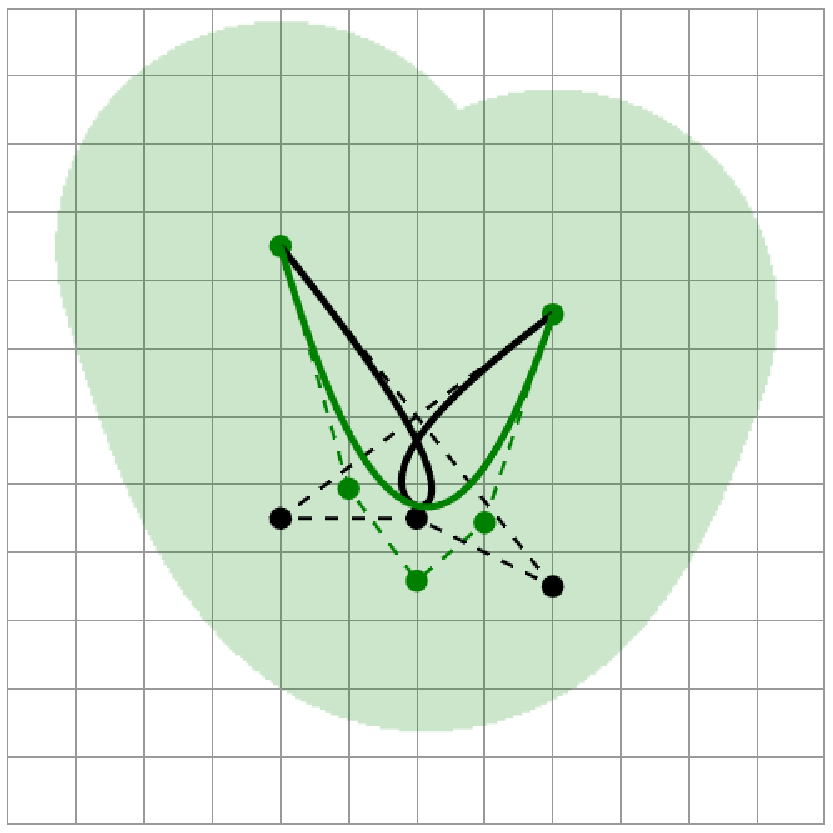} &
\includegraphics[width=0.165\textwidth]{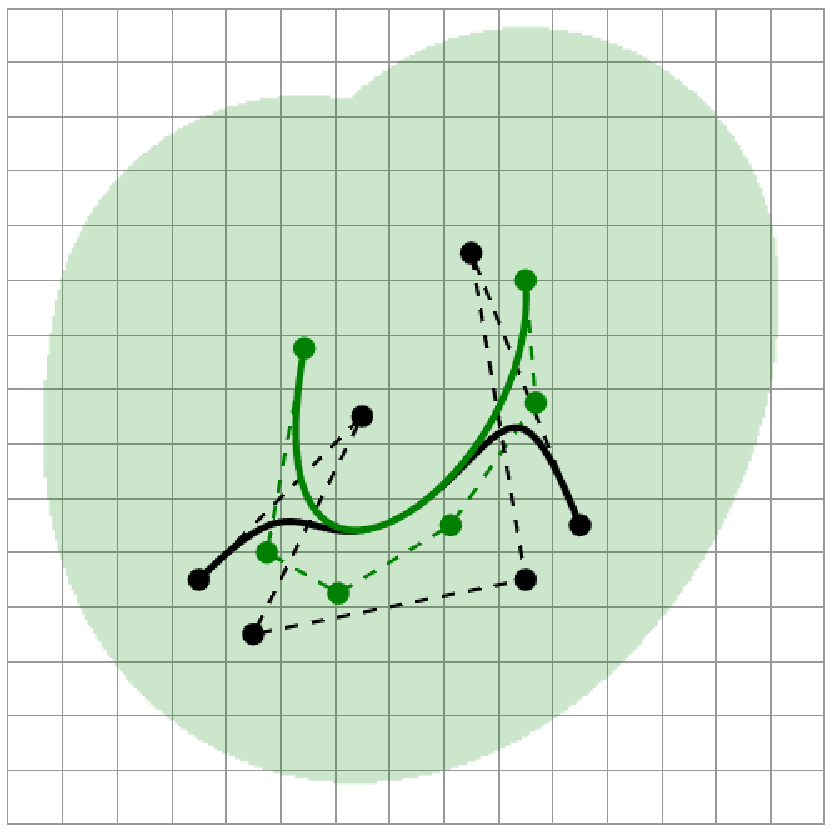} &
\includegraphics[width=0.165\textwidth]{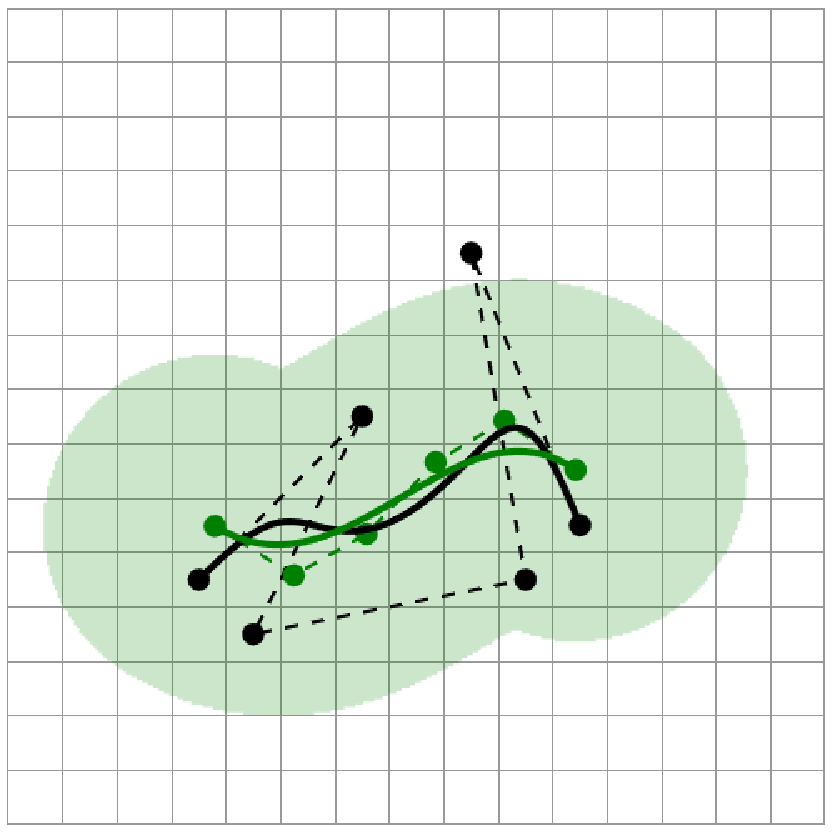} &
\includegraphics[width=0.165\textwidth]{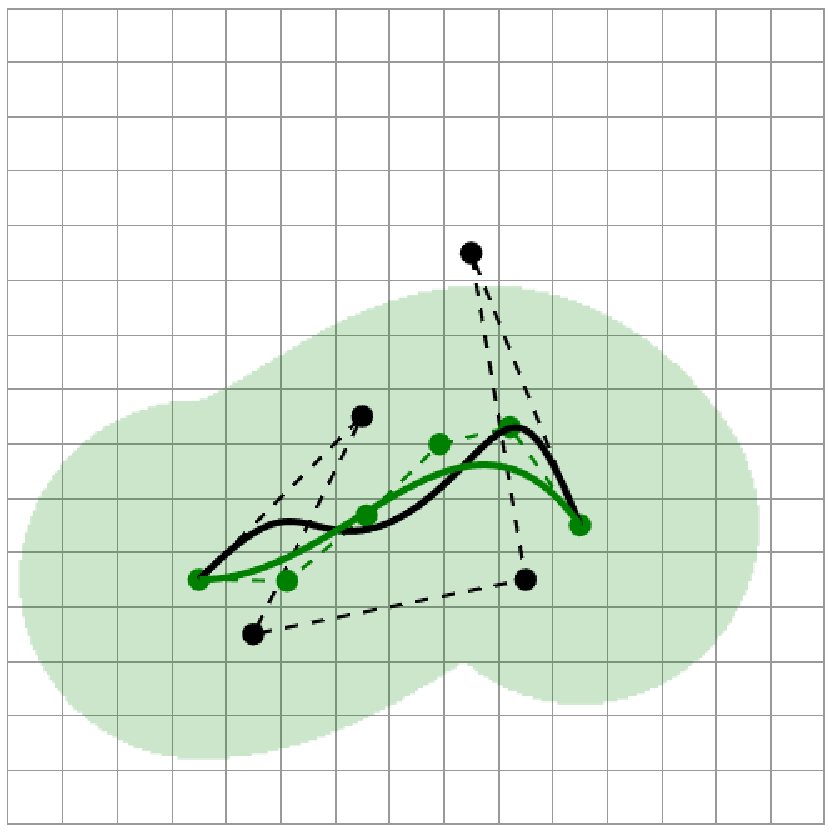} 
\\
\includegraphics[width=0.165\textwidth]{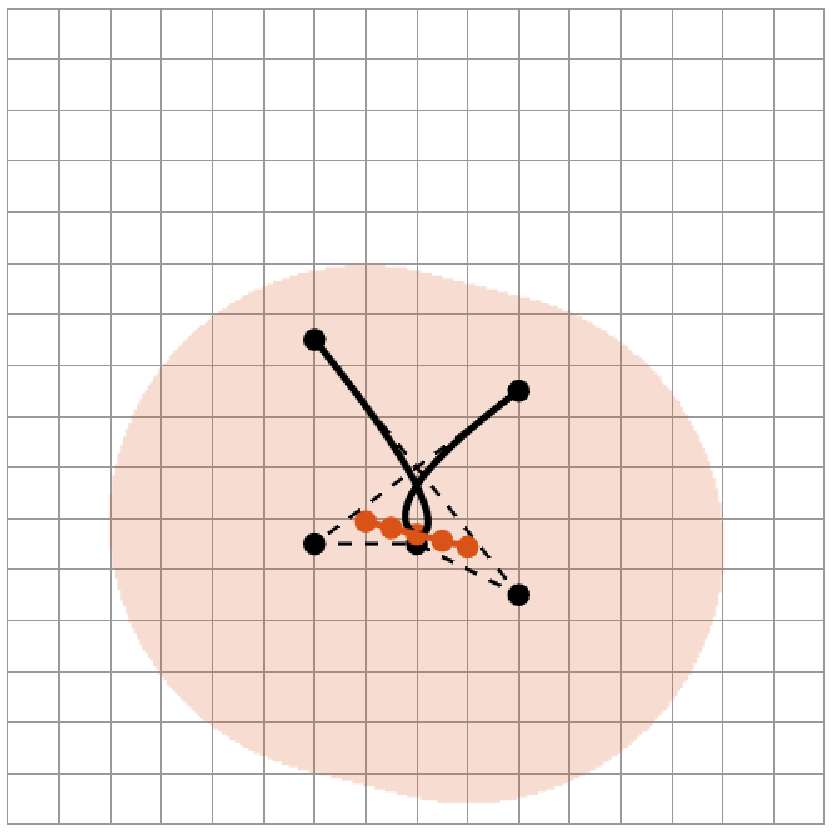} &
\includegraphics[width=0.165\textwidth]{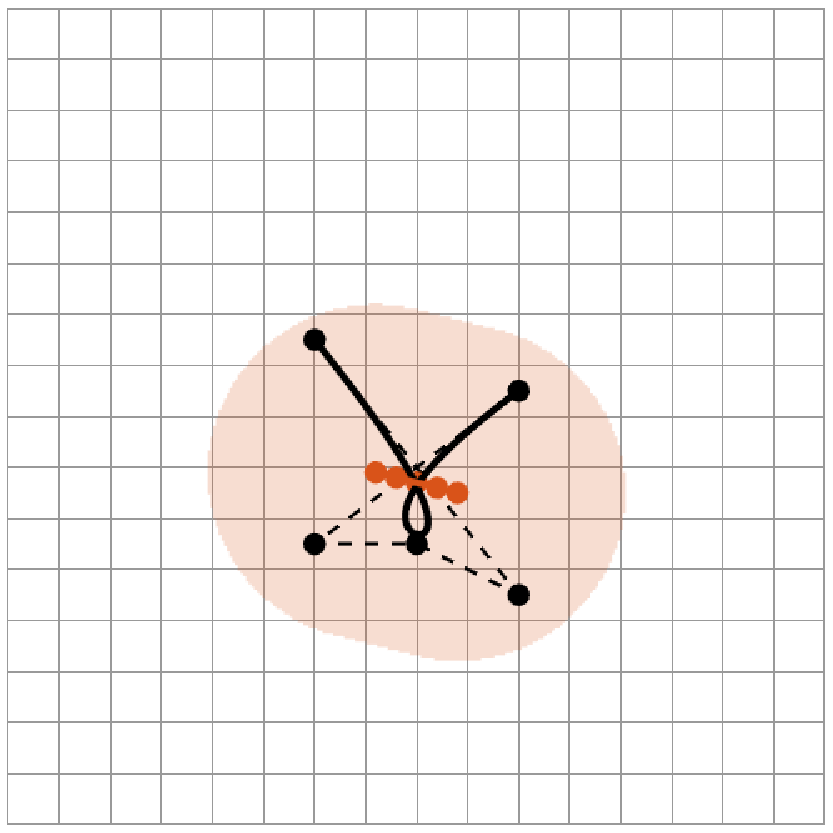} &
\includegraphics[width=0.165\textwidth]{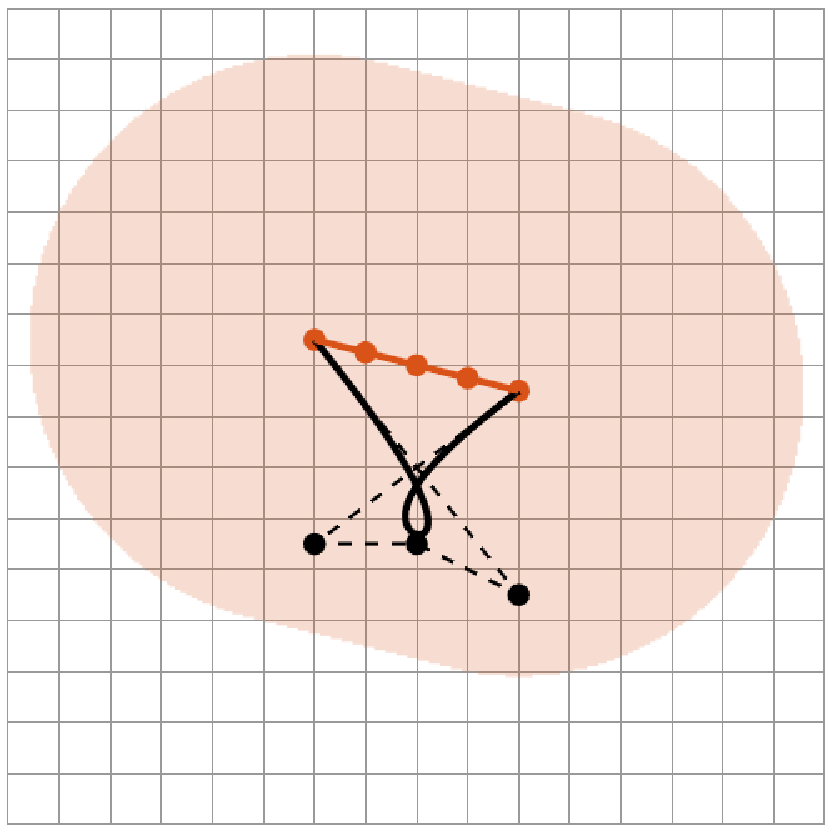} &
\includegraphics[width=0.165\textwidth]{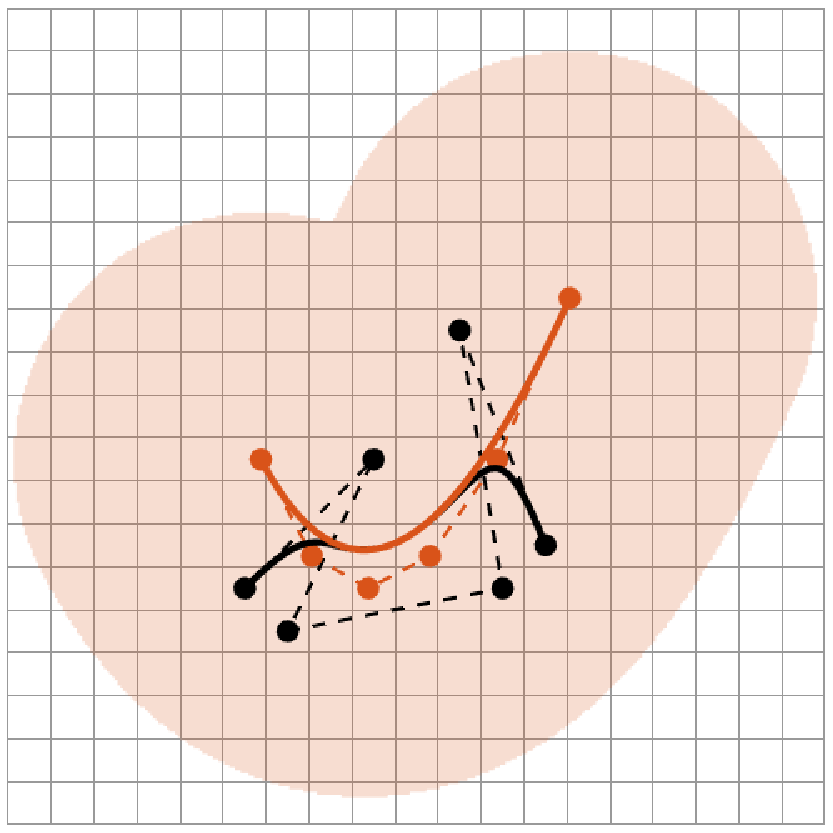} &
\includegraphics[width=0.165\textwidth]{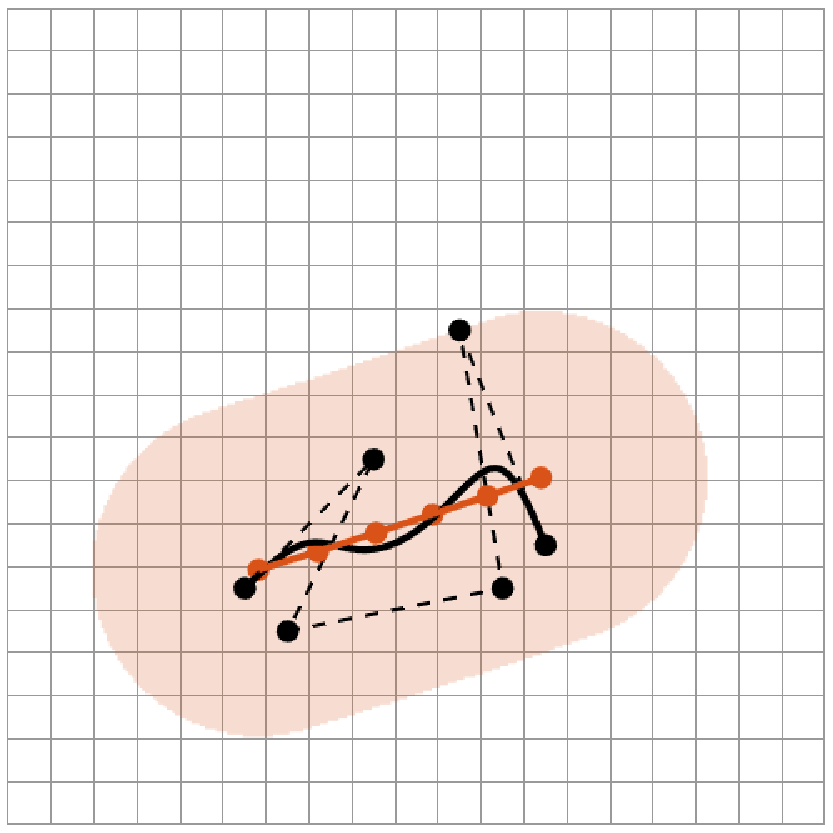} &
\includegraphics[width=0.165\textwidth]{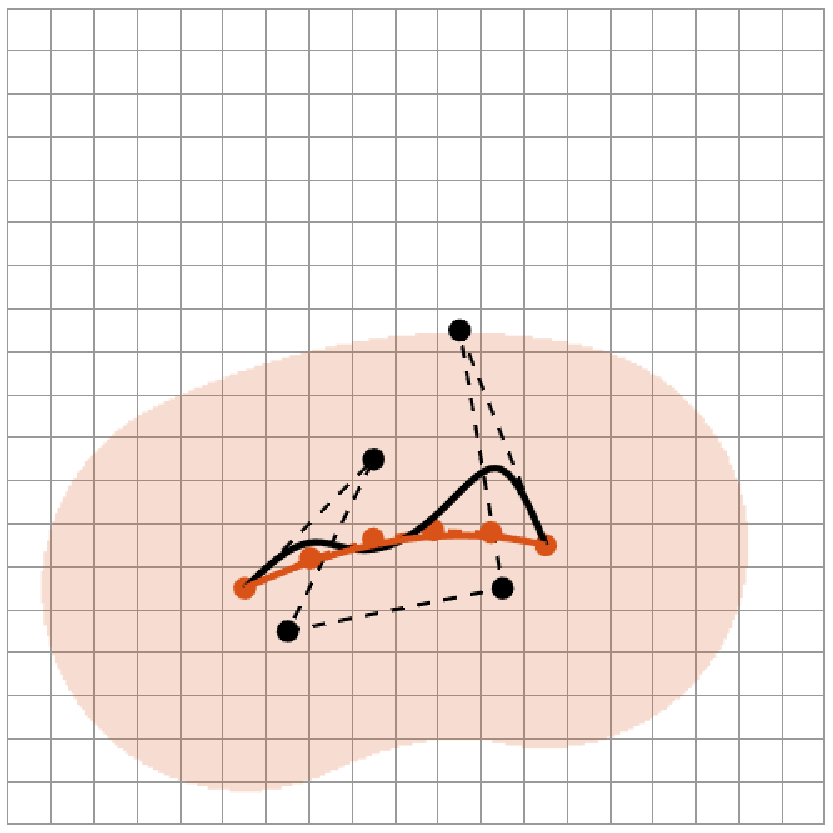} 
\\[-1mm]
\scriptsize{(a)} & \scriptsize{(b)} & \scriptsize{(c)} & \scriptsize{(d)} & \scriptsize{(e)} & \scriptsize{(f)}
\end{tabular}
\vspace{-1mm}
\caption{The maximum control-point distance defines a geometric bound on  a B\'ezier curve (black) relative to another B\'ezier curve (blue, green, red), for example,  its degree reduction. (top) Degree-one, (middle) degree-two  and (bottom) degree-three reduction based on (a, d) Taylor reduction, (b, e) least squares reduction, (c, f) uniform matching reduction. Here, the reduced B\'ezier curves are elevated again to have the same number of control points.} 
\label{fig.RelativeBezierBound}
\end{figure*}

\subsection{Hausdorff and Maximum Distances of B\'ezier Curves}

As an alternative to norm-induced algebraic B\'ezier metrics, one can also compare Bezier curves based on set-theoretic distance measures as follows. 
\begin{definition}\label{def.HaussdorffMaximumDistance}
(\emph{B\'ezier Haussdorff \& Maximum Distances}) 
The Haussdorff  and (parameterwise) maximum distances between two B\'ezier curves $\bcurve_{\bpoint_0, \ldots, \bpoint_n}(t)$ and $\bcurve_{\mpoint_0, \ldots, \mpoint_m}(t)$ over the unit interval $[0,1]$ are, respectively,  defined as
\begin{align}
&\!\bdistH(\bcurve_{\bpoint_0, \ldots, \bpoint_n}, \bcurve_{\mpoint_0, \ldots, \mpoint_m})  \nonumber \\
& \hspace{4mm} : = \max \!\plist{\!
\begin{array}{@{}c@{}}
\max\limits_{t_p \in [0,1]} \min\limits_{t_q \in [0,1]} \norm{\bcurve_{\bpoint_0, \ldots, \bpoint_n}(t_p) \! - \! \bcurve_{\mpoint_0, \ldots, \mpoint_m}(t_q)\!}, \\
\max\limits_{t_q \in [0,1]} \min\limits_{t_p \in [0,1]} \norm{\bcurve_{\bpoint_0, \ldots, \bpoint_n}(t_p) \!- \!\bcurve_{\mpoint_0, \ldots, \mpoint_m}(t_q)\!}
\end{array}
\!}\!\!, \!\!\! \! \label{eq.HaussdorffDistance}
\\
&\!\bdistM(\bcurve_{\bpoint_0, \ldots, \bpoint_n}, \bcurve_{\mpoint_0, \ldots, \mpoint_m}) \!\! := \! \! \max_{t \in [0,1]}\! \norm{\bcurve_{\bpoint_0, \ldots, \bpoint_n}\!(t) \!-\! \bcurve_{\mpoint_0, \ldots, \mpoint_m}\!(t)\!}.  \!\!\!\! \!\label{eq.MaximumDistance}
\end{align}
\end{definition}

\noindent Unfortunately, both the Haussdorff and  maximum distances of B\'ezier curves do not accept an analytic solution in terms of control points in general, but can be analytically bounded above by the maximum distance of B\'ezier control points.

\subsection{Control-Point Distance of B\'ezier Curves}

We introduce a new analytic B\'ezier metric that defines a relative geometric bound between B\'ezier curves, see \reffig{fig.RelativeBezierBound}.  

\begin{definition} \label{def.CtrlDistance}
(\emph{Bezier Control-Point Distance})
The maximum \emph{control-point distance} of $\cdegree^{\text{th}}$-order B\'ezier curves  is defined as
\begin{align}\label{eq.CtrlDistance}
\bdistC \plist{\bcurve_{\bpoint_0, \ldots, \bpoint_n}, \bcurve_{\mpoint_0, \ldots, \mpoint_n}} := \max_{i= 0, \ldots, n} \norm{\bpoint_i - \mpoint_i}.
\end{align}
\end{definition}

\begin{proposition}\label{prop.DistanceOrder}
(\emph{B\'ezier Distance Order})
The Frobenius, Hausdorff, and parameterwise \& control-pointwise maximum distances of $\cdegree^{\text{th}}$-order B\'ezier curves satisfy
\begin{subequations}\label{eq.DistanceOrder}
\begin{align}
\bdistH(\bcurve_{\bpoint_0, \ldots, \bpoint_n}, \bcurve_{\mpoint_0, \ldots, \mpoint_n}) &\leq \bdistM(\bcurve_{\bpoint_0, \ldots, \bpoint_n}, \bcurve_{\mpoint_0, \ldots, \mpoint_n}), 
\\ 
& \leq \bdistC(\bcurve_{\bpoint_0, \ldots, \bpoint_n},\bcurve_{\mpoint_0, \ldots, \mpoint_n}), 
\\
& \leq \bdistF(\bcurve_{\bpoint_0, \ldots, \bpoint_n},\bcurve_{\mpoint_0, \ldots, \mpoint_n}), 
\\  
& \leq  \sqrt{\cdegree} \,  \bdistC(\bcurve_{\bpoint_0, \ldots, \bpoint_n},\bcurve_{\mpoint_0, \ldots, \mpoint_n}).
\end{align} 
\end{subequations}
\end{proposition}
\begin{proof}
See \refapp{app.DistanceOrder}.
\end{proof}

\begin{proposition}\label{prop.RelativeBezierBound}
(Relative B\'ezier Bound)
In the $\cdim$-dimensional Euclidean space $\R^{\cdim}$, an $n^{\text{th}}$-order B\'ezier curve is contained in the dilation of an another $n^{\text{th}}$-order B\'ezier curve by their maximum control-point distance, i.e.,
\begin{align}
\bcurve_{\bpoint_0, \ldots, \bpoint_n}([0,1]) & \nonumber \\
& \hspace{-10mm} \subseteq \bcurve_{\mpoint_0, \ldots, \mpoint_n}([0,1]) \oplus  \ball_{\cdim}(\bdistC(\bcurve_{\bpoint_0, \ldots, \bpoint_n}, \bcurve_{\mpoint_0, \ldots, \mpoint_n})\!)\!, 
\\
& \hspace{-10mm} \subseteq \conv\plist{\mpoint_0, \ldots, \mpoint_n} \oplus  \ball_{\cdim}(\bdistC(\bcurve_{\bpoint_0, \ldots, \bpoint_n}, \bcurve_{\mpoint_0, \ldots, \mpoint_n})\!)\!, \!\!\!\!
\end{align}
where $\ball_{\cdim}(r):= \clist{\vect{x} \in \R^{\cdim} \big| \norm{\vect{x}} \leq r }$ denotes the $\cdim$-dimensional closed Euclidean ball of radius $r \geq 0$  centered at the origin. 
\end{proposition}
\begin{proof}
See \refapp{app.RelativeBezierBound}.
\end{proof}

\noindent The ordering (a.k.a. equivalence) relation of B\'ezier distances in \refprop{prop.DistanceOrder}  makes the maximum control-point  distance  a computationally efficient tool for discriminative comparison of B\'ezier curves independent of their degree $\cdegree$, whereas
the Frobenius-norm distance  tends to increase with increasing $\cdegree$. 
Moreover, similar to the convexity property in \refpropty{propty.BezierConvexity}, the relative bound of B\'ezier curves via the maximum control-point distance in \refprop{prop.RelativeBezierBound} offers an alternative way of constraining B\'ezier control points for safety and constraint verification in motion planning.   
We continue below with how different B\'ezier metrics behave under B\'ezier degree elevation and reduction operations.

\section{Bezier Degree Elevation \& Reduction}
\label{sec.BezierElevationReduction}

In this section, we briefly summarize the degree elevation and reduction operations of B\'ezier curves for (approximately) representing them with more or fewer control points.
In particular, degree reduction is another building block of high-order B\'ezier approximations with multiple low-order B\'ezier segments.
Accordingly, we introduce a new degree reduction method for approximating B\'ezier geometry more accurately.

\subsection{Degree Elevation}
\label{sec.DegreeElevation}

Degree elevation generates an exact representation of B\'ezier curves with more control points, as illustrated in \mbox{\reffig{fig.DegreeElevation}}.

\begin{definition}\label{def.DegreeElevation}
\emph{(Degree Elevation)}
A B\'ezier curve $\bcurve_{\mpoint_0, \ldots, \mpoint_m}$ of higher degree $m$ with control points $\mpoint_0, \ldots, \mpoint_m$ is said to be the \emph{degree elevation} of another B\'ezier curve  $\bcurve_{\bpoint_0, \ldots, \bpoint_n}$ of lower degree $n \leq m$ with control points $\bpoint_0, \ldots, \bpoint_n$ if and only if the curves are parameterwise identical, i.e.,
\begin{align}
\bcurve_{\mpoint_0, \ldots, \mpoint_m}(t) = \bcurve_{\bpoint_0, \ldots, \bpoint_n}(t) \quad \forall t \in \R.
\end{align} 
\end{definition}

B\'ezier degree elevation can be analytically  computed as:

\begin{proposition}(\emph{Elevated Control Points})\label{prop.ElevationControlPoint}
A B\'ezier curve  $\bcurve_{\mpoint_0, \ldots, \mpoint_m}$ of degree $m$  is the degree elevation of a B\'ezier curve  $\bcurve_{\bpoint_0, \ldots, \bpoint_n}$ of degree $n \leq m$
if and only if the control point matrices $\mpmat_{m} = \blist{\mpoint_0, \ldots, \mpoint_m}$ and $ \bpmat_{n}= \blist{\bpoint_0, \ldots, \bpoint_n}$ satisfy
\begin{align} \label{eq.ElevationControlPoints}
\mpmat_{m}= \bpmat_{n}  \emat(n ,m),
\end{align}
where the degree elevation matrix $\emat(n, m)$ is defined as
\begin{align} \label{eq.ElevationMatrix}
\emat(n, m) := \tfbasis_{\mbasis}^{\bbasis}(n) \mat{I}_{(n+1) \times (m+1)} \tfbasis_{\bbasis}^{\mbasis}(m),
\end{align}
and $\mat{I}_{n+1 \times (m+1)}$ is the  $(n+1) \times (m+1)$ rectangular identify matrix with ones in the main diagonal and zeros elsewhere.  
\end{proposition}
\begin{proof}
See \refapp{app.ElevationControlPoint}.
\end{proof}
\noindent Observe that \refeq{eq.ElevationMatrix} leverages the change of basis between Bernstein and monomial bases because degree elevation of monomial curves is trivial.

Higher-order Bernstein basis vectors can also be obtained from lower ones via degree elevation, which offers another way of determining the elevation matrix.

\begin{proposition}\label{prop.ElevationMatrixBernstein}
(\emph{Elevated Bernstein Basis})
For any $n \leq m$, the elevation matrix $\emat(n,m)$ relates Bernstein basis vectors as
\begin{align}\label{eq.BernsteinBasisElevationMatrix}
\bbasis_{n}(t) = \emat(n,m) \bbasis_{m}(t) \quad  \forall t \in \R.
\end{align} 
Hence, $\emat(n,m)$ can be determined as
\begin{align}\label{eq.ElevationMatrixBernstein}
\emat(n,m) = \bbmat_{n}(t_0, \ldots, t_m) \bbmat_{m}(t_0, \ldots, t_m)^{-1},
\end{align}
where $t_0, \ldots, t_m \in \R$ are an arbitrary selection of pairwise distinct curve  parameters, i.e., $t_i \neq t_j$ for all $i \neq j$. 
\end{proposition}
\begin{proof}
See \refapp{app.ElevationMatrixBernstein}.
\end{proof}

\begin{figure}[t]
\centering
\begin{tabular}{@{}cc@{}}
\includegraphics[width=0.2\textwidth]{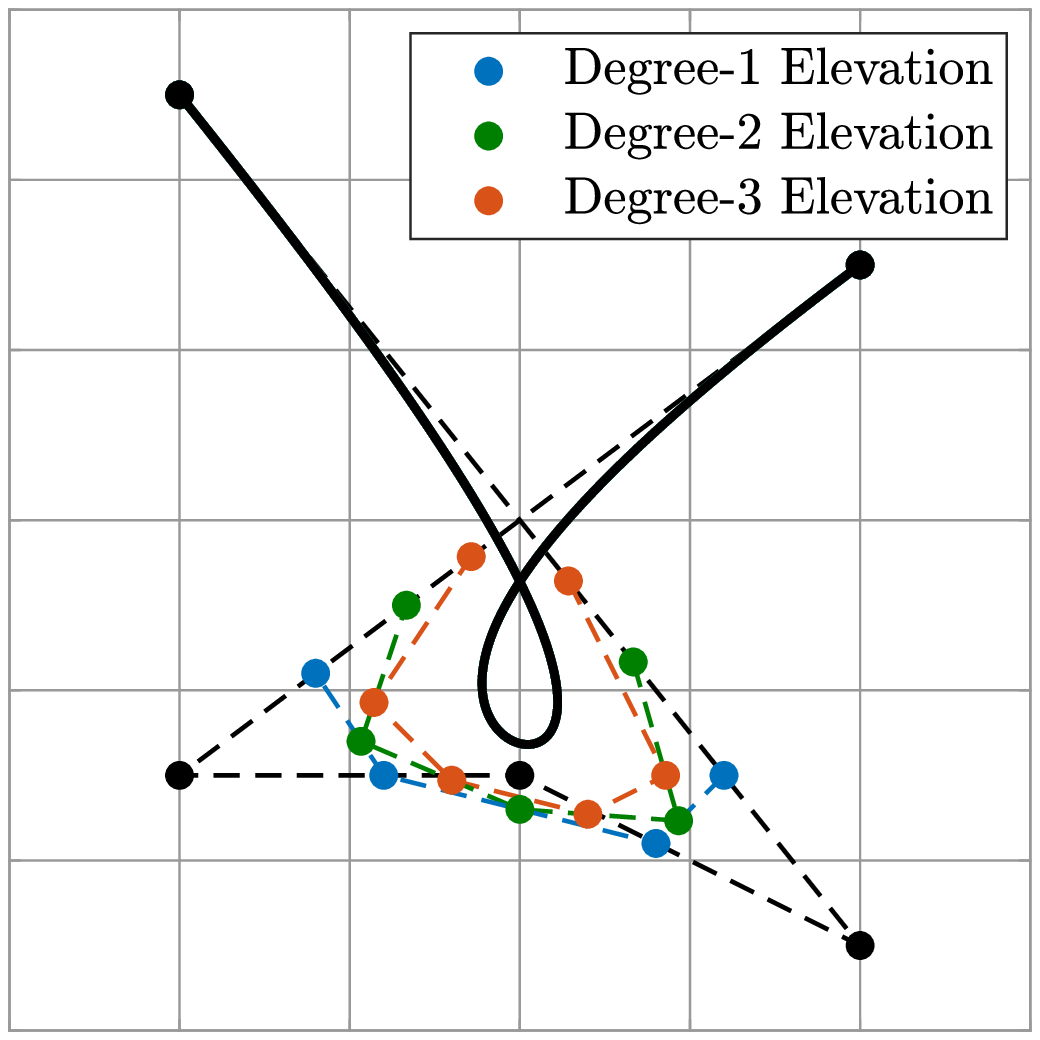} & \includegraphics[width=0.2\textwidth]{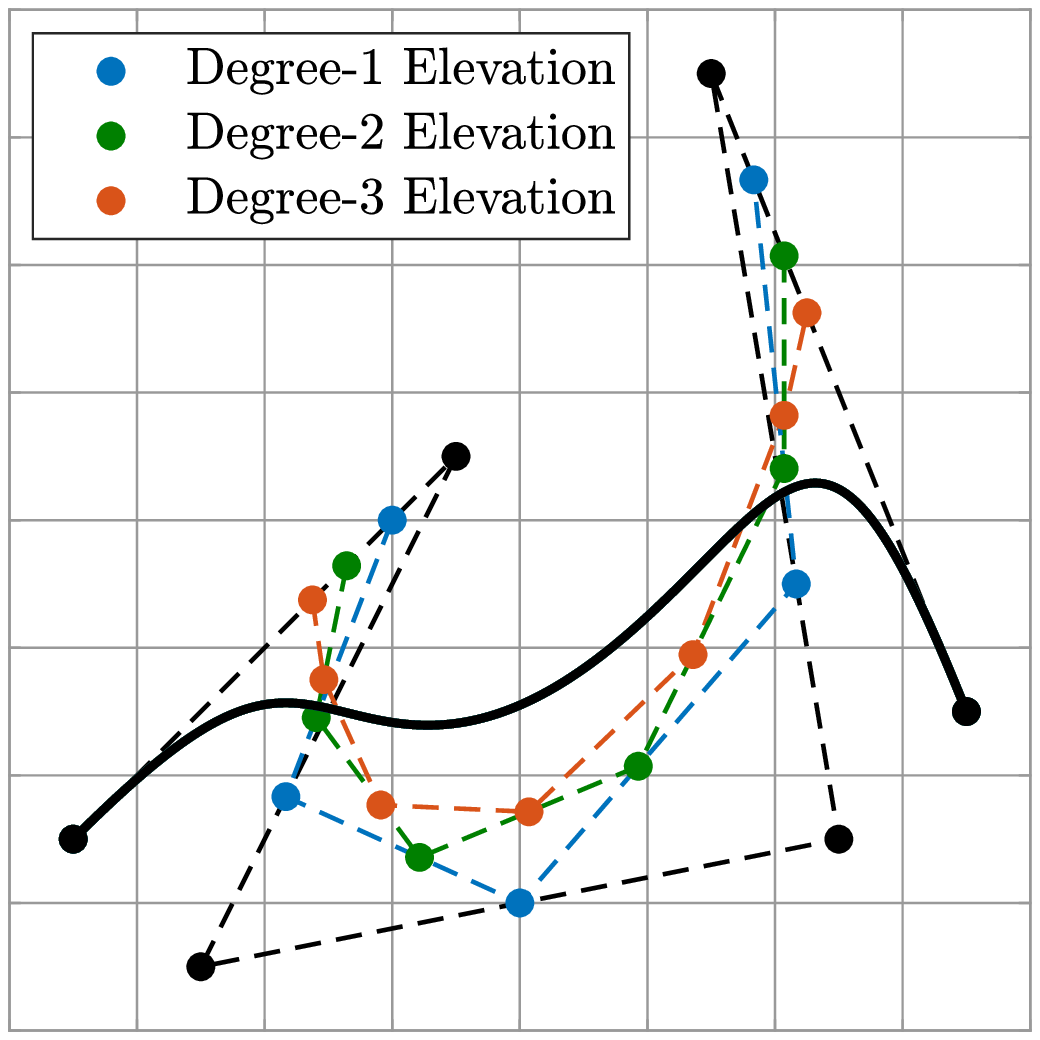}
\end{tabular}
\vspace{-1mm}
\caption{Degree elevation represents B\'ezier curves exactly by more control points that asymptotically converge to the curve itself as the amount of degree elevation goes to infinity.}
\label{fig.DegreeElevation}
\end{figure}

\begin{figure*}[t]
\centering
\begin{tabular}{@{}c@{\hspace{0.5mm}}c@{\hspace{0.5mm}}c@{\hspace{0.5mm}}c@{\hspace{0.5mm}}c@{\hspace{0.5mm}}c@{}}
\includegraphics[width=0.165\textwidth]{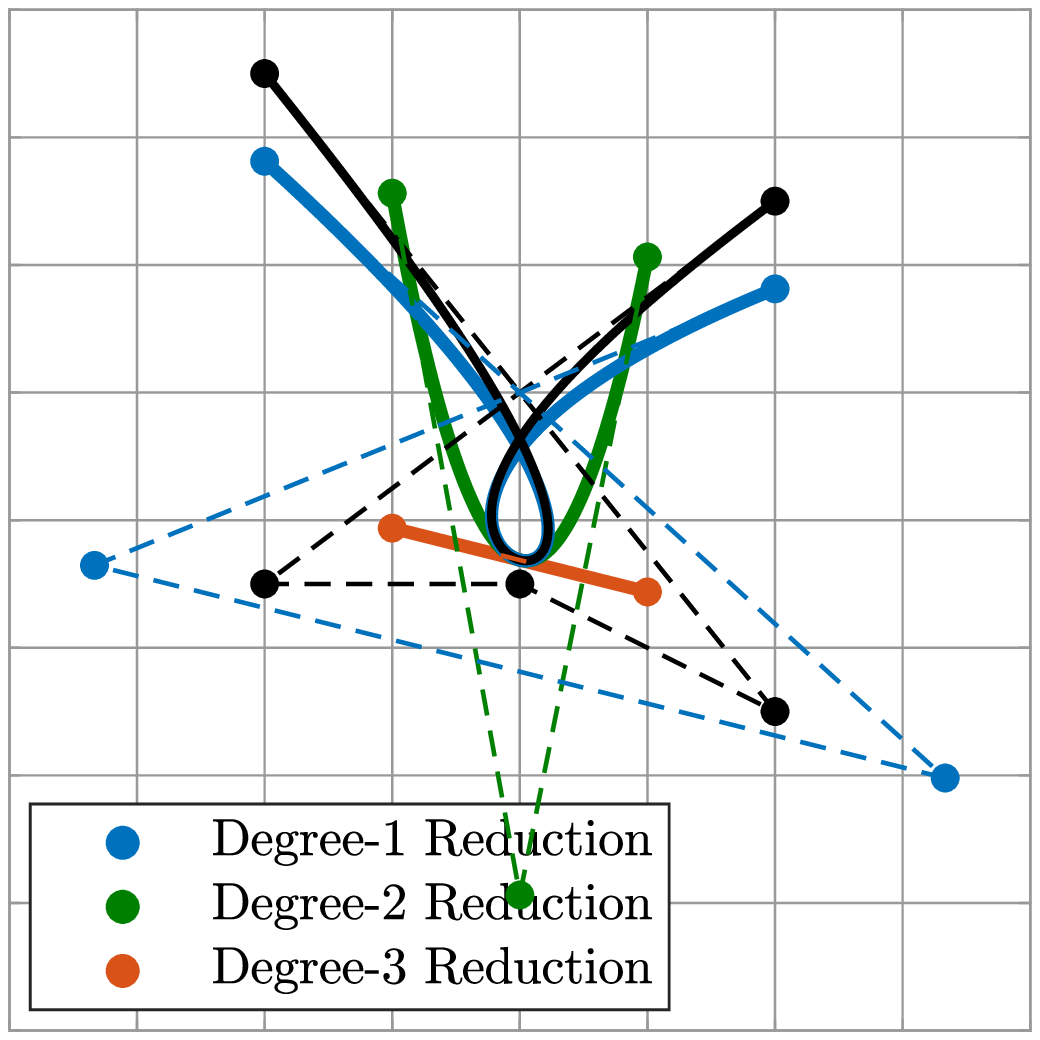} &
\includegraphics[width=0.165\textwidth]{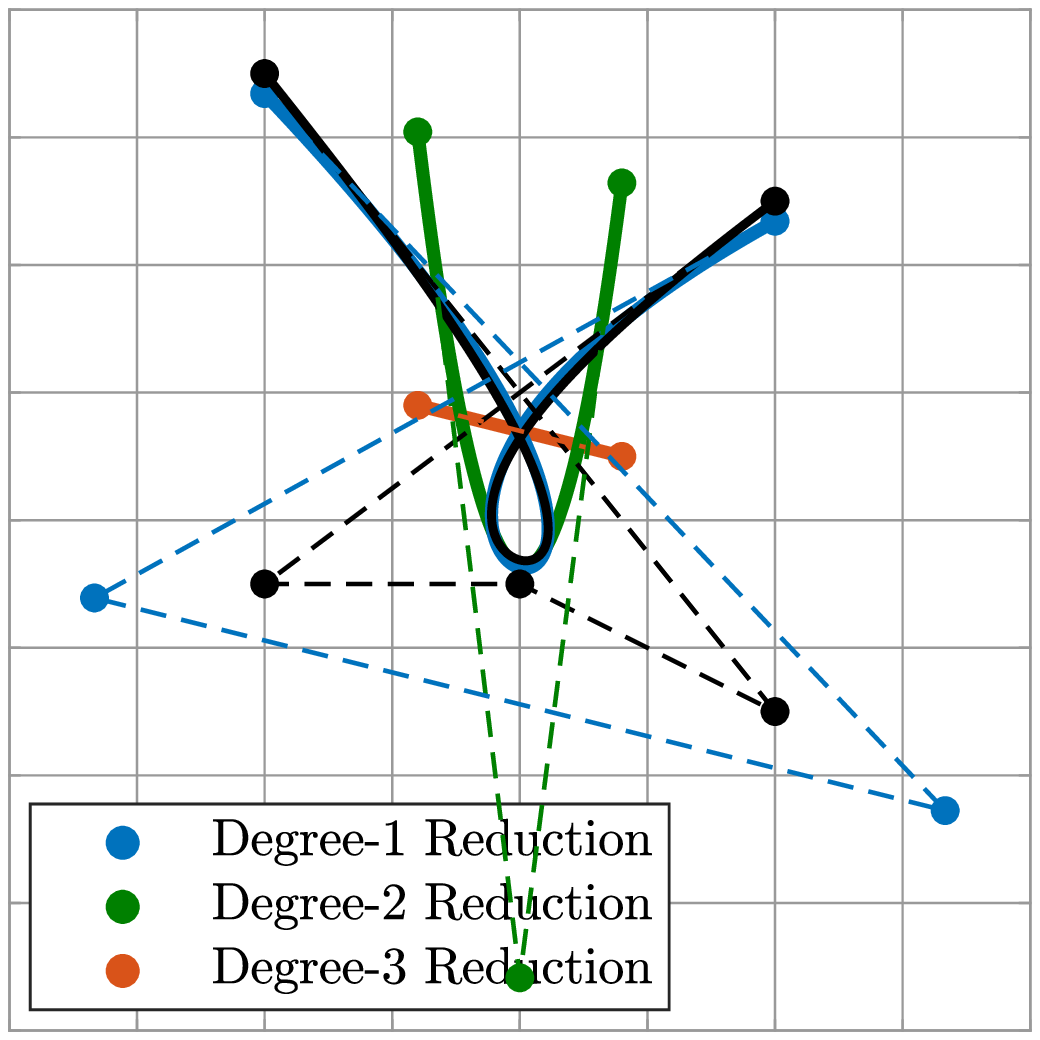} &
\includegraphics[width=0.165\textwidth]{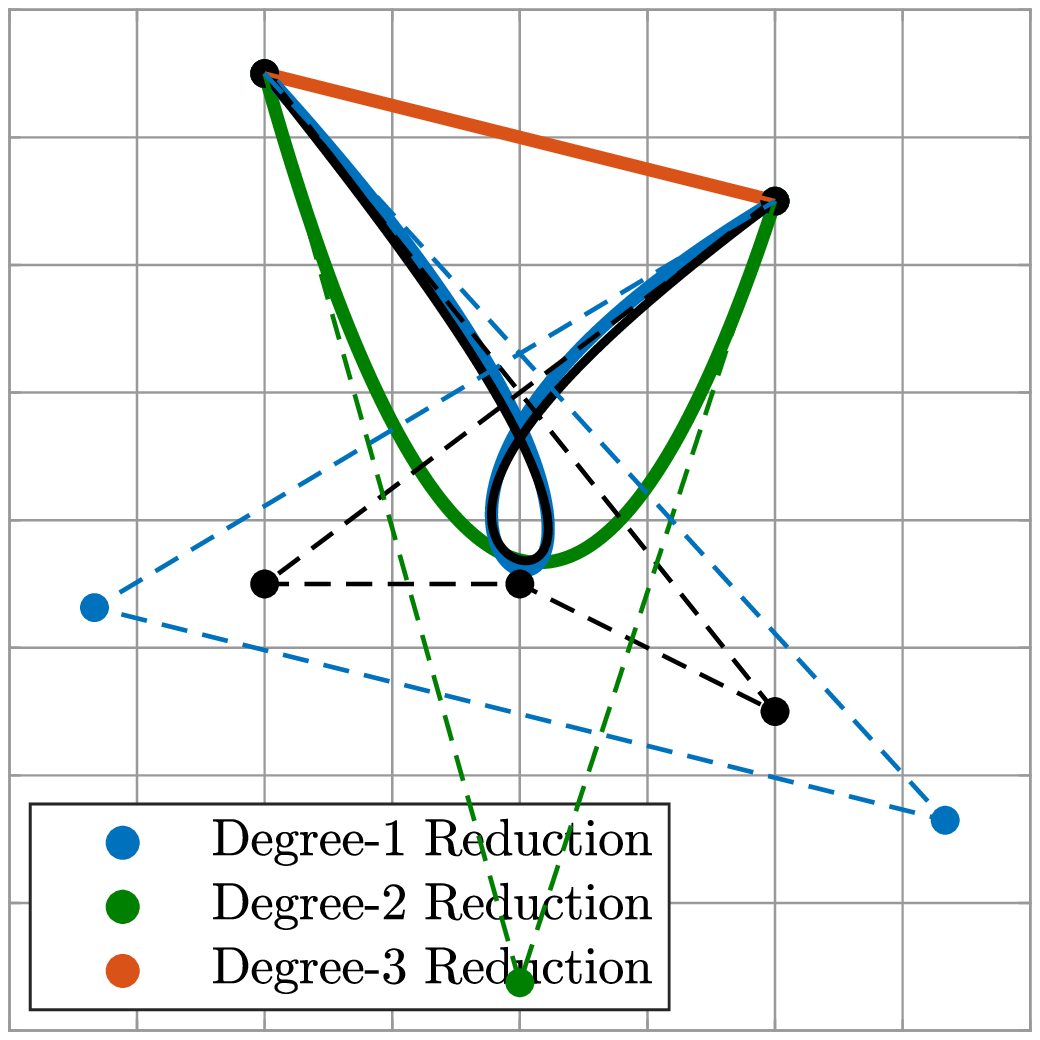} &
\includegraphics[width=0.165\textwidth]{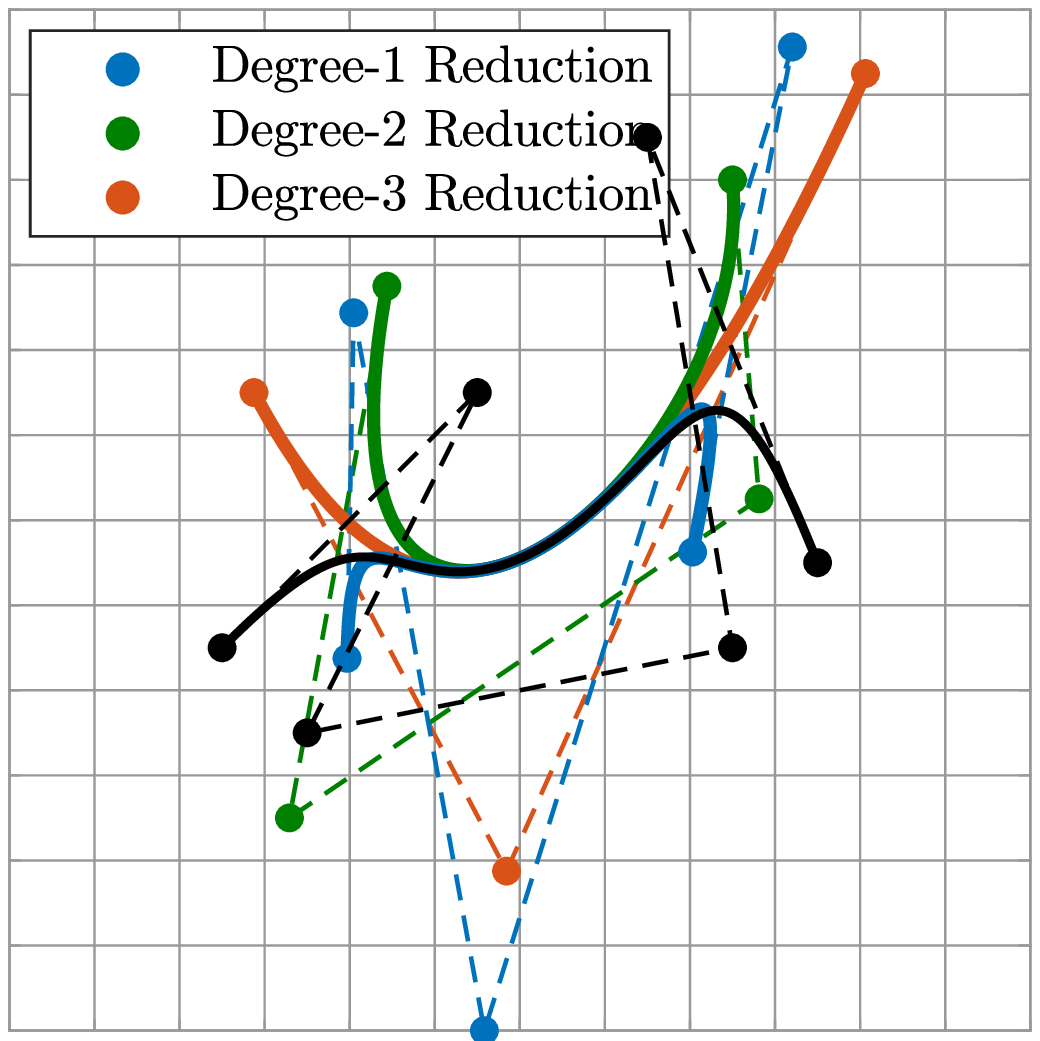} &
\includegraphics[width=0.165\textwidth]{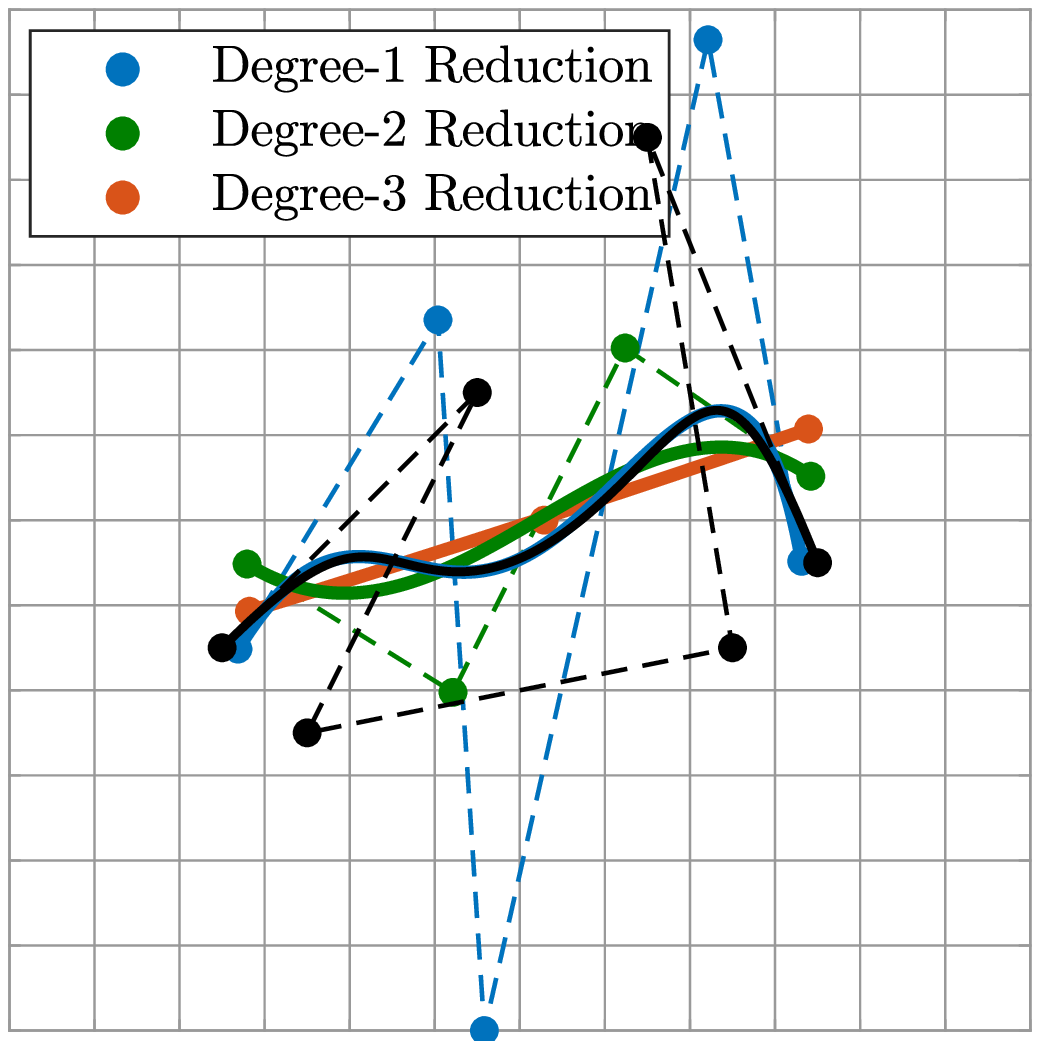} &
\includegraphics[width=0.165\textwidth]{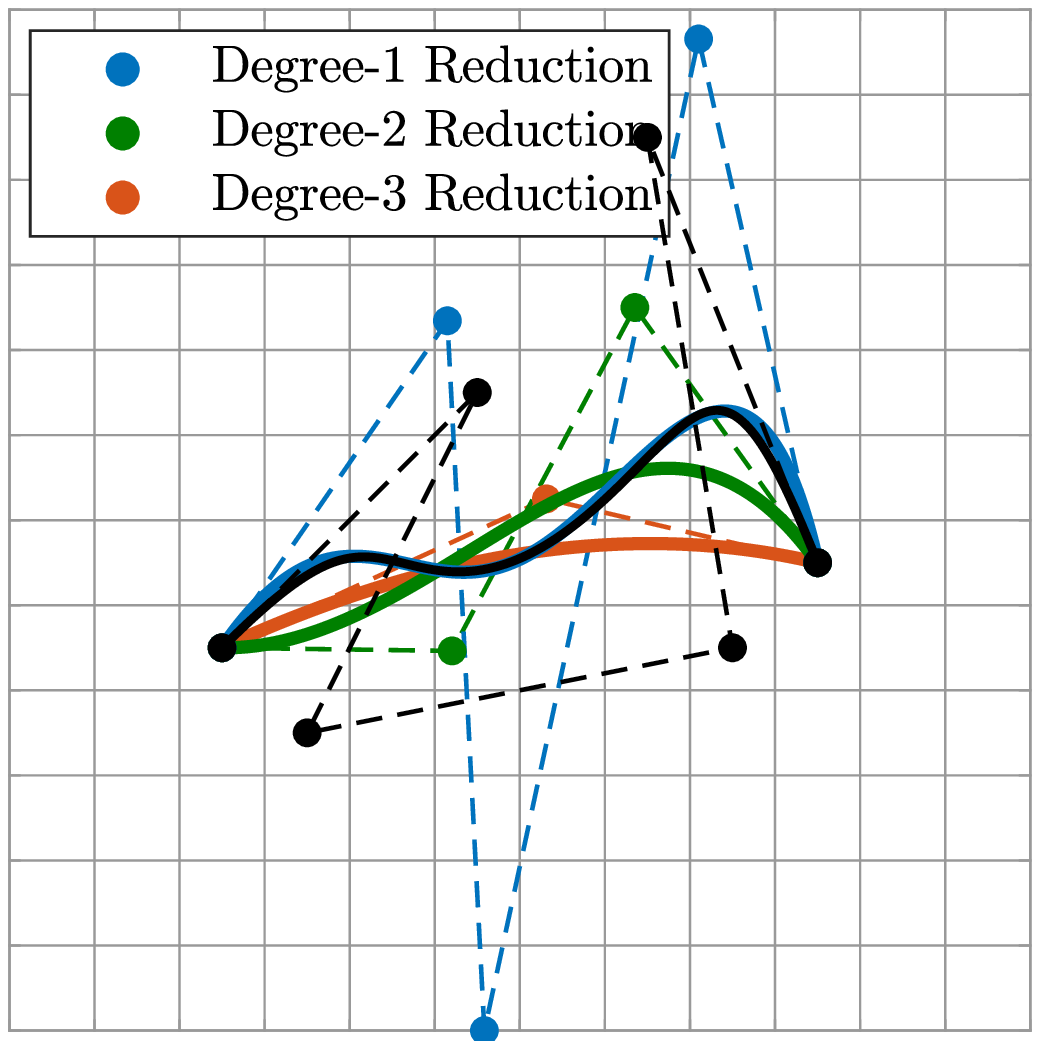} 
\\[-1mm]
\scriptsize{(a)} & \scriptsize{(b)} & \scriptsize{(c)} & \scriptsize{(d)} & \scriptsize{(e)} & \scriptsize{(f)}
\end{tabular}
\vspace{-4mm}
\caption{Degree reduction approximately represents B\'ezier curves with fewer control points. (a, d) Taylor degree reduction with a Taylor offset $\toffset = 0.5$, (b, e) Least squares reduction, (c, f) Uniform matching reduction with uniformly spaced parameters $t_0, \ldots, t_m$ over the unit interval, i.e., $t_i = \frac{i}{m}$ for $i = 0, \ldots, m$. }
\label{fig.DegreeReduction}
\vspace{-3mm}
\end{figure*}

In fact, the elements of the degree elevation matrix can be determined explicitly \cite{lee_park_BAMS1997, farin_CurvesSurfaces2002}:
\begin{proposition} \label{prop.ElevationMatrixElements}
(\emph{Elevation Matrix Elements})
For any $n \leq m$, the elements of the elevation matrix $\emat(n,m)$ are given by
\begin{align}\label{eq.ElevationMatrixElements}
\blist{\emat(n,m)}_{i+1,j+1} = \left\{ 
\begin{array}{@{}c@{\,\,}l@{}}
\frac{\binom{n}{i}\binom{m-n}{j-i}}{\binom{m}{j}} & \text{, if } m-n \geq j-i \geq 0 \\
0 & \text{, otherwise,}
\end{array}
\right. 
\end{align}
where $i = 0, \ldots, n$ and $j = 0, \ldots, m$.
\end{proposition}
\begin{proof}
See \refapp{app.ElevationMatrixElements}.
\end{proof}

\subsubsection{Important Elevation Matrix Properties}

The degree elevation matrices are full rank and  have unit column sum \cite{lee_park_BAMS1997}.

\begin{proposition}\label{prop.ElevationMatrixFullRank} (\emph{Full Rank Elevation Matrix})
For any $n \leq m$, the elevation matrix $\emat(n,m) \in \R^{(n+1) \times (m+1)}$ is full rank of $n+1$, i.e., $\rank(\emat(n,m)) = n+1$.
\end{proposition}
\begin{proof}
See \refapp{app.ElevationMatrixFullRank}.
\end{proof}

\begin{proposition}\label{prop.ElevationMatrixRowColumnSum} (\emph{Elevation Matrix Row \& Column Sum})
For any $n \leq m$, the sum of each column of the elevation matrix $\emat(n,m)$ is one, whereas each of its rows sums to $\frac{m+1}{n+1}$, i.e.,
\begin{subequations}
\begin{align}
\mat{1}_{1\times (n+1)} \, \emat(n,m) &= \mat{1}_{1 \times (m+1)}, \label{eq.ElevationMatrixColumnSum}
\\
\emat(n,m) \mat{1}_{(m+1) \times 1} &=  \tfrac{m+1}{n+1} \mat{1}_{(n+1) \times 1}. \label{eq.ElevationMatrixRowSum}
\end{align} 
\end{subequations}
where $\mat{1}_{n \times m}$ denotes the $n \times m$ matrix of all ones.  
\end{proposition}
\begin{proof}
See \refapp{app.ElevationMatrixRowColumnSum}.
\end{proof}

\subsubsection{B\'ezier Metrics under Degree Elevation}
Different B\'ezier metrics behave differently under degree elevation: while the L2-norm distance stays constant, the Frobenius norm distance might increase, whereas the maximum control-point distance is nonincreasing under elevation. 

\begin{proposition} \label{prop.L2DistanceElevation}
(\emph{Invariance of B\'ezier L2-norm distance under degree elevation})
The L2-norm distance of B\'ezier curves, $\bcurve_{\bpoint_0, \ldots, \bpoint_n} (t)$ and $\bcurve_{\mpoint_0, \ldots, \mpoint_m}(t)$ over the unit interval $[0,1]$, are preserved under degree elevation, i.e.,
\begin{align}
\bdistL(\bcurve_{\blist{\bpoint_0, \ldots, \bpoint_n} \emat(n,k)}, \bcurve_{\blist{\mpoint_0, \ldots, \mpoint_m} \emat(m, h)}\!) &=\bdistL(\bcurve_{\bpoint_0, \ldots, \bpoint_n}, \bcurve_{\mpoint_0, \ldots, \mpoint_m}), 
\end{align}
for any $k \geq n$ and $ h \geq m$.
\end{proposition}
\begin{proof}
See \refapp{app.L2DistanceElevation}.
\end{proof}

\begin{proposition}\label{prop.FrobeniusDistanceElevation}
(\emph{Elevated Frobenius Distance})
Under degree elevation, the Frobenius distance of $n^{\text{th}}$-order B\'ezier curves satisfies for any $m \geq n \in \N$ that
\begin{align}
&    \bdistF(\bcurve_{[\bpoint_0, \ldots, \bpoint_n] \emat(n,m)}, \bcurve_{[\mpoint_0, \ldots, \mpoint_n]\emat(n,m)})^2   \nonumber \\
& \hspace{32mm}\leq \tfrac{m+1}{n+1} \bdistF(\bcurve_{\bpoint_0, \ldots, \bpoint_n}, \bcurve_{\mpoint_0, \ldots, \mpoint_n})^2 .
\end{align} 
\end{proposition}
\begin{proof}
See \refapp{app.FrobeniusDistanceElevation}.
\end{proof}

\begin{proposition} \label{prop.CtrlDistanceElevation}
(\emph{Nonincreasing Elevated Control-Point Distance})
The maximum control-point distance of $n^{\text{th}}$-order B\'ezier curves is non-increasing under degree elevation, i.e., 
\begin{align}
\bdistC(\bcurve_{\blist{\bpoint_0, \ldots, \bpoint_n}\emat(n,m)}, \bcurve_{\blist{\mpoint_0, \ldots, \mpoint_n} \emat(n,m)}) \leq \bdistC(\bcurve_{\bpoint_0, \ldots, \bpoint_n}, \bcurve_{\mpoint_0, \ldots, \mpoint_n}),   
\end{align} 
for any $m \geq n \in \N$.
\end{proposition}
\begin{proof}
See \refapp{app.CtrlDistanceElevation}.
\end{proof}
\noindent Thus, the ordering relation of B\'ezier distances in \refprop{prop.DistanceOrder}, the relative geometric bound of B\'ezier curves in \refprop{prop.RelativeBezierBound},  and the nonincreasing property under degree elevation  in \refprop{prop.CtrlDistanceElevation} make the maximum  control-point distance an analytic  and intuitive metric for comparing B\'ezier~curves.

Finally, it is useful to note that as elevation degree goes to infinity, one has all B\'ezier curve points as its control points.
\begin{proposition} [\cite{farin_CurvesSurfaces2002}]
(\emph{Asymptotic Behavior of Degree Elevation})
As the elevation degree goes to infinity, the elevated B\'ezier control points become the B\'ezier curve points, i.e.,
\begin{align}
\lim _{m \rightarrow \infty} [\bpoint_{0}, \ldots, \bpoint_n]\emat(n,m) = \blist{\bcurve_{\bpoint_0, \ldots, \bpoint_n}(t)}_{0 \leq t \leq 1} \,.
\end{align}
\end{proposition}

\subsection{Degree Reduction}
\label{sec.DegreeReduction}

As opposed to degree elevation, B\'ezier degree reduction aims to approximately represent a B\'ezier curve with fewer control points, as illustrated in \reffig{fig.DegreeReduction}.
Hence, degree reduction is naturally defined as the inverse of degree elevation.\reffn{fn.UnifiedDegreeReduction}

\begin{definition} \label{def.DegreeReduction}
(\emph{Degree Reduction})
A B\'ezier curve $\bcurve_{\mpoint_0, \ldots, \mpoint_m}$ of lower degree $m$ with control points $\mpmat_{m}=\blist{\mpoint_0, \ldots, \mpoint_m}$ is said to be a \emph{degree reduction} of another B\'ezier curve $\bcurve_{\bpoint_0, \ldots, \bpoint_n}$ of higher degree $n \geq m$ with control points $\bpmat_{n} = \blist{\bpoint_0, \ldots, \bpoint_n}$ if and only if the control points are related to each other by
\begin{align}
\mpmat_{m} = \bpmat_{n} \rmat(n,m),
\end{align}
where $\mat{R}(n,m)\! \in \! \R^{(n+1) \times (m+1)}$ denotes a degree reduction~matrix that is a right inverse of the elevation matrix $\emat(m,n)$, i.e.,
\begin{align}
\emat(m,n)\mat{R}(n,m) = \mat{I}_{(m+1) \times (m+1)}.
\end{align}
\end{definition}
\noindent That is to say, the degree elevation from $m$ to $n$ followed by a degree reduction from $n$ to $m$ preserves B\'ezier curves; but, the reverse is not correct in general.
Also note that the right inverse of the elevation matrix is not unique, which allows many alternative ways of constructing a reduction matrix.

\addtocounter{footnote}{1}
\footnotetext{\label{fn.UnifiedDegreeReduction}Many existing notions of B\'ezier degree reduction methods that are defined in terms of different B\'ezier distances (possibly with end-point constraints) can be unified using the inverse of degree elevation \cite{sunwoo_lee_CAGD2004}.}

\subsubsection{Least Squares Reduction}

A standard choice for degree reduction is the pseudo-inverse of the elevation matrix \cite{farin_CurvesSurfaces2002}.

\begin{definition}\label{def.LeastSquaresReduction}
(\emph{Least Squares Reduction})
The \emph{least squares} reduction matrix $\rmat_{L2}(n,m)$ is defined as the pseudo-inverse of the elevation matrix $\emat(m,n)$ that is explicitly given by
\begin{align}
\rmat_{L2}(n,m) = \tr{\emat(m,n)} (\emat(m,n) \tr{\emat(m,n)})^{-1}.
\end{align}
\end{definition}
\noindent Note that $\emat(m,n) \tr{\emat(m,n)}$ is invertible for any $m \leq n \in \N$ because $\emat(m,n)$ is full rank  (\refprop{prop.ElevationMatrixFullRank}), which implies $\emat(m,n)\rmat_{L2}(n,m) = \mat{I}_{(m+1) \times (m+1)}$. 
An example of least squares reduction is presented in \ref{fig.DegreeReduction}(b,e), where the B\'ezier end-points are not preserved after degree reduction.

The least squares degree reduction is known to be optimal in the sense of the  L2-  and Frobenius-norm  distances \mbox{\cite{lee_park_BAMS1997, lutterkort_peters_reif_CAGD1999}}.

\begin{proposition}\label{prop.LeastSquaresReductionOptimality}
(\emph{Optimality of Least Squares Reduction})
With respect to the L2-norm and Frobenius-norm distances, the optimal $m^{\text{th}}$-order Bezier curve $\bcurve_{\mpoint_0, \ldots, \mpoint_m}(t)$  with control points $\mpmat_{m} = [\mpoint_0, \ldots, \mpoint_m]$ that is closest to a higher $\text{n}^{\text{th}}$-order B\'ezier curve $\bcurve_{\bpoint_0, \ldots, \bpoint_{n}}(t)$ with control points $\bpmat_{m} = \blist{\bpoint_0, \ldots, \bpoint_n}$  is given by the least squares reduction, i.e.,
\begin{subequations}
\begin{align}
\mpmat_{m} &=  \argmin_{\mpoint_0, \ldots, \mpoint_m \in \R^{\cdim}} \bdistL(\bcurve_{\bpoint_0, \ldots, \bpoint_n}, \bcurve_{\blist{\mpoint_0, \ldots, \mpoint_m} \emat(m,n)}),
\\
& =  \argmin_{\mpoint_0, \ldots, \mpoint_m \in \R^{\cdim}} \bdistF(\bcurve_{\bpoint_0, \ldots, \bpoint_n}, \bcurve_{\blist{\mpoint_0, \ldots, \mpoint_m} \emat(m,n)}),
\\
& =  \bpmat_{n}\rmat_{L2}(n,m). 
\end{align} 
\end{subequations}
\end{proposition}
\begin{proof}
See \refapp{app.LeastSquaresReductionOptimality}.
\end{proof}

\subsubsection{Taylor Reduction}
Another classical degree reduction is Taylor approximation that preserves local derivatives.

\begin{definition} \label{def.TaylorReduction}
(\emph{Taylor Reduction}) 
The Taylor reduction matrix $\mat{R}_{\tbasis, \toffset}(n,m)$ for approximating an $n^\text{th}$-order Bezier curve around  $\toffset \in \R$ by a lower $m^\text{th}$-order Bezier curve is defined as%
\begin{align}\label{eq.ReductionMatrix}
\mat{R}_{\tbasis, \toffset}(n,m) := \tfbasis_{\tbasis}^{\bbasis}(n, \toffset) \mat{I}_{(n+1) \times (m+1)} \tfbasis_{\bbasis}^{\tbasis}(m, \toffset),
\end{align}
where $\tfbasis_{\tbasis}^{\bbasis}$ and $\tfbasis_{\bbasis}^{\tbasis}$ are the Taylor-to-Bernstein and the Bern-stein-to-Taylor basis transformation matrices  in \reflem{lem.BasisTransformation}. 
\end{definition}
\noindent In other words, Taylor reduction perform a basis transformation from Bernstein to Taylor basis, and ignores some higher-order Taylor basis elements, and then comes back to Bernstein basis.\footnote{
The Taylor reduction matrix can be derived as follows:
\begin{align}
\bcurve_{\bpoint_0, \ldots, \bpoint_n}(t) &= \blist{\bpoint_0, \ldots, \bpoint_n} \bbasis_{n}(t)  
= \blist{\bpoint_0, \ldots, \bpoint_n} \tfbasis_{\tbasis}^{\bbasis}(n, \toffset) \tbasis_{n, \toffset}(t) \nonumber
\\
& \approx \blist{\bpoint_0, \ldots, \bpoint_n} \tfbasis_{\tbasis}^{\bbasis}(n, \toffset) \mat{I}_{(n+1) \times (m+1)} \tbasis_{m, \toffset}(t) \nonumber
\\
& = \blist{\bpoint_0, \ldots, \bpoint_n} \tfbasis_{\tbasis}^{\bbasis}(n, \toffset) \mat{I}_{(n+1) \times (m+1)} \tfbasis_{\bbasis}^{\tbasis}(m, \toffset) \bbasis_{m}(t) \nonumber
\\
& = \blist{\mpoint_0, \ldots, \mpoint_m} \bbasis_{m}(t) 
=\bcurve_{\mpoint_0, \ldots, \mpoint_m} (t) \nonumber
\end{align}
}
Hence, it has a strong bias and local expressiveness around the Taylor offset $\toffset$, as illustrated in \reffig{fig.DegreeReduction}(a,d).

\begin{proposition} \label{prop.TaylorReductionInverse}
(\emph{Taylor Reduction as Elevation Inverse})
The Taylor reduction matrix $\mat{R}_{\tbasis, \toffset}(n,m)$ is a right inverse of the elevation matrix $\emat(m,n)$, i.e., for any $m \leq n \in \N$
\begin{align}
\emat(m,n) \mat{R}_{\tbasis, \toffset}(n,m) = \mat{I}_{(m+1) \times (m+1)}.
\end{align}
\end{proposition}
\begin{proof}
See \refapp{app.TaylorReductionInverse}.
\end{proof}

It is important to observe  in \reffig{fig.DegreeReduction} that both the least-squares and Taylor reduction methods offer less freedom in controlling the resulting shape of B\'ezier approximations; for example, the end points of the original curve are not preserved after degree reduction, which is essential for path smoothing with boundary conditions \cite{ravankar_etal_Sensors2018}.

\subsubsection{Parameterwise Matching Reduction}

In order to accurately approximate curve shape and geometry, we propose a new parameterwise matching reduction method that preserves a finite set of curve points after degree reduction.

\begin{definition}\label{def.MatchingReduction}
(\emph{Parameterwise Matching Reduction})
For B\'ezier degree reduction from a higher degree $n$ to a lower degree $m \leq n$, the \emph{parameterwise matching degree reduction matrix} $\rmat_{t_0, \ldots, t_m}(n,m)$ associated with pairwise distinct parameters $t_0,  \ldots, t_m \in \R$ (i.e., $t_i \neq t_j$ for all $i\neq j$) is defined as 
\begin{align} \label{eq.MatchingReduction}
\rmat_{t_0, \ldots, t_m}(n,m) := \bbmat_{n}(t_0, \ldots, t_m) \bbmat_{m}(t_0, \ldots, t_m)^{-1},
\end{align}
where $\bbmat_{n}(t_0, \ldots, t_m)$ is the Bernstein basis matrix  in \refeq{eq.BasisMatrix}.
\end{definition}
\noindent For example, a numerically stable choice of  $t_0, \ldots, t_m$ over the unit interval is the uniformly spaced parameters in $[0,1]$, i.e., $t_i = \frac{i}{m}$ for $i = 0, \ldots, m$. 
We call the corresponding reduction operation as the \emph{uniform matching reduction}. 

As expected, the parameterwise matching reduction of B\'ezier curves keeps curve points unchanged at $t_0, \ldots, t_m$.
\begin{proposition}\label{prop.MatchingReduction}
(\emph{Preserved Points of Matching Reduction}) For any pairwise distinct $t_0, \ldots, t_m \in \R$, a B\'ezier curve $\bcurve_{\bpoint_0, \ldots, \bpoint_n}(t)$ of degree $n$  and its parameterwise matching reduction $\bcurve_{\mpoint_0, \ldots, \mpoint_m}(t)$ of degree $m \leq n$ with control points
\begin{align}
[\mpoint_0, \ldots, \mpoint_m] = [\bpoint_0, \ldots, \bpoint_n] \rmat_{t_0, \ldots, t_m}(n,m)
\end{align}
match at the curve parameters $t_0, \ldots, t_m$, i.e.,
\begin{align}
\bcurve_{\mpoint_0, \ldots, \mpoint_m}(t) = \bcurve_{\bpoint_0, \ldots, \bpoint_n}(t) \quad \forall t = t_0, \ldots, t_m.
\end{align} 
\end{proposition}
\begin{proof}
See \refapp{app.MatchingReduction}.
\end{proof}

\bigskip

\begin{proposition}\label{prop.MatchingReductionInverse}
(\emph{Matching Reduction as Elevation Inverse})
For any $m \leq n$ and pairwise distinct reals $t_0, \ldots, t_m \in \R$, the parameterwise matching degree matrix $\rmat_{t_0, \ldots, t_m}(n,m)$ is a right inverse of the elevation matrix $\emat(m,n)$, 
\begin{align}
\emat(m,n)\rmat_{t_0, \ldots, t_m}(n,m) = \mat{I}_{(m+1) \times (m+1)}.
\end{align} 
\end{proposition}
\begin{proof}
See \refapp{app.MatchingReductionInverse}
\end{proof}

\begin{figure*}[b]
\vspace{-0mm}
\centering
\begin{tabular}{@{}c@{\hspace{0.5mm}}c@{\hspace{0.5mm}}c@{\hspace{0.5mm}}c@{\hspace{0.5mm}}c@{\hspace{0.5mm}}c@{\hspace{0.5mm}}c@{}}
\rotatebox{90}{\scriptsize{\hspace{0.1mm}Taylor Reduction}}&
\includegraphics[width=0.161\textwidth]{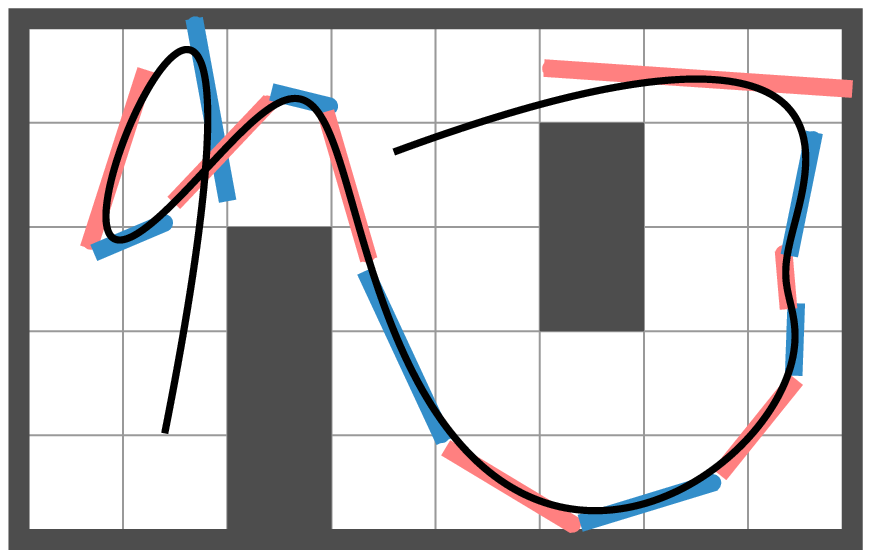} &
\includegraphics[width=0.161\textwidth]{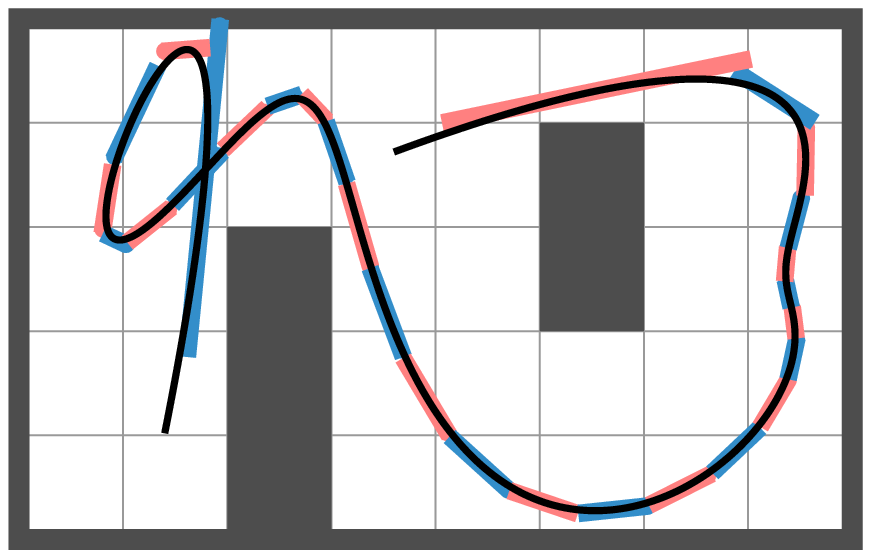} &
\includegraphics[width=0.161\textwidth]{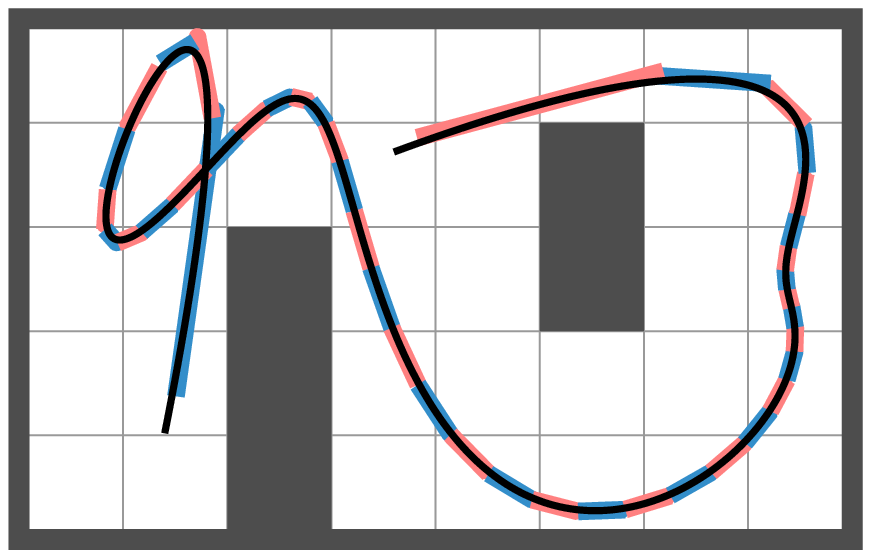} &
\includegraphics[width=0.161\textwidth]{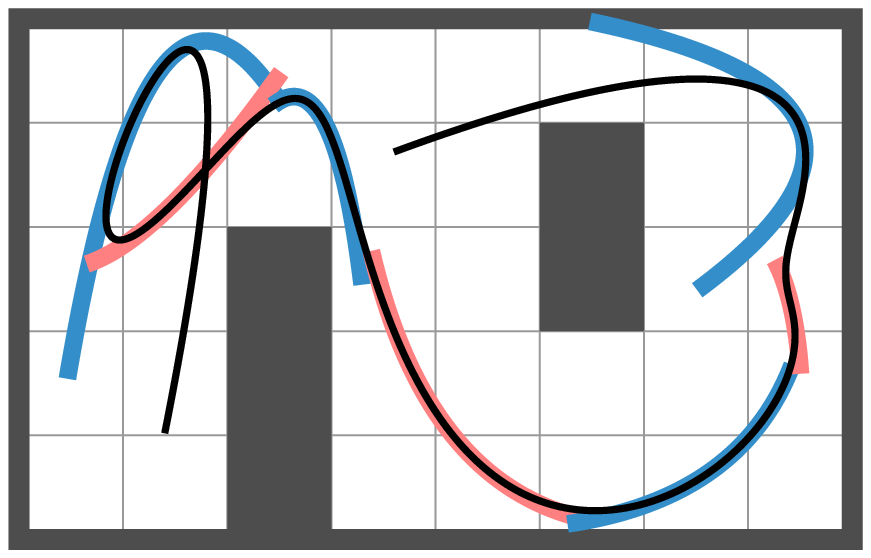} & 
\includegraphics[width=0.161\textwidth]{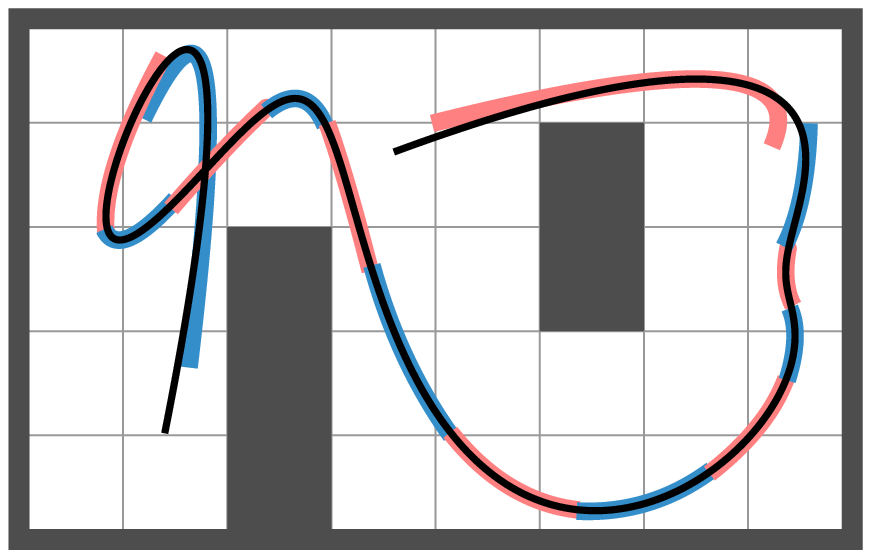} &
\includegraphics[width=0.161\textwidth]{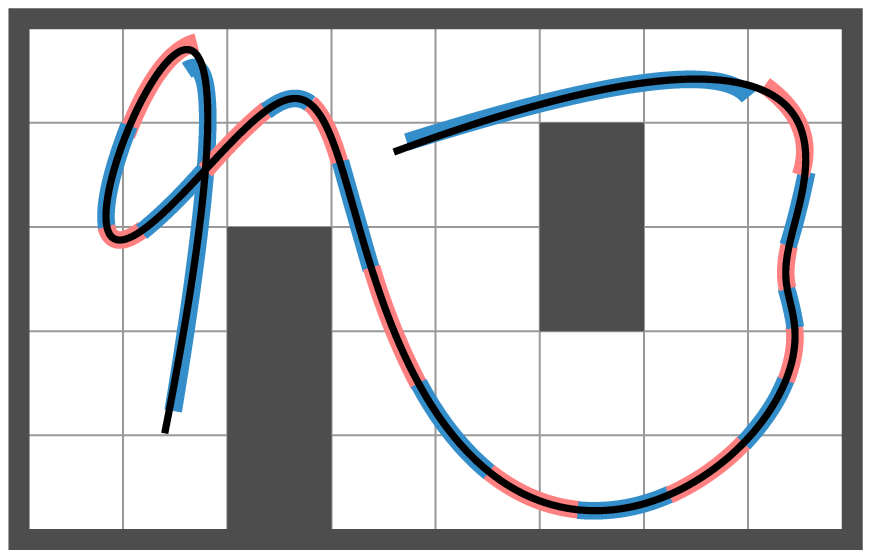} 
\\
\rotatebox{90}{\scriptsize{\hspace{2mm}Least Squares}} &
\includegraphics[width=0.161\textwidth]{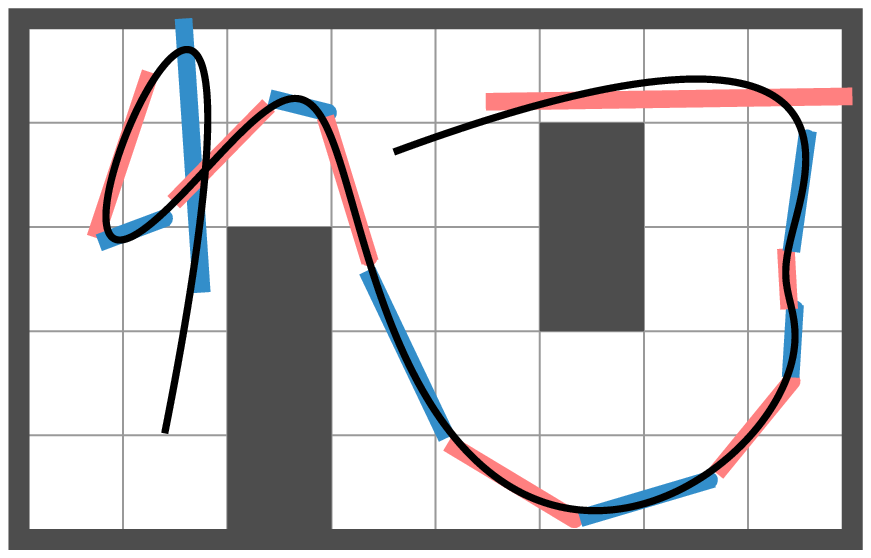} &
\includegraphics[width=0.161\textwidth]{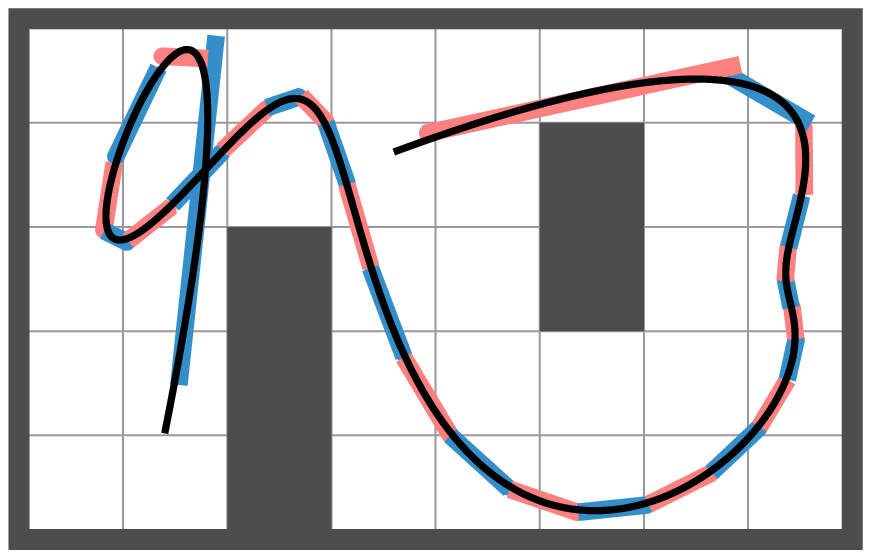} &
\includegraphics[width=0.161\textwidth]{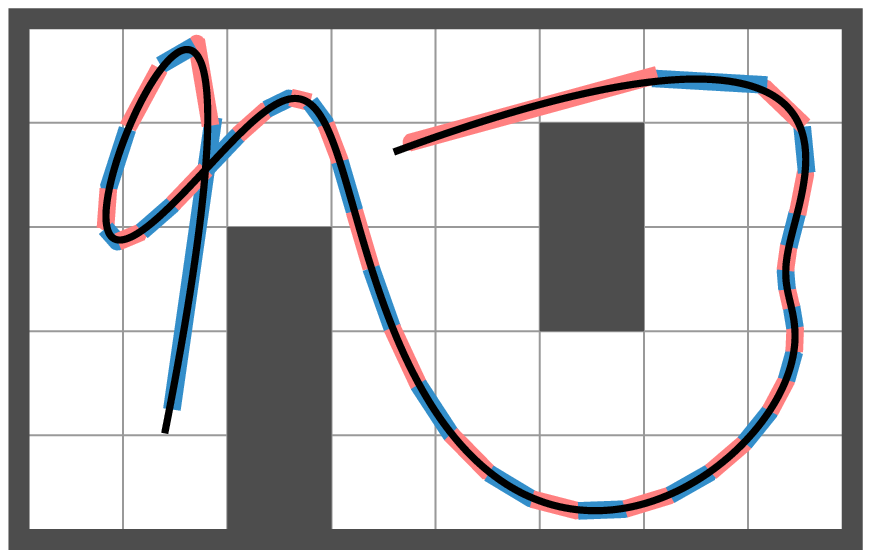} &
\includegraphics[width=0.161\textwidth]{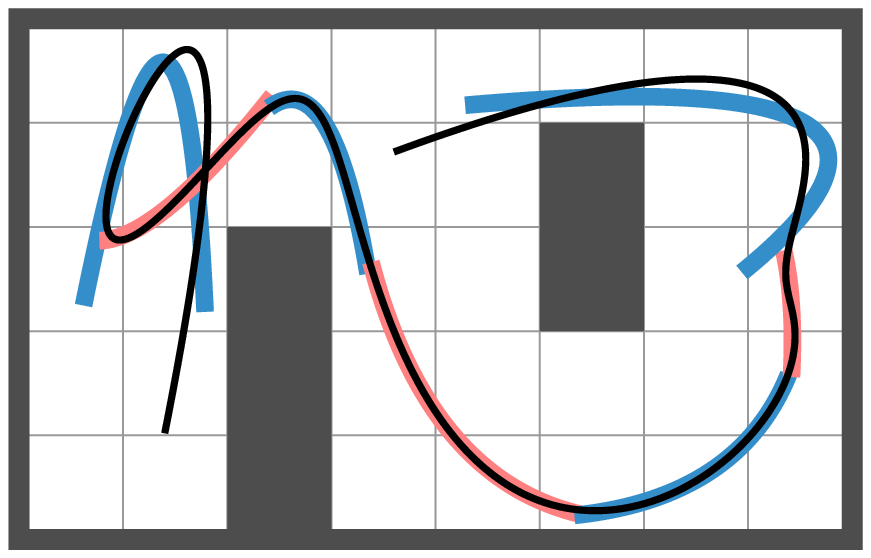} & 
\includegraphics[width=0.161\textwidth]{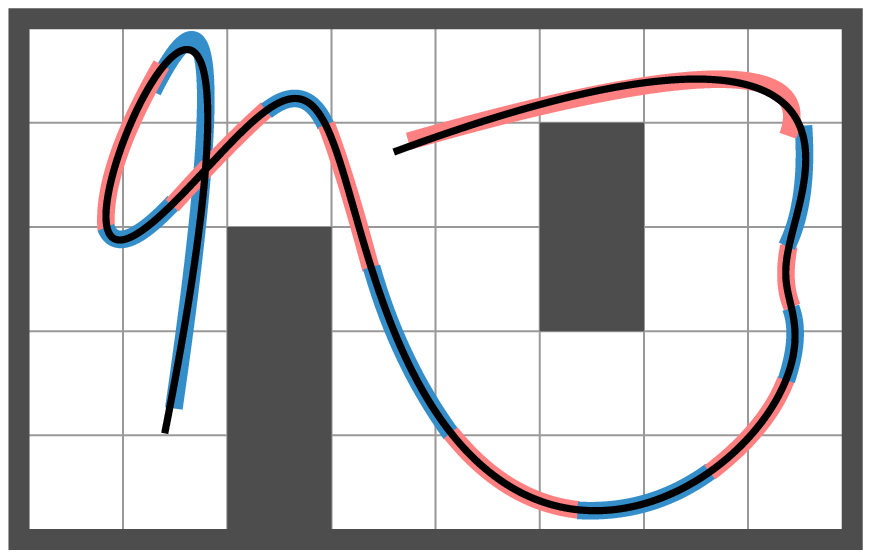} &
\includegraphics[width=0.161\textwidth]{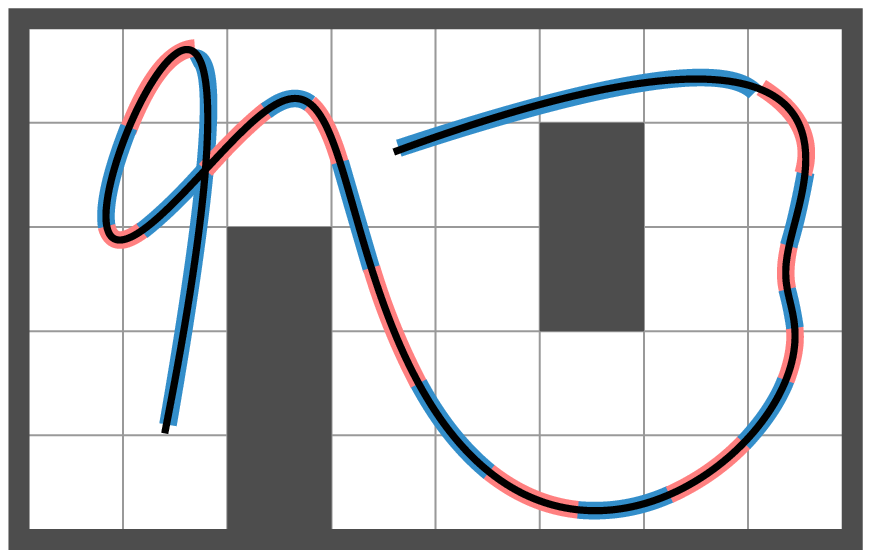} 
\\
\rotatebox{90}{\scriptsize{\hspace{-0.75mm}Uniform Matching}} & 
\includegraphics[width=0.161\textwidth]{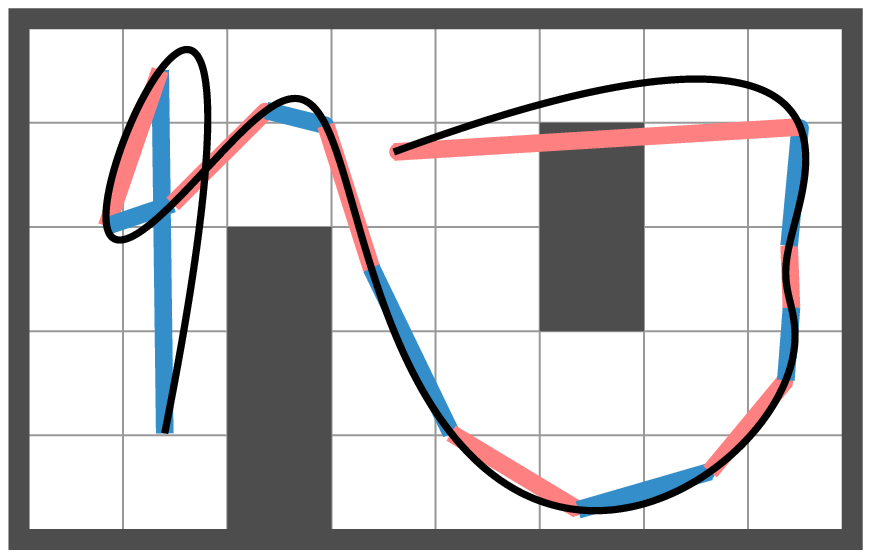} &
\includegraphics[width=0.161\textwidth]{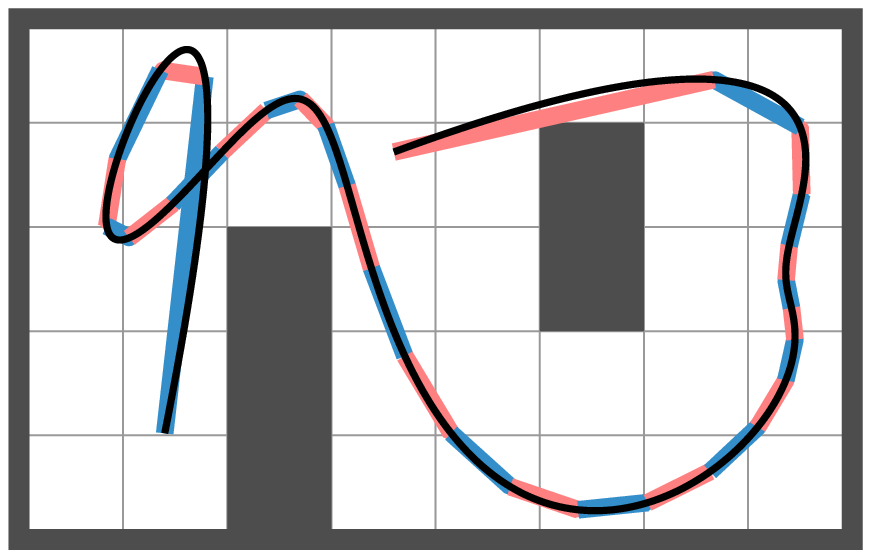} &
\includegraphics[width=0.161\textwidth]{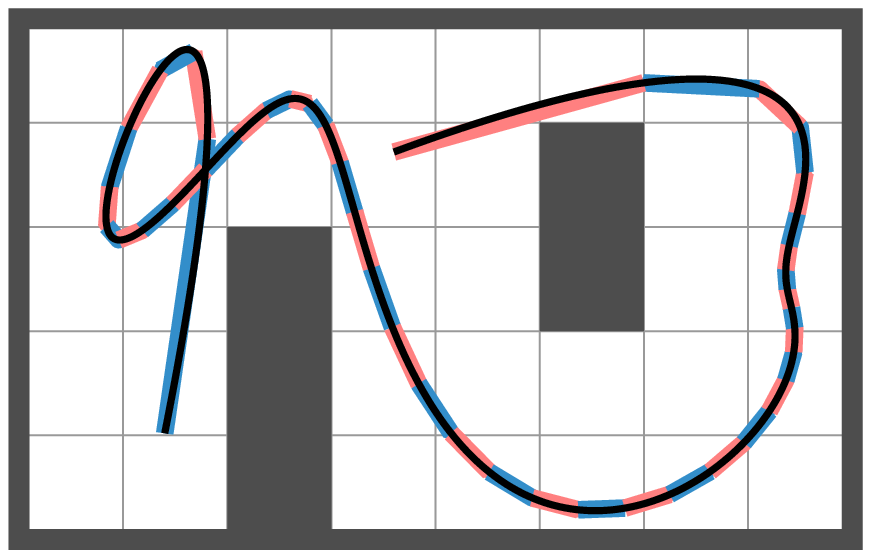} &
\includegraphics[width=0.161\textwidth]{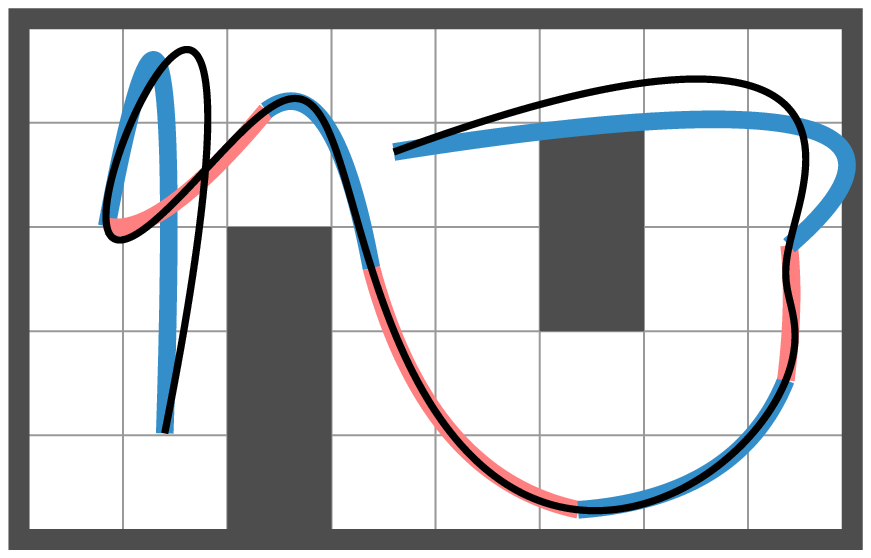} & 
\includegraphics[width=0.161\textwidth]{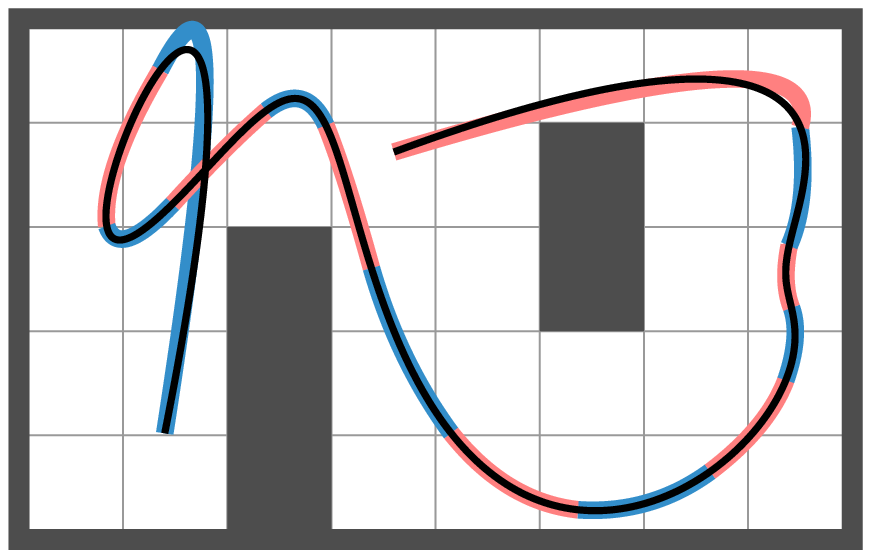} &
\includegraphics[width=0.161\textwidth]{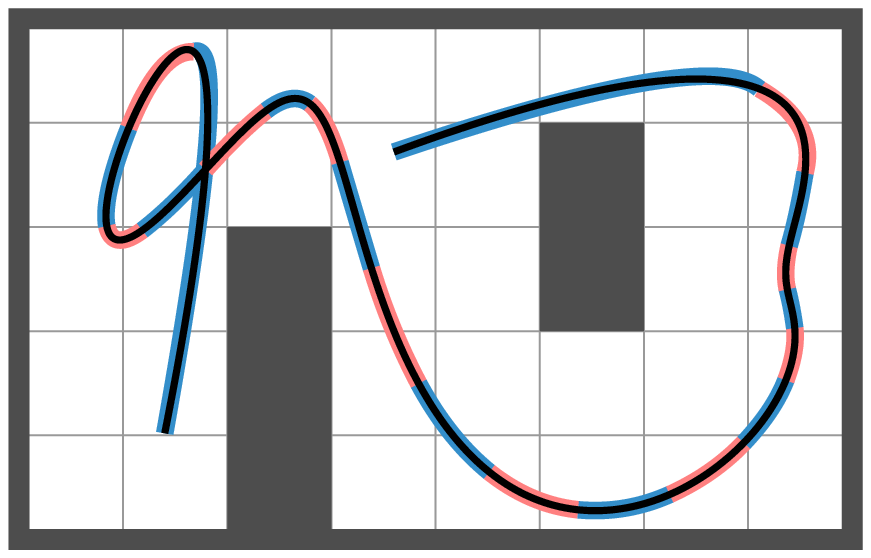} 
\\[-1.5mm]
& \scriptsize{$2(n-1)$ Linear Segments} & \scriptsize{$4(n-1)$ Linear Segments } & \scriptsize{$6(n-1)$ Linear Segments} & \scriptsize{$(n-1)$ Quad Segments} & \scriptsize{$2(n-1)$ Quad Segments} & \scriptsize{$3(n-1)$ Quad Segments}   
\\[-1mm]
 & \scriptsize{(a)} & \scriptsize{(b)} & \scriptsize{(c)} & \scriptsize{(d)} & \scriptsize{(e)} & \scriptsize{(f)}
\end{tabular}
\vspace{-3mm}
\caption{Approximation of a B\'ezier curve of degree $n = 8$  by (a, b, c) linear and (d, e, f) quadratic B\'ezier segments using (top) Taylor, (middle) least squares, and (bottom) uniform matching reduction. The Taylor offset is set to be 0.5, and a uniform partition of the unit interval is used with (a) $2(n-1)$, (b) $4(n-1)$, (c) $6(n-1)$ partition elements for linear approximations, and (d) $(n-1)$, (e) $2(n-1)$, (f) $3(n-1)$ partition elements for quadratic approximations. B\'ezier approximation rule: An $n^\text{th}$-order B\'ezier curve can be approximated accurately by $3(n-1)$ quadratic and $6(n-1)$ linear uniform matching B\'ezier curves.}
\label{fig.BezierApproximation}
\end{figure*}

Another interesting connection between degree elevation and matching reduction is their shared matrix form.
\begin{proposition} \label{prop.MatchingReductionElevationSharedForm}
(\emph{Shared Form of Matching Reduction and Elevation Matrix})
For any $n,m \in \N$ and any pairwise distinct reals $t_0, \ldots, t_m \in \R$, the Bernstein basis matrices satisfy%
\footnote{It is important to highlight that for any distinct  $t_0, \ldots, t_n \in \R$, the inverse of the Bernstein matrix can be computed analytically using the Bernstein-to-monomial basis transformation $\bbasis_{n}(t_0, \ldots, t_n) = \tfbasis_{\mbasis}^{\bbasis} \mbasis(t_0, \ldots, t_n)$ as 
\begin{align*}
\bbmat_{n}(t_0, \ldots, t_n)^{-1} = \mbmat_{n}(t_0, \ldots, t_n)^{-1} \tfbasis_{\bbasis}^{\mbasis}(n)
\end{align*}
since the inverse of the monomial (a.k.a. Vandermonde) matrix is analytically available \cite{neagoe_SPL1996}, and the elements of Bernstein-to-monomial basis transformation $\tfbasis_{\bbasis}^{\mbasis}(n)$ can be determined explicitly, see \refapp{app.BasisTransformation}.
}
\begin{align}
\!\!\!\bbmat_{n}(t_0, \ldots, t_m)\bbmat_{m}(t_0, \ldots, t_m)^{-1} \!\!= \!\!
\left\{
\begin{array}{@{}c@{}l@{}}
\emat(n,m) & \text{, if } n \leq m, \\
\!\rmat_{t_0, \ldots, t_m}\!(n,m) & \text{, if } n \geq m,
\end{array}
\right. \!\!\!\!
\end{align}
where $\emat(n,n) = \rmat_{t_0, \ldots, t_m}(n,n) = \mat{I}_{(n+1) \times (n+1)}$.
\end{proposition}
\begin{proof}
It follows from  \refprop{prop.ElevationMatrixBernstein} and \refdef{def.MatchingReduction}.
\end{proof}

A critical property of matching reduction is that the degree-one reduction error can be determined analytically.
\begin{proposition} \label{prop.MatchingReductionDifference}
(\emph{Degree-One Matching Reduction Error})
For any pairwise distinct $t_0, \ldots t_n \in \R$,  the difference between a B\'ezier curve $\bcurve_{\bpoint_0, \ldots, \bpoint_{n+1}} (t)$ and its  parameterwise matching degree reduction  $\bcurve_{\mpoint_0, \ldots, \mpoint_{n}} (t)$, with control points 
\begin{align}
[\mpoint_0, \ldots, \mpoint_n] = [\bpoint_0, \ldots, \bpoint_{n+1}] \rmat_{t_0, \ldots, t_n}(n+1, n),
\end{align}
is given by
\begin{align}
\bcurve_{\bpoint_0, \ldots, \bpoint_{n+1}} (t) - \bcurve_{\mpoint_0, \ldots, \mpoint_{n}} (t) = \Delta \bpoint \prod_{i=0}^{n}(t - t_i),
\end{align}
where
\begin{align}
\Delta \bpoint = \sum_{i=0}^{\cdegree + 1} (-1)^{n+1-i}\scalebox{1.2}{$\binom{n+1}{i}$} \bpoint_i.
\end{align}
\end{proposition}
\begin{proof}
See \refapp{app.MatchingReductionDifference}.
\end{proof}

\noindent It is important to observe that the degree-one matching reduction difference vector $\Delta \bpoint$ is independent of the selection of matching parameters $t_0, \ldots, t_\cdegree$ where the reduction error is zero.\footnote{Finding optimal matching parameters that minimize the peak reduction error is an open research problem and outside the scope of this paper. We observe from \reffig{fig.MatchingReductionError} that optimal matching parameters should be nonuniformly spaced with a bias towards the ends points. In this paper, we consider the uniformly spaced matching parameters and their adaptive selection based on binary search in \refsec{sec.BezierAdaptiveApproximation}.}  
Moreover, the polynomial product form of the degree-one matching reduction error, illustrated in \reffig{fig.MatchingReductionError}, plays a key role in determining how many local low-order B\'ezier segments are needed for approximating high-order B\'ezier curves accurately,  as discussed below in \refsec{sec.BezierAdaptiveApproximation}.

\begin{figure}[t]
\centering
\includegraphics[width=0.43\textwidth]{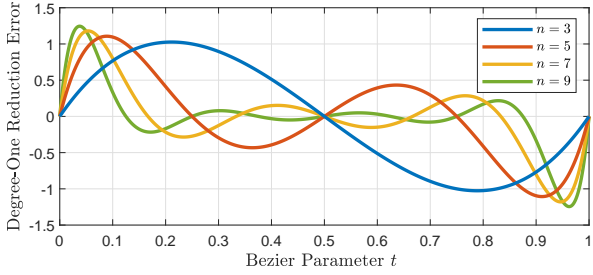}
\vspace{-2mm} 
\caption{Scaled degree-one reduction error function $\varepsilon_{n}(t)$ for uniform matching reduction of B\'ezier curves that is defined as $\varepsilon_{n} (t) =  \frac{\prod_{i=0}^{n} t - \frac{i}{n}}{ \prod_{i=0}^{n} \frac{1}{2n} - \frac{i}{n}}$.}
\label{fig.MatchingReductionError}
\end{figure}

\section{Adaptive Degree Reduction  and Splitting \\ of B\'ezier Curves}
\label{sec.BezierAdaptiveApproximation}

In this section, we consider the problem of approximating high-order B\'ezier curves by multiple lower-order B\'ezier segments. 
We first describe how B\'ezier approximations can be performed over a given finite partition of the unit interval, and then  
propose linear and binary search methods for adaptive approximation of high-order B\'ezier curves by lower-order B\'ezier segments at any desired accuracy level. 
We also present a rule of thumb for accurate B\'ezier discretization.

\subsection{B\'ezier Approximation over a Partition of the Unit Interval}
\label{sec.BezierApproximationPartition}

Consider an $n^{\text{th}}$-order B\'ezier curve $\bcurve_{\bpoint_0, \ldots, \bpoint_n}(t)$ with control points $\bpmat_{n} = [\bpoint_0, \ldots, \bpoint_n]$ defined over the unit interval $[0,1]$.
Suppose $T = [t_0, \ldots, t_k]$ is an ordered list of distinct parameters that defines a partition of the unit interval into $k$ splits, i.e., $0 = t_0 < t_1 < \ldots < t_{k-1} < t_k = 1$.
Accordingly, the B\'ezier curve $\bcurve_{\bpoint_0, \ldots, \bpoint_n}(t)$  can be locally approximated over each parameter subinterval $[t_{i-1}, t_i]$ by a lower $m^{\text{th}}$-order Bezier curve $\bcurve_{\mpoint_{0,i}, \ldots, \mpoint_{m,i}}(t)$ whose control points $\mpmat_{m,i} = [\mpoint_{0,i}, \ldots, \mpoint_{m,i}]$ is obtained based on a choice  of a degree reduction matrix $\rmat(n,m)$ (see \refdef{def.DegreeReduction}) as
\begin{align}
\mpmat_{m,i} = \mathrm{DegreeReduction}(\bpmat_{n,i}, m) := \bpmat_{n,i} \rmat(n,m),
\end{align}
using the reparametrization $\bcurve_{\bpoint_{0,i}, \ldots, \bpoint_{n,i}}(t)$ of $\bcurve_{\bpoint_0, \ldots, \bpoint_n}(t)$ from $[t_{i-1}, t_i]$ to the unit interval $[0,1]$    with new control points $\bpmat_{n,i}= [\bpoint_{0,i}, \ldots, \bpoint_{n,i}]$ (see \reflem{lem.PolynomialReparametrization}) that are obtained as
\begin{align}
\, \bpmat_{n,i} &= \mathrm{Reparameterize} (\bpmat_i, [t_{i-1}, t_i]),
\\
& \hspace{-3mm} := \bbmat_{n}\plist{0, \tfrac{1}{n}, \ldots,  \tfrac{n-1}{n},1} \bbmat_{n}\plist{s_i(0), s_i(\tfrac{1}{n}), \ldots,  s_i(\tfrac{n-1}{n}),s_i(1)\!}^{\!-1}\!\!\!,
\end{align}    
where $s_i(t) = \bcurve_{t_{i-1}, t_i} (t) =  t_{i-1} (1- t) + t_i \, t $ is the B\'ezier parameter scaling function.
Hence, as described in \refalg{alg.BezierApproximation}, the  high-order B\'ezier curve $\bcurve_{\bpoint_0, \ldots, \bpoint_n}([0,1])$ can be approximated by a collection of low-order B\'ezier segments $\bcurve_{\mpoint_{0,i}, \ldots, \mpoint_{m,i}}([0,1])$  constructed  over each partition element $[t_{i-1}, t_i]$ such that $\bcurve_{\mpmat_{m,i}}([0,1])$ approximates $\bcurve_{\bpmat_{n}}([t_{i-1}, t_i])$.

\begin{algorithm}[t]
\caption{B\'ezier Approximation over a Partition \\  \mbox{\hspace{20mm}}of the Unit Interval}\label{alg.BezierApproximation}
\KwIn{%
$\bpmat_n \in \R^{\cdim \times n}$: Bezier Control Points \\ 
\hspace{10mm} $m \in \N$: Reduction Degree \\ 
\hspace{10mm} $T=[t_0, \ldots, t_k]$: Partition of the Unit Interval
}
\KwOut{%
\mbox{$\mpmat_{m,1}, \ldots \mpmat_{m,k}$: List of Reduced Control Points}
\\
\mbox{\hspace{11.5mm} where $\bcurve_{\mpmat_{m,i}}([0,1])$ approximates $\bcurve_{\bpmat_{n}}([t_{i-1}, t_i])$}
\vspace{-6mm}
\\
\hrulefill
}
$k \gets \mathrm{length}(T) - 1$ \tcp*[f]{\small{Number of Segments}}\\
\For{$i \gets 1$ \KwTo $k$}{
$\bpmat_{n,i} \gets \textrm{Reparametrize}(\bpmat_n, [t_{i-1}, t_i])$  \\
$\mpmat_{m,i}  \gets \textrm{DegreeReduction}(\bpmat_{n,i}, m)$ \\
}
\Return{$\mpmat_{m,1}, \ldots, \mpmat_{m, k}$}
\end{algorithm}

In \reffig{fig.BezierApproximation}, we illustrate approximating an $8^{\text{th}}$-order B\'ezier curve with linear and quadratic B\'ezier segments over the uniform partitions of the unit interval using Taylor, least squares, and uniform matching reductions. 
As seen in \reffig{fig.BezierApproximation}, the uniform matching reduction performs better in approximately representing the original curve shape than the least squares reduction which performs better than Taylor reduction. 
Also notice that the end points of the original curve are kept unchanged only under the uniform matching reduction.

\subsection{B\'ezier Approximation Rule}
\label{sec.BezierApproximationRule}

A practical question of approximating high-order B\'ezier curves by a finite number of low-order B\'ezier segments is what the required number of local B\'ezier segments is for an accurate B\'ezier discretization. 
As expected, the answer  depends on the desired level of approximation accuracy and the degree of B\'ezier curves.  
In this part, we provide an answer based on the structural form of the B\'ezier approximation error of degree-one matching reduction (\refprop{prop.MatchingReductionDifference}), and the \mbox{numerical analysis of B\'ezier approximations in \refsec{sec.NumericalAnalysis}}.

\begin{BezierApproximationRule}
An $n^{\text{th}}$-order B\'ezier curve can be accurately approximated by $3(n-1)$ quadratic and $6(n-1)$ linear B\'ezier curves obtained via uniform matching reduction.
\end{BezierApproximationRule}

According to \refprop{prop.MatchingReductionDifference}, an $n^{\text{th}}$-order B\'ezier curve can be written as the sum of an $(n-1)^{\text{th}}$-order reduced B\'ezier curve and a degree-one matching reduction error  which is a polynomial of order $n$ in the product form.
Note that the $(n-1)^{\text{th}}$-order reduced B\'ezier curve can be better approximated with the same number of low-order B\'ezier segments than the original $n^{\text{th}}$-order B\'ezier curve.  
Hence, one can determine the required number of low-order B\'ezier segments for accurately approximating high-order B\'ezier curves by exploiting the functional form of the reduction error.  
As seen in \reffig{fig.MatchingReductionError}, the degree-one matching reduction error of $n^\text{th}$-order B\'ezier curves has $(n-1)$ extreme (local maximum and minimum) points. 
This implies that a proper approximation of $n^\text{th}$-order B\'ezier curves structurally requires at least $(n-1)$ quadratic B\'ezier curves which has a single extremum. 
Because of the asymmetry of the approximation error around each extreme point, one needs at least $2(n-1)$ quadratic segments.
Since the uniform matching reduction uses uniformly spaced parameters for approximation, we observe in our numerical studies in \refsec{sec.NumericalAnalysis} that the asymmetry around each extreme point can be better handled with $3(n-1)$ quadratic patches in practice.
Our numerical analysis also shows that approximating $n^{\text{th}}$-order B\'ezier curves by $3(n-1)$ quadratic segments ensures a normalized (i.e., scale invariant) approximation error below the order of $10^{-3}$ for computing important curve features such as curve length, distance-to-point/line, and maximum velocity/acceleration.   
Similarly, since a quadratic polynomial structurally requires at least two linear curve segments for a proper representation of its unique extremum, we also observe from our numerical studies that approximating $n^{\text{th}}$-order B\'ezier curves by $6(n-1)$ linear B\'ezier segments yields a normalized approximation error in the order of $10^{-3}$.

\begin{algorithm}[b]
\caption{\mbox{Adaptive Degree Reduction and Splitting} via Linear Search}\label{alg.AdaptiveApproximationLinearSearch}
\KwIn{%
$\bpmat_n \in \R^{\cdim \times n}$: Bezier Control Points \\ 
\hspace{10mm} $m \in \N$: Reduction Degree \\ 
\hspace{10mm} $\varepsilon > 0 $: Approximation Tolerance
}
\KwOut{%
$T=[t_0, \ldots, t_k]$: Split Intervals \\
\mbox{\hspace{11.5mm} $\mpmat_{m,1}, \ldots \mpmat_{m,k}$: List of Reduced Control Points}\\
\mbox{\hspace{11.5mm} where $\bcurve_{\mpmat_{m,i}}([0,1])$ approximates $\bcurve_{\bpmat_{n}}([t_{i-1}, t_i])$}
\vspace{-6mm}
\\
\hrulefill
}
$k \gets 1$ \tcp*[f]{\small{Number of Segments}}\\
$ T \gets [0, 1]$ \tcp*[f]{\small{Initial Partition}} \\
\While{$k < \mathrm{length}(T)$}{
\For{$i \gets 1$ \KwTo $k$}{
$\bpmat_{n,i} \gets \textrm{Reparametrize}(\bpmat_n, [t_{i-1}, t_i])$  \\
$\mpmat_{m,i}  \gets \textrm{DegreeReduction}(\bpmat_{n,i}, m)$ \\
\If{$\mathrm{BezierDistance}(\bpmat_{n,i}, \mpmat_{m,i} \emat(m,n)) > \varepsilon$}{
$T \gets [0, \frac{1}{k+1}, \ldots, \frac{k}{k+1}, 1]$ \\
\textbf{break}
}
}
$k \gets k + 1$ \\
}
$k \gets \mathrm{length}(T) - 1$ \tcp*[f]{\small{Number of Segments}} \\
\Return{$T, \mpmat_{m,1}, \ldots, \mpmat_{m, k}$}

\end{algorithm}

\begin{algorithm}[t]
\caption{\mbox{Adaptive Degree Reduction and Splitting} via Binary Search}\label{alg.AdaptiveApproximationBinarySearch}
\KwIn{%
$\bpmat_n \in \R^{\cdim \times n}$: Bezier Control Points \\ 
\hspace{10mm} $m \in \N$: Reduction Degree \\ 
\hspace{10mm} $\varepsilon > 0$: Approximation Tolerance
}
\KwOut{%
$T=[t_0, \ldots, t_k]$: Split Intervals \\
\mbox{\hspace{11.5mm} $\mpmat_{m,1}, \ldots \mpmat_{m,k}$: List of Reduced Control Points}\\
\mbox{\hspace{11.5mm} where $\bcurve_{\mpmat_{m,i}}([0,1])$ approximates $\bcurve_{\bpmat_{n}}([t_{i-1}, t_i])$}
\vspace{-6mm}
\\
\hrulefill
} 

$k \gets 1$  \tcp*[f]{\small{Segment Counter}}\\
$T \gets [0,1]$ \tcp*[f]{\small{Initial Partition}} \\
\While{$k < \mathrm{length}(t)$}{
$\bpmat_{n,k} \gets \textrm{Reparametrize}(\bpmat_n, [t_{k-1}, t_k])$  \\
$\mpmat_{m,k}  \gets \textrm{DegreeReduction}(\bpmat_{n,i}, m)$ \\
\If{$\mathrm{BezierDistance}(\bpmat_{n,k}, \mpmat_{m,k} \emat(m,n)) > \varepsilon$}{
$T \gets [t_0, \ldots, t_{k-1},  \frac{t_{k-1} + t_k}{2},t_k]$
}
\Else{
$k \gets k + 1$
}
}
$k \gets \mathrm{length}(T) - 1$ \tcp*[f]{\small{Number of Segments}}\\
\Return{$T, \mpmat_{m,1}, \ldots, \mpmat_{m, k}$}
\end{algorithm}

\subsection{Adaptive B\'ezier Approximation via B\'ezier Metrics}

The B\'ezier approximation rule above holds for any B\'ezier curve in general, and ensures a proper structural representation of high-order B\'ezier curves at a certain level of accuracy.
However, high-order B\'ezier curves might have redundant control points, for example, consider the degree elevation of B\'ezier curves in \reffig{fig.DegreeElevation}, and also different application settings might require different levels of approximation accuracy.
Hence, it is desirable to perform B\'ezier approximation that is tailored to individual B\'ezier curves  and can adaptively select the required number of B\'ezier segments and the partition of the unit interval based on the desired level of accuracy. 
In this part, we extend B\'ezier approximations over a given partition of the unit interval by incorporating a search strategy to automatically determine a partition of the unit interval in order to achieve a desired level of measurable approximation accuracy.

\begin{figure*}[b]
\centering
\begin{tabular}{@{}c@{\hspace{0.5mm}}c@{\hspace{0.5mm}}c@{\hspace{0.5mm}}c@{\hspace{0.5mm}}c@{\hspace{0.5mm}}c@{\hspace{0.5mm}}c@{}}
\rotatebox{90}{\scriptsize{Linear Srch \scalebox{0.9}{$\varepsilon = .5$}}}&
\includegraphics[width=0.161\textwidth]{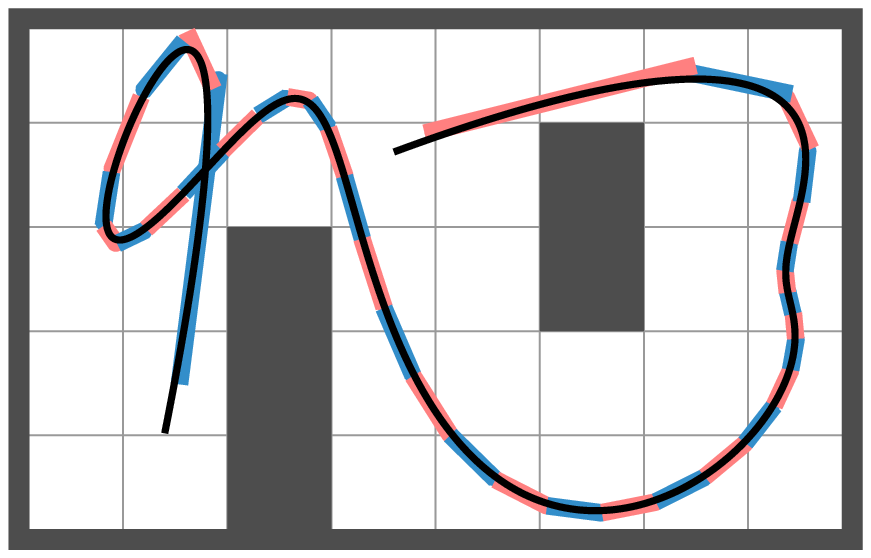} &
\includegraphics[width=0.161\textwidth]{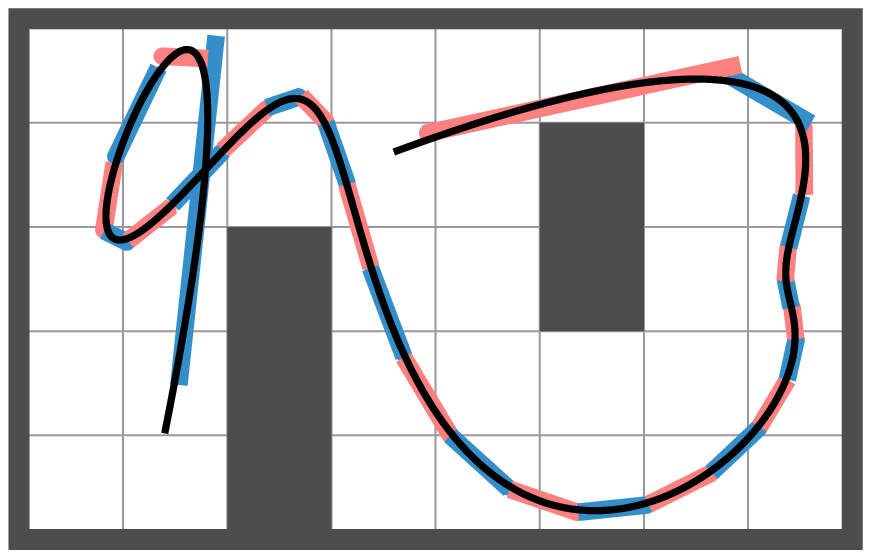} &
\includegraphics[width=0.161\textwidth]{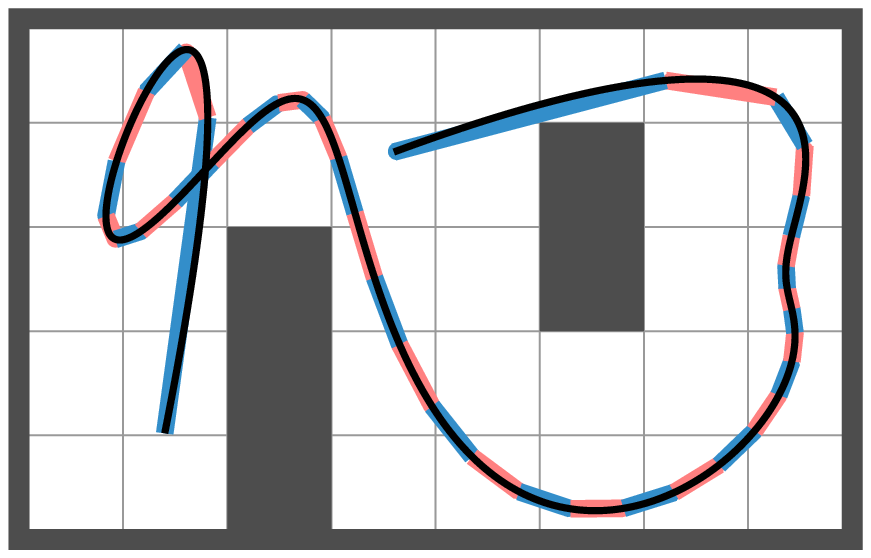} &
\includegraphics[width=0.161\textwidth]{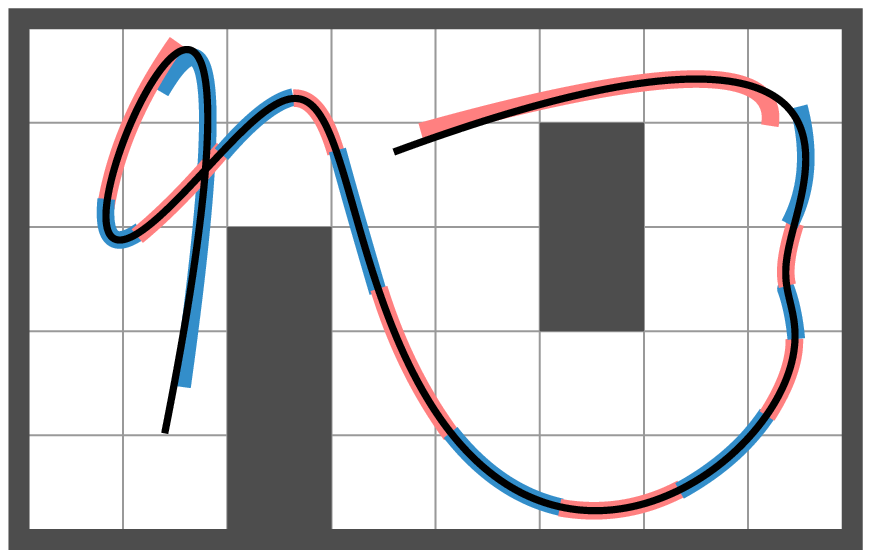} &
\includegraphics[width=0.161\textwidth]{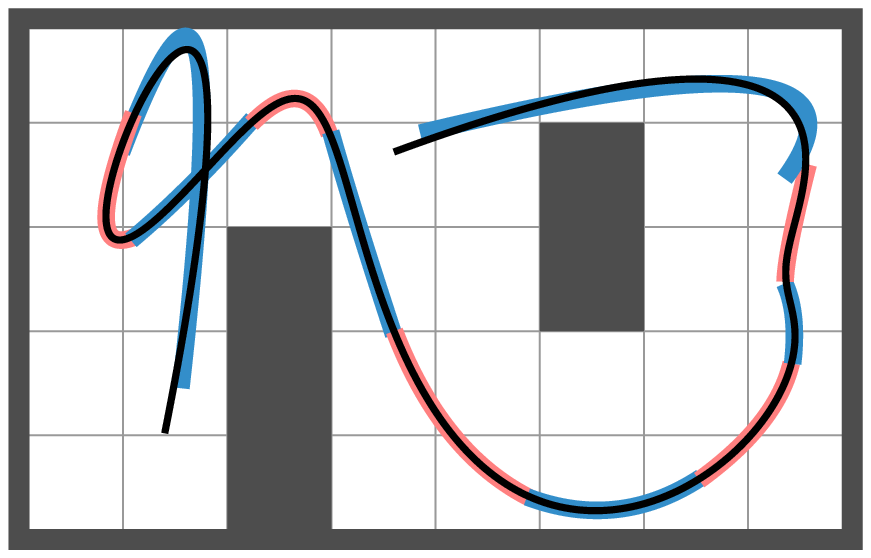} &
\includegraphics[width=0.161\textwidth]{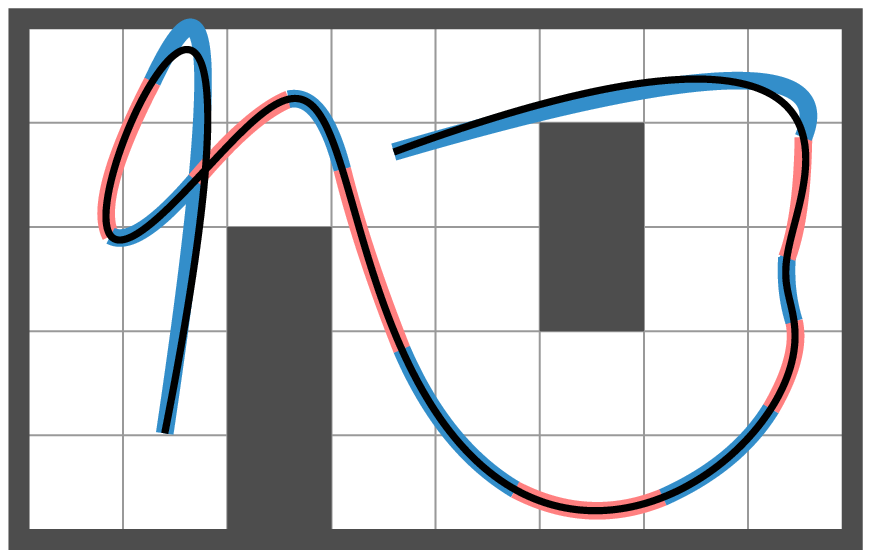} 
\\
\rotatebox{90}{\scriptsize{Linear Srch \scalebox{0.9}{$\varepsilon = .1$}}}&
\includegraphics[width=0.161\textwidth]{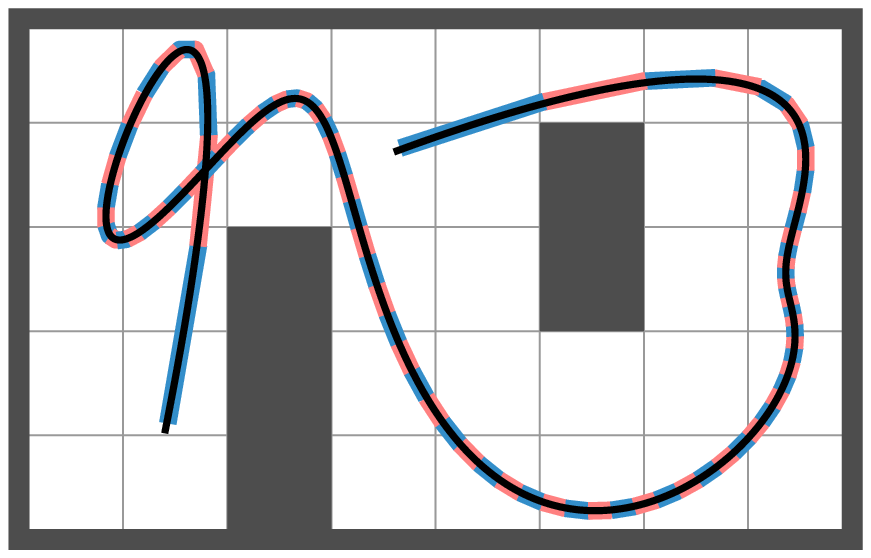} &
\includegraphics[width=0.161\textwidth]{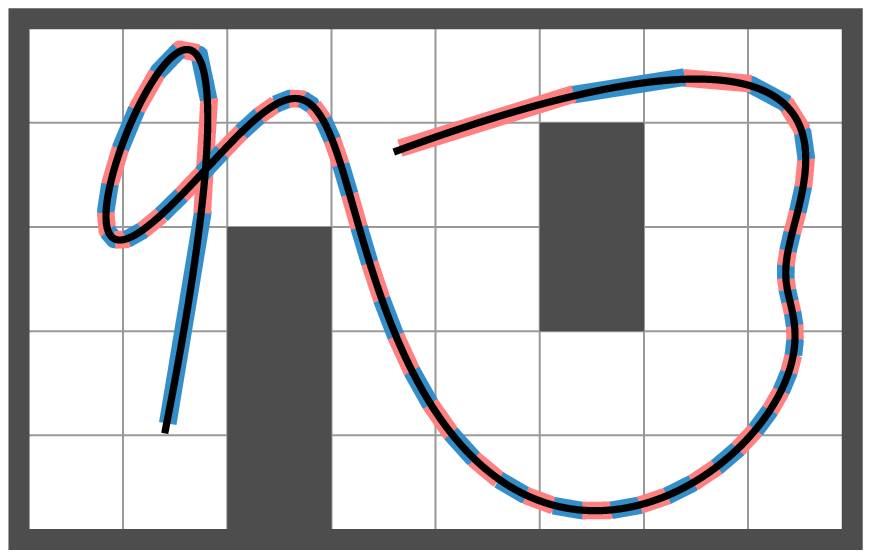} &
\includegraphics[width=0.161\textwidth]{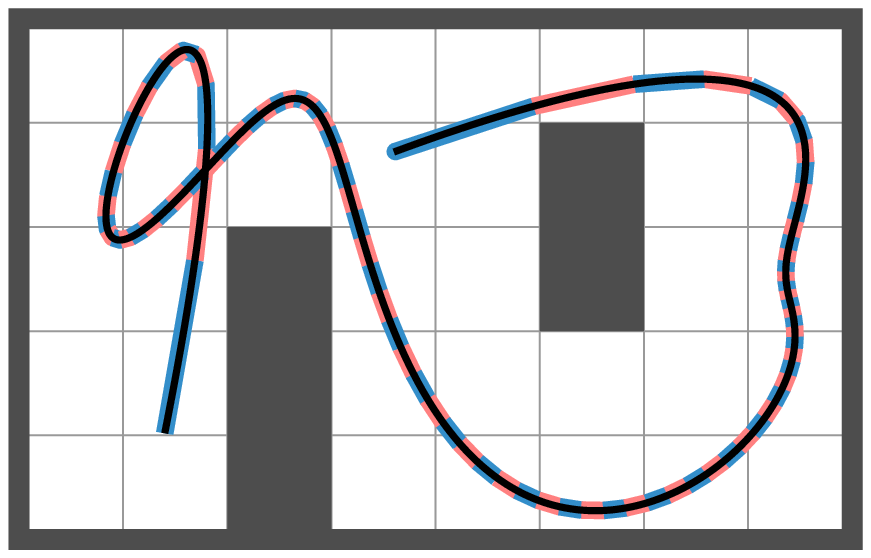} &
\includegraphics[width=0.161\textwidth]{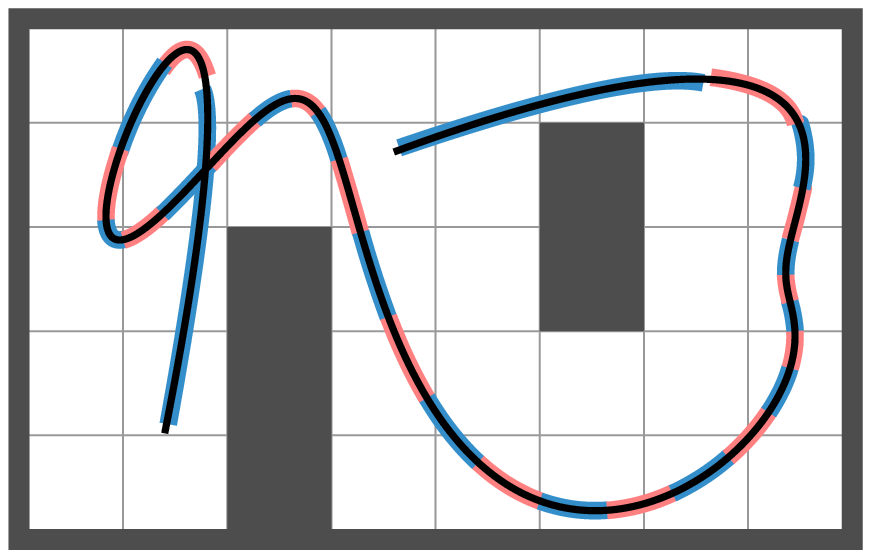} &
\includegraphics[width=0.161\textwidth]{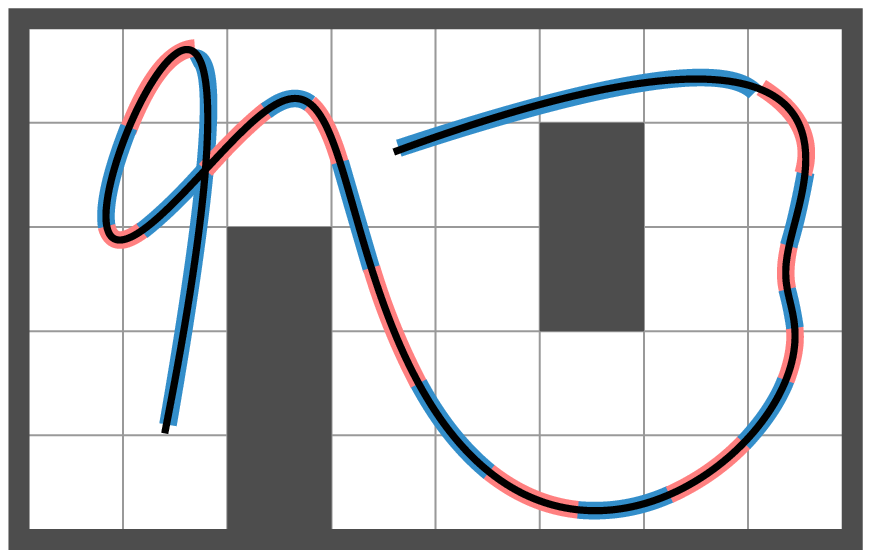} &
\includegraphics[width=0.161\textwidth]{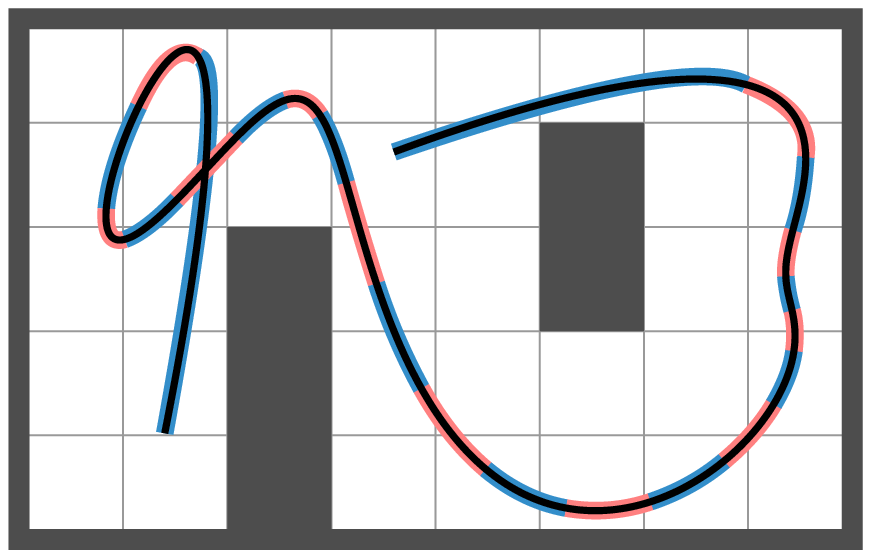} 
\\
\rotatebox{90}{\scriptsize{Binary Srch \scalebox{0.9}{$\varepsilon = .1$}}}&
\includegraphics[width=0.161\textwidth]{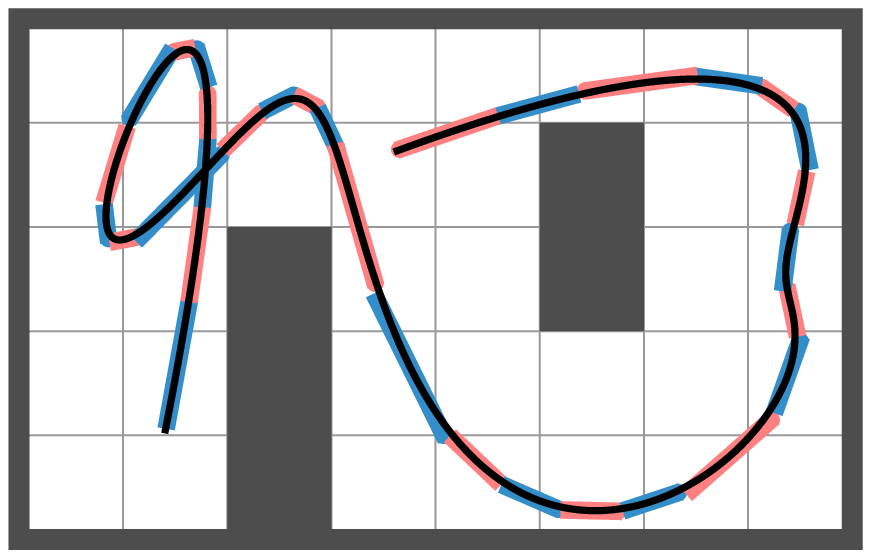} &
\includegraphics[width=0.161\textwidth]{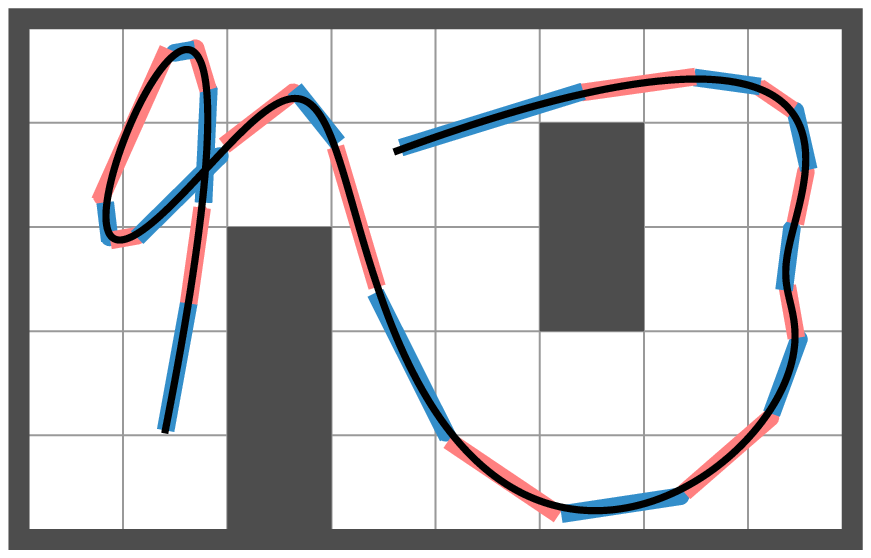} &
\includegraphics[width=0.161\textwidth]{figures/bezierAdaptiveApprox_Obstacle_3_UniformMatch_Binary_1_100.eps} &
\includegraphics[width=0.161\textwidth]{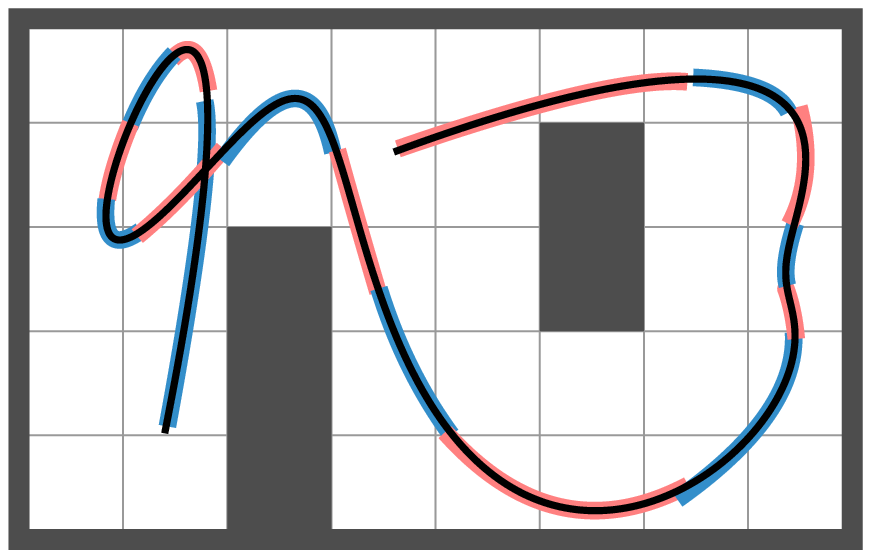} &
\includegraphics[width=0.161\textwidth]{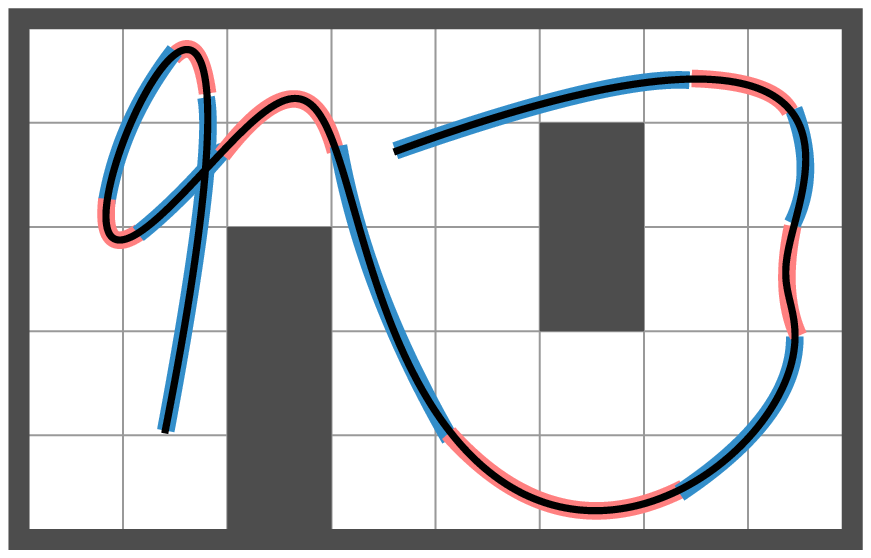} &
\includegraphics[width=0.161\textwidth]{figures/bezierAdaptiveApprox_Obstacle_3_UniformMatch_Binary_2_100.eps}
\\[-1mm]
& \scriptsize{Linear Taylor Reduction} & \scriptsize{Linear Least Squares} & \scriptsize{Linear Uniform Matching} & \scriptsize{Quadratic Taylor Reduction} & \scriptsize{Quadratic Least Squares} & \scriptsize{Quadratic Uniform Matching} 
\\[-1mm]
& \scriptsize{(a)} & \scriptsize{(b)} & \scriptsize{(c)} & \scriptsize{(d)} & \scriptsize{(e)} & \scriptsize{(f)}    
\end{tabular}
\vspace{-4mm}
\caption{Adaptive approximation of a $8^{\text{th}}$-order B\'ezier curve (black dashed line) by (a, b, c) linear and (d, e, f) quadratic B\'ezier curve segments (red and blue patches) that are automatically obtained via (top, middle) linear  and (bottom) binary search based on (a, d) Taylor approximation, (b, e) least squares reduction, and (c, f) uniform matching reduction. 
The B\'ezier splitting is automatically done based on a desired approximation tolerance $\varepsilon>0$ specified in term of the B\'ezier maximum control-point distance: (top) $ \varepsilon = 0.5$ units, (middle, bottom) $\varepsilon = 0.1$ units.}
\label{fig.AdaptiveBezierApproximation}
\end{figure*}

As discussed in \refsec{sec.BezierApproximationPartition}, an $n^{\text{th}}$-order B\'ezier curve $\bcurve_{\bpoint_0, \ldots, \bpoint_n}(t)$  can be locally approximated by $m^{\text{th}}$-order Bezier segments over each element of a $k$-partition $T = [t_0, \ldots, t_k]$ of the unit interval  by applying curve reparametrization and degree reduction as
\begin{align}
\bpmat_{n,i} &= \mathrm{Reparameterize} (\bpmat_i, [t_{i-1}, t_i]), \\
\mpmat_{m,i} &= \mathrm{DegreeReduction}(\bpmat_{n,i}, m), 
\end{align}
Hence, one can measure the quality of approximating  $\bcurve_{\bpoint_{0,i}, \ldots, \bpoint_{n,i}}(t)$  by $\bcurve_{\mpoint_{0,i}, \ldots, \mpoint_{m,i}}(t)$ using a B\'ezier metric $\bdist_{B}$ (\refdef{def.BezierMetric}) and degree elevation (\refprop{prop.ElevationControlPoint}) as
\begin{align}
&\mathrm{BezierDistance}(\bpmat_{n,i}, \mpmat_{m,i} \emat(m,n)) \nonumber \\ 
&\hspace{25mm}:= \bdist_{B}(\bcurve_{\bpoint_{0,i}, \ldots, \bpoint_{n,i}}, \bcurve_{[\mpoint_{0,i}, \ldots, \mpoint_{m,i}] \emat(m,n)}).
\end{align}
Accordingly, in \refalg{alg.AdaptiveApproximationLinearSearch} and \refalg{alg.AdaptiveApproximationBinarySearch}, respectively, we  present a linear- and  a binary-search approach for automatically finding a proper partition $T=[t_0, \ldots, t_k]$ of the unit interval (and the associated  control points $\mpmat_{m,1}, \ldots, \mpmat_{m,k}$ of local B\'ezier segments) where the distance between the actual curve and its degree reduction is below a certain desired approximation tolerance $\varepsilon > 0$.
Note that linear search assumes uniform partitions of the unit interval whereas binary search might result in a nonuniform partition of the unit interval depending on the shape of the input B\'ezier curve.
As a result, as illustrated in \reffig{fig.AdaptiveBezierApproximation}, binary search often achieves the same level of approximation quality as linear search by using a significantly less number of B\'ezier segments.
Another important observation  in \reffig{fig.AdaptiveBezierApproximation} is that although uniform matching still outperforms least squares and Taylor approximations, the quality of adaptive B\'ezier approximation is less dependent on the choice of a reduction method. 
Finally, as expected, the required number of B\'ezier segments increases exponentially with the increasing approximation quality, which is further discussed in the \mbox{following \refsec{sec.NumericalAnalysis}}.

\section{Numerical Analysis \small{of} B\'ezier~Approximations} 
\label{sec.NumericalAnalysis}

In this section, we provide numerical evidence to show the effectiveness of uniform matching reduction over least squares and Taylor reductions by investigating how B\'ezier approximation accuracy depends on the number of curve segments. 
We also demonstrate  how the automatically adjusted number of curve segments in adaptive B\'ezier approximation depends on B\'ezier degree and  approximation tolerance.

\subsection{Approximation Accuracy vs. Number of Segments}

To investigate the role of number of segments in approximation accuracy, we consider the following  B\'ezier features:

\begin{itemize}
\item  \emph{Curve Length:} The arc length of B\'ezier curves is an essential criterion in optimal motion planning to  find motion trajectories that reduce travel distance. 

\item \emph{Distance-to-Point:} The distance of a B\'ezier curve to a point is often used in constrained motion planning for determining parameterwise B\'ezier intersections and the maximum velocity/acceleration along B\'ezier curves.

\item \emph{Distance-to-Line:} The distance of a B\'ezier curve to a line segment (or a polyline/polygon) is a common distance-to-collision measure in safe motion planning. 
\item \emph{Maximum Curvature:} Curvature-constrained motion planning of nonholonomic systems  requires an efficient computation of maximum curvature of B\'ezier curves. 
\end{itemize}
The aforementioned B\'ezier features can be determined analytically only for linear and quadratic B\'ezier curves. 
For high-order B\'ezier curves, we suggest computing these curve features efficiently using B\'ezier approximations by linear and quadratic B\'ezier segments. 
Since these curve features are nonnegative, to determine the approximation quality, we define the normalized approximation error of a B\'ezier feature using its actual and approximate calculations as
\begin{align}
\text{Approximation Error} = \frac{|\text{Feature}_{\text{Approx}} - \text{Feature}_{\text{Actual}}|}{\text{Feature}_{\text{Approx}} + \text{Feature}_{\text{Actual}}}, \!\!\!
\end{align} 
where the actual curve features are computed using a dense discrete samples of B\'ezier curves.

\begin{figure*}
\centering
\begin{tabular}{@{\hspace{-0.1mm}}c@{\hspace{0mm}}c@{\hspace{0mm}}c@{\hspace{0mm}}c@{\hspace{0mm}}c@{\hspace{0mm}}c@{}}
\includegraphics[width=0.1675\textwidth]{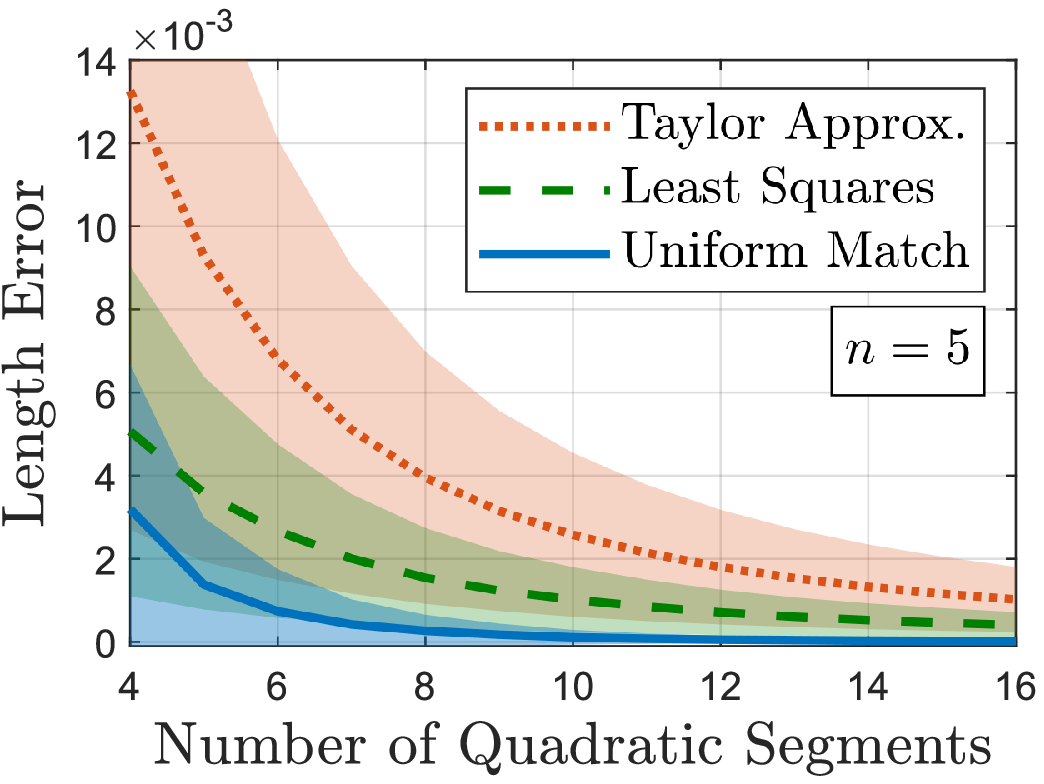} &
\includegraphics[width=0.1675\textwidth]{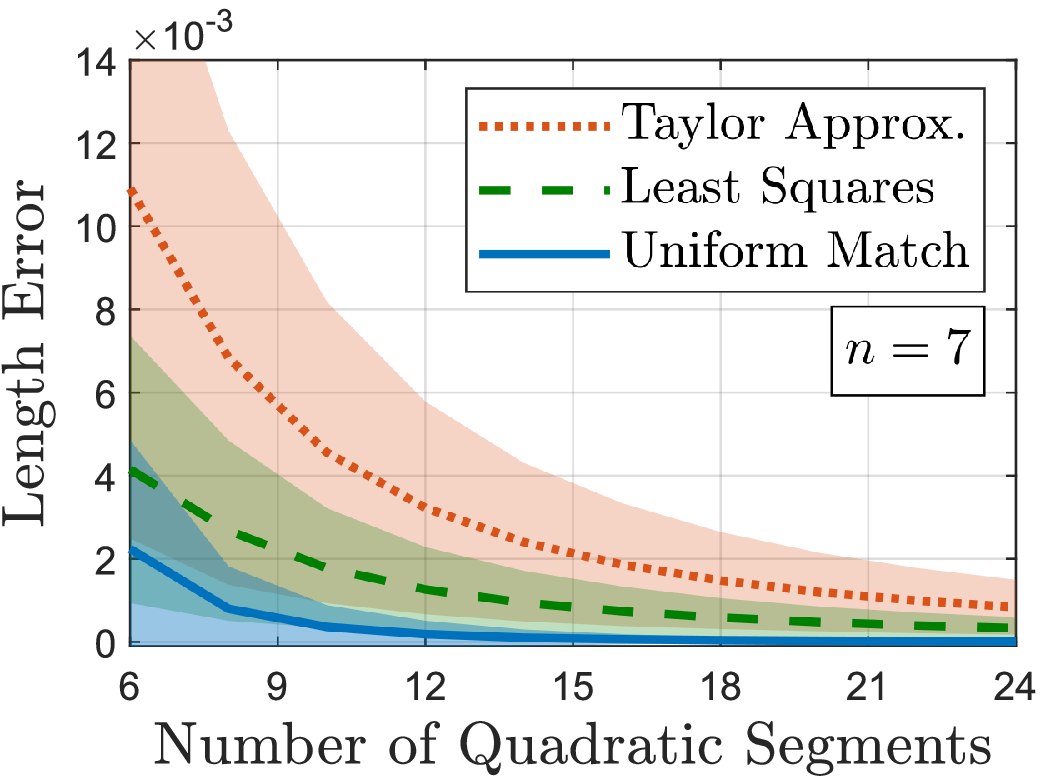} &
\includegraphics[width=0.1675\textwidth]{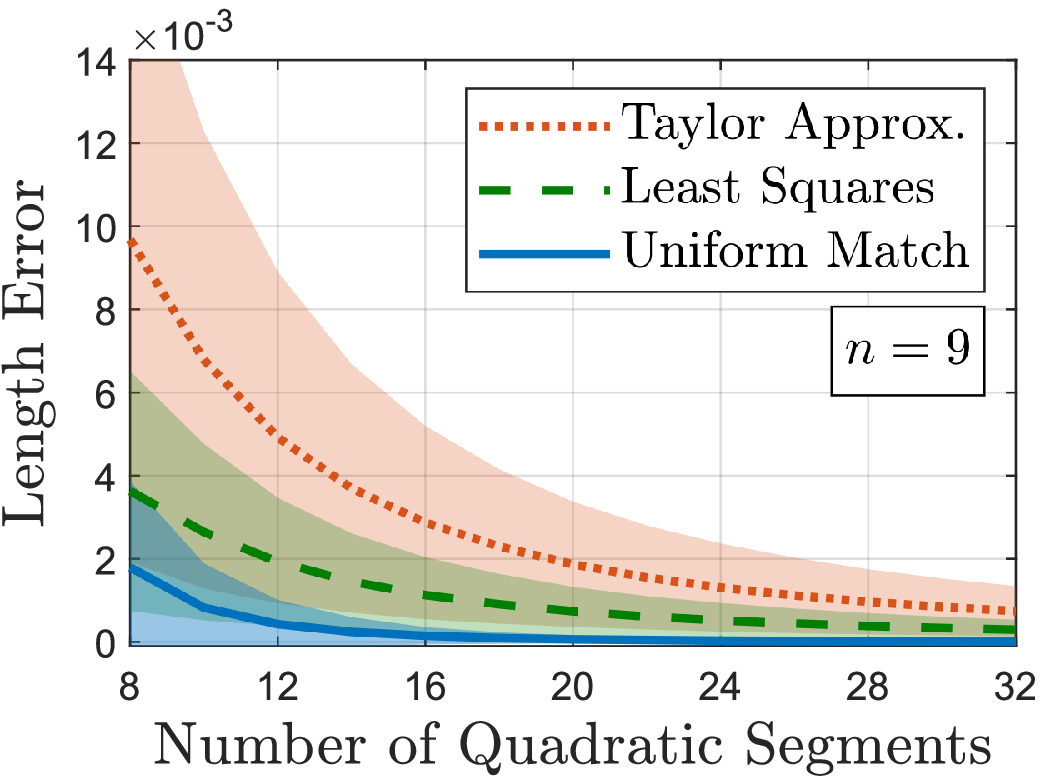} &
\includegraphics[width=0.1675\textwidth]{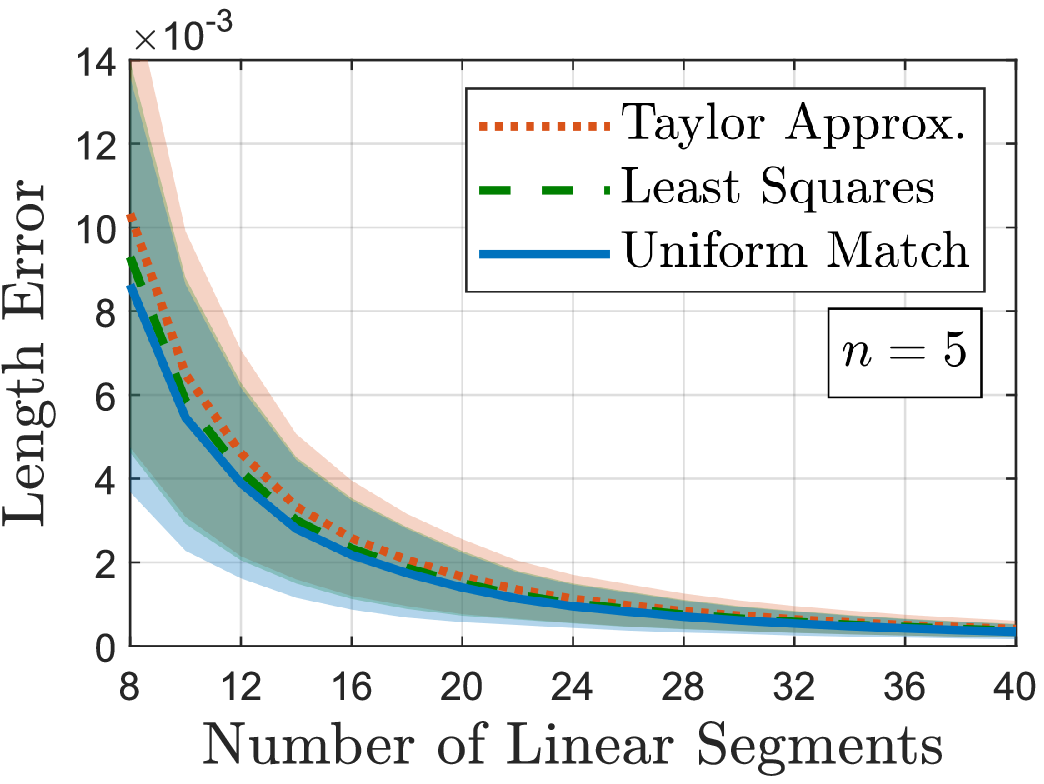} &
\includegraphics[width=0.1675\textwidth]{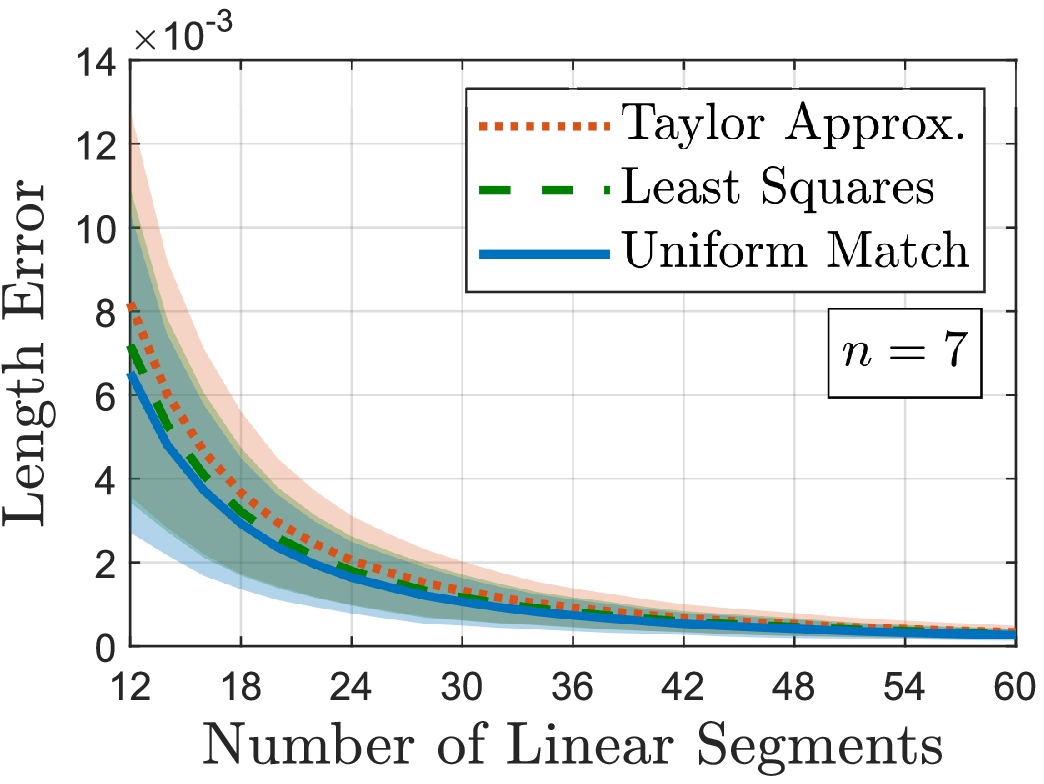} &
\includegraphics[width=0.1675\textwidth]{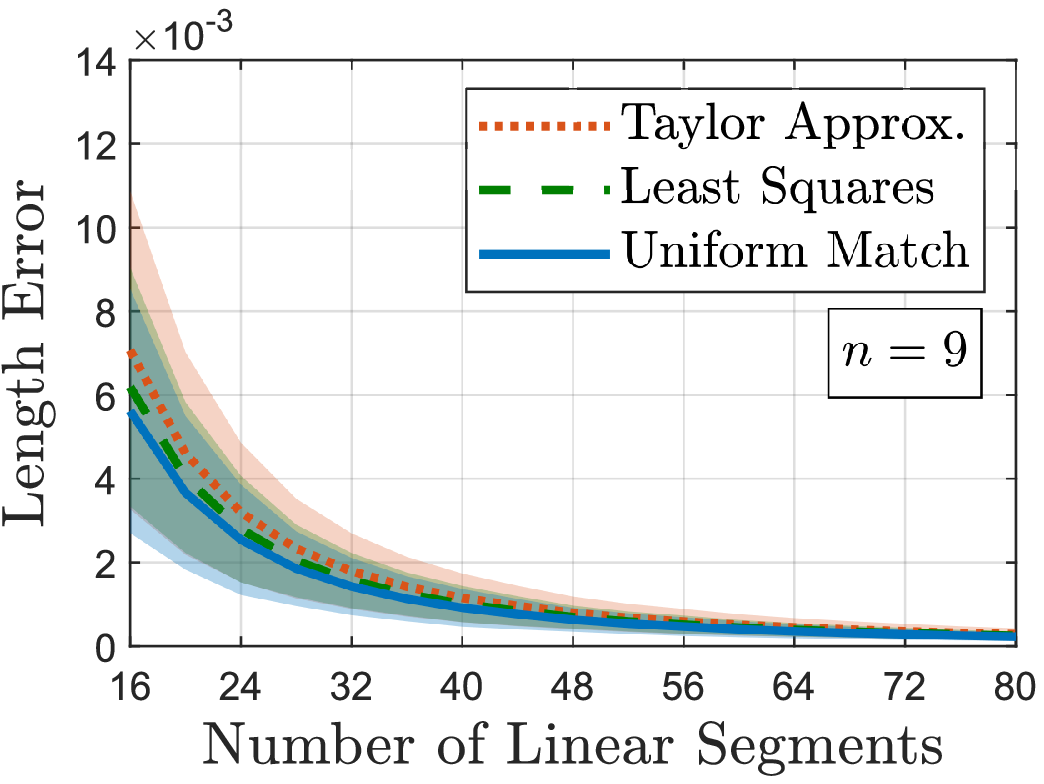} 
\\[-1.5mm]
\scriptsize{(a)} & \scriptsize{(b)} & \scriptsize{(c)} & \scriptsize{(d)} & \scriptsize{(e)} & \scriptsize{(f)}
\end{tabular}
\vspace{-3.5mm}
\caption{Normalized length error statistics of approximating $n^{\text{th}}$-order B\'ezier curves by (a, b, c) quadratic and (d, e, f) linear B\'ezier segments: (a, d) $n = 5$, (b, e) $n=7$, (c, f) $n = 9$, where the mean and the standard deviation of the error are presented by a line and a shaded region, respectively.}
\label{fig.LengthErrorStatistics}
\end{figure*}

\begin{figure*}
\centering
\begin{tabular}{@{\hspace{0mm}}c@{\hspace{0mm}}c@{\hspace{0mm}}c@{\hspace{0mm}}c@{\hspace{0mm}}c@{\hspace{0mm}}c@{\hspace{0mm}}c@{}}
\includegraphics[width=0.1675\textwidth]{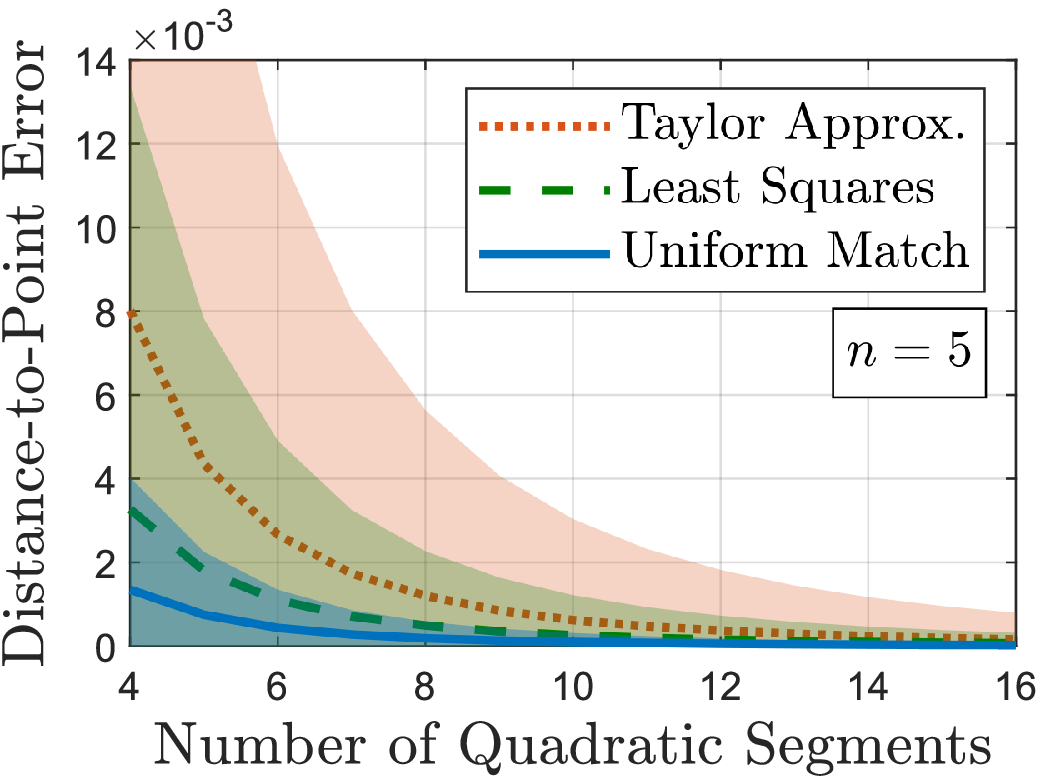} &
\includegraphics[width=0.1675\textwidth]{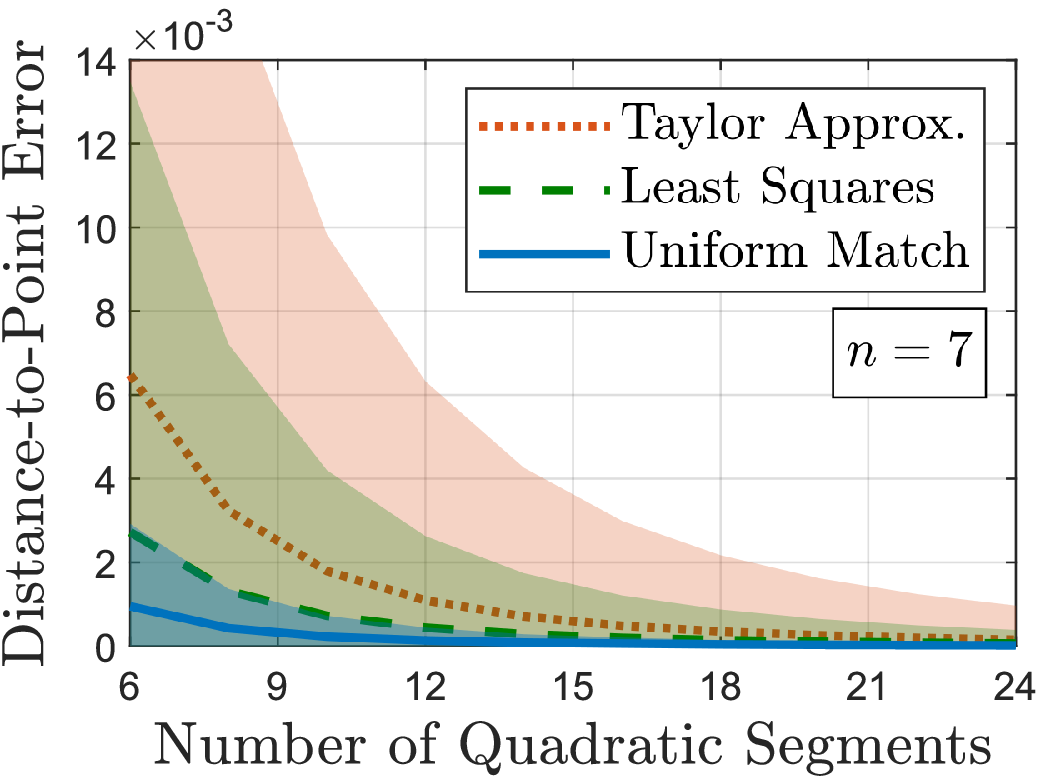} &
\includegraphics[width=0.1675\textwidth]{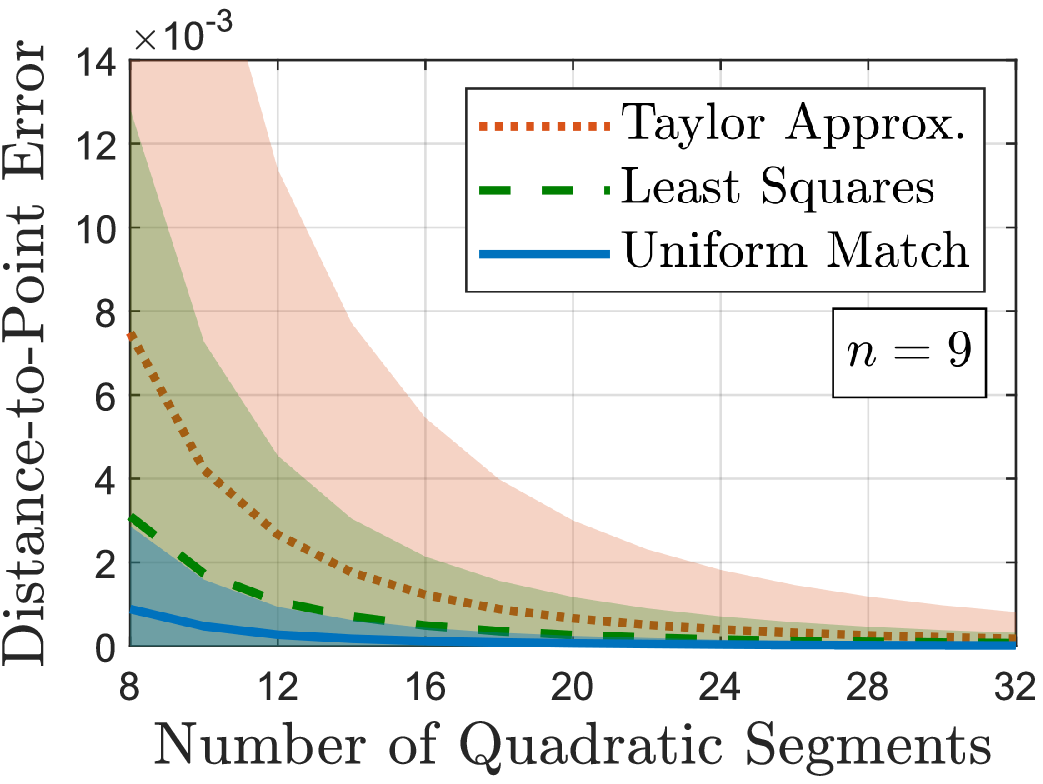} &
\includegraphics[width=0.1675\textwidth]{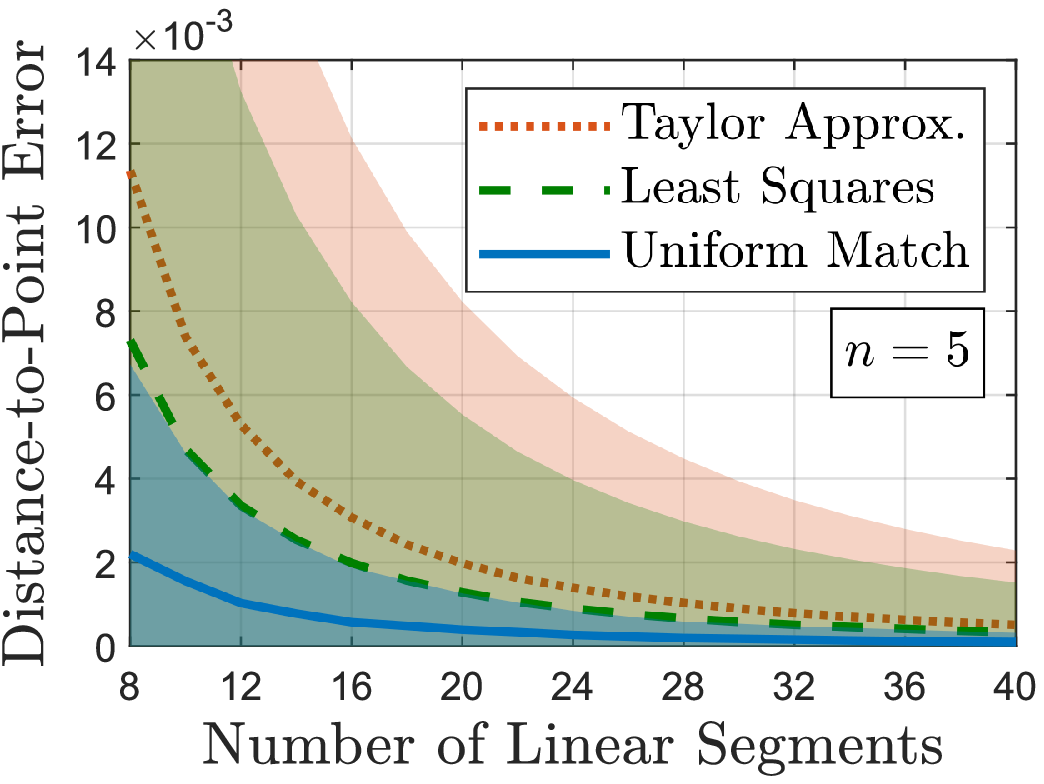} &
\includegraphics[width=0.1675\textwidth]{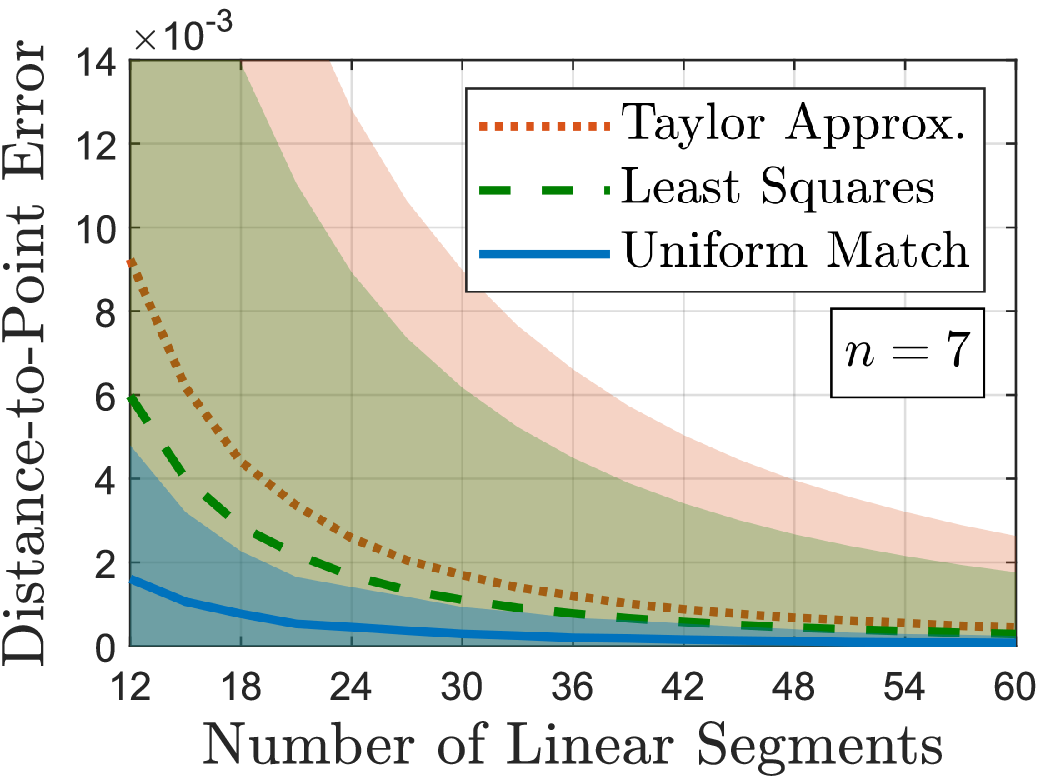} &
\includegraphics[width=0.1675\textwidth]{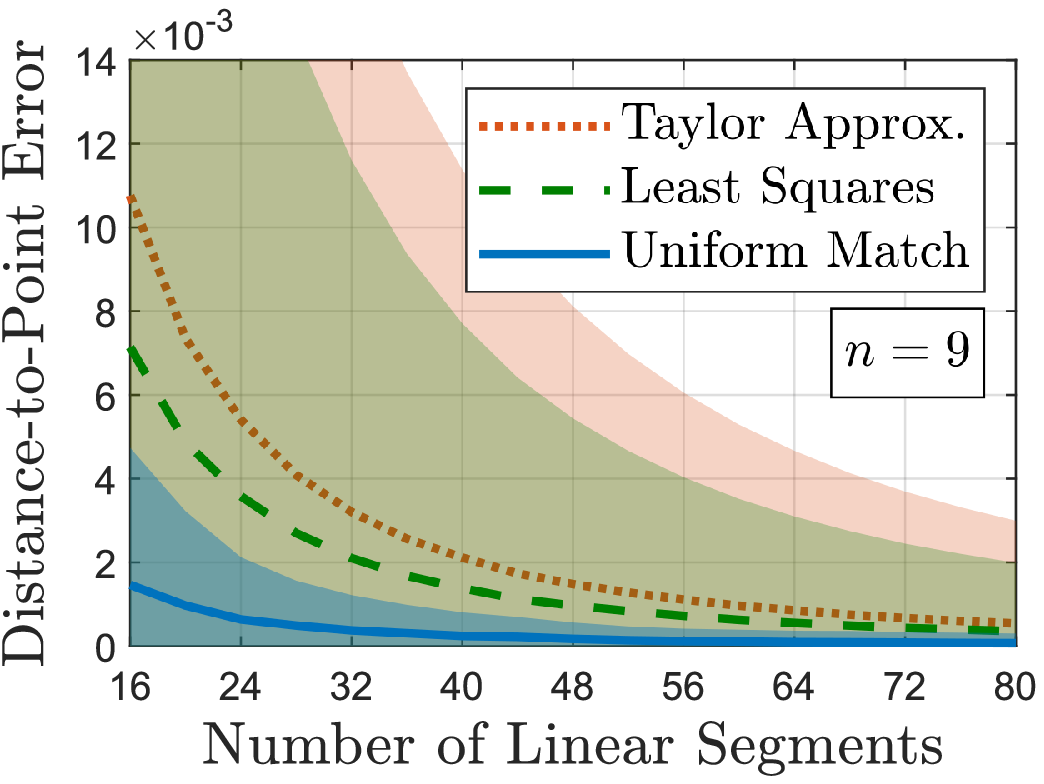}
\\[-1.5mm]
\scriptsize{(a)} & \scriptsize{(b)} & \scriptsize{(c)} & \scriptsize{(d)} & \scriptsize{(e)} & \scriptsize{(f)}
\end{tabular}
\vspace{-3.5mm}
\caption{Normalized distance-to-point error statistics of approximating $n^{\text{th}}$-order B\'ezier curves by (a, b, c) quadratic and (d, e, f) linear B\'ezier segments: (a, d) $n = 5$, (b, e) $n=7$, (c, f) $n = 9$, where the mean and the standard deviation of the error are presented by a line and a shaded region, respectively.}
\label{fig.DistanceToPointErrorStatistics}
\end{figure*}

\begin{figure*}
\centering
\begin{tabular}{@{\hspace{0mm}}c@{\hspace{0mm}}c@{\hspace{0mm}}c@{\hspace{0mm}}c@{\hspace{0mm}}c@{\hspace{0mm}}c@{\hspace{0mm}}c@{}}
\includegraphics[width=0.1675\textwidth]{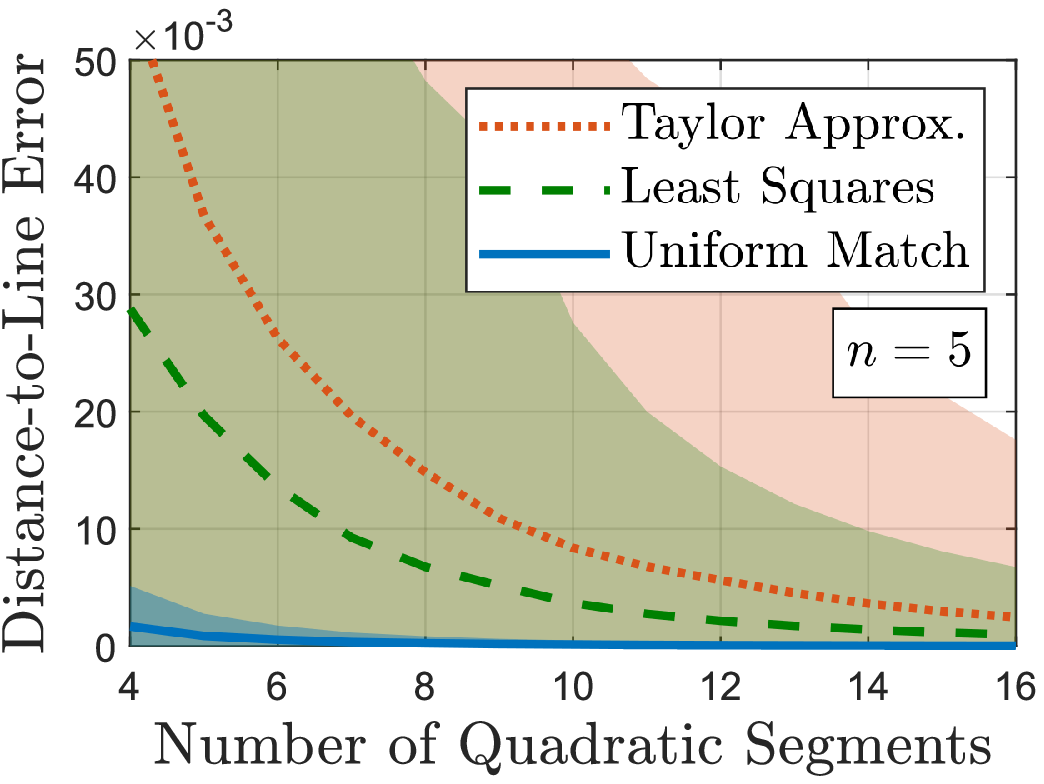} &
\includegraphics[width=0.1675\textwidth]{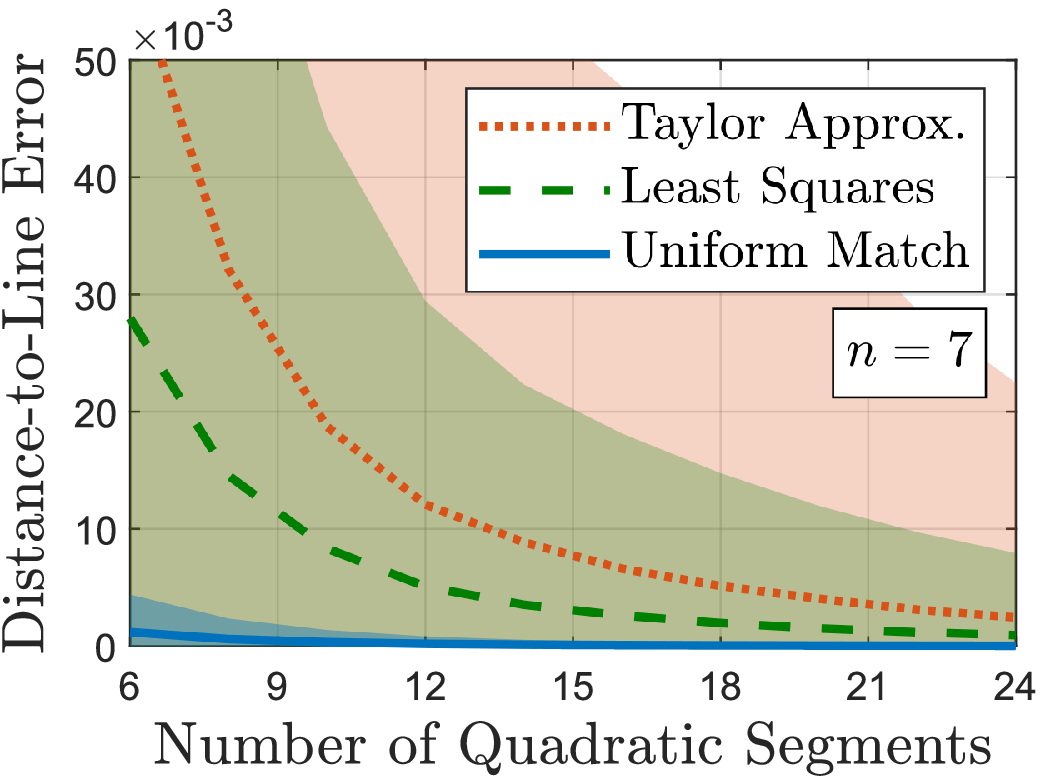} &
\includegraphics[width=0.1675\textwidth]{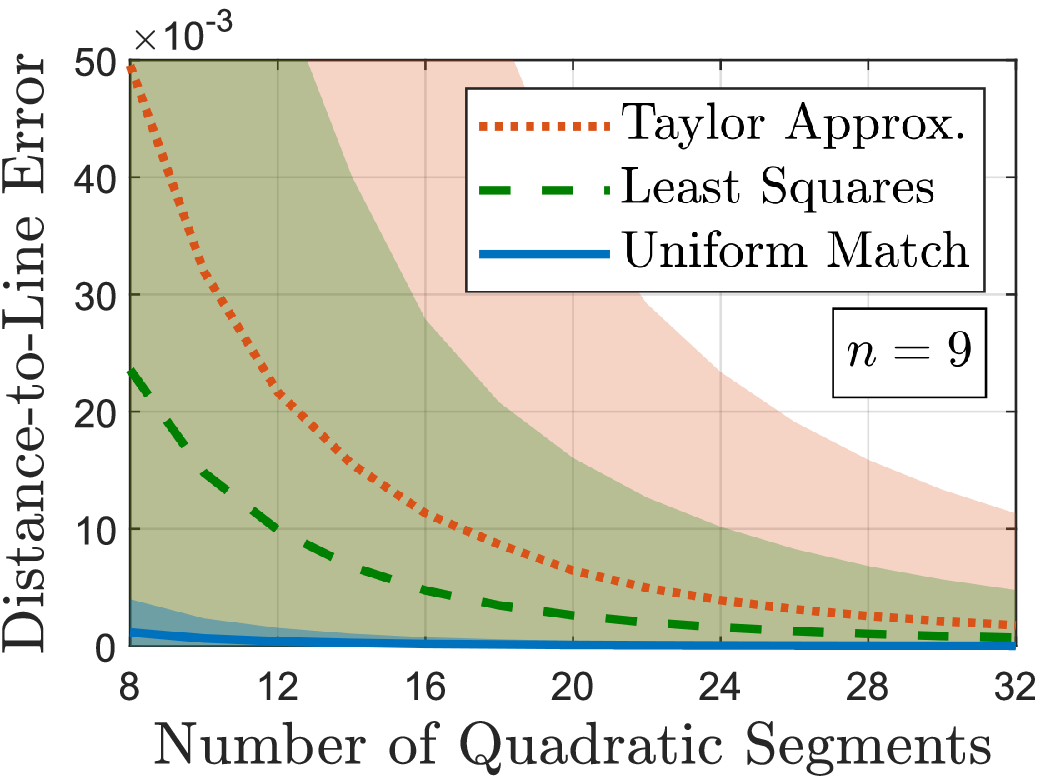} &
\includegraphics[width=0.1675\textwidth]{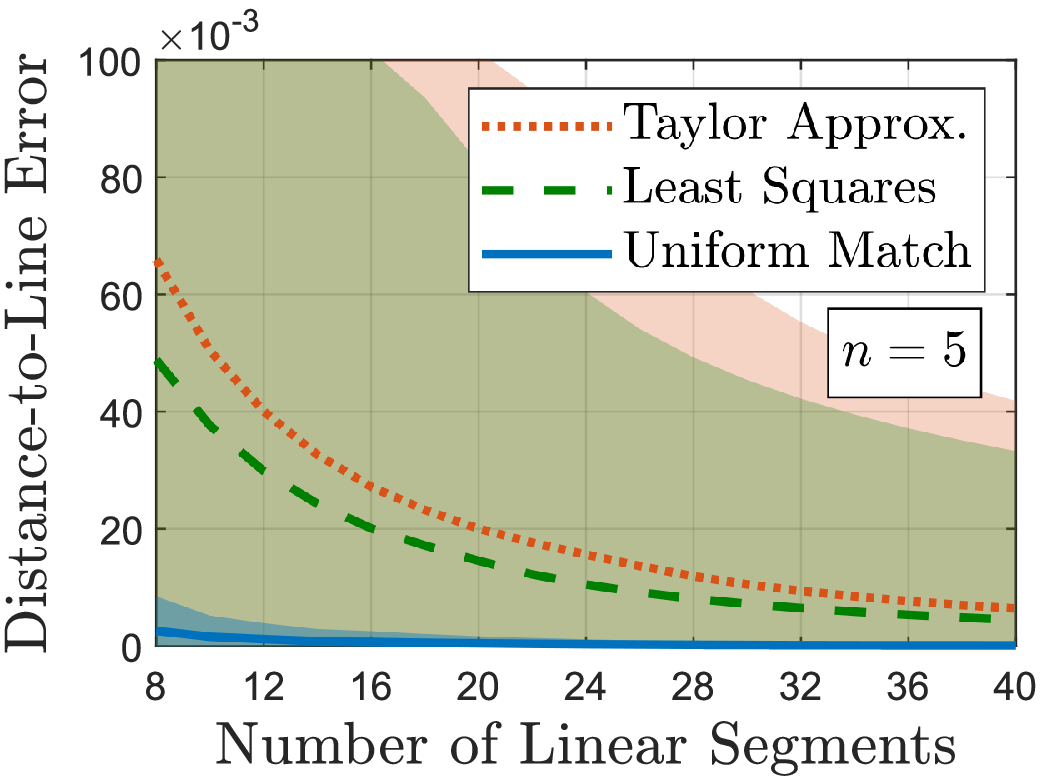} &
\includegraphics[width=0.1675\textwidth]{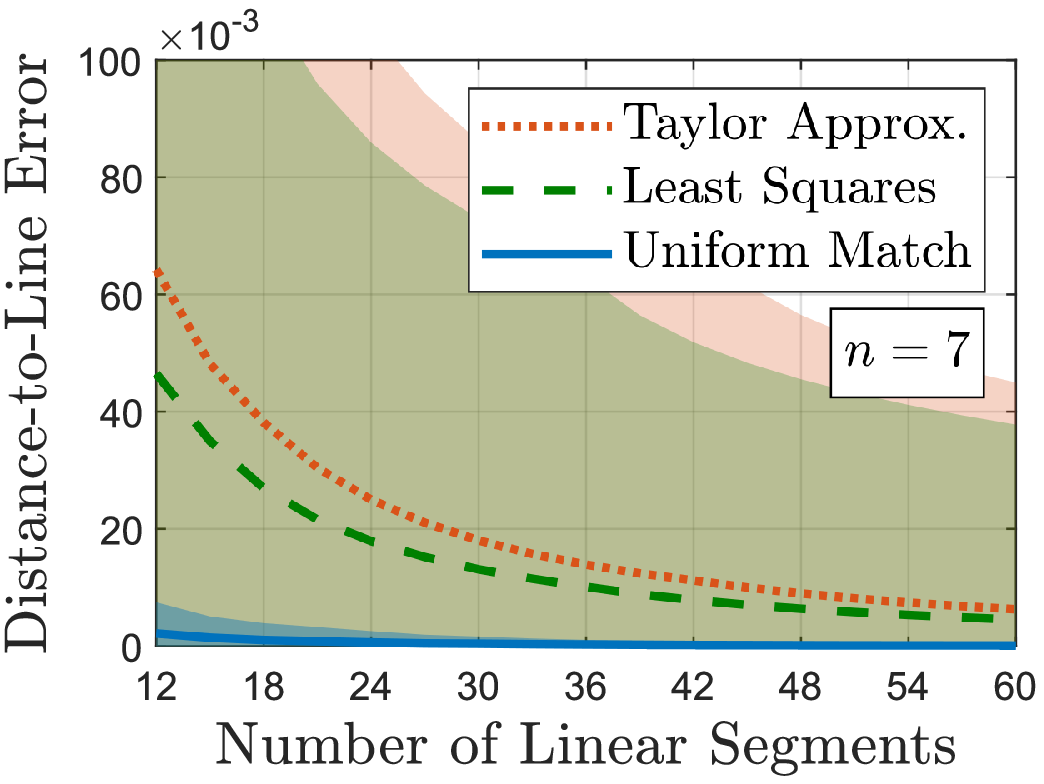} &
\includegraphics[width=0.1675\textwidth]{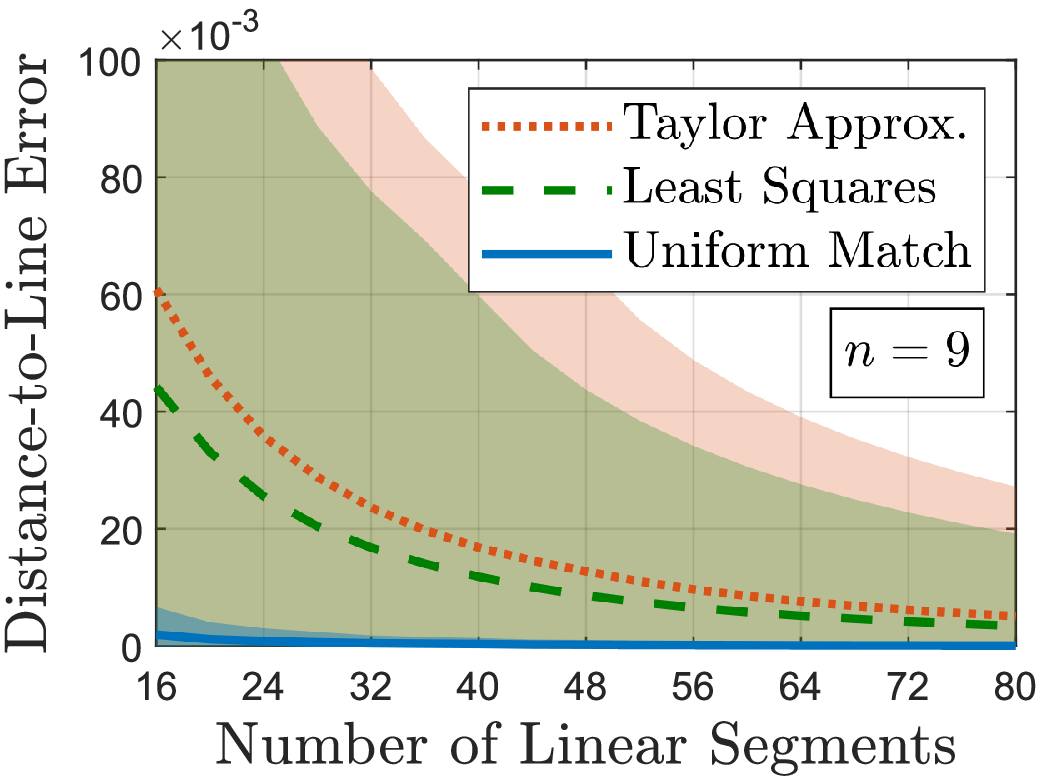} 
\\[-1.5mm]
\scriptsize{(a)} & \scriptsize{(b)} & \scriptsize{(c)} & \scriptsize{(d)} & \scriptsize{(e)} & \scriptsize{(f)}
\end{tabular}
\vspace{-3.5mm}
\caption{Normalized distance-to-line error statistics of approximating $n^{\text{th}}$-order B\'ezier curves by (a, b, c) quadratic and (d, e, f) linear B\'ezier segments: (a, d) $n = 5$, (b, e) $n=7$, (c, f) $n = 9$, where the mean and the standard deviation of the error are presented by a line and a shaded region, respectively.}
\label{fig.DistanceToLineErrorStatistics}
\end{figure*}

\begin{figure}
\centering
\begin{tabular}{@{\hspace{0mm}}c@{\hspace{0mm}}c@{\hspace{0mm}}c@{}}
\includegraphics[width=0.165\textwidth]{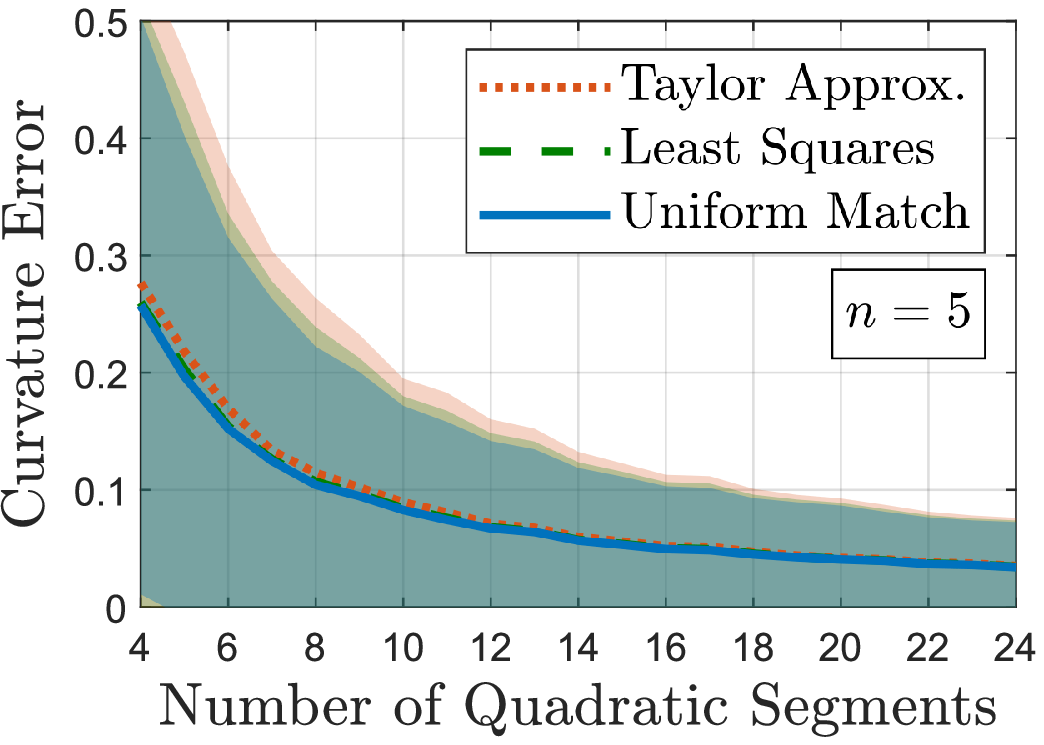} &
\includegraphics[width=0.165\textwidth]{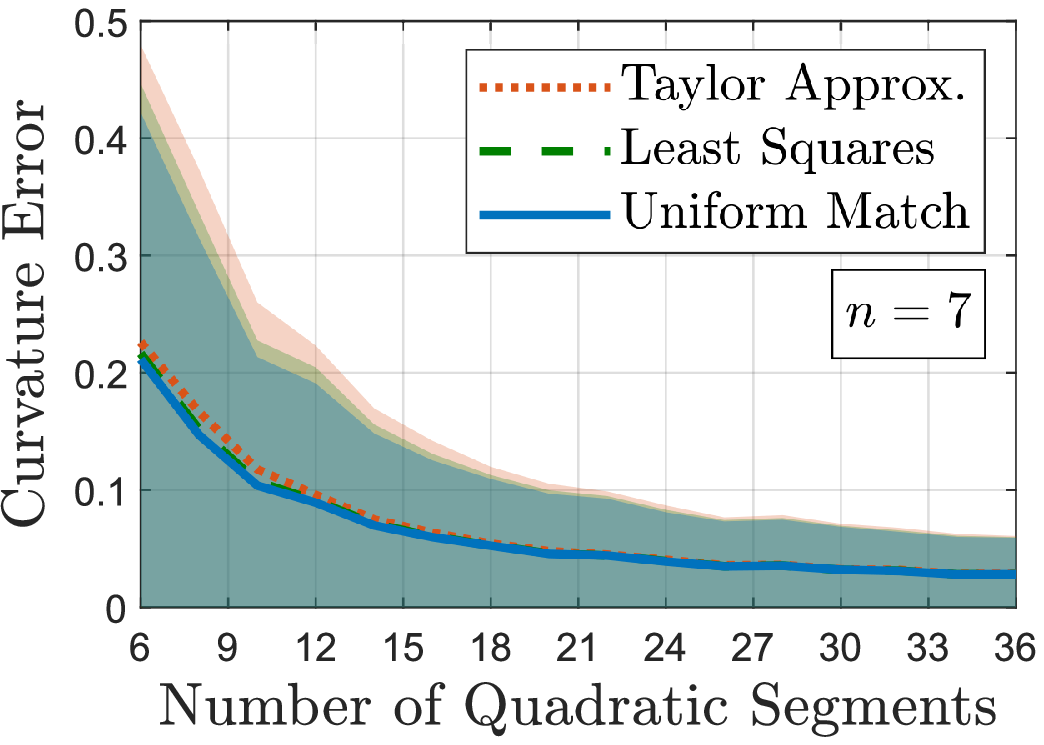} &
\includegraphics[width=0.165\textwidth]{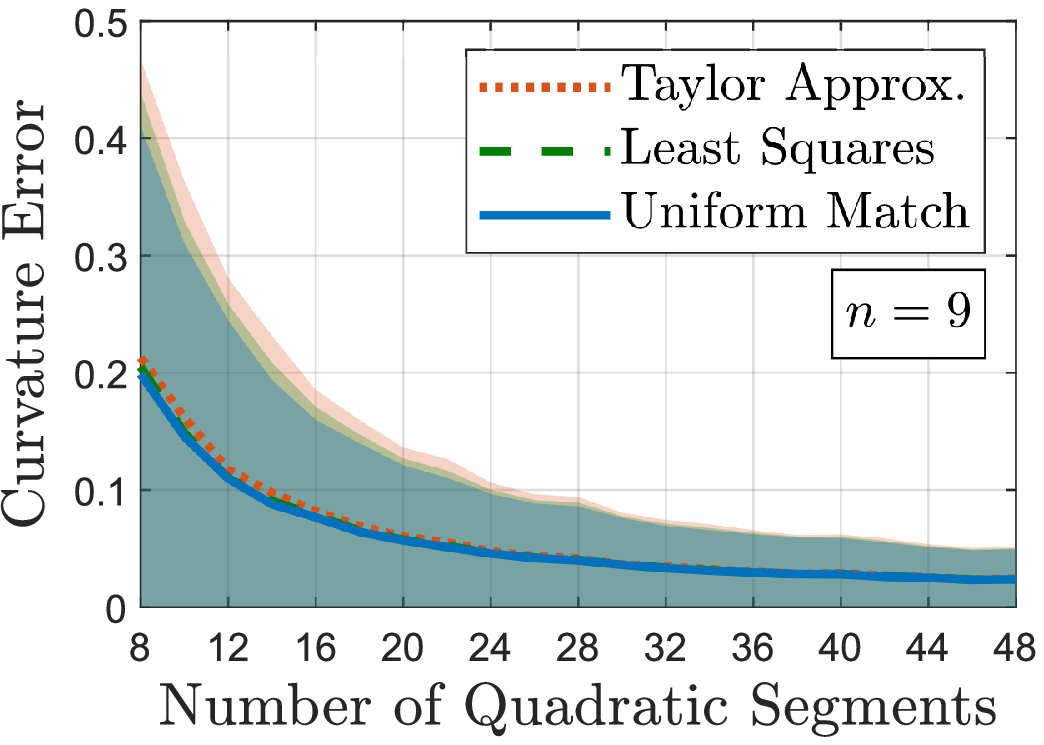} 
\\[-1.5mm]
\scriptsize{(a)} & \scriptsize{(b)} & \scriptsize{(c)}
\end{tabular}
\vspace{-3.5mm}
\caption{Normalized maximum curvature error statistics of quadratic approximations of B\'ezier curves for different number of segments:  (a) $n = 5$, (b) $n=7$, (c) $n = 9$, where the mean and the standard deviation of the error are presented by a line and a shaded region, respectively. Note that for numerical stability we set an upper bound of $1000$ units on the maximum curvature, and any sample case with a larger maximum curvature is rejected.}
\label{fig.MaximumCurvatureErrorStatistics}
\end{figure}

\begin{figure*}[t]
\begin{tabular}{@{}c@{\hspace{1mm}}c@{\hspace{1mm}}c@{\hspace{1mm}}c@{\hspace{1mm}}c@{\hspace{1mm}}c@{}}
\includegraphics[width=0.1625\textwidth]{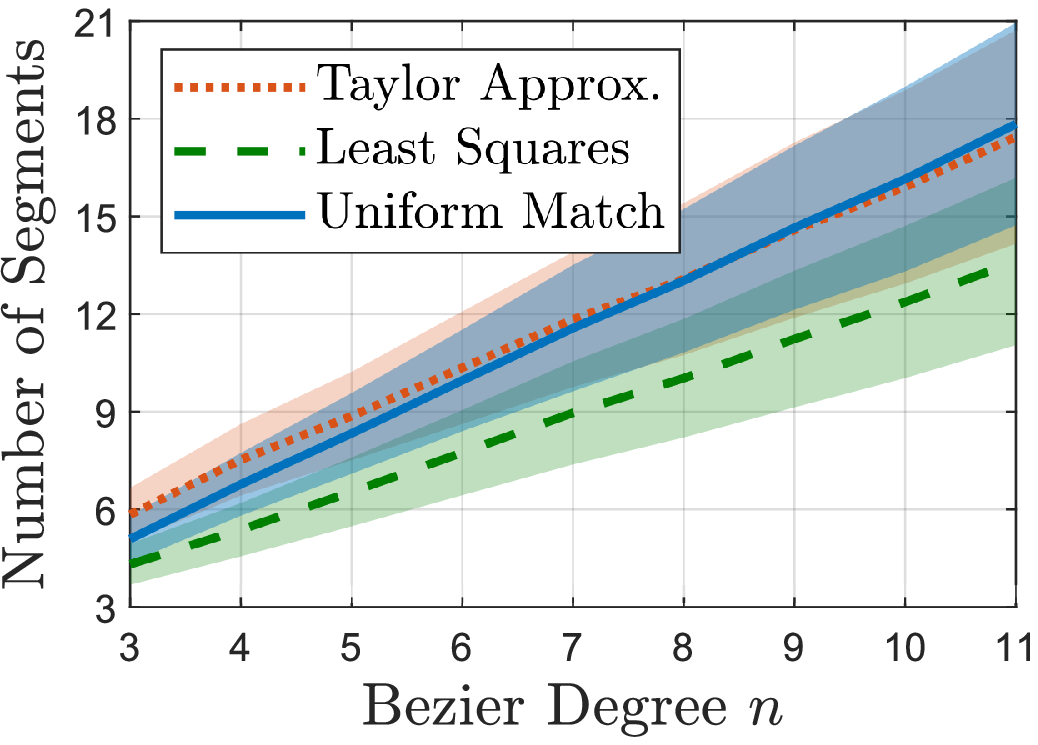} & 
\includegraphics[width=0.1625\textwidth]{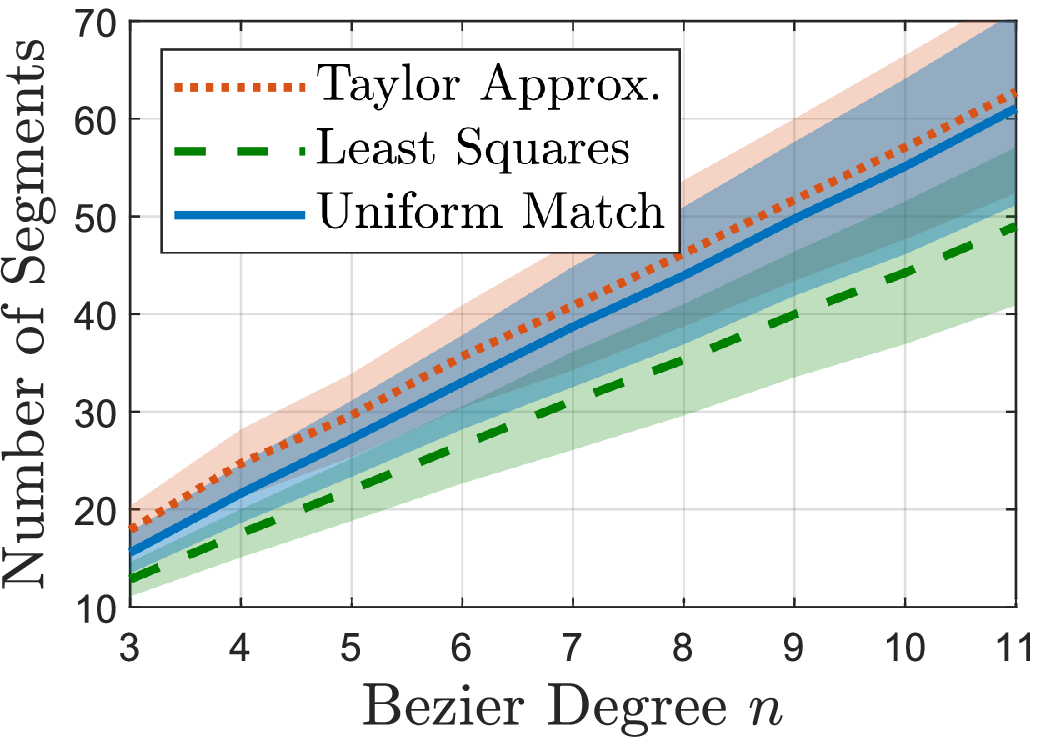} &
\includegraphics[width=0.1625\textwidth]{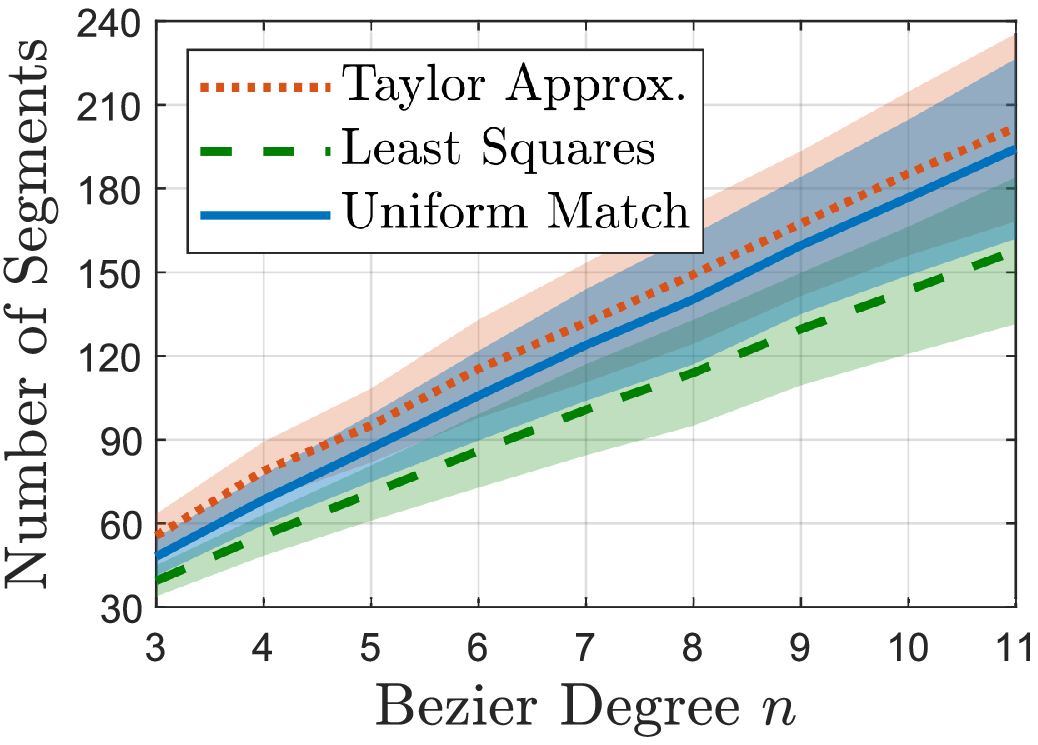} &
\includegraphics[width=0.1625\textwidth]{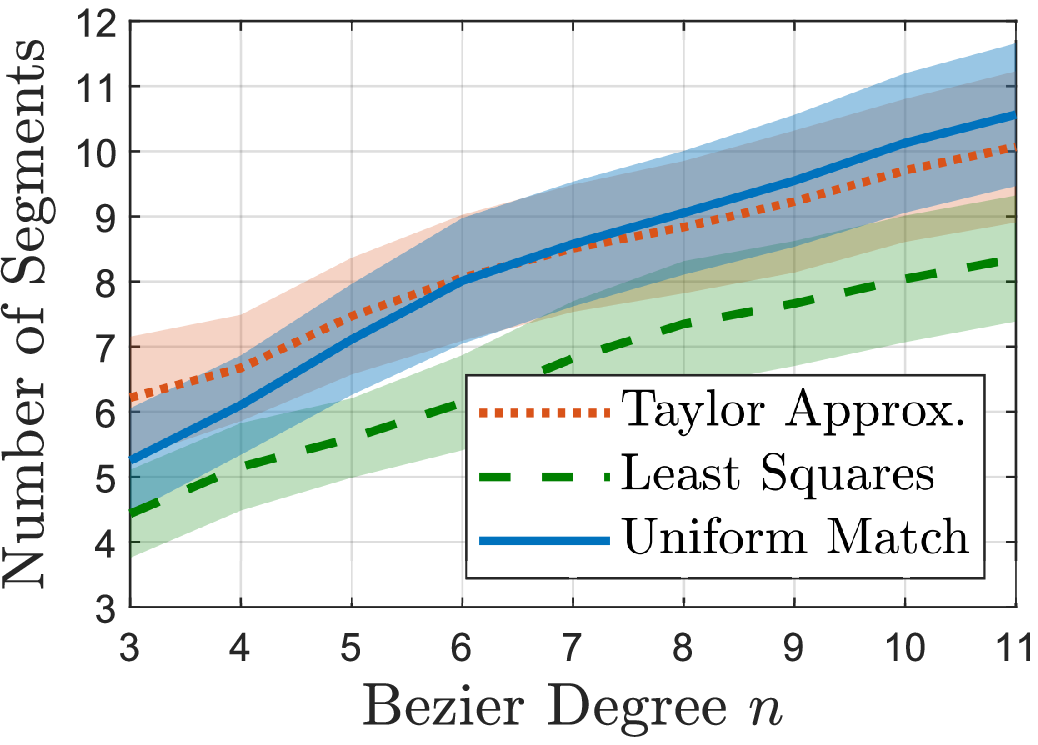} &
\includegraphics[width=0.1625\textwidth]{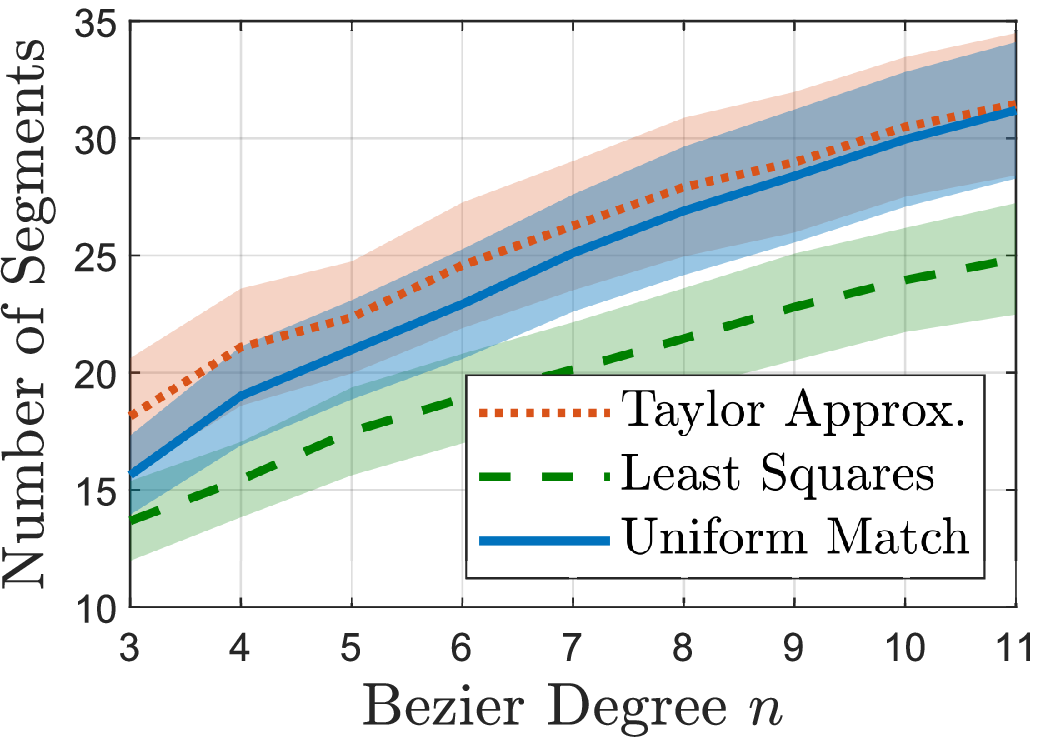} &
\includegraphics[width=0.1625\textwidth]{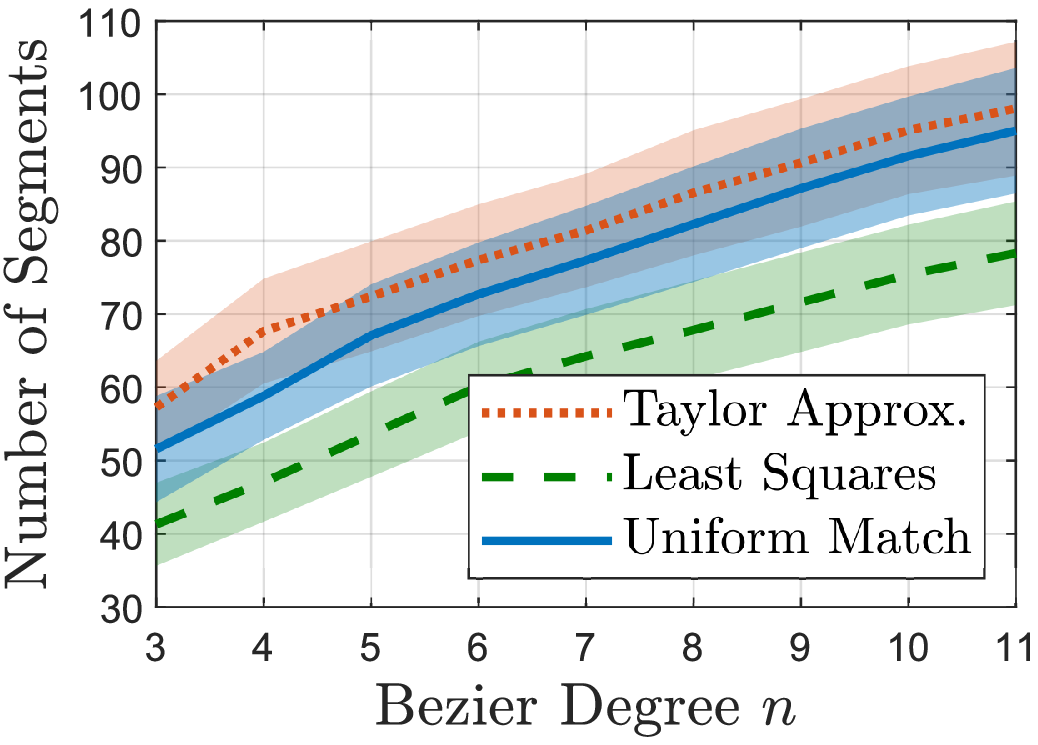} 
\\
\includegraphics[width=0.1625\textwidth]{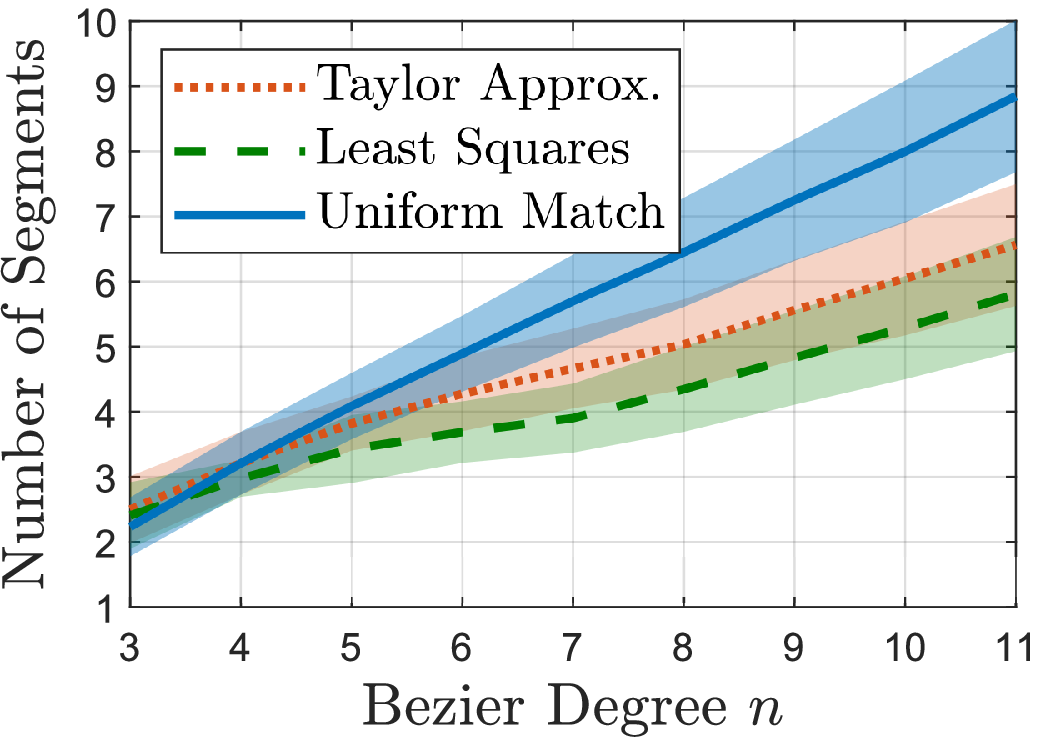} & 
\includegraphics[width=0.1625\textwidth]{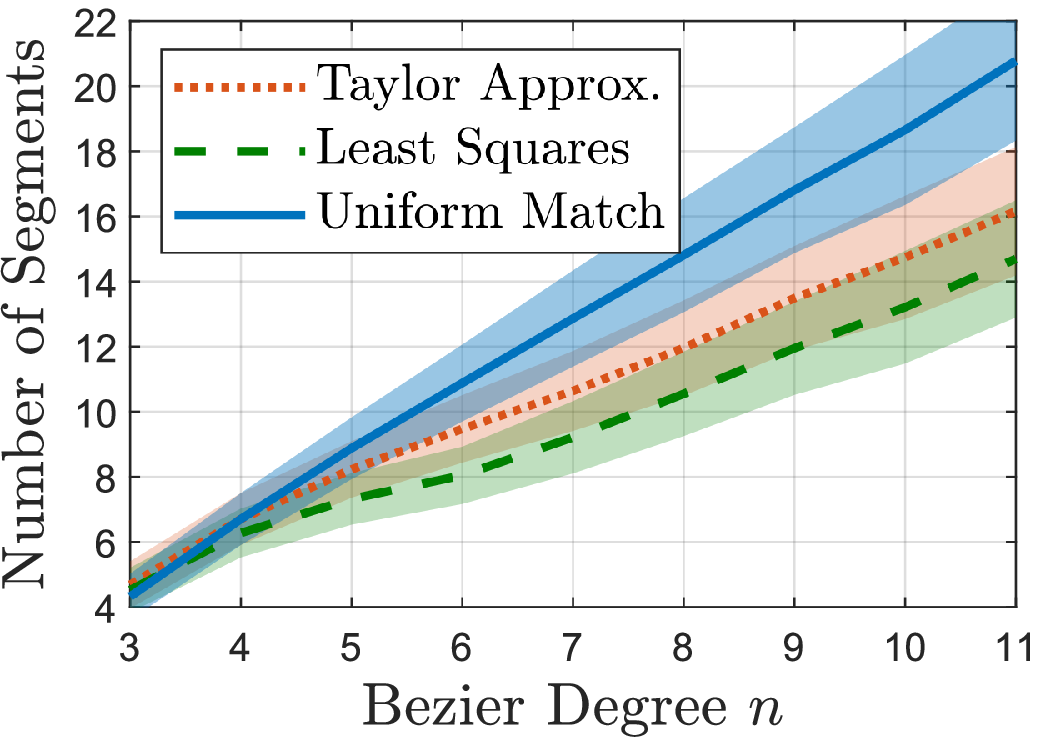} &
\includegraphics[width=0.1625\textwidth]{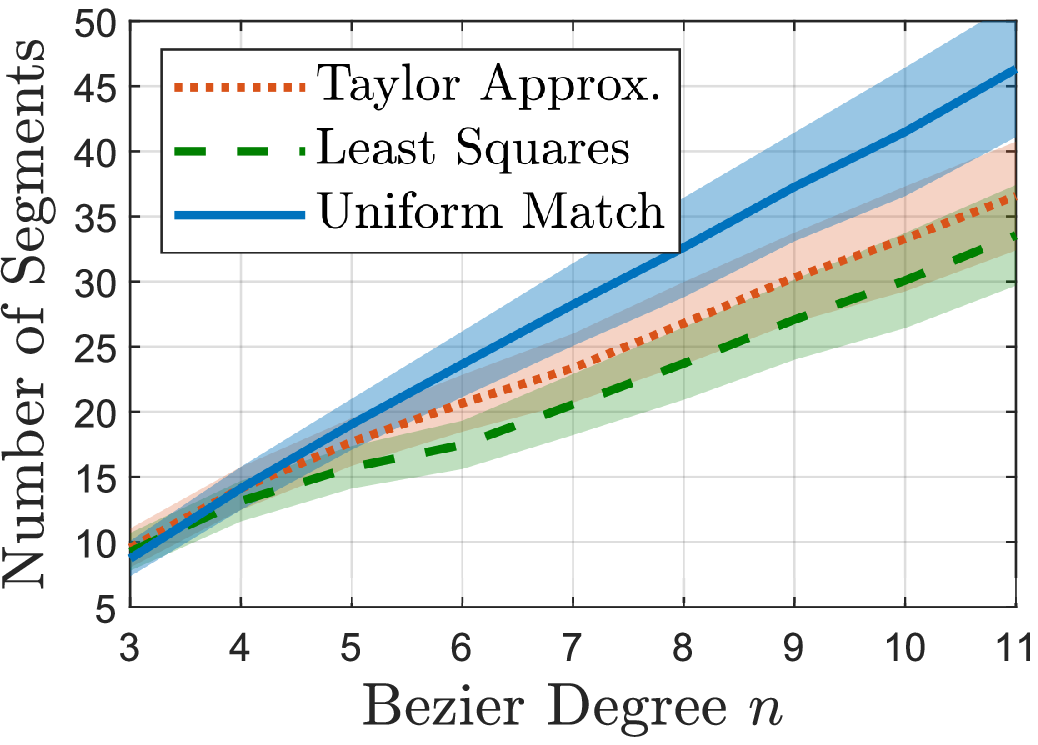} &
\includegraphics[width=0.1625\textwidth]{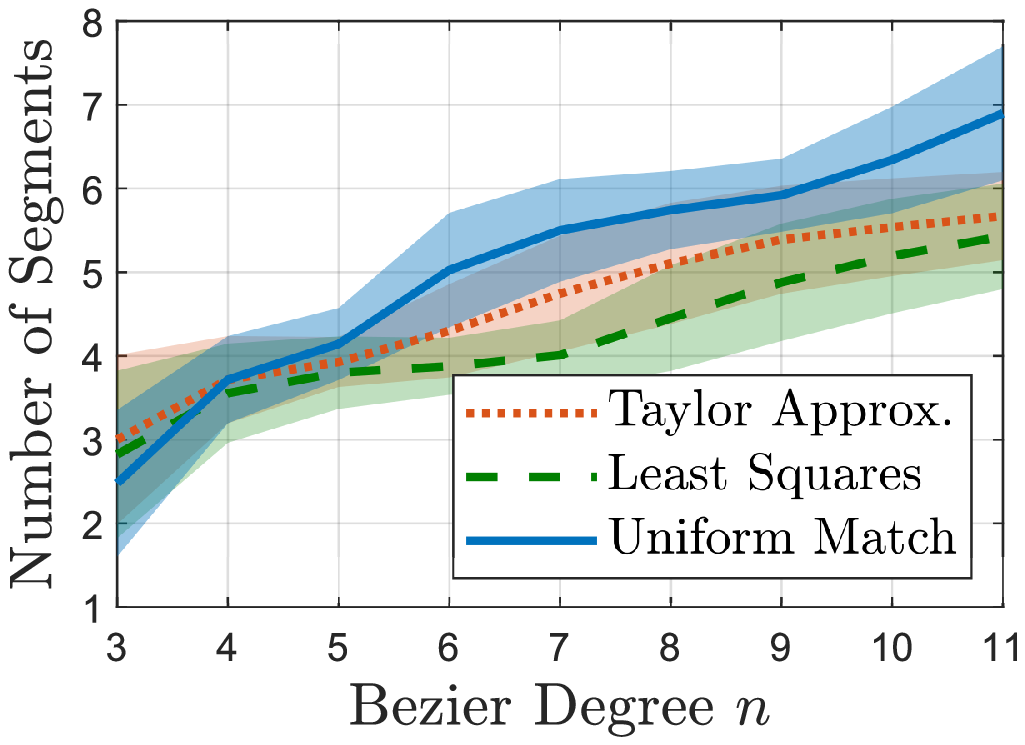} &
\includegraphics[width=0.1625\textwidth]{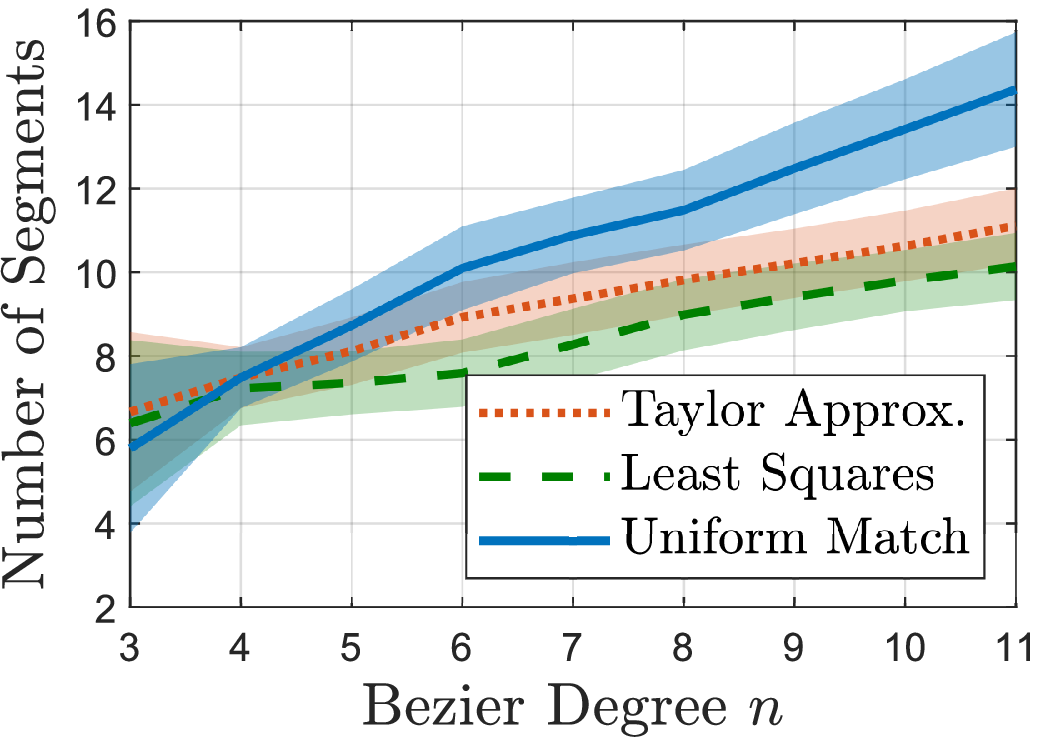} &
\includegraphics[width=0.1625\textwidth]{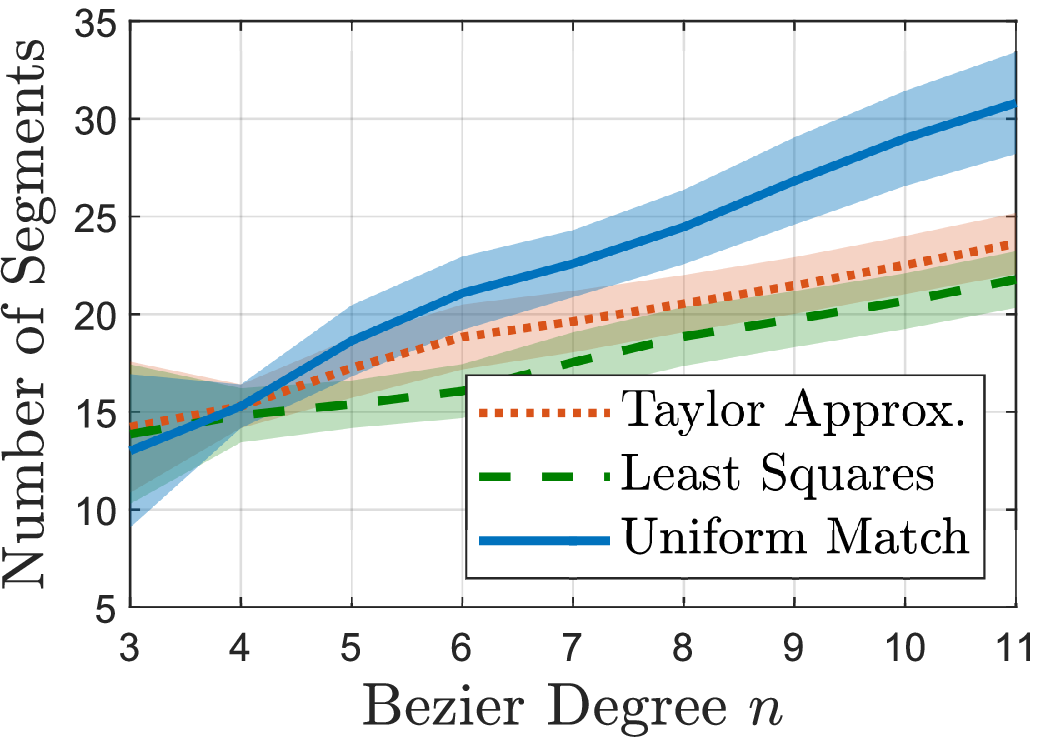} 
\\[-1.5mm]
\scriptsize{(a)}& \scriptsize{(b)}& \scriptsize{(c)}& \scriptsize{(d)}& \scriptsize{(e)}& \scriptsize{(f)} 
\end{tabular}
\vspace{-3.5mm}
\caption{Relation between the number of segments and B\'ezier degree in adaptive B\'ezier approximation by (top) linear and (bottom) quadratic segments using (a, b, c) linear  and (d, e, f) binary  search. Here, the approximation tolerance is measured by the maximum control-point distance and set to be (a, d) 0.1, (b, e) 0.01, and (c, f) 0.001 units.}
\label{fig.NumberOfSegmentsBezierOrder}
\end{figure*}

\begin{figure*}[t]
\begin{tabular}{@{}c@{\hspace{1mm}}c@{\hspace{1mm}}c@{\hspace{1mm}}c@{\hspace{1mm}}c@{\hspace{1mm}}c@{}}
\includegraphics[width=0.1625\textwidth]{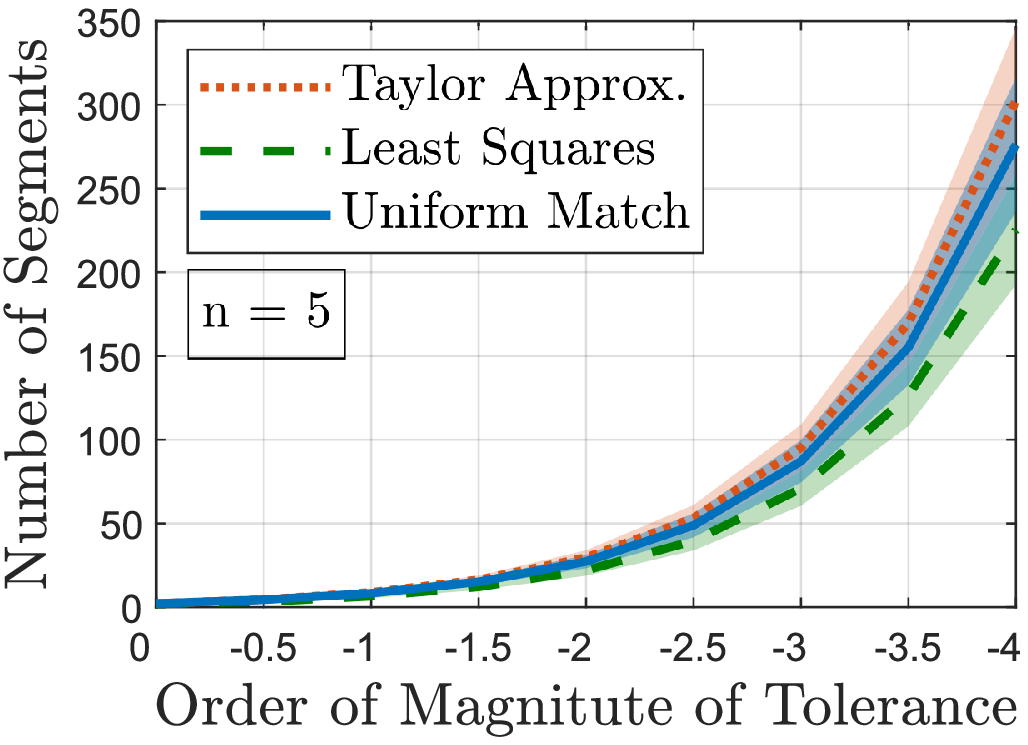} &
\includegraphics[width=0.1625\textwidth]{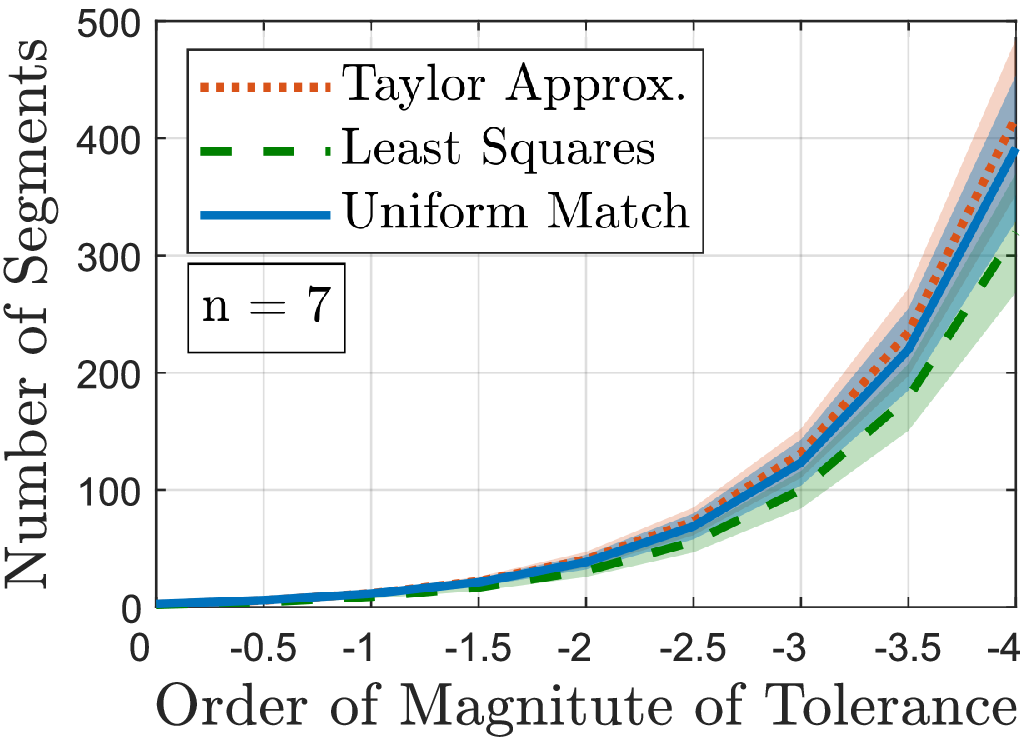} &
\includegraphics[width=0.1625\textwidth]{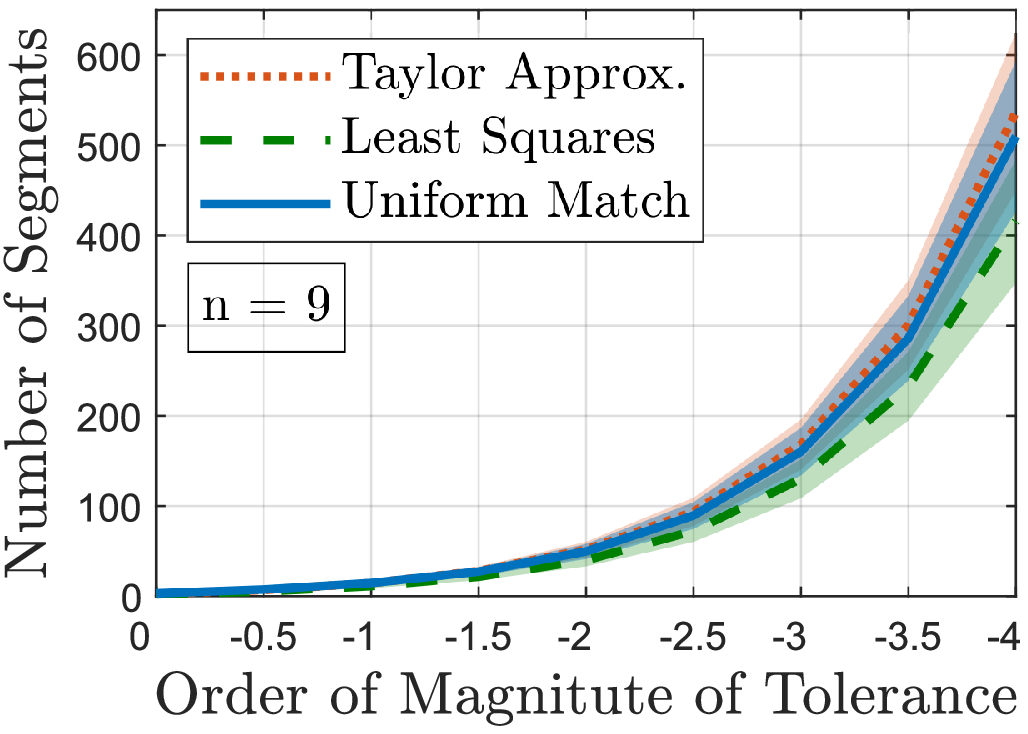} &
\includegraphics[width=0.1625\textwidth]{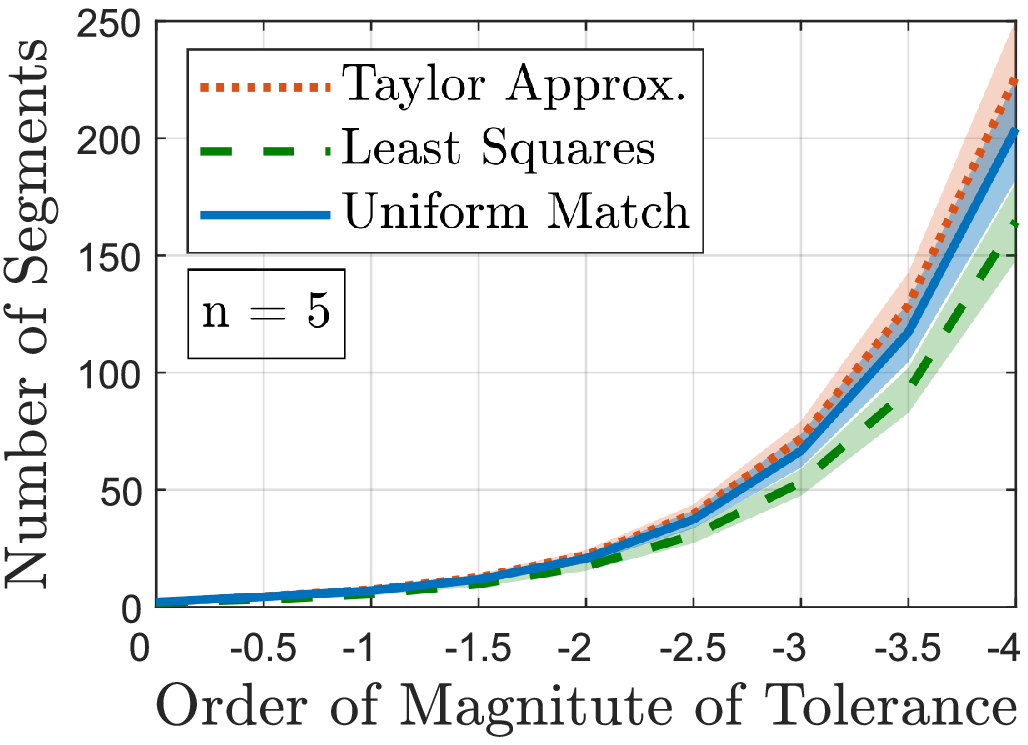} &
\includegraphics[width=0.1625\textwidth]{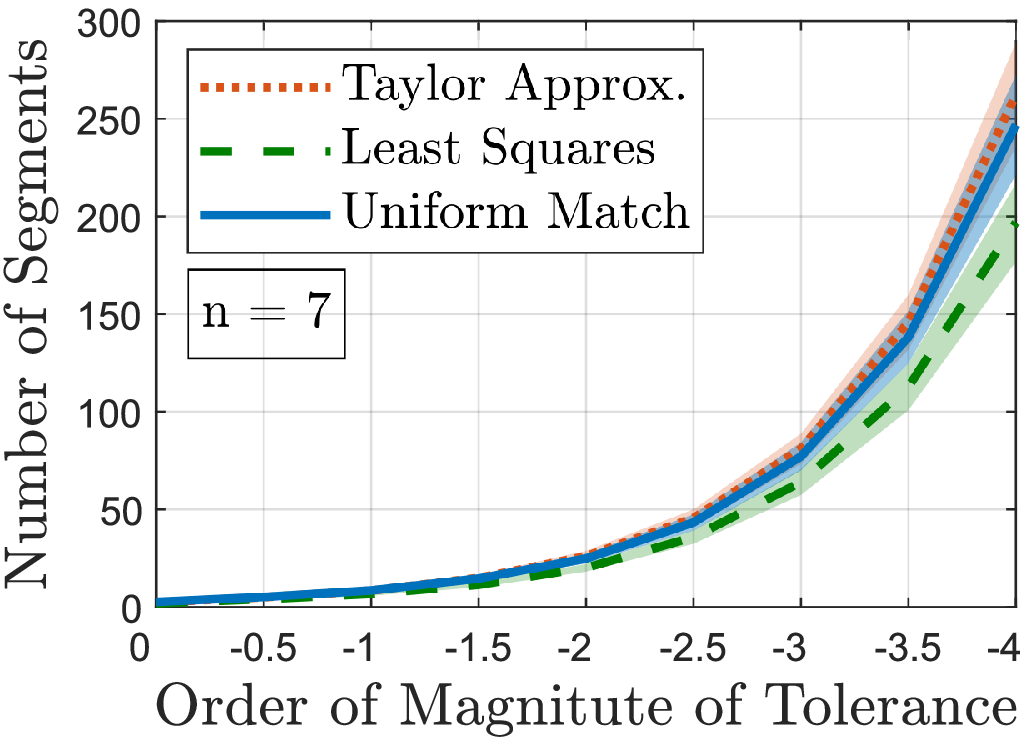} &
\includegraphics[width=0.1625\textwidth]{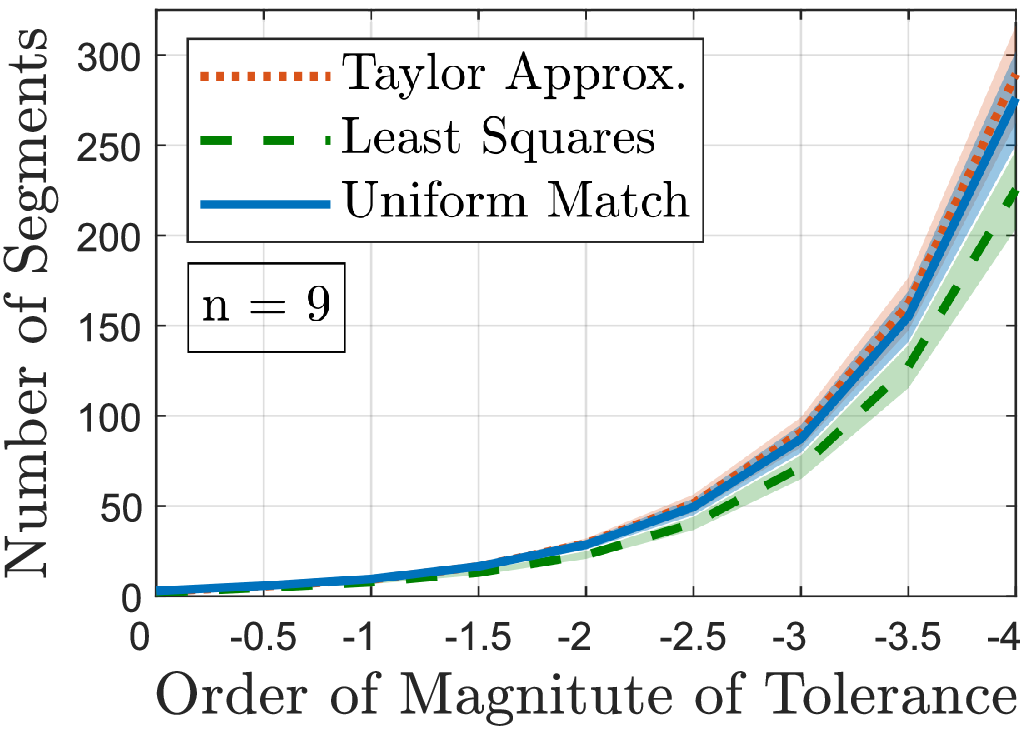} 
\\
\includegraphics[width=0.1625\textwidth]{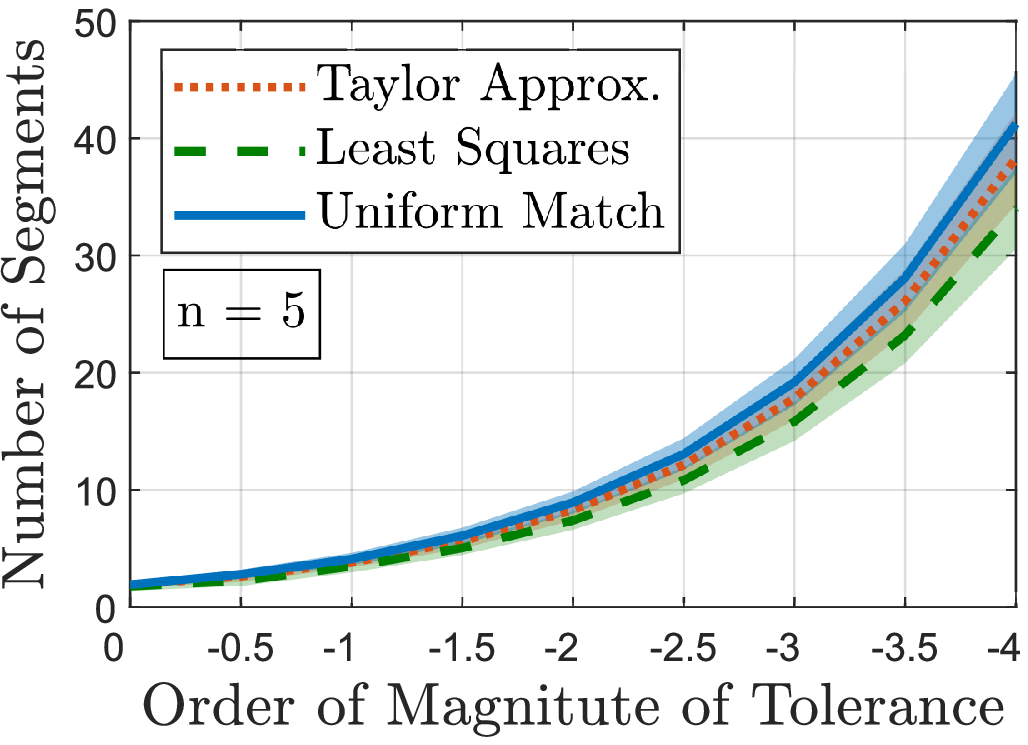} &
\includegraphics[width=0.1625\textwidth]{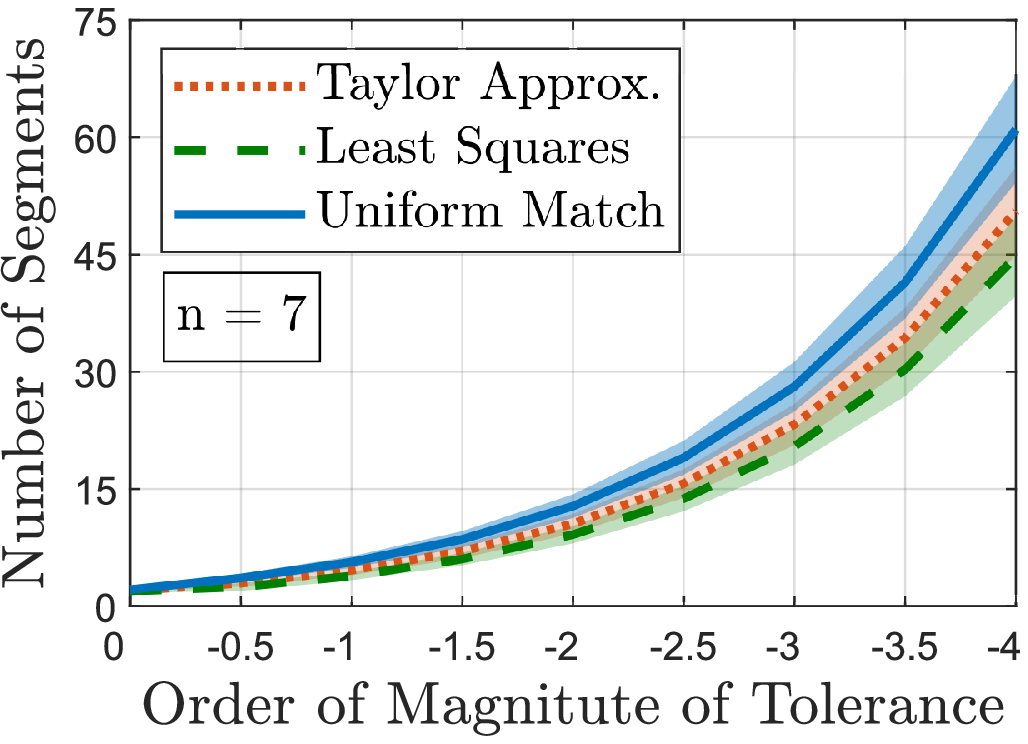} &
\includegraphics[width=0.1625\textwidth]{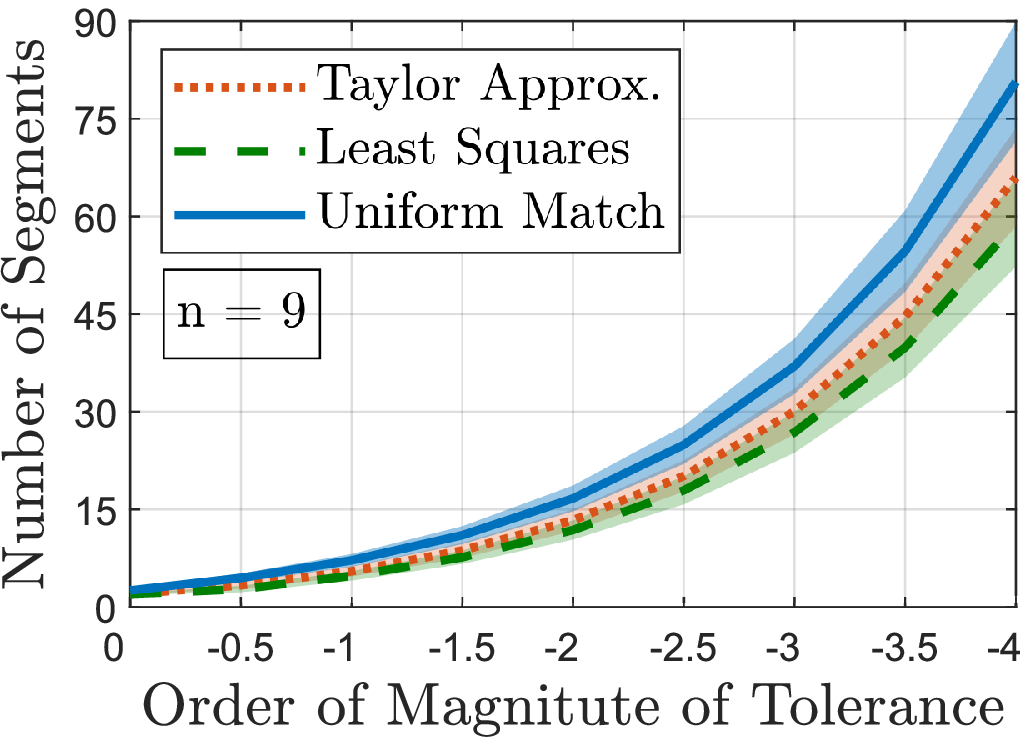} &
\includegraphics[width=0.1625\textwidth]{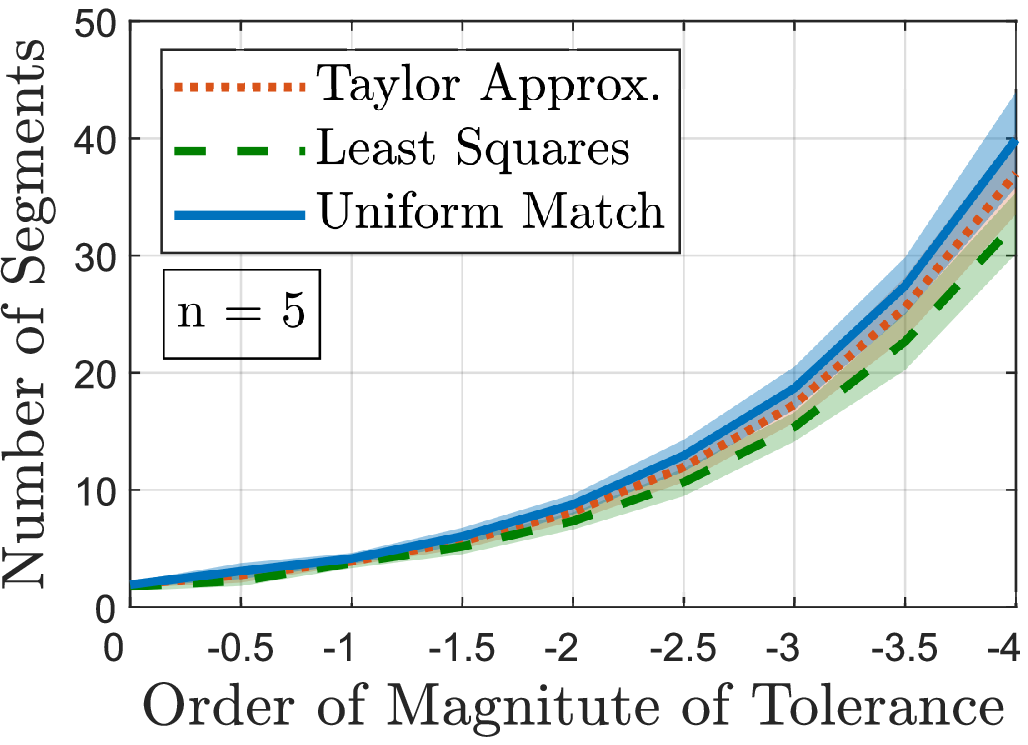} &
\includegraphics[width=0.1625\textwidth]{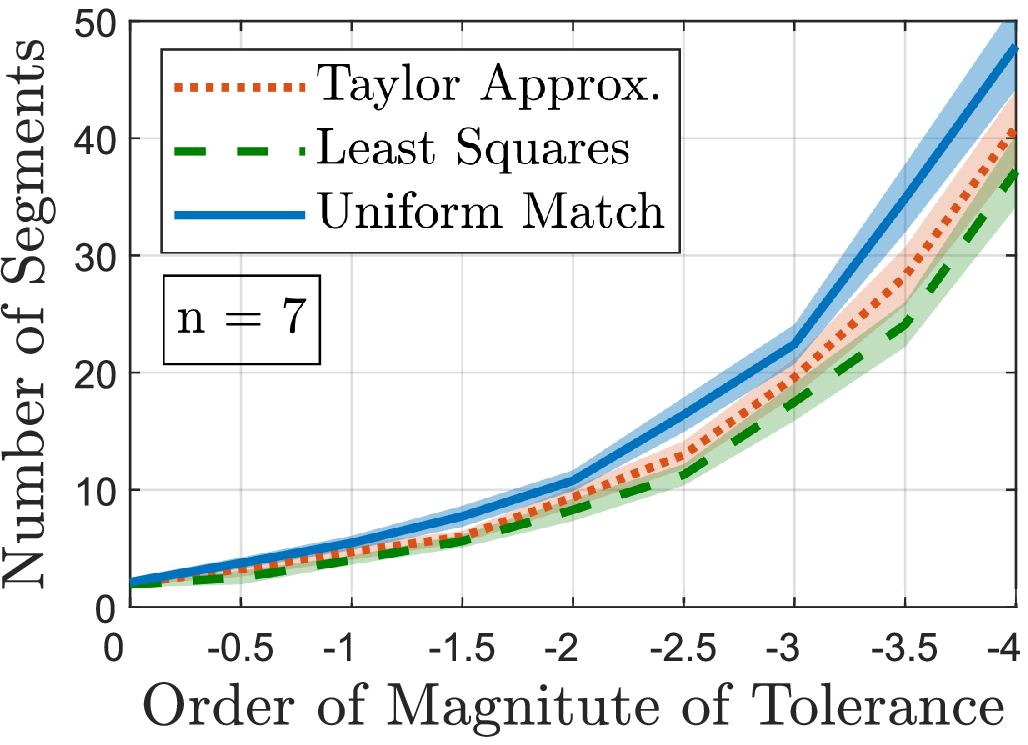} &
\includegraphics[width=0.1625\textwidth]{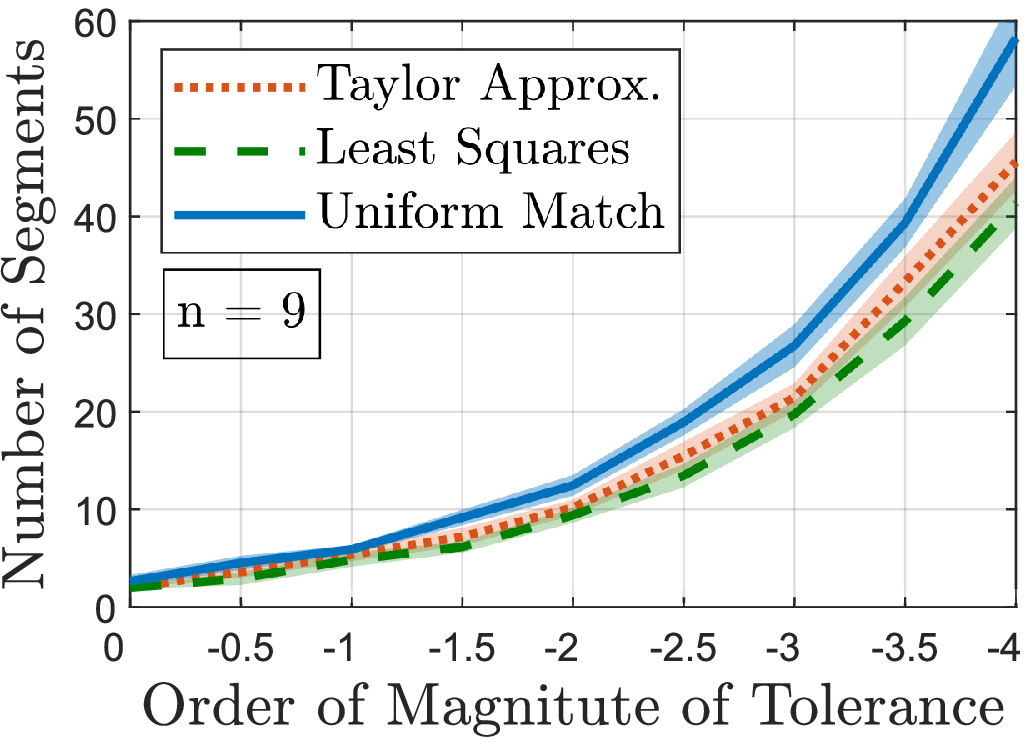} 
\\[-1.5mm]
\scriptsize{(a)}& \scriptsize{(b)}& \scriptsize{(c)}& \scriptsize{(d)}& \scriptsize{(e)}& \scriptsize{(f)} 
\end{tabular}
\vspace{-3.5mm}
\caption{Relation between the number of segments and approximation tolerance in adaptive B\'ezier approximation by (top) linear and (bottom) quadratic segments using (a, b, c) linear  and (d, e, f) binary  search for B\'ezier order (a, d) $n = 5$, (b, e) $n=7$, and (c, f) $n= 9$. Here, the approximation tolerance is measured by the maximum control-point distance.}
\label{fig.NumberOfSegmentsApproximationTolerance}
\end{figure*}

To determine approximation error statistics, we randomly generate B\'ezier control points that are uniformly distributed over the unit box $[0,1] \times [0,1]$.
For the distance-to-point criterion, we select the origin as the point of interest; and for distance-to-line, we select the horizontal side of the unit box (i.e., the line segment joining the origin $(0,0)$ to point $(1,0)$). 
In Fig. \ref{fig.LengthErrorStatistics}-\ref{fig.MaximumCurvatureErrorStatistics}, we provide sample statistics (mean and standard deviation) of normalized approximation errors of curve length, distance-to-point, distance-to-line, and maximum curvature versus the number of segments.
It is visibly clear that the B\'ezier approximation with uniform matching reduction achieves significantly better performance in capturing curve length, distance-to-point and distance-to-line compared to the least squares and Taylor approximations.
Especially, the end-point preservation property of uniform matching reduction plays a key role for its superior performance for the distance-to-point/line criteria presented in Fig. \ref{fig.DistanceToPointErrorStatistics}-\ref{fig.DistanceToLineErrorStatistics}.
We observe in \reffig{fig.LengthErrorStatistics} that B\'ezier approximations with linear segments have comparable accuracy for all three reduction methods, which can be explained by the limited representation power of linear curve segments.
On the other hand, uniform matching reduction shows a superior performance for B\'ezier approximations with quadratic segments.
Finally, as illustrated in \reffig{fig.MaximumCurvatureErrorStatistics}, we see that all B\'ezier degree reduction methods perform equally well for approximating the maximum curvature of B\'ezier curves.
This can be explained by the limited expresiveness of quadratic segments for approximating the first and second derivatives of B\'ezier curves since curvature is a function of the first and second curve derivatives.

\subsection{Number of Segments vs. B\'ezier Degree and Tolerance}

In this part, we numerically study how the number of segments automatically determined in adaptive B\'ezier approximation depends on the order of the B\'ezier curve and the approximation tolerance (specified in terms of the maximum control-point distance).
We consider randomly generated B\'ezier control points over the unit box $[0,1] \times [0,1]$.
To ensure scale invariance, we rescale B\'ezier control points to have a sample variance of unity. 
In \reffig{fig.NumberOfSegmentsBezierOrder} we present  the  average number of segments used in adaptive approximation of B\'ezier curves of different orders.
For a fixed choice of an approximation tolerance, we observe that the number of segments grows linearly with the B\'ezier degree for linear search whereas the grow rate is sublinear for binary search.  
This is strongly aligned with the B\'ezier approximation rule proposed in \refsec{sec.BezierApproximationRule}.
Finally, as illustrated in \reffig{fig.NumberOfSegmentsApproximationTolerance}, the average number of segments used in adaptive B\'ezier approximation grows exponentially with the negated order of magnitude of approximation tolerance $\varepsilon$, because the higher the accuracy the higher the spatial resolution.

\section{Conclusion}
\label{sec.Conclusions}

In this paper, we introduce a novel adaptive B\'ezier approximation method that automatically splits and performs degree reduction on high-order B\'ezier curves to approximately represent them by multiple low-order B\'ezier segments at any given approximation tolerance measured by a B\'ezier metric.
Accordingly, we propose a new maximum control-point distance for efficient and informative comparison of B\'ezier curves.
We show that the maximum control-point distance defines a tight upper bound on standard B\'ezier metrics such as Hausdorff, parameterwise maximum, and Frobenious-norm distance of B\'ezier curves and can be used to geometrically bound B\'ezier curves with respect to each others. 
To better maintain the original curve shape, we also propose a new parameterwise matching reduction method that allows one to preserve a certain set of curve points (e.g., end points) after degree reduction.
The matching reduction shows a superior approximation performance compared to standard least squares and Taylor approximations.
Based on the explicit form of degree-one matching reduction error, we also suggest a rule of thumb for approximating $n^{\text{th}}$-order B\'ezier curves by $3(n-1)$ quadratic and $6(n-1)$ linear B\'ezier segments.
Our extensive numerical studies demonstrates the effectiveness of the proposed methods and the validity of our Bezier approximation rule. 
Work now in progress targets applying these B\'ezier approximation tools in sensor-based reactive motion planning and trajectory optimization of nonholonomically constrained mobile robots and autonomous vehicles \cite{arslan_koditschek_IJRR2019}.

\appendices 

\section{Polynomial Basis Transformation Matrices}
\label{app.ExplicitBasisTransformation}

In this part, we provide the explicit formulas for the elements of polynomial basis transformation matrices.

\begin{lemma}[\cite{farouki_CAGD2012}] \label{lem.MonomialBernsteinTransformation}
\emph{(Monomial \& Bernstein Basis Transformation)} 
The transformation matrices between monomial and Bernstein basis vectors, i.e.,
\begin{subequations} \label{eq.MonomialBernsteinTransformation}
\begin{align}
\mbasis_{\cdegree}(t) &= \tfbasis_{\bbasis}^{\mbasis} (\cdegree) \bbasis_{\cdegree}(t), 
\\
\bbasis_{\cdegree}(t) &= \tfbasis_{\mbasis}^{\bbasis} (\cdegree) \mbasis_{\cdegree}(t),
\end{align}
\end{subequations}
are explicitly given by\footnote{The transformation matrices between monomial and Bernstein bases are upper triangular with positive diagonal elements and so are invertible.}
\begin{subequations}\label{eq.MonomialBernsteinTranformationMatrix}
\begin{align}
\blist{\tfbasis_{\mbasis}^{\bbasis}(\cdegree)} _{i+1,j+1} &= 
\left \{
\begin{array}{r@{\,\,}l}
(-1)^{(j-i)} \binom{\cdegree}{j}\binom{j}{i}, & \text{ if } i \leq j
\\
0, & \text{ otherwise,} 
\end{array}
 \right.
\\
\blist{\tfbasis_{\bbasis}^{\mbasis}(\cdegree)} _{i+1,j+1} &= \left \{
\begin{array}{r@{\,\,}l}
\frac{\binom{j}{i}}{\binom{\cdegree}{i}},  & \text{ if } i \leq j
\\
0, & \text{ otherwise,} 
\end{array}
 \right.
\end{align}
\end{subequations}
where $i,j \in \blist{0, 1, \ldots, \cdegree}$, and they are the inverse of each other
\begin{align}
\tfbasis_{\bbasis}^{\mbasis}(\cdegree)^{-1} = \tfbasis_{\mbasis}^{\bbasis}(\cdegree).
\end{align}
\end{lemma}

\begin{lemma}\label{lem.MonomialTaylorTransformation}
\emph{(Monomial-Taylor Basis Transformation)}
The transformation between monomial and Taylor bases, i.e.,   
\begin{subequations} \label{eq.MonomialTaylorTransformation}
\begin{align}
\mbasis_{\cdegree}(t) &= \tfbasis_{\tbasis}^{\mbasis}(\cdegree, \toffset) \tbasis_{\cdegree, \toffset}(t),
\\
\tbasis_{\cdegree, \toffset}(t) &= \tfbasis_{\mbasis}^{\tbasis}(\cdegree, \toffset) \mbasis_{\cdegree}(t),
\end{align}
\end{subequations}
are explicitly given by\footnote{The transformation matrices between monomial and Taylor bases are lower triangular with all ones in the main diagonal and so are invertible.}
\begin{subequations} \label{eq.MonomialTaylorTranformationMatrix}
\begin{align}
\blist{\tfbasis_{\mbasis}^{\tbasis}(\cdegree, \toffset)}_{i+1,j+1} &= \left \{ \begin{array}{@{}c@{\,\,}l}
\binom{i}{j}(-\toffset)^{i-j}  & \text{, if } i\geq j \\
0 & \text{, otherwise,}
\end{array}
\right . 
\\
\blist{\tfbasis_{\tbasis}^{\mbasis}(\cdegree, \toffset)}_{i+1,j+1} &= \left \{ \begin{array}{@{}c@{}l}
\binom{i}{j}(\toffset)^{i-j}  & \text{, if } i\geq j \\
0 & \text{, otherwise,}
\end{array}
\right . 
\end{align}
\end{subequations}
where $i,j \in \{0, \ldots, \cdegree\}$, and they are the inverse of each other, 
\begin{align}
\tfbasis_{\tbasis}^{\mbasis}(\cdegree, \toffset)^{-1} =  \tfbasis_{\mbasis}^{\tbasis}(\cdegree, \toffset). 
\end{align}
\end{lemma}
\begin{proof}
See \refapp{app.MonomialTaylorTransformation}.
\end{proof}

Accordingly, the transformation matrices between the Bernstein and Taylor bases, i.e.,
\begin{subequations}
\begin{align}
\tbasis_{\cdegree, \toffset}(t) &= \tfbasis_{\bbasis}^{\tbasis}(\cdegree, \toffset) \bbasis_{\cdegree}(t), 
\\
\bbasis_{\cdegree}(t) & = \tfbasis_{\tbasis}^{\bbasis}(\cdegree, \toffset) \tbasis_{\cdegree, \toffset}(t),
\end{align}
\end{subequations}
can be determined using the monomial basis as
\begin{subequations}\label{eq.BernsteinTaylorTranformation}
\begin{align}
\tfbasis_{\bbasis}^{\tbasis}(\cdegree, \toffset) = \tfbasis_{\bbasis}^{\mbasis}(\cdegree) \tfbasis_{\mbasis}^{\tbasis}(\cdegree, \toffset),
\\
\tfbasis_{\tbasis}^{\bbasis}(\cdegree, \toffset) = \tfbasis_{\tbasis}^{\mbasis}(\cdegree, \toffset) \tfbasis_{\mbasis}^{\bbasis}(\cdegree),
\end{align}
\end{subequations}
where  $\tfbasis_{\bbasis}^{\tbasis}(\cdegree, \toffset) ^{-1} = \tfbasis_{\tbasis}^{\bbasis}(\cdegree, \toffset)$.

\section{On Reparametrization of Polynomial Curves}
\label{app.CurveReparametrication}

In this part, we show how affine reparametrization of polynomial curves can be performed explicitly via Taylor~basis.

\begin{lemma}\label{lem.BasisReparametrization}
The Bernstein, monomial and Taylor basis vectors of degree $\cdegree$ (associated with a Taylor offset $\toffset \in \R$) can be affinely reparametrized from interval $[a,b]$ to $[c,d]$ (with $a < b$ and $c < d$) as
{\small
\begin{subequations}
\begin{align}
\!\bbasis_{\cdegree}(t) &=  \tfbasis_{\tbasis}^{\bbasis}(\cdegree, \widehat{\toffset}) \diag \plist{\!\mbasis_{\cdegree}(\tfrac{b-a}{d-c})\!\!}  \tfbasis_{\bbasis}^{\tbasis}(\cdegree, \toffset) \bbasis_{\cdegree} \!\plist{\!\tfrac{b-a}{d-c} t + \tfrac{a d - b c}{d - c}\!\!},   \!\!\!
\\
\!\mbasis_{\cdegree}(t) &=  \tfbasis_{\tbasis}^{\mbasis}(\cdegree, \widehat{\toffset}) \diag \plist{\!\mbasis_{\cdegree}(\tfrac{b-a}{d-c})\!\!}  \tfbasis_{\mbasis}^{\tbasis}(\cdegree, \toffset) \mbasis_{\cdegree}\! \plist{\!\tfrac{b-a}{d-c} t + \tfrac{a d - b c}{d - c}\!\!}\!,\!\!    
\\
\!\tbasis_{\cdegree, \widehat{\toffset}}(t) & =   \diag \plist{\!\mbasis_{\cdegree}(\tfrac{b-a}{d-c})\!\!}   \tbasis_{\cdegree, \toffset} \plist{\tfrac{b-a}{d-c} t + \tfrac{a d - b c}{d - c}\!},   
\end{align}
\end{subequations}
}%
where $\diag$ denotes the diagonal matrix with diagonal entries specified with its argument,  and the reparametrized Taylor offset is given by the associated affine transformation as   
\begin{align}
\widehat{\toffset} =  \tfrac{d -c}{b-a} t_a - \tfrac{ a d - b c}{ b- a}.
\end{align}
\end{lemma}
\begin{proof}
See \refapp{app.BasisReparametrization}.
\end{proof}

\begin{lemma}\label{lem.CurveReparametrization}
\emph{(Polynomial Curve Reparametrization)} B\'ezier, monomial and Taylor curves  of degree $\cdegree \in \N$ with respective control point matrices  $\bpmat_{\cdegree}=\blist{\bpoint_0, \ldots, \bpoint_{\cdegree}}$, $\mpmat_{\cdegree} = \blist{\mpoint_0, \ldots, \mpoint_\cdegree}$, and $\tpmat_{\cdegree}=\blist{\tpoint_0, \ldots, \tpoint_\cdegree}$ (and a Taylor offset $\toffset \in \R$) can be affinely reparametrized from interval $[a,b]$ to $[c,d]$ (with $a < b$ and $c < d$) as
\begin{subequations}
\begin{align}
\bcurve_{\widehat{\bpoint}_0, \ldots, \widehat{\bpoint}_\cdegree}(t) &= \bcurve_{\bpoint_0, \ldots, \bpoint_\cdegree} \plist{\tfrac{b-a}{d-c} t + \tfrac{a d - b c}{d - c}},
\\
\mcurve_{\widehat{\mpoint}_0, \ldots, \widehat{\mpoint}_\cdegree}(t) &= \mcurve_{\mpoint_0, \ldots, \mpoint_\cdegree} \plist{\tfrac{b-a}{d-c} t + \tfrac{a d - b c}{d - c}},
\\
\tcurve_{\widehat{\tpoint}_0, \ldots, \widehat{\tpoint}_\cdegree}(t, \widehat{\toffset}) &= \tcurve_{\tpoint_0, \ldots, \tpoint_\cdegree} \plist{\tfrac{b-a}{d-c} t + \tfrac{a d - b c}{d - c}, \toffset},
\end{align}
\end{subequations}
with the corresponding reparametrized control point matrices $\widehat{\bpmat}_{\cdegree} = \blist{\widehat{\bpoint}_0, \ldots, \widehat{\bpoint}_\cdegree}$, $\widehat{\mpmat}_{\cdegree} = \blist{\widehat{\mpoint}_0, \ldots, \widehat{\mpoint}_\cdegree}$, and $\widehat{\tpmat}_{\cdegree} = \blist{\widehat{\tpoint}_0, \ldots, \widehat{\tpoint}_\cdegree}$ that are given by 
\begin{subequations}
\begin{align}
\widehat{\bpmat}_{\cdegree} &= \bpmat_{\cdegree} \tfbasis_{\tbasis}^{\bbasis}(\cdegree, \toffset\!) \diag\plist{\!\mbasis_{\cdegree}(\tfrac{d-c}{b-a})\!} \tfbasis_{\bbasis}^{\tbasis}(\cdegree, \widehat{\toffset}),
\\
\widehat{\mpmat}_{\cdegree} &= \mpmat_{\cdegree} \tfbasis_{\tbasis}^{\mbasis}(\cdegree, \toffset\!) \diag\plist{\!\mbasis_{\cdegree}(\tfrac{d-c}{b-a})\!} \tfbasis_{\mbasis}^{\tbasis}(\cdegree, \widehat{\toffset}),
\\
\widehat{\tpmat}_{\cdegree} &= \tpmat_{\cdegree} \diag\plist{\mbasis_{\cdegree}(\tfrac{d-c}{b-a})}, 
\end{align}
\end{subequations}
where  $\widehat{\toffset} = \tfrac{d -c}{b-a} \toffset - \tfrac{ a d - b c}{ b- a}$.
\end{lemma}
\begin{proof}
See \refapp{app.CurveReparametrization}.
\end{proof}

\section{Matching Reduction in Monomial Basis}

The matching reduction matrix can be explicitly computed using the monomial basis.

\begin{lemma}\label{lem.MatchingReductionMatrixMonomial}
(\emph{Matching Reduction in Monomial Basis})
For any $n \geq m \in \N$ and distinct $t_0, \ldots, t_m \in \R$, the parameterwise matching reduction matrix $\rmat_{t_0, \ldots, t_m}(n,m)$ can be computed using monomial basis as 
\begin{align}
\rmat_{t_0, \ldots, t_{m}}(n,m) &  = \tfbasis_{\mbasis}^{\bbasis}(n) \mbmat_{n}(t_0, \ldots, t_m) \mbmat_{m}(t_0, \ldots, t_m)^{-1} \tfbasis_{\bbasis}^{\mbasis}(m), \nonumber
\\
&  =
 \tfbasis_{\mbasis}^{\bbasis}(n) 
\blist{
\begin{array}{@{}c@{}}
\mat{I}_{(m+1) \times (m+1)} \\
\alpha_1(t_0, \ldots, t_m) \\
\vdots \\
\alpha_{n-m}(t_0, \ldots, t_m)
\end{array}
}
\tfbasis_{\bbasis}^{\mbasis}(m),
\end{align}
with row vectors $\alpha_{i}(t_0, \ldots, t_m) = [\alpha_{i,0}, \ldots, \alpha_{i,m}]$ that are recursively defined as
\begin{align}\label{eq.MatchingReductionRecursion}
\alpha_{i+1, k} &= \alpha_{i, k-1} + \alpha_{i,m} \alpha_{1,k},
\end{align}
where base conditions of  $\alpha_{i,-1} = 0$ and $\alpha_{1,k}$ satisfying
\begin{align}\label{eq.MatchingReductionBase}
t^{m+1} - \sum_{k=0}^{m} \alpha_{1,k} t^k =  \prod_{k = 0}^{m} (t - t_k).
\end{align}
\end{lemma}
\begin{proof}
See \refapp{app.MatchingReductionMatrixMonomial}.
\end{proof}

\section{Analytic Properties of  Quadratic B\'ezier Curves}
\label{app.LinearQuadraticBezierCurves}

%
%
%
%

Motion planning with B\'ezier curves (around obstacles) requires determining critical geometric curve properties such as arc length, maximum absolute curvature, distance to a point or a line segment, and intersection with a halfspace.
Fortunately, the low-degree of quadratic B\'ezier curves allows for simple analytic expressions of these curve properties enabling computationally efficient constrained motion planning, because quadratic bezier curves and derivatives have simple forms, 
\begin{subequations} \label{eq.QuadraticBezier}
\begin{align}
\bcurve_{\bpoint_0, \bpoint_1, \bpoint_2}(t) &= (\bpoint_2 - 2\bpoint_1 + \bpoint_0) t^2 + 2(\bpoint_1 - \bpoint_0) t + \bpoint_0, \label{eq.QuadraticBezierCurve}
\\
\bcurve_{\bpoint_0, \bpoint_1,\bpoint_2}'(t)   &= 2(\bpoint_1 - \bpoint_0) + 2t (\bpoint_2 - 2\bpoint_1 + \bpoint_0), \label{eq.QuadraticBezierFirstDerivative}
\\
\bcurve_{\bpoint_0, \bpoint_1,\bpoint_2}''(t) &= 2(\bpoint_2 - 2 \bpoint_1 + \bpoint_0). \label{eq.QuadraticBezierSecondDerivative}
\end{align}
\end{subequations}

\begin{proposition}[\cite{ahn_et_al_JCAM2014}] \label{prop.QuadBezierArcLength} 
The arc length of a quadratic B\'ezier curve $\bcurve_{\bpoint_0,\bpoint_1,\bpoint_2}(t)$ over an interval $[t_1, t_2]$ is given by 
\begin{align}
\clength(\bcurve_{\bpoint_0,\bpoint_1,\bpoint_2}([t_1, t_2])) =  2(I(t_2) - I(t_1))
\end{align}
where $I(t) = \int R(t) \mathrm{d} t$ is the integral of $R(t) = \sqrt{a t^2 + b t + c}$ that is given for $a > 0$ as  \cite{gradshteyn_ryzhik_IntegralsSeriesProducs2014}
\begin{align}
I(t) 
&  =   \frac{2 a t + b}{4a} R(t) + \frac{4 a c - b^2}{8a\sqrt{a}} \ln |2 \sqrt{a} R(t) + 2 a t + b|
\end{align}
with
\begin{subequations} \label{eq.QuadBezierLengthCoefs}
\begin{align}
a &= \norm{\bpoint_2 - 2\bpoint_1 + \bpoint_0}^2 \\
b & = 2 \tr{(\bpoint_1 - \bpoint_0)} (\bpoint_2 - 2 \bpoint_1 + \bpoint_0) \\
c & = \norm{\bpoint_1 - \bpoint_0}^2
\end{align}
\end{subequations}
Otherwise (i.e., $a=0$), one has $I(t) = \sqrt{c} t = \norm{\bpoint_1 - \bpoint_0} t$.

\end{proposition}
\begin{proof}
By definition, the arc length of a curve is given by the integral of the norm of its rate of change, i.e., 
\begin{align}
L(\bcurve_{\bpoint_0,\bpoint_1,\bpoint_2}([t_1, t_2])) &= \int_{t_1}^{t_2} \norm{\bcurve_{\bpoint_0,\bpoint_1,\bpoint_2} '(t)} \mathrm{d}t 
\\
& \hspace{-8mm}=   \int_{t_1}^{t_2} 2 \norm{(\bpoint_1 - \bpoint_0) + t (\bpoint_2 - 2 \bpoint_1 + \bpoint_0)}  \mathrm{d}t 
\!\!\\
& \hspace{-8mm} = 2 \int_{t_1}^{t_2} \sqrt{a t^2 + b t + c} \mathrm{d}t  
\end{align}
where $a,b, c$ are defined as in \refeq{eq.QuadBezierLengthCoefs}. 
Also note that $a = 0$ implies $b = 0$, and so $\int \sqrt{a t^2 + b t + c} \mathrm{d} t = \sqrt{c} t$.
Hence, the result follows.
\end{proof}

\begin{proposition}[\cite{sapidis_frey_CAGD1992, deddi_everett_lazard_TechReport2000}]
\label{prop.kappamax}
The maximum absolute curvature of a planar quadratic B\'ezier curve $\bcurve_{\bpoint_0,\bpoint_1,\bpoint_2}(t)$, associated with $\bpoint_0, \bpoint_1, \bpoint_2 \in \R^{2}$, over an interval $[t_1,t_2]$ satisfies
\begin{align} \label{eq.QuadBezierMaxAbsCurv}
\max_{t \in [t_1,t_2]} | \kappa(t) | = 
\left \{
\begin{array}{rl}
| \kappa(t_1) | & \text{, if } t_\kappa^* < t_1 \\
| \kappa(t_\kappa^*) | & \text{, if } t_1 \leq t_\kappa^* \leq  t_2 \\
| \kappa(t_2) | & \text{, if } t_\kappa^* > t_2
\end{array} 
\right.
\end{align}
where the quadratic B\'ezier curvature $\kappa(t)$ is given by
\begin{align}
\kappa(t) = \frac{\det([\bpoint_1 - \bpoint_0, \bpoint_2 - \bpoint_1])}{2\norm{(\bpoint_1-\bpoint_0)(1-t) +(\bpoint_2 - \bpoint_1)t}^3},    
\end{align}
and the optimal curve parameter $t_{\kappa}^*$ and the maximum absolute curvature $|\kappa (t_\kappa^*)|$ over $t \in \R$ are 
\begin{align}
t^*_\kappa &= \argmax\limits_{t \, \in \, \R} | \kappa(t) | = \frac{(\bpoint_0 - \bpoint_1)^{T} (\bpoint_2 - 2\bpoint_1 + \bpoint_0)}{\norm{\bpoint_2 - 2\bpoint_1 + \bpoint_0}^2}, \label{eq.tQuadBezierMaxAbsCurv}
\\
|\kappa (t_\kappa^*)| &= \cfrac{\norm{\bpoint_2 - 2\bpoint_1 + \bpoint_0}^3}{2 \det([\bpoint_1 - \bpoint_0, \bpoint_2 - \bpoint_1 ])^2}. 
\end{align}
\end{proposition}
\begin{proof}
The quadratic B\'ezier curvature can be determined as
\begin{align}
\!\!\kappa(t) &= \frac{\det\plist{\blist{\bcurve_{\bpoint_0, \bpoint_1,\bpoint_2 }'(t), \bcurve_{\bpoint_0, \bpoint_1,\bpoint_2 }''(t)}}}{\norm{\bcurve_{\bpoint_0, \bpoint_1,\bpoint_2 }'(t)}^3} 
\\
& \hspace{-2mm}= \frac{\det{\blist{(\bpoint_1 - \bpoint_0) + t(\bpoint_2 - 2\bpoint_1 + \bpoint_0), (\bpoint_2 - 2\bpoint_1 + \bpoint_0)}}}{2\norm{(\bpoint_1 - \bpoint_0) + t(\bpoint_2 - 2\bpoint_1 + \bpoint_0)}^3} \!\! 
\\
& \hspace{-2mm}= \frac{\det{\blist{(\bpoint_1 - \bpoint_0), (\bpoint_2 - 2\bpoint_1 + \bpoint_0)}}}{2\norm{(\bpoint_1 - \bpoint_0) + t(\bpoint_2 - 2\bpoint_1 + \bpoint_0)}^3} \nonumber \\
& \hspace{10mm}+ t \underbrace{\frac{\det{\blist{(\bpoint_2 - 2\bpoint_1 + \bpoint_0), (\bpoint_2 - 2\bpoint_1 + \bpoint_0)}}}{2\norm{(\bpoint_1 - \bpoint_0) + t(\bpoint_2 - 2\bpoint_1 + \bpoint_0)}^3}}_{=0}\\
& = \frac{\det{\blist{(\bpoint_1 - \bpoint_0), (\bpoint_2 - 2\bpoint_1 + \bpoint_0)}}}{2\norm{(\bpoint_1 - \bpoint_0) + t(\bpoint_2 - 2\bpoint_1 + \bpoint_0)}^3} 
\end{align}
Hence, the maximum absolute curvature $|\kappa(t)|$ is achieved when $\norm{(\bpoint_1 - \bpoint_0) + t(\bpoint_2 - 2\bpoint_1 + \bpoint_0)}^2$ is minimized, which is a convex function of $t$ and its unique minimum can be determined by setting its derivative to zero as
\begin{align}
0 & = \frac{\mathrm{d}}{\mathrm{d}t} \norm{(\bpoint_1 - \bpoint_0) + t(\bpoint_2 - 2\bpoint_1 + \bpoint_0)}^2  \\
&= 2 \tr{(\bpoint_2 - 2\bpoint_1 + \bpoint_0)}\plist{(\bpoint_1 - \bpoint_0) + t(\bpoint_2 - 2\bpoint_1 + \bpoint_0)}
\end{align} 
which corresponds to \refeq{eq.tQuadBezierMaxAbsCurv}.
Hence, due to the convexity of quadratic $\norm{(\bpoint_1 - \bpoint_0) + t(\bpoint_2 - 2\bpoint_1 + \bpoint_0)}^2$, the maximum absolute curvature is realized at $t_{\kappa}^{*}$ if $t_{\kappa}^{*} \in [t_1, t_2]$; otherwise, the maximum value is at the closest boundary  of the interval $[t_1, t_2]$ to $t_{\kappa}^{*}$ as described in \refeq{eq.QuadBezierMaxAbsCurv}.  

Finally, using the fact that $\det([\vect{a}, \vect{b}])^2 = \norm{\vect{a}}^2 \norm{\vect{b}}^2 - \plist{\tr{\vect{a}} \vect{b}}^2$ for any $\vect{a}, \vect{b} \in \R^{2}$, one can verify that
\begin{align}
|\kappa(t^*_{\kappa})| &= \tfrac{|\det\plist{\blist{\bpoint_1 - \bpoint_0, \bpoint_2 - \bpoint_1}}| \norm{\bpoint_2 - 2\bpoint_1 + \bpoint_0}^3}{2 \norm{(\bpoint_1 - \bpoint_0)\norm{\bpoint_2 - 2\bpoint_1 + \bpoint_0} - \frac{\tr{(\bpoint_0 - \bpoint_1)}(\bpoint_2 - 2\bpoint_1 + \bpoint_0)(\bpoint_2 - 2\bpoint_1 + \bpoint_0)}{\norm{\bpoint_2 - 2\bpoint_1 + \bpoint_0}}}^3} 
\\
& = \tfrac{|\det\plist{\blist{\bpoint_1 - \bpoint_0, \bpoint_2 - \bpoint_1}}| \norm{\bpoint_2 - 2\bpoint_1 + \bpoint_0}^3}{2 \plist{\norm{\bpoint_1 - \bpoint_0}^2 \norm{\bpoint_2 - 2\bpoint_1 + \bpoint_0}^2 - \plist{\tr{\bpoint_1 - \bpoint_0} (\bpoint_2 - 2 \bpoint_1 + \bpoint_0)}^2}^{\frac{3}{2}}} 
\\
& = \tfrac{|\det\plist{\blist{\bpoint_1 - \bpoint_0, \bpoint_2 - \bpoint_1}}| \norm{\bpoint_2 - 2\bpoint_1 + \bpoint_0}^3}{2|\det\plist{\blist{\bpoint_1 - \bpoint_0, \bpoint_2 - \bpoint_1}}|^3} 
\\
&= \tfrac{\norm{\bpoint_2 - 2\bpoint_1 + \bpoint_0}^3}{2\det\plist{\blist{\bpoint_1 - \bpoint_0, \bpoint_2 - \bpoint_1}}^2}
\end{align}
which completes the proof.
\end{proof}

\begin{proposition}
\label{prop.QuadBezierDist}
The distance of a quadratic Bézier curve $\bcurve_{\bpoint_0, \bpoint_1, \bpoint_2}(t)$ to the origin  over an interval $[t_1, t_2]$ satisfies 
\begin{align}
\min_{t \in [t_1,t_2]} \norm{\bcurve_{\bpoint_0, \bpoint_1, \bpoint_2}(t)} = \min_{t \in T_{\bpoint_0, \bpoint_1, \bpoint_2}} \norm{\bcurve_{\bpoint_0, \bpoint_1, \bpoint_2}(t)},
\label{eq.QuadBezierDist}
\end{align}
where the finite set of time instances is given by\footnote{The roots of a cubic equation can be determined analytically \cite{nickalls_richard_TMG_1993}.}
\begin{align}
T_{\bpoint_0, \bpoint_1, \bpoint_2}([t_1, t_2]) = \clist{ t \in [t_1,t_2] \,\Big|\, \sum\nolimits_{i=0}^3 c_i t^i \!=\!0 } \cup \clist{t_1,t_2}, \!\!\!
\label{eq.tQuadBezierDist}
\end{align}
with
\begin{subequations} \label{eq.QuadBezierDistCoefs}
\begin{align}
    c_0 &= \tr{(\bpoint_1 - \bpoint_0)} \bpoint_0, \label{eq.c0coefficient} \\
    c_1 &= \tr{(\bpoint_2 - 2\bpoint_1 + \bpoint_0)}\bpoint_0  + 2\norm{\bpoint_1 - \bpoint_0}^2, \label{eq.c1coefficient} \\
    c_2 &= 3\tr{(\bpoint_2 - 2\bpoint_1 + \bpoint_0)} (\bpoint_1 - \bpoint_0), \label{eq.c2coefficient} \\
    c_3 &= \norm{\bpoint_2 - 2\bpoint_1 + \bpoint_0}^2. \label{eq.c3coefficient}
\end{align}
\end{subequations}
\end{proposition}
\begin{proof}
One can verify using \refeq{eq.QuadraticBezier} that the critical points of $\norm{\bcurve_{\bpoint_0, \bpoint_1, \bpoint_2}(t)}$ over $\R$ satisfy the following cubic equation
\begin{align}
0 & = \tr{\bcurve_{\bpoint_0, \bpoint_1, \bpoint_2}(t)} \bcurve_{\bpoint_0, \bpoint_1, \bpoint_2}'(t) 
= 2(c_3 t^3 + c_2 t^2 + c_1 t + c_0 )\!\!
\end{align}
where the coefficients $c_0, c_1, c_2, c_3$ are defined as in \refeq{eq.QuadBezierDistCoefs}.
Hence, the quadratic bezier distance to the origin is realized in one of the critical points in $[t_1, t_2]$ or on the boundary.
\end{proof}

\begin{remark}
The distance of a  bezier curve $\bcurve_{\bpoint_0, \ldots, \bpoint_{\cdegree}}(t)$  to a point $\mpoint \in \R^{\cdim}$  or another  bezier curve $\bcurve_{\mpoint_0, \ldots, \mpoint_{\cdegree}}(t)$ (with parameter-wise correspondence) can be formulated as its distance to the origin because
\begin{align}
\bcurve_{\bpoint_0, \ldots, \bpoint_\cdegree}(t) - \mpoint &=  \bcurve_{\bpoint_0 - \mpoint, \ldots, \bpoint_{\cdegree} - \mpoint} (t) \\
\bcurve_{\bpoint_0, \ldots, \bpoint_\cdegree}(t) - \bcurve_{\mpoint_0, \ldots, \mpoint_\cdegree}(t) &= \bcurve_{\bpoint_0 - \mpoint_0, \ldots, \bpoint_\cdegree - \mpoint_{\cdegree}} (t)
\end{align} 
\end{remark}

\begin{proposition} \label{prop.QuadBezierDist2Line}
The distance of a quadratic B\'ezier curve $\bcurve_{\bpoint_0, \bpoint_1, \bpoint_2} (t)$ defined over the internal $[t_1, t_2]$ to a linear B\'ezier curve $\bcurve_{\mpoint_0, \mpoint_1} (k)$ defined over the interval [$k_1, k_2]$ can be analytically determined as 
\begin{align}
& \min_{t \in [t_1, t_2]} \min_{k \in [k_1, k_2]} \norm{\bcurve_{\bpoint_0, \bpoint_1, \bpoint_2} (t) - \bcurve_{\mpoint_0, \mpoint_1}(k)} \nonumber 
\\ & \hspace{25mm}=  \min_{t \in \widehat{T}} \norm{\bcurve_{\bpoint_0, \bpoint_1, \bpoint_2} (t) - \bcurve_{\mpoint_0, \mpoint_1}(k^*(t))}
\end{align}
using a finite set of critical time instances $\widehat{T}$ defined in terms of $T$ in \refeq{eq.tQuadBezierDist} as
\begin{align}
\widehat{T} &= T_{\bpoint_0-\mpoint_0, \bpoint_1 - \mpoint_0, \bpoint_2 - \mpoint_0 } ([t_1, t_2]) \nonumber \\
& \hspace{20mm}\cup T_{\bpoint_0-\mpoint_1, \bpoint_1 - \mpoint_1, \bpoint_2 - \mpoint_1} ([t_1, t_2]) \nonumber \\
& \hspace{40mm}  \cup T_{\hat{\bpoint}_0, \hat{\bpoint}_1, \hat{\bpoint}_2}([t_1, t_2])
\end{align}
and the optimal line parameter $k^*(t)$ that is given by 
\begin{align}
k^*(t) &=  \min\plist{\max\plist{k_1, \tfrac{\tr{\plist{\mpoint_1 - \mpoint_0}}}{\norm{\mpoint_1 - \mpoint_0}^2}\plist{\bcurve_{\bpoint_0, \bpoint_1, \bpoint_2}(t) - \mpoint_0}}, k_2}. \!\! \label{eq.QuadBezierDist2LineOptimal}
\end{align}
where 
\begin{align}
\hat{\bpoint}_i = \plist{\mat{I} - \frac{(\mpoint_1 - \mpoint_0)\tr{(\mpoint_1 - \mpoint_0)}}{\norm{\mpoint_1 - \mpoint_0}^2}}(\bpoint_i - \mpoint_0).
\end{align}
\end{proposition}
\begin{proof}
Due to convexity, for any $t \in \R$, the line parameter
\begin{align}
k(t) = \tfrac{\tr{\plist{\mpoint_1 - \mpoint_0}}}{\norm{\mpoint_1 - \mpoint_0}^2}\plist{\bcurve_{\bpoint_0, \bpoint_1, \bpoint_2}(t) - \mpoint_0}
\end{align}
minimizes $\norm{\bcurve_{\bpoint_0, \bpoint_1, \bpoint_2}(t) - \bcurve_{\mpoint_0, \mpoint_1}(k)}$  over $\R$, i.e., 
\begin{align}
0 =  \tr{(\mpoint_1 - \mpoint_0)}\plist{\bcurve_{\bpoint_0, \bpoint_1, \bpoint_2}(t) - \mpoint_0 - k(t) (\mpoint_1 - \mpoint_0)}.
\end{align} 
Hence, the optimal solution $k^*(t)$ over the interval $[k_1, k_2]$ is given by \refeq{eq.QuadBezierDist2LineOptimal} since the optimal solution of a quadratic optimization problem is realized at $k(t)$ if $k(t) \in [k_1, k_2]$; otherwise, the optimum is at the closest interval boundary.

Similarly, the optimal bezier parameter is either related with a boundary point $\mpoint_0$ and $\mpoint_1$ of the line segment (corresponding to $T_{\bpoint_0-\mpoint_0, \bpoint_1 - \mpoint_0, \bpoint_2 - \mpoint_0 } ([t_1, t_2])$ and $T_{\bpoint_0-\mpoint_0, \bpoint_1 - \mpoint_0, \bpoint_2 - \mpoint_0 } ([t_1, t_2])$) or the minimum of $\norm{\bcurve_{\bpoint_0, \bpoint_1, \bpoint_2}(t) - \bcurve_{\mpoint_0, \mpoint_1}(k(t))}$ where 
{\small
\begin{align}
\bcurve_{\mpoint_0, \mpoint_1}(k(t)) 
&= \mpoint_0 +  \tfrac{(\mpoint_1 - \mpoint_0)\tr{(\mpoint_1 - \mpoint_0)}}{\norm{\mpoint_1 - \mpoint_0}^2}(\bcurve_{\bpoint_0, \bpoint_1, \bpoint_2}(t) - \mpoint_0)
\\
\bcurve_{\bpoint_0, \bpoint_1, \bpoint_2}(t) - \bcurve_{\mpoint_0, \mpoint_1}(k(t)) 
& = \plist{\mat{I} - \tfrac{(\mpoint_1 - \mpoint_0)\tr{(\mpoint_1 - \mpoint_0)}}{\norm{\mpoint_1 - \mpoint_0}^2}}(\bcurve_{\bpoint_0, \bpoint_1, \bpoint_2}(t) - \mpoint_0)
\\
& =  \bcurve_{\hat{\bpoint}_0, \hat{\bpoint}_1, \hat{\bpoint}_2} (t)
\end{align}
}
which completes the proof.
\end{proof}

\begin{proposition}
The intersection of quadratic B\'ezier curve $\bcurve_{\bpoint_0, \bpoint_1, \bpoint_2}(t)$ defined over an interval $[t_a, t_b]$ with a halfspace $H_{\vect{a},\vect{b}} = \clist{\vect{x} \in \R^{\cdim} | \tr{\vect{a}}(\vect{x} - \vect{b}) \leq 0}$ satisfies
\begin{align}
&\clist{ t \in [t_a, t_b] \big| \tr{\vect{a}} (\bcurve_{\bpoint_0, \bpoint_1, \bpoint_2}(t) - \vect{b}) \geq 0} \nonumber 
\\ 
&\hspace{10mm} = \bigcup \clist{[t_i, t_{i+1}] | 1 \leq i < |T|, \bcurve_{\bpoint_0, \bpoint_1, \bpoint_2}(\frac{t_i + t_{i+1}}{2}) \in H_{\vect{a}, \vect{b}}}
\end{align}
where $(t_1, t_2, \ldots, t_{|T|} )$ is the ascendingly ordered tuple of     
\begin{align}
T = \clist{ t \in [t_a, t_b] \big| \tr{\vect{a}} (\bcurve_{\bpoint_0, \bpoint_1, \bpoint_2}(t) - \vect{b}) =0 } \cup \clist{t_a, t_b}
\end{align}
where $\tr{\vect{a}} (\bcurve_{\bpoint_0, \bpoint_1, \bpoint_2}(t) - \vect{b}) =0$ is a quadratic equation.
\end{proposition}
\begin{proof}
By definition, the roots of $\tr{\vect{a}} (\bcurve_{\bpoint_0, \bpoint_1, \bpoint_2}(t) - \vect{b})$ determines the Bezier parameters over $\R$ where the curve intersects the halfspace boundary.
Hence, the ascendingly sorted elements of $T$ define a partition of the interval $[t_a, t_b]$ whose consecutive pairs define the part of the bezier curve on the opposite sides of the halfspace.
Therefore, the result follows.   
\end{proof}

\section{Proofs}

\subsection{Proof of \reflem{lem.InvertibleBasisMatrix}}
\label{app.InvertibleBasisMatrix}

\begin{proof}
The result follows from that $n^{\text{th}}$-order Bernstein polynomials, as well as monomials and Taylor polynomials of degree less than or equal to $\cdegree$, define a basis of $\cdegree + 1$ linearly independent polynomials for $n^{\text{th}}$-order polynomials \cite{farouki_CAGD2012}.

Alternatively, one can verify the result using polynomial basis transformations as follows.
The monomial basis matrix $\mbmat_{\cdegree}(t_{0}, \ldots, t_\cdegree)$, by definition, equals to the Vandermonde matrix, which is nonsingular for distinct $t_0, \ldots, t_\cdegree$ \cite{neagoe_SPL1996}. 
The Bezier and Taylor basis matrices are also nonsingular due to the change of basis relation, i.e.,
\begin{subequations}
\begin{align}
\bbmat_{\cdegree}(t_0, \ldots, t_\cdegree) & = \tfbasis_{\mbasis}^{\bbasis}(\cdegree) \mbmat_{\cdegree}(t_0, \ldots, t_ \cdegree) ,
\\
\tbmat_{\cdegree, \toffset}(t_0, \ldots,  t_\cdegree) & = \tfbasis_{\tbasis}^{\bbasis}(\cdegree, \toffset) \mbmat_{\cdegree}(t_0, \ldots, t_ \cdegree) ,
\end{align}
\end{subequations}
where $\tfbasis_{\mbasis}^{\bbasis}(\cdegree)$ and $\tfbasis_{\tbasis}^{\bbasis}(\cdegree, \toffset)$ are invertible triangular basis transformation matrices (see Lemmas \ref{lem.MonomialBernsteinTransformation} \& \ref{lem.MonomialTaylorTransformation}).
\end{proof}

\subsection{Proof of \reflem{lem.BasisTransformation}}
\label{app.BasisTransformation}

\begin{proof}
Consider the basis transformation matrix $\tfbasis_{\mbasis}^{\bbasis}(\cdegree)$ from monomial to Bernstein basis. 
It follows by definition that 
\begin{align}
\bbmat_{\cdegree}(t_0, \ldots, t_\cdegree) = \tfbasis_{\mbasis}^{\bbasis}(\cdegree) \mbmat_{\cdegree}(t_0, \ldots, t_\cdegree).
\end{align}   
Since the monomial basis matrix $\mbmat_{\cdegree}(t_0, \ldots, t_\cdegree)$  is invertible for any distinct $t_0, \ldots, t_\cdegree \in \R$ (\reflem{lem.InvertibleBasisMatrix}), we obtain
\begin{align}
\tfbasis_{\mbasis}^{\bbasis}(\cdegree) = \bbmat_{\cdegree}(t_0, \ldots, t_\cdegree) \mbmat_{\cdegree}(t_0, \ldots, t_\cdegree)^{-1}.
\end{align}
Similarly, the result can be verified for any change of basis between Bernstein, Taylor and monomial bases, which completes the proof. 
\end{proof}

\subsection{Proof of \reflem{lem.PolynomialCurveEquivalence}}
\label{app.PolynomialCurveEquivalence}

\begin{proof}
Let us focus on the equivalence of B\'ezier curves to monomial curves.
The equivalence of B\'ezier and monomial curves means that  for any distinct $t_0, \ldots, t_\cdegree \in \R$ one has
\begin{align} 
\bcurve_{\bpoint_0, \ldots, \bpoint_\cdegree}(t_0, \ldots, t_\cdegree) & = \mcurve_{\bpoint_0, \ldots, \bpoint_\cdegree}(t_0, \ldots, t_\cdegree)  
\\ 
\bpmat_{\cdegree} \bbmat_{\cdegree}(t_0, \ldots, t_\cdegree) &=  \mpmat_{\cdegree} \mbmat_{\cdegree}(t_0, \ldots, t_\cdegree)
\\
& =  \mpmat_{\cdegree} \tfbasis_{\bbasis}^{\mbasis}(\cdegree)\bbmat_{\cdegree}(t_0, \ldots, t_\cdegree).
\end{align}
Since the Bernstein basis matrix $\bbmat_{\cdegree}(t_0, \ldots, t_\cdegree)$ is invertible for any distinct $t_0, \ldots, t_\cdegree$ (\reflem{lem.InvertibleBasisMatrix}), one can conclude that
\begin{align}
\bpmat_{\cdegree} = \mpmat_{\cdegree} \tfbasis_{\bbasis}^{\mbasis}(\cdegree)
\end{align}
which can be similarly extended for other polynomial curve equivalence relations. 
\end{proof}

\subsection{Proof of \reflem{lem.PolynomialReparametrization}}
\label{app.PolynomialReparametrization}

\begin{proof}
For any $t_0, \ldots, t_\cdegree \in \R$, by definition, affine Bezier reparametrization satisfies for $\widehat{t}_i = \tfrac{d -c}{b-a} t_i - \tfrac{ a d - b c}{ b- a}$ that 
\begin{align}
\widehat{\bpmat}_{\cdegree} \bbmat_{\cdegree}(\widehat{t}_0, \ldots, \widehat{t}_{\cdegree})=  \bpmat_{\cdegree} \bbmat_{\cdegree}(t_0, \ldots, t_\cdegree). 
\end{align}
Hence, we have the result since Bernstein basis matrices are invertible for any distinct parameters (\reflem{lem.InvertibleBasisMatrix}), which also extends in a similar way to Taylor and monomial curves.
\end{proof}


\subsection{Proof of \refprop{prop.L2Distance}}
\label{app.L2Distance}

\begin{proof}
Using the following properties of Bernstein polynomials \cite{farin_CurvesSurfaces2002},
\begin{align}
\int_{0}^{1} \!\!\bpoly_{i,n}(t) \diff t = \scalebox{1.25}{$\frac{1}{n+1}$}, \, \text{ and } \, 
\bpoly_{i,n}(t) \bpoly_{j,m}(t)  = \frac{\binom{n}{i}\binom{n}{j}}{\binom{n+m}{i+j}} \bpoly_{i+j, n+m}(t) \nonumber 
\end{align}
one can verify the result as
\begin{align}
\!\!\bdistL(\bcurve_{\bpoint_0, \ldots, \bpoint_n}, \bcurve_{\mpoint_0, \ldots, \mpoint_n})^2  &= \int_{0}^{1} \!\!\norm{\bcurve_{\bpoint_0 - \mpoint_0, \ldots, \bpoint_n - \mpoint_n}(t) }^2 \diff t,  
\\
& \hspace{-29mm} = \sum_{i=0}^{n}\sum_{j=0}^{n} \tr{\plist{\bpoint_i - \mpoint_i}} \plist{\bpoint_j - \mpoint_j} \int_{0}^{1}\!\! \bpoly_{i,n}(t) \bpoly_{j,n}(t) \diff t,
\\
& \hspace{-29mm}  = \sum_{i=0}^{n}\sum_{j=0}^{n} \tr{\plist{\bpoint_i - \mpoint_i}} \plist{\bpoint_j - \mpoint_j}  \frac{\binom{n}{i}\binom{n}{j}}{\binom{2n}{i+j}} \int_{0}^{1}\!\! \bpoly_{i+j, 2n}(t) \diff t,  \!\!\!
\\
& \hspace{-29mm}  = \sum_{i=0}^{n}\sum_{j=0}^{n} \frac{1}{2n+1}\frac{\binom{n}{i}\binom{n}{j}}{\binom{2n}{i+j}}  \tr{\plist{\bpoint_i - \mpoint_i}} \plist{\bpoint_j - \mpoint_j}, 
\\
& \hspace{-29mm}   = \trace\plist{(\bpmat_{\cdegree}- \mpmat_{\cdegree})\bwmat_{\cdegree} \tr{(\bpmat_{\cdegree} - \mpmat_{\cdegree})}},
\end{align}
which completes the proof.
\end{proof}

\subsection{Proof of \refprop{prop.DistanceOrder}}
\label{app.DistanceOrder}

\begin{proof}
By \refdef{def.HaussdorffMaximumDistance}, the B\'ezier parameterwise-maximum distance defines an upper bound on the B\'ezier Haussdoff distance, i.e.,
\begin{align}
\bdistH(\bcurve_{\bpoint_0, \ldots, \bpoint_n}, \bcurve_{\mpoint_0, \ldots, \mpoint_n}) &\leq \bdistM(\bcurve_{\bpoint_0, \ldots, \bpoint_n}, \bcurve_{\mpoint_0, \ldots, \mpoint_n}). 
\end{align}
Similarly, the equivalence relation of the Frobenius distance and the control-point distances is evident from \refdef{def.FrobeniusDistance} and \refdef{def.CtrlDistance} as
\begin{align}
\bdistC(\bcurve_{\bpoint_0, \ldots, \bpoint_n}, \bcurve_{\mpoint_0, \ldots, \mpoint_n}) \hspace{-20mm}& \hspace{20mm} = \max_{i=0, \ldots, n} \norm{\bpoint_i - \mpoint_i} ,
\\
& \leq \bdistF(\bcurve_{\bpoint_0, \ldots, \bpoint_n}, \bcurve_{\mpoint_0, \ldots, \mpoint_n})  = \sqrt{\sum\nolimits_{i=0}^{n} \norm{\bpoint_i - \mpoint_i}^2}, \!\!\!
\\
& \leq \sqrt{\cdegree}  \bdistC(\bcurve_{\bpoint_0, \ldots, \bpoint_n}, \bcurve_{\mpoint_0, \ldots, \mpoint_n}).
\end{align}
Hence, the result follows from Jensen's equality for the squared Euclidean distance because
\begin{align}
\norm{\bcurve_{\bpoint_0, \ldots, \bpoint_\cdegree}(t) - \bcurve_{\mpoint_0, \ldots, \mpoint_\cdegree}(t)}^2 &= \norm{\bcurve_{\bpoint_0 - \mpoint_0, \ldots, \bpoint_\cdegree - \mpoint_{\cdegree}}(t)}^2, 
\\
& \hspace{-31mm}=\norm{\sum_{i=0}^{\cdegree} \bpoly_{i,\cdegree}(t) (\bpoint_i - \mpoint_i)}^2
\leq \sum_{i=0}^{\cdegree} \bpoly_{i,\cdegree}(t) \norm{\bpoint_i - \mpoint_i}^2 ,
\\
&  \hspace{-31mm}\leq \max_{i} \norm{\bpoint_i - \mpoint_i}^2.
\end{align}
Note that the Bernstein polynomials sum to one over $t \in [0,1]$, i.e., $\sum_{i=0}^{\cdegree} \bpoly_{i, \cdegree}(t) = 1$ for all $t \in [0,1]$ (\refpropty{propty.BezierConvexity}).
\end{proof}

\subsection{Proof of \refprop{prop.RelativeBezierBound}}
\label{app.RelativeBezierBound}

\begin{proof}
The ordering relation of B\'ezier metrics in \refprop{prop.DistanceOrder} and the convexity of B\'ezier curves in \refpropty{propty.BezierConvexity} imply that
{
\begin{align}
&\hspace{-3mm} \min_{t' \in [0,1]} \! \norm{\bcurve_{\bpoint_0, \ldots, \bpoint_n}(t) - \bcurve_{\mpoint_0, \ldots, \mpoint_n}(t')}  \leq \norm{\bcurve_{\bpoint_0, \ldots, \bpoint_n}(t) - \bcurve_{\mpoint_0, \ldots, \mpoint_n}(t)},  \nonumber
\\
& \hspace{5mm} \leq \bdistM( \bcurve_{\bpoint_0, \ldots, \bpoint_n}, \bcurve_{\mpoint_0, \ldots, \mpoint_n} )
\leq \bdistC (\bcurve_{\bpoint_0, \ldots, \bpoint_n}, \bcurve_{\mpoint_0, \ldots, \mpoint_n}), \!\!\!    
\end{align}  
}%
and so the result follows.
\end{proof}

\subsection{Proof of \refprop{prop.ElevationControlPoint}}
\label{app.ElevationControlPoint}

\begin{proof}
The sufficiency of elevated control points in \refeq{eq.ElevationControlPoints} can be verified as
\begin{subequations}
\begin{align}
\bcurve_{\mpoint_0, \ldots, \mpoint_m}(t) &= \mpmat_m \bbasis_{m}(t) 
\\
&= \bpmat_{n} \emat(n,m) \bbasis_{m}(t),
\\
& = \bpmat_{n} \tfbasis_{\mbasis}^{\bbasis}(n) \mat{I}_{(n+1) \times (m+1)} \tfbasis_{\bbasis}^{\mbasis}(m) \bbasis_{m}(t), \!\!\!
\\
& = \bpmat_{n} \tfbasis_{\mbasis}^{\bbasis}(n) \mat{I}_{(n+1) \times (m+1)} \mbasis_{m}(t), \!\!\!
\\
& = \bpmat_{n}  \tfbasis_{\mbasis}^{\bbasis}(n) \mbasis_{n}(t), 
\\
& = \bpmat_{n} \bbasis_{n}(t) = \bcurve_{\bpoint_0, \ldots, \bpoint_{n}}(t).
\end{align}
\end{subequations}

To show the necessity of parameterwise coincidence, consider distinct parameters $t_0, \ldots, t_m \in \R$ with $t_i \neq t_j$ for all $i \neq j$. 
Since square Bernstein basis matrices of distinct parameters are invertible (\reflem{lem.InvertibleBasisMatrix}), using the coinciding curve points at $t_0, \ldots, t_m$, i.e.,
\begin{align}
\bpmat_{n} \bbmat_{n}(t_0, \ldots, t_m) = \mpmat_{m} \bbmat_{m}(t_0, \ldots, t_m),
\end{align}
one can obtain an explicit expression for $\emat(n,m)$ as 
\begin{align}
\mpmat_{m} &= \bpmat_{n} \bbmat_{n}(t_0, \ldots, t_m)\bbmat_{m}(t_0, \ldots, t_m)^{-1}, \!\!\!
\\
& = \bpmat_{n} \emat(n,m),
\end{align}
which can be further simplified using the Bernstein-to-monomial basis transformation as
\begin{align}
\emat(n,m) \hspace{-2mm}&  \hspace{2mm}= \bbmat_{n}(t_0, \ldots, t_m)\bbmat_{m}(t_0, \ldots, t_m)^{-1} ,
\\
& =  \tfbasis_{\mbasis}^{\bbasis}(n) \mbmat_{n}(t_0, \ldots, t_m)\mbmat_{m}(t_0, \ldots, t_m)^{-1} \tfbasis_{\bbasis}^{\mbasis}(m) , \!\!\!
\\
& = \tfbasis_{\mbasis}^{\bbasis}(n) \mat{I}_{(n+1) \times (m+1)} \tfbasis_{\bbasis}^{\mbasis}(m) ,
\end{align} 
which completes the proof.
\end{proof}

\subsection{Proof of \refprop{prop.ElevationMatrixBernstein}}
\label{app.ElevationMatrixBernstein}

\begin{proof}
By definition in \refeq{eq.ElevationMatrix}, we have
\begin{align}
\emat(n,m) \bbasis_{m}(t) & = \tfbasis_{\mbasis}^{\bbasis}(n) \mat{I}_{(n+1) \times (m+1)} \tfbasis_{\bbasis}^{\mbasis}(m) \bbasis_{m}(t) 
\\
& = \tfbasis_{\mbasis}^{\bbasis}(n) \mat{I}_{(n+1) \times (m+1)} \mbasis_{m}(t)
\\
& = \tfbasis_{\mbasis}^{\bbasis}(n)\mbasis_{n}(t) = \bbasis_{n}(t).
\end{align}
Therefore, the result follows since the Bernstein basis matrix $\bbmat_{m}(t_0, \ldots, t_m)$ is invertible (\reflem{lem.InvertibleBasisMatrix}).
\end{proof}

\subsection{Proof of \refprop{prop.ElevationMatrixElements}}
\label{app.ElevationMatrixElements}

\begin{proof}
We below provide a proof by induction. 

$\bullet$ Base Case: ($n \leq m \leq n+1$): 
If $m = n$, then one trivially has $\emat(n,n) = \mat{I}_{(n+1) \times (n+1)}$. 
If $m = n+1$, then
\begin{align}\label{eq.ElevationMatrixSpecial}
\blist{\emat(n, n+1)}_{i+1,j+1} = \left \{
\begin{array}{@{}c@{\,}l@{}}
1 - \frac{j}{n + 1} & \text{, if } j = i, \\
\frac{j}{n + 1} & \text{, if }   j = i + 1, \\
0 & \text{, otherwise,}
\end{array} 
\right .
\end{align}
which follows from \refeq{eq.BernsteinBasisElevationMatrix} and the  following degree-one elevation property of Bernstein polynomials \cite{farouki_CAGD2012}
\begin{align}
\bpoly_{i,n}(t) = \plist{1- \tfrac{i}{n+1}}\bpoly_{i,n+1}(t) + \tfrac{i+1}{n+1} \bpoly_{i+1,n+1}(t).
\end{align}
Also note that $\binom{n}{j}/\binom{n+1}{j} = 1 - \frac{j}{n+1}$ and $\binom{n}{j-1}/\binom{n+1}{j} = \frac{j}{n+1}$.
Hence, the result holds for the base case.

$\bullet$ Induction Step ($m>n+1$): Suppose the results holds for $\emat(n, m-1)$, then one can determine $\emat(n,m)$ as
\begin{align}
\emat(n,m) = \emat(n,m-1) \emat(m-1,m),
\end{align}
because the degree elevation operation preserves the original B\'ezier curve exactly. 
Hence, it follows from \refeq{eq.ElevationMatrixSpecial} that
{\small
\begin{align}
\hspace{0mm}\blist{\emat(n,m)}_{i+1, j+1} \!& =\! \sum_{k=0}^{m} \!\blist{\emat(n,m-1)}_{i+1, k+1} \blist{\emat(m-1, m)}_{k+1, j+1}, \!\!\!\!\!
\\
&\hspace{-13mm} = \blist{\emat(n,m-1)}_{i+1, j} \blist{\emat(m-1, m)}_{j, j+1} \nonumber \\
& \hspace{+0mm} + \blist{\emat(n,m-1)}_{i+1, j+1} \blist{\emat(m-1, m)}_{j+1, j+1}, \!\!
\\
&\hspace{-13mm} =   \tfrac{j}{m} \blist{\emat(n,m-1)}_{i+1, j} + \plist{1 - \tfrac{j}{m}} \blist{\emat(n,m-1)}_{i+1, j+1}. \!\!\! 
\end{align}
}%
Note that $\blist{\emat(n,m-1)}_{i+1, j+1} \neq 0$ iff $0 \leq j -i \leq m-n-1$; and $\blist{\emat(n,m-1)}_{i+1, j} \neq 0$ iff $1 \leq j-i \leq m-n$. 
Hence, we complete the induction step by checking the following cases:

\smallskip

$\circ$ If $j-i > m-n$ or $j-i < 0$, then $\blist{\emat(n,m-1)}_{i+1, j+1} = 0$ and $\blist{\emat(n,m-1)}_{i+1, j} = 0$, and so
\begin{align}
\blist{\emat(n,m)}_{i+1, j+1} = 0.
\end{align}

\smallskip

$\circ$ If $j-i = m - n$, then $\blist{\emat(n,m-1)}_{i+1, j+1} = 0$ and so
\begin{align}
\blist{\emat(n,m)}_{i+1, j+1} = \tfrac{j}{m} \blist{\emat(n,m-1)}_{i+1, j} = \tfrac{j}{m} \frac{\binom{n}{i}}{\binom{m-1}{j-1}} = \frac{\binom{n}{i}}{\binom{m}{j}}. 
\end{align}

\smallskip

$\circ$ If $j-i = 0$, then  $\blist{\emat(n,m-1)}_{i+1, j} = 0$ and so
\begin{align}
\blist{\emat(n,m)}_{i+1, j+1} &= \plist{1 - \tfrac{j}{m}} \blist{\emat(n,m-1)}_{i+1, j+1}, \\
&= \plist{1 - \tfrac{j}{m}} \frac{\binom{n}{i}}{\binom{m-1}{j}} = \frac{\binom{n}{i}}{\binom{m}{j}}.
\end{align}

\smallskip
$\circ$ Otherwise (i.e., $0< j-i < m-n$), we have
\begin{align}
\blist{\emat(n,m)}_{i+1, j+1} &= \tfrac{j}{m} \blist{\emat(n,m-1)}_{i+1, j}  \nonumber 
\\ & \hspace{8mm}+ \plist{1 - \tfrac{j}{m}} \blist{\emat(n,m-1)}_{i+1, j+1},\!\! 
\\
& \hspace{-8mm} = \tfrac{j}{m} \frac{\binom{n}{i} \binom{m-n-1}{j-i-1}}{\binom{m-1}{j-1}} + \plist{1 - \tfrac{j}{m}}  \frac{\binom{n}{i} \binom{m-n-1}{j-i}}{\binom{m-1}{j}},
\\
& \hspace{-8mm}=   \frac{\binom{n}{i} \binom{m-n-1}{j-i-1}}{\frac{m}{j}\binom{m-1}{j-1}} +  \frac{\binom{n}{i} \binom{m-n-1}{j-i}}{ \frac{m}{m - j}\binom{m-1}{j}},
\\
& \hspace{-8mm} = \frac{\binom{n}{i}}{\binom{m}{j}} \plist{\! \binom{m-n-1}{j-i-1} + \binom{m-n-1}{j-i}\!},
\\
& \hspace{-8mm} = \frac{\binom{n}{i} \binom{m-n}{j-i}}{\binom{m}{j}},
\end{align}
which completes the proof.
\end{proof}

\subsection{Proof of \refprop{prop.ElevationMatrixFullRank}}
\label{app.ElevationMatrixFullRank}

\begin{proof}
The result can be verified using  either \refeq{eq.ElevationMatrix} or \refeq{eq.ElevationMatrixBernstein} with the fact that if a square matrix $\mat{B}$ is full rank (i.e., invertible), then $\rank(\mat{A}\mat{B}) = \rank(\mat{A})$ for any matrix  $\mat{A}$ that is conformable for the multiplication $\mat{A} \mat{B}$ \cite{horn_johnson_MatrixAnalyis2012}.  
\end{proof}

\subsection{Proof of \refprop{prop.ElevationMatrixRowColumnSum}}
\label{app.ElevationMatrixRowColumnSum}

\begin{proof}
The column-sum property of the elevation matrix follows from \refprop{prop.ElevationMatrixBernstein}, 
\begin{align}
\emat(n,m) &= \bbmat_{n}(t_0, \ldots, t_m) \bbmat_{m}(t_0, \ldots, t_m)^{-1},
\end{align}
and  the convexity of Bernstein polynomials (\refpropty{propty.BezierConvexity}),
\begin{align}
\mat{1}_{1 \times (n+1)} \bbmat_{n}(t_0, \ldots, t_m) & = \mat{1}_{1 \times (m+1)}, \\
\mat{1}_{1 \times (m+1)} \bbmat_{m}(t_0, \ldots, t_m) & = \mat{1}_{1 \times (m+1)} \bbmat_{m}(t_0, \ldots, t_m)^{-1}, \\
 &   = \mat{1}_{1 \times (m+1)},
\end{align}
where  $t_0, \ldots, t_m \in \R$ are any distinct reals.

\smallskip

The row-sum property of the elevation matrix can be proven by induction as follows. 

 $\bullet$ Base Case ($n \leq m \leq n+1)$: If $m = n$,  then one has  $\emat(n,n) = \mat{I}_{(n+1)\times (n+1)}$ and so the result holds. 
For $m=n+1$, 
\begin{align}
\blist{\emat(n,n+1)}_{i+1, j+1} = \left \{ 
\begin{array}{@{}c@{\,\,}l@{}}
1 - \frac{i}{n+1} & \text{, if } j = i, \\
\frac{i+1}{n+1} & \text{, if } j = i+1, \\
0 & \text{, otherwise}.
\end{array}
\right.
\end{align}
Hence, the row sum of $\emat(n,n+1)$ is $\frac{n+2}{n+1}$, i.e.,
\begin{align}
\sum_{j=0}^{n+1} \blist{\emat(n,n+1)}_{i+1, j+1} = 1  - \tfrac{i}{n+1} + \tfrac{i+1}{n+1} = \tfrac{n+2}{n+1}.
\end{align}

 $\bullet$ Induction ($m > n+1$). 
Suppose that the result holds for $\emat(n,m-1)$. 
Hence, using $\emat(n,m) = \emat(n,m-1)\emat(m-1,m)$, one can conclude that the row sum of multiplication of two matrices is the multiplication of the their row sums, i.e.,
{\small
\begin{align}
\hspace{-1.5mm}\sum_{j = 0}^{m} \blist{\emat(n,m)}_{i+1, j+1} \!\! & \nonumber \\
& \hspace{-19mm} = \!\!\sum_{j = 0}^{m} \sum_{k=0}^{m-1} \!\blist{\emat(n,m-1)}_{i+1, k+1} \blist{ \emat(m-1,m)}_{k+1, j+1}, \!\! 
\\
& \hspace{-19mm}= \plist{\sum_{k=0}^{m-1} \blist{\emat(n,m-1)}_{i+1, k+1}\!\!}\!\! \plist{\sum_{j = 0}^{m} \blist{ \emat(m-1,m)}_{k+1, j+1}\!\!}, \!\!
\\
& \hspace{-19mm} = \frac{m}{n+1}\frac{m+1}{m} = \frac{m+1}{n+1},
\end{align}
}%
which completes the proof.
\end{proof}

\subsection{Proof of \refprop{prop.L2DistanceElevation}}
\label{app.L2DistanceElevation}

\begin{proof}
The result directly follows from \refdef{def.L2Distance} of the  L2-norm distance because degree elevation exactly represents B\'ezier curves with more control points (\refdef{def.DegreeElevation}).
\end{proof}

\subsection{Proof of \refprop{prop.FrobeniusDistanceElevation}}
\label{app.FrobeniusDistanceElevation}

\begin{proof}
Using the column- and row-sum property of the elevation matrix in  \refprop{prop.ElevationMatrixRowColumnSum}, one can obtain the result by applying Jensen's  inequality as
\begin{align}
\bdistF(\bcurve_{[\bpoint_0, \ldots, \bpoint_n] \emat(n,m)}, \bcurve_{[\mpoint_0, \ldots, \mpoint_n]\emat(n,m)}) \hspace{-24mm}& \hspace{24mm} \nonumber 
\\
&= \norm{(\mat{P}_n - \mat{Q}_n) \emat(n,m)}_F^2, 
\\
&  = \sum_{j=0}^{m} \norm{(\mat{P}_n - \mat{Q}_n) \emat(n,m)_{j+1}}^2,
\\
& \leq \sum_{j=0}^{m} \sum_{i=0}^{n} [\emat(n,m)]_{i+1,j+1} \norm{\bpoint_i - \mpoint_i}^2, 
\\
&  = \sum_{i=0}^{n}\norm{\bpoint_i - \mpoint_i}^2 \sum_{j=0}^{m}[\emat(n,m)]_{i+1,j+1},  \!\!
\\
& =  \tfrac{m+1}{n+1}\bdistF(\bcurve_{\bpoint_0, \ldots, \bpoint_n}, \bcurve_{\mpoint_0, \ldots, \mpoint_n})^2,
\end{align}
where $\emat(n,m)_{j+1}$ denotes the $(j+1)^\text{th}$-column of $\emat(n,m)$. 
\end{proof}

\subsection{Proof of \refprop{prop.CtrlDistanceElevation}}
\label{app.CtrlDistanceElevation}

\begin{proof}
Let $\emat(n,m) = \blist{\vect{e}_0, \ldots, \vect{e}_m}$. 
Then, the result can be verified using Jensen's inequality and the unit column sum property of the elevation matrix (\refprop{prop.ElevationMatrixRowColumnSum}) as follows: 
{\small
\begin{align}
\bdistC(\bcurve_{\blist{\bpoint_0, \ldots, \bpoint_n}\emat(n,m)}, \bcurve_{\blist{\mpoint_0, \ldots, \mpoint_n} \emat(n,m)}) & = \max_{i=0\ldots m} \! \norm{ (\mat{P} - \mat{Q}) \vect{e}_i}, \!\!\!
\\
& \hspace{-30mm} \leq \max_{j = 0, \ldots, n} \norm{\bpoint_j - \mpoint_j} = \bdistC(\bcurve_{\bpoint_0, \ldots, \bpoint_n}, \bcurve_{\mpoint_0, \ldots, \mpoint_n}),
\end{align}
}%
where  the Jensen's inequality and $\tr{\mat{1}} \vect{e}_i = 1$ imply that
\begin{align}
\norm{ (\mat{P} - \mat{Q}) \vect{e}_i} \leq \max_{j = 0, \ldots, n} \norm{\bpoint_j - \mpoint_j}
\end{align}
and this completes the proof.
\end{proof}

\subsection{Proof of \refprop{prop.LeastSquaresReductionOptimality}}
\label{app.LeastSquaresReductionOptimality}

\begin{proof}
Using the following matrix identities \cite{petersen_pedersen_MatrixCookbook2012}
{\small
\begin{align}
\frac{\partial}{\partial \mat{X}} \trace\plist{\mat{X}\mat{A}} & = \tr{\mat{A}}, \quad  \text{ and }  \quad 
\frac{\partial}{\partial \mat{X}} \trace\plist{\mat{X} \mat{A} \tr{\mat{X}}} = \mat{X} (\mat{A} + \tr{\mat{A}}), 
\end{align}
}%
and the explicit form of the L2-norm distance in \refprop{prop.L2Distance}, one can verify the optimality of the least squares reduction with respect to the L2-norm distance as follows
{\small
\begin{align}
 &\frac{\partial}{\partial\mpmat_{m}} \bdistL(\bcurve_{\bpoint_0, \ldots, \bpoint_n}, \bcurve_{\blist{\mpoint_0, \ldots, \mpoint_m} \emat(m,n)})^2 \nonumber \\
& \hspace{5mm} = \frac{\partial}{\partial \mpmat_m} \trace\plist{\!\!(\bpmat_n  - \mpmat_m \emat(m,n)) \bwmat_n \tr{(\bpmat_n - \mpmat_m \emat(m,n))\!}\!}, \!\!\! \!
\\
& \hspace{5mm} = 2 (\mpmat_m \emat(m,n) - \bpmat_n)\bwmat_n,
\end{align}
}%
which equals to zero for $\mpmat_{m} = \bpmat_{n} \rmat(n,m)$. 
Thus, the global optimality follows from the convexity of the squared L2-norm distance.

Similarly, due to its strong relation with linear least squares, the Frobenius-norm distance of B\'ezier curves 
\begin{align}
\bdistF(\bcurve_{\bpoint_0, \ldots, \bpoint_n }, \bcurve_{[\mpoint_0, \ldots, \mpoint_m] \emat(m,n)}) = \norm{ \bpmat_n  - \mpmat_m \emat(m,n)}_F,
\end{align}
is minimized via the pseudo-inverse $\emat(m,n)^{+}$ of $\emat(m,n)$ at 
\begin{align}
\mpmat_m = \bpmat_{n} \emat(m,n)^{+} = \bpmat_{n} \rmat_{L2}(n,m),
\end{align} 
which completes the proof.
\end{proof}

\subsection{Proof of \refprop{prop.TaylorReductionInverse}}
\label{app.TaylorReductionInverse}

\begin{proof}
The result can be verified using \refeq{eq.ElevationMatrix} and \refeq{eq.ReductionMatrix} as
{\small
\begin{align}
& \hspace{-2mm}\emat(m,n) \mat{R}_{\tbasis, \toffset}(n,m) \nonumber \\
&= \tfbasis_{\mbasis}^{\bbasis}(m) \mat{I}_{(m+1) \times (n+1)} \tfbasis_{\bbasis}^{\mbasis}(n) \tfbasis_{\tbasis}^{\bbasis}(n, \toffset) \mat{I}_{(n+1) \times (m+1)} \tfbasis_{\bbasis}^{\tbasis}(m, \toffset), \!\!\!\!
\\
& =\tfbasis_{\mbasis}^{\bbasis}(m) \underbrace{\mat{I}_{m \times n}  \tfbasis_{\tbasis}^{\mbasis}(n, \toffset) \mat{I}_{n \times m}}_{= \tfbasis_{\tbasis}^{\mbasis}(m, \toffset)} \tfbasis_{\bbasis}^{\tbasis}(m, \toffset), 
\\
& = \tfbasis_{\mbasis}^{\bbasis}(m)\tfbasis_{\tbasis}^{\mbasis}(m, \toffset)\tfbasis_{\bbasis}^{\tbasis}(m, \toffset),
\\
&= \tfbasis_{\mbasis}^{\bbasis}(m) \tfbasis_{\bbasis}^{\mbasis}(m) = \mat{I}_{(m+1) \times (m+1)}.
\end{align}
}%
which completes the proof.
\end{proof}

\subsection{Proof of \refprop{prop.MatchingReduction}}
\label{app.MatchingReduction}

\begin{proof}
Since $\bbmat(t_0, \ldots, t_m)^{-1} \bbmat_{m}(t_0, \ldots, t_m) = \mat{I}_{(m+1)\times (m+1)}$,   we have for any $t_i \in \clist{t_0, \ldots, t_m}$ that
\begin{subequations}
\begin{align}
\bcurve_{\mpoint_0, \ldots, \mpoint_m}(t_i) \hspace{-12mm}&\hspace{12mm}= [\mpoint_{0}, \ldots, \mpoint_m] \bbasis_{m}(t_i) ,
\\
&=  [\bpoint_0, \ldots, \bpoint_n] \rmat_{t_0, \ldots, t_m}(n,m)  \bbasis_{m}(t_i)  ,
\\
&=
[\bpoint_0, \ldots, \bpoint_n] \bbmat_{n}(t_0, \ldots, t_m) \bbmat_{m}(t_0, \ldots, t_m)^{-1}  \bbasis_{m}(t_i), \!\!\!
\\
& = [\bpoint_0, \ldots, \bpoint_n] \bbasis_n(t_i)  = \bcurve_{\bpoint_0, \ldots, \bpoint_n}(t_i).
\end{align}
\end{subequations}
Thus, the matching reduction preserves the curve at $t_i$.
\end{proof}

\subsection{Proof of \refprop{prop.MatchingReductionInverse}}
\label{app.MatchingReductionInverse}

\begin{proof}
Consider some additional distinct parameters $t_{m+1}, \ldots, t_{n} \in \R $ that are different from $ t_0, \ldots, t_m$.
Then the result can be verified using \refprop{prop.ElevationMatrixBernstein} as
{\small
\begin{align}
\! \emat(m,n)\rmat_{t_0, \ldots, t_m}(n,m) & = \emat(m,n) \bbmat_{n}(t_0, \ldots, t_m) \bbmat_{m}(t_0, \ldots, t_m)^{-1}\!\!\!, \!\!\!\!
\\
& = \bbmat_{m}(t_0, \ldots, t_m) \bbmat_{m}(t_0, \ldots, t_m)^{-1},
\\
& = \mat{I}_{(m+1) \times (m+1)},
\end{align}
}%
which completes the proof.
\end{proof}

\subsection{Proof of \refprop{prop.MatchingReductionDifference}}
\label{app.MatchingReductionDifference}

\begin{proof}
The matching reduction matrix $\rmat_{t_0, \ldots, t_n}(n+1, n)$ can be expressed in the monomial basis using the basis transformation between Bernstein and monomial bases as
{\small
\begin{align}
\rmat_{t_0, \ldots, t_n}(n+1, n) & = \bbmat_{n+1}(t_0, \ldots, t_n) \bbmat_{n}(t_0, \ldots, t_n) ^{-1}, 
\\
& \hspace{-17mm}=  \tfbasis_{\mbasis}^{\bbasis}(n+1) \mbmat_{n+1}(t_0, \ldots, t_n) \mbmat_{n}(t_0, \ldots, t_n) ^{-1} \tfbasis_{\bbasis}^{\mbasis}(n), \\
& \hspace{-17mm}=  \tfbasis_{\mbasis}^{\bbasis}(n+1) \blist{\begin{array}{c}
\mat{I}_{(n+1)\times(n+1)} 
\\
\blist{t_0^{n+1}, \ldots, t_n^{n+1}} \mbmat_{n}(t_0, \ldots, t_n)^{-1}
\end{array}} 
\tfbasis_{\bbasis}^{\mbasis}(n). \!\!\!
\end{align}
}%
 
Hence,  the degree-one matching reduction difference can be written for $\mpmat_{n} = [\mpoint_0, \ldots, \mpoint_n]$  and $\bpmat_{n+1} = [\bpoint_0, \ldots, \bpoint_{n+1}]$ as 
{\footnotesize
\begin{align}
&\bcurve_{\bpoint_0, \ldots, \bpoint_{\cdegree+1}}(t) - \bcurve_{\mpoint_0, \ldots, \mpoint_{\cdegree}}(t) = \bpmat_{n+1} \bbasis_{\cdegree+1}(t) - \mpmat_{n} \bbasis_{\cdegree}(t)  ,
\\
& \hspace{2mm}= \bpmat_{n+1} \tfbasis_{\mbasis}^{\bbasis}(\cdegree+1) \mbasis_{\cdegree+1}(t) \nonumber \\
& \hspace{5mm} - \bpmat_{n+1} \tfbasis_{\mbasis}^{\bbasis}(\cdegree+1) 
\left[\begin{array}{@{}c@{}}
\mat{I}_{(\cdegree+1) \times (\cdegree+1)} \\
\blist{t_0^{n+1}, \ldots, t_n^{n+1}} \mbmat_{n}(t_0, \ldots, t_n)^{-1}
\end{array}
\right]
\tfbasis_{\bbasis}^{\mbasis}(\cdegree) \bbasis_{\cdegree}(t),  \!\!\!
\\
&\hspace{2mm}= \bpmat_{n+1} \tfbasis_{\mbasis}^{\bbasis}(\cdegree+1)\!\plist{\!\!\mbasis_{\cdegree+1}(t)\! -\! \left[\begin{array}{@{}c@{}}
\mat{I}_{(\cdegree+1) \times (\cdegree+1)} \!\!\!\! \\
\blist{t_0^{n+1}, \ldots, t_n^{n+1}} \mbmat_{n}(t_0, \ldots, t_n)^{-1}\!
\end{array}
\right] \mbasis_{\cdegree}(t)\!\!\!},\!\!\!
\\
& \hspace{2mm} = \bpmat_{n+1} \tfbasis_{\mbasis}^{\bbasis}(\cdegree+1) \!\!\blist{\begin{array}{@{}c@{}}
0 \\
\vdots \\
0 \\
1
\end{array}}\!\!\plist{t^{n+1}\! - [t_0^{n+1}, \ldots, t_n^{n+1}] \mbmat_{n}(t_0, \ldots, t_n)^{-1} \mbasis_{n}(t)\!}. \!\!\!
\end{align}
}%
Now observe that for any distinct $t_0, \ldots, t_n \in \R$ one has 
{\small
\begin{align}
t^{n+1} - [t_0^{n+1}, \ldots, t_n^{n+1}] \mbmat_{n}(t_0, \ldots, t_n)^{-1} \mbasis_{n}(t) = \prod_{i=0}^{n}(t - t_i),
\end{align}
}%
which is zero at $t = t_0, \ldots, t_n$. 
We also have from \reflem{lem.MonomialBernsteinTransformation}
{\small
\begin{align}
\blist{\tfbasis_{\mbasis}^{\bbasis}(\cdegree+1)}_{i+1,\cdegree+2} &= (-1)^{\cdegree+1 - i} \binom{\cdegree+1}{\cdegree+1} \binom{\cdegree+1}{i},
\\
& = (-1)^{\cdegree+1 - i} \binom{\cdegree+1}{i}.
\end{align}
}%
Hence, the matching reduction difference is given by  
{\small
\begin{align}
\bcurve_{\bpoint_0, \ldots, \bpoint_{\cdegree+1}}(t) - \bcurve_{\mpoint_0, \ldots, \mpoint_{\cdegree}}(t) \!=\! \plist{\sum_{i=0}^{\cdegree + 1}\! (-1)^{n+1-i}\binom{n+1}{i} \bpoint_i\!\!}\!\! \prod_{i=0}^{\cdegree}(t- t_i),
\end{align}
}%
which completes the proof.
\end{proof}

\subsection{Proof of \reflem{lem.MonomialTaylorTransformation}}
\label{app.MonomialTaylorTransformation}

\begin{proof}
The monomial-to-Taylor basis transformation $\tfbasis_{\mbasis}^{\tbasis}$ directly follows from the binomial formula, 
\begin{align}
(t - \toffset)^{i} = \sum_{j=0}^{i} \tbinom{i}{j} (-\toffset)^{i-j} t^{j}.
\end{align}

Similarly, the Taylor-to-monomial basis transformation $\tfbasis_{\mbasis}^{\tbasis}$ can be obtained using the binomial formula as 
\begin{align}
t^{i} = (t- \toffset + \toffset)^{i} = \sum_{j=0}^{i} \binom{i}{j} \toffset^{i-j} (t - \toffset)^{j}.
\end{align}

Finally, the monomial-to-Taylor and Taylor-to-monomial transformations are inverses of each other since they are lower triangular matrices with all ones in the main diagonal (i.e., $\det\plist{\tfbasis_{\mbasis}^{\tbasis}(\cdegree, \toffset)} = \det\plist{\tfbasis_{\tbasis}^{\mbasis}(\cdegree, \toffset)} = 1$), and 
\begin{align}
\mbasis_{\cdegree}(t) = \tfbasis_{\tbasis}^{\mbasis}(\cdegree, \toffset) \tbasis_{\cdegree, \toffset} (t) = \tfbasis_{\tbasis}^{\mbasis}(\cdegree, \toffset) \tfbasis_{\mbasis}^{\tbasis}(\cdegree, \toffset) \mbasis_{\cdegree}(t)   
\\
\tbasis_{\cdegree, \toffset}(t) = \tfbasis_{\mbasis}^{\tbasis}(\cdegree, \toffset) \mbasis_{\cdegree} (t) = \tfbasis_{\mbasis}^{\tbasis}(\cdegree, \toffset) \tfbasis_{\tbasis}^{\mbasis}(\cdegree, \toffset) \tbasis_{\cdegree, \toffset}(t)   
\end{align}
hold for all $t \in \R$.
\end{proof}

\subsection{Proof of \reflem{lem.BasisReparametrization}}
\label{app.BasisReparametrization}

\begin{proof}
For Taylor basis reparametrization, the result follows from the definition of monomial and Taylor basis because
\begin{align}
\plist{\tfrac{b-a}{d-c} t + \tfrac{a d - b c}{d - c} - \toffset}^k 
& =  \plist{\tfrac{b-a}{d-c}}^{k} \plist{t - \tfrac{d-c}{b-a} \toffset + \tfrac{a d - b c}{b - a}}^{k} 
\\ 
&=  \plist{\tfrac{b-a}{d-c}}^{k} \plist{t - \widehat{\toffset}}^k
\end{align}
For Bernstein basis reparametrization, the results can be verified using the change of basis between Bernstein and Taylor bases as 
\begin{align}
\bbasis_{\cdegree}(t) \hspace{-4mm}& \hspace{4mm}= \tfbasis_{\tbasis}^{\bbasis}(\cdegree, \widehat{\toffset}) \tbasis_{\cdegree,\widehat{\toffset}} (t) 
\\
& =  \tfbasis_{\tbasis}^{\bbasis}(\cdegree, \widehat{\toffset}) \diag \plist{\mbasis_{\cdegree}(\tfrac{b-a}{d-c})}   \tbasis_{\cdegree, \toffset} \plist{\tfrac{b-a}{d-c} t + \tfrac{a d - b c}{d - c}}
\\
& =  \tfbasis_{\tbasis}^{\bbasis}(\cdegree, \widehat{\toffset}) \diag \plist{\mbasis_{\cdegree}(\tfrac{b-a}{d-c})}  \tfbasis_{\bbasis}^{\tbasis}(\cdegree, \toffset) \bbasis_{\cdegree} \plist{\tfrac{b-a}{d-c} t + \tfrac{a d - b c}{d - c}} \!\!\!
\end{align}
which also extends to the monomial basis reparametrization in a similar way and so completes the proof.
\end{proof}

\subsection{Proof of \reflem{lem.CurveReparametrization}}
\label{app.CurveReparametrization}

\begin{proof}
For B\'ezier curve reparametrization, the result follows from the Bernstein basis reparametrization in \reflem{lem.BasisReparametrization} as
\begin{align}
\bcurve_{\bpoint_0, \ldots, \bpoint_\cdegree}\plist{\tfrac{b-a}{d-c} t + \tfrac{a d - b c}{d - c}} 
& = \bpmat_{\cdegree} \bbasis_{\cdegree} \plist{\tfrac{b-a}{d-c} t + \tfrac{a d - b c}{d - c}} 
\\
& \hspace{-20mm} =   \bpmat_{\cdegree} \tfbasis_{\tbasis}^{\bbasis}(\cdegree, \toffset) \diag\plist{\mbasis_{\cdegree}(\tfrac{d-c}{b-a})} \tfbasis_{\bbasis}^{\tbasis}(\cdegree, \widehat{\toffset})\bbasis_{\cdegree}\plist{t} 
\\
& \hspace{-20mm} = \widehat{\bpmat}_{\cdegree} \bbasis_{\cdegree}\plist{t}
\end{align}
which similarly extends to the monomial and Taylor curve reparametrization as well. 
\end{proof}

\subsection{Proof of \reflem{lem.MatchingReductionMatrixMonomial}}
\label{app.MatchingReductionMatrixmonomial}

\begin{proof}
It follow from the Bernstein-to-monomial basis transformation that
\begin{align}
\rmat_{t_0, \ldots, t_m}(n,m) &= \bbmat_{n}(t_0, \ldots, t_m) \bbmat_{m}(t_0, \ldots, t_m)^{-1}, \label{eq.DemotionMatrix2}
\\
& \hspace{-16mm} =  \tfbasis_{\mbasis}^{\bbasis}(n)\mbmat_{n}(t_0, \ldots, t_m) \mbmat_{m}(t_0, \ldots, t_m)^{-1} \tfbasis_{\bbasis}^{\mbasis}(m),
\\
& \hspace{-16mm} =  \tfbasis_{\mbasis}^{\bbasis}(n)
\left [
\begin{array}{@{}c@{}}
\mat{I}_{(m+1) \times (m+1)} \\
\blist{t_0^{m+1}, \ldots, t_m^{m+1}} \mbmat_{m}(t_0, \ldots, t_m)^{-1}
\\
\vdots
\\
\blist{t_0^{n}, \ldots, t_m^{n}} \mbmat_{m}(t_0, \ldots, t_m)^{-1}
\end{array}
\right ]
 \tfbasis_{\bbasis}^{\mbasis}(m). \!\!\!
\end{align}

To complete the proof, we show below  that the rows of the middle matrix following the identity matrix satisfy the recursion in \refeq{eq.MatchingReductionRecursion} with the base case of \refeq{eq.MatchingReductionBase}.
Hence, we first consider the base case where
\begin{align}
[\alpha_{1, 0}, \ldots, \alpha_{1,m}] = \blist{t_0^{m+1}, \ldots, t_m^{m+1}} \mbasis_{m}(t_0, \ldots, t_m)^{-1}, 
\\
\blist{t_0^{m+1}, \ldots, t_m^{m+1}} = [\alpha_{1, 0}, \ldots, \alpha_{1,m}] \mbasis_{m}(t_0, \ldots, t_m),
\end{align}  
which can be equivalently written as
\begin{align}
t^{m+1} - \sum_{k=0}^{m}  \alpha_{1,k} t^k = 0 \quad \forall t = t_0, \ldots, t_m. 
\end{align}
Since parameters $t_0, \ldots, t_m$ are distinct, $\prod_{k=0}^{m}(t - t_k)$ is the unique polynomial of order $m+1$ whose roots are $t_0, \ldots, t_m$ with the unity coefficient of the monomial $t^{m+1}$.  
Therefore, we obtain the base case in \refeq{eq.MatchingReductionBase} as
\begin{align}
t^{m+1} - \sum_{k=0}^{m}  \alpha_{1,k} t^k = \prod_{k=0}^{m}(t - t_k).
\end{align}

\noindent Now consider the $(m+i+1)^{\text{th}}$-row,
\begin{align}
[\alpha_{i+1, 0}, \ldots, \alpha_{i+1,m}] &= \blist{t_0^{m+i+1}, \ldots, t_m^{m+i+1}} \mbmat_{m}(t_0, \ldots, t_m)^{-1}, 
\\
\blist{t_0^{m+i+1}, \ldots, t_m^{m+i+1}} &= [\alpha_{i+1, 0}, \ldots, \alpha_{i+1,m}] \mbmat_{m}(t_0, \ldots, t_m), \!\!\!
\end{align}
which is equivalent to  
\begin{align}
\sum_{k=0}^{m} \alpha_{i+1,k} t^{k}  & = t^{m+i+1} = t^{m+i} t,
\\
& = t \sum_{k=0}^{m} \alpha_{i,k} t^{k} = \alpha_{i, m} t^{m+ 1} + \sum_{k=0}^{m-1} \alpha_{i,k} t^{k+1},
\\
& = \sum_{k=0}^{m} (\alpha_{i,m} \alpha_{1,k} + \alpha_{i, k-1}) t^k,
\end{align}
where $\alpha_{i,-1} = 0$. 
This implies the recursion relation in \refeq{eq.MatchingReductionRecursion} and so the result follows.
\end{proof}


\bibliographystyle{IEEEtran}
\bibliography{references.bib}

\vfill

\end{document}